\pdfoutput=1
\documentclass{article}
\usepackage{iclr2027_conference,times}

\usepackage{amsmath,amssymb,amsthm,mathtools,bm}
\usepackage{graphicx,booktabs,array}
\usepackage[hypertexnames=false]{hyperref}
\usepackage[protrusion=true,expansion=false]{microtype}
\usepackage{xcolor}
\usepackage{needspace}
\definecolor{PaperNavy}{HTML}{243B53}
\definecolor{PaperTeal}{HTML}{12635D}
\hypersetup{
  colorlinks=true,
  linkcolor=PaperNavy,
  citecolor=PaperTeal,
  urlcolor=PaperTeal,
  pdfborder={0 0 0},
  bookmarksopen=false,
  bookmarksnumbered=true,
  pdftitle={Awakening of the Buddha: Subspace Learning During Population-Loss Plateaus},
  pdfauthor={Akash Kumar}
}

\iclrfinalcopy
\makeatletter
\renewcommand{\@maketitle}{\vbox{\hsize\textwidth
  {\LARGE\scshape\@title\par}
  \vskip 1em
  {\normalsize\normalfont\@author\par}
  \vskip 0.3in minus 0.1in}}
\makeatother
\usepackage{xurl}

\newtheorem{theorem}{Theorem}[section]
\newtheorem{lemma}[theorem]{Lemma}
\newtheorem{proposition}[theorem]{Proposition}
\newtheorem{corollary}[theorem]{Corollary}
\theoremstyle{definition}

\theoremstyle{remark}

\title{Awakening of the Buddha:\\ Subspace Learning During\\ Population-Loss Plateaus}
\author{Akash Kumar\\Department of Computer Science \& Engineering\\University of California-San Diego}
\date{}

\begin{document}
\maketitle

\begin{abstract}
Population loss can remain nearly constant while a neural network learns a
substantially more predictive representation. We establish this separation
for two-layer ReLU and leaky-ReLU networks trained on Gaussian inputs by
simultaneous fixed-step population gradient descent on all parameters.
For structured additive teachers whose links are positive mixtures of
Gaussian-damped cubics in $H^1(\gamma)$, we give explicit conditions under
which small IID Gaussian initialization yields a high-probability guarantee:
at a checkpoint during a high-loss plateau, minimum alignment between the
rank-$r$ teacher subspace and the leading $r$-dimensional eigenspace of the
predictor's average gradient outer product (AGOP) increases by at least
$1/2$, and the minimum refit MSE under unchanged coefficient budgets
decreases by more than $0.399$, both relative to initialization.
The same trajectory subsequently attains a trained loss below every value
in the plateau window. A complementary result treats unequal-weight cubic
teachers and small additive Sobolev perturbations using projected-feature
refits. For SwiGLU networks with an exactly fitted intercept, we prove
leading-AGOP alignment during a loss plateau at fixed width and dimension
as Gaussian initialization vanishes, for square-integrable teachers with
nonzero Hermite content of degree one, two, or three. A rank-one cubic
specialization also gives simultaneous unrestricted-refit gains at a
prescribed width. An approximation lower bound further shows that certain
interaction targets retain nonzero error when ridge neurons are restricted
to shared orthogonal axes within the teacher subspace. Population-moment
experiments with ReLU students across 21 teachers and 50 initializations
per teacher complement the analysis.

\end{abstract}

\begin{flushright}
\begin{minipage}{0.72\textwidth}
\raggedleft
\textit{``There are no covenants between lions and men...''}\\[0.3em]
--- Akhille\'us to H\'ekt\={o}r, H\'om\={e}ros, \textit{Ili\'as} 22.262
\end{minipage}
\end{flushright}

\clearpage
\section{Introduction}
\label{sec:introduction}

A nearly constant prediction loss can conceal substantial changes in a
network's representation. Early ReLU alignment \citep{maennel2018quantizes},
\emph{silent alignment} of tangent-kernel eigenstructure
\citep{atanasov2022silent}, and hidden feature amplification in parity
learning \citep{barak2022hidden} provide precedents. We ask:
can an ordinary nonlinear training trajectory recover \emph{every}
direction of a target subspace while its population loss remains high,
and can the learned features already support a substantially better predictor?

We study simultaneous population gradient descent on all parameters of
two-layer ReLU and leaky-ReLU networks. Gaussian inputs and a target
depending on an unknown $r$-dimensional subspace connect our setting to
Gaussian multi-index representation learning
\citep{damian2022representations,bietti2025gradient}; small initialization
connects it to early alignment and feature acquisition
\citep{boursier2025early,kunin2025alternating}.
We track the weakest recovered direction using the leading
$r$-dimensional eigenspace of the current predictor's average gradient
outer product (AGOP), including all cross-neuron terms. To measure
predictive information separately, we refit the actual feature bank
under identical normalization and coefficient budgets before and after learning.

For scalar teacher links, let $\gamma=N(0,1)$ and
$h_k=\mathrm{He}_k/\sqrt{k!}$ be the orthonormal probabilists' Hermite
polynomials. Our Gaussian Sobolev regularity class is
\begin{equation}
 \mathcal G=H^1(\gamma)
 =\left\{g=\sum_{k=0}^{\infty}\widehat g_k h_k:
   \sum_{k=0}^{\infty}(1+k)|\widehat g_k|^2<\infty\right\}.
 \label{eq:gaussian-sobolev-links}
\end{equation}
Here $\widehat g_k=\mathbb E[g(Z)h_k(Z)]$, $Z\sim\gamma$, and the series
converges in $L^2(\gamma)$. Equivalently, $g$ and its weak derivative
belong to $L^2(\gamma)$. The class contains all globally Lipschitz links
and polynomials; multivariate links use $H^1(\gamma_r)$ with
$\gamma_r=N(0,I_r)$. Positive learning results require additional teacher
structure; the SwiGLU alignment result below allows $L^2$ teachers.

Small initialization separates output amplitude from feature geometry
while all parameters train; refitting never enters the updates.
For positive mixtures of
Gaussian-damped cubic links, including $h_3$ and nonpolynomial links
\eqref{eq:mixture-teacher}, our first result combines all-direction
recovery and improved refitting during a high-loss window with a later
decrease in the trained loss.

\Needspace{11\baselineskip}
\paragraph{Informal Theorem~\ref{thm:main} (ReLU and leaky ReLU).}
\textit{For our positive-mixture teachers, fixed-step population GD from
small Gaussian initialization learns the teacher subspace while loss
stays near one. With high probability, minimum AGOP alignment improves
by at least $1/2$ and same-budget refitting reduces MSE by more than
$0.399$ during the plateau. The same trajectory later lowers its own loss.}

To test whether feature learning during a plateau extends beyond
piecewise-linear students, we turn to SwiGLU. Its smooth gate and
trainable value response can change direction while their product stays
small, yielding leading-direction AGOP alignment for a broader teacher class.

\paragraph{Informal Theorem~\ref{swb:thm:alignment} (SwiGLU).}
\textit{For square-integrable teachers with a nonzero Hermite component
of degree one, two, or three, SwiGLU exhibits a related separation.
At fixed width, dimension and GD step, with the intercept fitted exactly,
its leading AGOP direction can approach the teacher subspace arbitrarily
closely as Gaussian initialization shrinks, while loss remains near
its initial value.}

For ReLU and leaky ReLU, Theorems~\ref{app:sw-theorem}
and~\ref{app:aj-theorem} give the full conditions and extensions. For restricted
rank-one teachers at a specified width,
Corollary~\ref{cor:swiglu-joint-main} adds simultaneous unrestricted-refit
improvement; its pure-$h_3$ example gives alignment and refit gains of at
least $0.5$ and $0.7$ with high asymptotic probability.

Figure~\ref{fig:main-trajectories} illustrates fifty initializations per
teacher. For the normalized
$\operatorname{SiLU}(X_1)X_2$ teacher, mean minimum-AGOP alignment rises from
$0.028$ to $0.986$ and mean refit MSE falls from $0.829$ to $0.455$
by update 200, while loss remains near one. Neuron-angle and axis-snapping
diagnostics reveal mixed teacher coordinates: subspace recovery need not
entail specialization to teacher axes. These interaction teachers lie
outside the ReLU plateau theorems. Figure~\ref{fig:swiglu-main} reports
SwiGLU students over twenty seeds per teacher.

\begin{figure}[p]
 \centering
 \includegraphics[width=\linewidth]{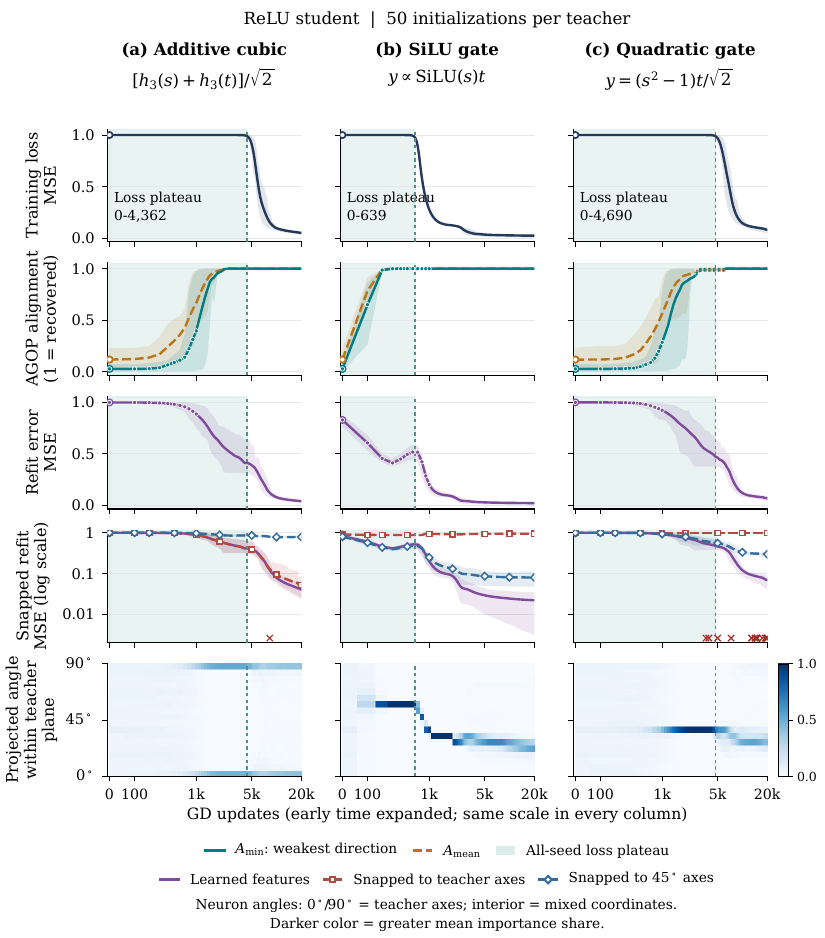}
 \caption{\textbf{Features improve before loss falls; gated teachers retain mixed neurons.}
 Fifty initializations of one ReLU-student configuration per normalized
 teacher ($s=X_1$, $t=X_2$).
 \emph{Read down each column:} training loss stays near one, AGOP subspace
 alignment rises, and a same-budget readout refit improves before loss release.
 Shading marks the loss-only plateau shared by all fifty seeds; the labeled
 boundaries are the earliest individual plateau endpoints.
 Horizontal position is $\log(1+n/100)$: early time is expanded, ticks show
 actual updates, and initialization and all 20,000 updates remain visible.
 For SiLU, mean $A_{\min}$ rises $0.028\to0.986$ and refit MSE falls
 $0.829\to0.455$ by update 200 while loss remains near one; refit improvement
 is not monotone.
 Row four projects each weight's teacher-plane component onto its nearest
 axis in the indicated frame, retaining outside-plane components and biases,
 then refits. Bottom-row color shows mean importance share of projected
 neuron directions \emph{within} the teacher plane: $0^\circ/90^\circ$
 are its two axes; interior angles mix their coordinates. This is not the
 angle to the plane and does not measure the outside-plane component.
 Curves use all fifty seeds; light bands are pointwise 10th--90th percentiles,
 not confidence intervals. Dotted alignment bounds retain unresolved spectra;
 red crosses mark snapped-refit solver tolerance misses.
 Diagnostics are recorded every 100 updates (early dots); connecting lines
 add no observations. Appendix~\ref{app:experiments} gives numerical brackets,
 the plateau rule and eighteen further teachers.}
 \label{fig:main-trajectories}
 \label{fig:frame-contrast}
\end{figure}

\clearpage
\begin{figure}[p]
\centering
\includegraphics[width=\textwidth]{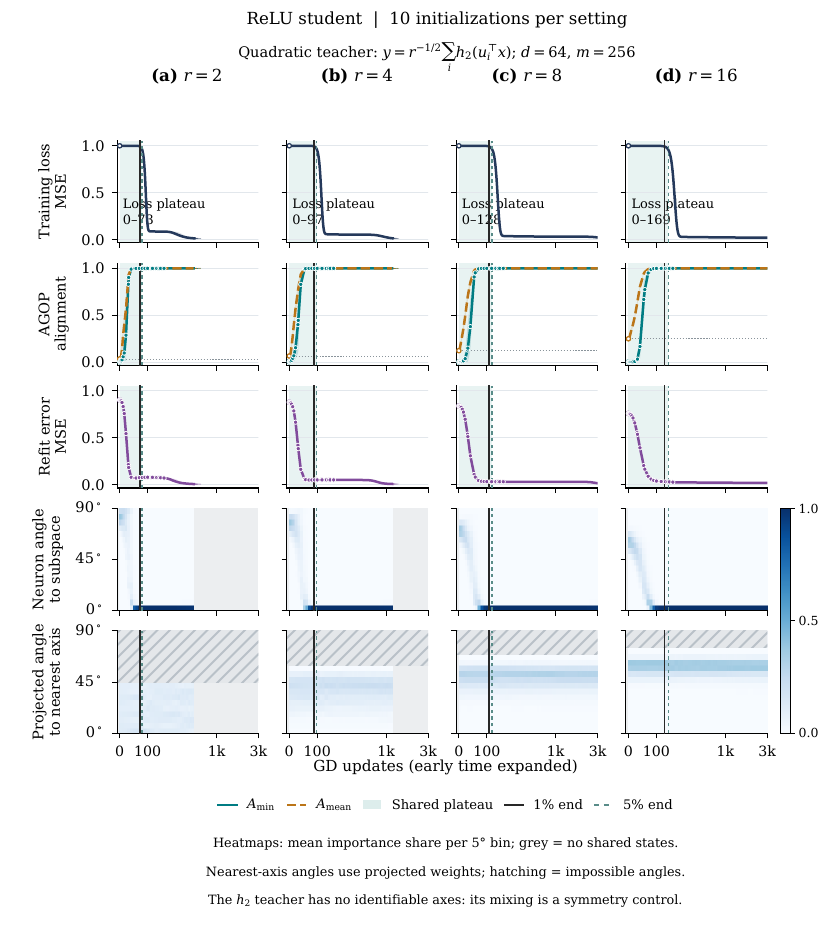}
\caption{\textbf{Subspace recovery before loss release at increasing teacher rank.}
ReLU students ($d=64$, $m=256$), ten initializations per column,
learn $y=r^{-1/2}\sum_i h_2(u_i^TX)$ at $r=2,4,8,16$.
Rows show training MSE, AGOP alignment, bounded-head refit MSE,
neuron angle \emph{to} the teacher subspace, and projected neuron
angle to its \emph{nearest axis}. The last two rows distinguish
subspace entry from individual-axis alignment. Unlike the bottom row of
Figure~\ref{fig:main-trajectories}, the bottom row here uses a projected
nearest-axis angle at arbitrary rank.
Means and 10th--90th percentile bands retain all ten seeds;
heatmaps average separately normalized importance-weighted histograms.
The dotted $r/d$ line is an isotropic reference for $A_{\rm mean}$.
Shading marks the common 5\% loss-ratio prefix, with a 1\% comparison
boundary; the loss also remains in $[0.95,1.05]$.
Grey heatmap regions have no shared saved state; hatching marks
geometrically impossible angles.
The quadratic teacher is invariant to rotations inside its subspace,
so its axes are not identifiable: this is a symmetry control for
coordinate mixing. Appendix~\ref{app:highrank} defines the angles,
protocol, numerical checks and additional teacher comparisons.}
\label{fig:highrank-relu-main}
\end{figure}

\clearpage
\section{Related work}
\label{sec:related}

\paragraph{Feature changes before loss reduction.}
Small-initialization ReLU networks can align before appreciable loss
reduction \citep{maennel2018quantizes}. Silent alignment describes a
related separation through the evolving tangent kernel, analytically
in linear networks and experimentally in nonlinear networks
\citep{atanasov2022silent}. Nonlinear early-alignment guarantees cover
classification under label-correlation conditions \citep{min2024early},
empirical gradient flow \citep{boursier2025early}, and orthogonal-input
regression under balanced initialization \citep{boursier2022orthogonal}.
Hidden progress in parity learning instead amplifies Fourier features
\citep{barak2022hidden}. Our result connects this early progress to every-direction AGOP
recovery, a bounded-readout gain, and later loss decrease on one
simultaneous-GD trajectory.

\paragraph{Plateaus, scale, and escape.}
Delayed transitions occur in exact deep-linear dynamics \citep{saxe2014exact}.
Alternating directional acquisition and feature growth have a proved
small-initialization limit in diagonal linear networks
\citep{kunin2025alternating}; optimal first-escape directions in deep
ReLU networks have a low-rank bias in deeper layers
\citep{bantzis2026saddle}.
Univariate ReLU plateaus also reflect activation-pattern dynamics
\citep{ainsworth2021plateau}. Scaling determines whether training is lazy
\citep{chizat2019lazy}, so our initialization and GD step are substantive
hypotheses.

\paragraph{Gaussian multi-index learning.}
Staged representation learning for polynomial targets
\citep{damian2022representations} assumes an expected Hessian of rank $r$,
which excludes our central odd links. Gaussian multi-index population
flow with infinitely faster nonparametric link fitting exhibits
Hermite-dependent phases \citep{bietti2025gradient}.
For suitable orthogonal additive targets and initialization, correlation-loss
flow gives all-index recovery with order-$r\log r$ neurons
\citep{simsek2025learning}. Unlike those decoupled neurons, ours remain
coupled through square loss, with trained biases and heads; recovery
concerns the current predictor's AGOP. The emphasis here is the joint
dynamical and spectral guarantee under simultaneous training; our
width prescriptions are sufficient existence bounds and do not improve
that coverage rate. Large first-layer steps followed by readout fitting
provide another regime
\citep{dandi2024giant}.

\paragraph{Hermite structure.}
Information exponents govern high-dimensional online search times
\citep{benarous2021online}. Here the mixture and Sobolev-neighborhood
hypotheses control both directional signal and approximation by learned
profiles; regularity alone is insufficient.

\paragraph{Gradient-based subspaces and AGOP.}
Gradient outer products support active-subspace approximation and
dimension reduction \citep{constantine2014active,yuan2025efficient}.
The neural feature ansatz relates weight geometry and AGOP
\citep{radhakrishnan2024mechanism}, with exact identities under balanced
deep-linear flow and nonlinear counterexamples \citep{tansley2026neural}.
We study the evolution of the AGOP eigenspace during the transient regime
in which the population loss remains nearly constant.
AGOP progress during flat loss is observed in recursive feature machines
for modular arithmetic \citep{mallinar2025emergence}; that alignment uses
the final learned matrix, whereas ours uses the known target subspace.
We control the full current-head AGOP and cutoff eigengap without
assuming proportionality to a weight Gram matrix; bounded refitting
separately measures predictive information.

\section{Population training and representation diagnostics}
\label{sec:setup}

Let $X\sim\gamma_d=N(0,I_d)$, and let $U=(u_1,\ldots,u_r)\in\mathbb R^{d\times r}$
have orthonormal columns, with $1\le r<d$. The teacher depends only on $U^TX$;
$P_U=UU^T$ and $P_{U^\perp}=I_d-P_U$ project onto its subspace and complement.
For functions, $\|g\|_2^2=\mathbb E[g(X)^2]$ and
$\|g\|_{H^1}^2=\|g\|_2^2+\mathbb E[\|\nabla g(X)\|^2]$,
with weak derivatives when needed. Unless stated otherwise, $\|y\|_2=1$.
Vector norms $\|\cdot\|_p$ are $\ell^p$ norms; unadorned norms are Euclidean
for vectors and operator norms for matrices. For symmetric $M$,
$\lambda_k(M)$ are decreasing eigenvalues, $\lambda_{\min}(M)$ is the least,
and $M\succeq H$ means $M-H$ is positive semidefinite.
Write $\varphi(t)=e^{-t^2/2}/\sqrt{2\pi}$, $\Phi(t)=\int_{-\infty}^t\varphi(z)\,dz$,
and $h_k(t)=(-1)^k\varphi^{(k)}(t)/(\sqrt{k!}\varphi(t))$ for the
orthonormal probabilists' Hermite polynomials; thus $h_3(t)=(t^3-3t)/\sqrt6$.

\paragraph{Network and optimization.}
For a fixed $0\le\alpha<1$, put $J_\alpha=1-\alpha$ and
$\sigma_\alpha(t)=\alpha t+J_\alpha\max\{t,0\}$; $\alpha=0$ gives ReLU.
With width $m$, scalar heads $A_{j,n}$, weights $W_{j,n}\in\mathbb R^d$,
and biases $B_{j,n}$ at update $n$, the network and loss are
\begin{equation}
 f_n(x)=\frac1m\sum_{j=1}^m A_{j,n}
      \sigma_\alpha(W_{j,n}^Tx+B_{j,n}),\qquad
 L_n=\mathbb E[(y(X)-f_n(X))^2].
 \label{eq:network}
\end{equation}
Every scalar entry of $A_{j,0},W_{j,0},B_{j,0}$ is independently
$N(0,s^2/d)$ for the same public $s>0$. Thus the output heads are
random and signed. All parameters undergo simultaneous Euclidean
gradient descent with fixed force step $h>0$ and raw step $\eta=mh/2$:
\begin{equation}
 \Theta_{n+1}=\Theta_n-\eta\nabla_\Theta L_n,
 \qquad \Theta_n=((A_{j,n},W_{j,n},B_{j,n}))_{j=1}^m.
 \label{eq:gd}
\end{equation}
For example, $A_{j,n+1}=A_{j,n}+h\mathbb E[(y-f_n)
\sigma_\alpha(W_{j,n}^TX+B_{j,n})]$. Expectations are over $X$ unless
initialization is specified; training uses exact population gradients.
The appendix gives the remaining coordinate updates and kink convention.

\paragraph{Full predictor AGOP.}
Define
\begin{equation}
 \mathsf G_n=\mathbb E[\nabla_x f_n(X)\nabla_x f_n(X)^T],
 \qquad
 \nabla_x f_n(X)=\frac1m\sum_j A_{j,n}
 \sigma_\alpha'(W_{j,n}^TX+B_{j,n})W_{j,n}.
 \label{eq:agop}
\end{equation}
This includes all pairs of neurons, with their current heads.
Whenever $\lambda_r(\mathsf G_n)>\lambda_{r+1}(\mathsf G_n)$,
let $P_n$ be its unique leading rank-$r$ orthogonal projector and set
\begin{equation}
 A_{\min,n}=\lambda_{\min}(U^TP_nU),\qquad
 A_{\mathrm{mean},n}=\frac1r\operatorname{tr}(U^TP_nU).
 \label{eq:alignment}
\end{equation}
These are the minimum and mean squared cosines of the principal
angles. The minimum score requires recovery of every teacher direction.
The theorems establish the required spectral gaps at both compared
checkpoints. A null eigenspace is never completed with an arbitrary
basis to define a favorable alignment.
Multiplying $f_n$ by a nonzero constant rescales its AGOP but preserves
its eigenspaces. These scores therefore measure directional geometry
separately from output amplitude.

\paragraph{Normalized-feature refit.}
The main diagnostic uses the full unprojected feature bank:
\begin{equation}
 \phi_{j,n}(x)=
 \frac{\sigma_\alpha(W_{j,n}^Tx+B_{j,n})}
      {\sqrt{\|W_{j,n}\|^2+B_{j,n}^2}},\qquad
 \mathcal R(n)=\inf_{\substack{\|v\|_2\le32/J_\alpha\\
                    \|v\|_1\le64\sqrt r/J_\alpha}}
 \left\|y-\sum_jv_j\phi_{j,n}\right\|_2^2.
 \label{eq:refit}
\end{equation}
A zero augmented row gives the zero feature. Both coefficient caps
are identical at every checkpoint and independent of $s$; there is no
additional intercept. The readout is fitted only to measure information
available in the features. It does not replace the trained heads in
\eqref{eq:network}.
By positive homogeneity, scaling a weight and its bias by the same
positive factor leaves the normalized feature unchanged. The diagnostic
thus compares feature shapes at a fixed readout budget.

\paragraph{Teacher family for the main theorem.}
For any probability measure $\mu$ on $[1/3,1]$, define
\begin{equation}
 \begin{split}
 g_\mu(t)&=\int_{1/3}^1
 v^{-7/2}(t^3-3vt)e^{-(v^{-1}-1)t^2/2}\,d\mu(v),\\
 N_\mu&=\|g_\mu\|_{L^2(\gamma_1)},\qquad
 q_\mu=g_\mu/N_\mu,\qquad
 y(x)=\frac1{\sqrt r}\sum_{i=1}^r q_\mu(u_i^Tx).
 \end{split}
 \label{eq:mixture-teacher}
\end{equation}
These links belong to $H^1(\gamma_1)$: their mean and first Hermite
coefficient vanish, $\sqrt6\le N_\mu<9$, and $\|y\|_2=1$.
Writing $\delta_v$ for a unit point mass at $v$, $\mu=\delta_1$ gives
$q_\mu=h_3$. Whenever $\mu([1/3,1))>0$,
the link has an infinite Hermite expansion, starting at degree three.
For atomic measures we abbreviate $g_{\delta_v}$ as $g_v$.
The entire link is used in the objective and gradients.
The positive optimized signal and the later-loss scale are
\begin{equation}
 K_*=\frac{J_\alpha}{2\sqrt r\,N_\mu}
 \max_{B\ge0}\left\{\sqrt{\frac B{1+B}}
   \int v^{-3/2}\varphi(\sqrt{B/v})\,d\mu(v)\right\},
 \qquad G_c=\frac{K_*^2}{32768}>0.
 \label{eq:signal}
\end{equation}

\paragraph{SwiGLU convention.}
Section~\ref{sec:swiglu-main} instead uses $\operatorname{SiLU}(t)=t/(1+e^{-t})$ and
\begin{equation}
 f_\theta(x)=b_0+\frac{\kappa}{m}\sum_{j=1}^m
 a_j\operatorname{SiLU}(w_j^Tx+b_j)(z_j^Tx+c_j),\qquad \kappa>0.
 \label{eq:swiglu-network}
\end{equation}
Here $\theta$ collects scalar heads $a_j$, gate/value weights $w_j,z_j\in\mathbb R^d$
and biases $b_j,c_j$; $\kappa$ is a fixed output scale.
Only the intercept is fitted exactly: $b_0=\mathbb E[y-\widetilde f_\theta]$,
where $\widetilde f_\theta=f_\theta-b_0$. The profiled loss is
$L(\theta)=\operatorname{Var}(y-\widetilde f_\theta)$.
This removes the error in fitting the target mean; the plateau concerns
learning its input-dependent variation.
All other parameters train jointly by $\dot\theta=-m\nabla L$ or
$\theta^{n+1}=\theta^n-mh\nabla L(\theta^n)$: the SwiGLU raw GD step
is $mh$, fixed independently of $\varepsilon$.
For $\theta_j=(a_j,(w_j,b_j),(z_j,c_j))$, initialize
$\theta_j(0)=\varepsilon\vartheta_j$ with independent
$\vartheta_j\sim N(0,\operatorname{diag}(\beta_0^2,s_0^2I_{2d+2}))$
and fixed $\beta_0,s_0>0$.
The illustrations use $\beta_0=1$, $s_0=1/\sqrt{d+1}$, and $\kappa=1$.
The SwiGLU theorem normalizes $\operatorname{Var}(y)=1$; its illustrations retain
$\mathbb E[y^2]=1$ and plot $\ell=L/\operatorname{Var}(y)$.
For its full AGOP $\mathsf G$, $A_{\rm top}=\min\{\|P_Ue\|^2:
\|e\|=1,\ \mathsf Ge=\lambda_1(\mathsf G)e\}$.

\section{Subspace acquisition during a loss plateau}
\label{sec:theory}

We state the equal-coefficient, common-link case; Appendix~\ref{app:sw-theorem}
gives the broader statement. The theorem assumes Gaussian inputs, exact
population gradients, the specified teacher, and the appendix's explicit
size, initialization and step conditions.

The acquisition checkpoint $N$ lies inside the loss-controlled window
$0\le n\le N_c$. At $N$, the features already support a better bounded
refit; at a later $N_2$, the original predictor beats every loss value
in that window. Thus improved representation and improved trained
prediction are quantified separately along the same trajectory.

\paragraph{Parameter scope.}
The appendix prescription is part of the theorem's hypotheses.
It supplies explicit occupancy, dimension, mesh, force, and
initialization inequalities in an acyclic order before the random
draw. In particular $m\ge r$ and $d\ge16r/\delta$, but these two
inequalities alone are insufficient. Compatible widths may exceed
$d$; arbitrary triples $(r,d,m)$ are not covered. Width must populate
rare Gaussian categories, and the raw scale and mesh can be
exponentially small in an already large acquisition horizon.
The result gives no practical complexity or sample-size bound.

\begin{theorem}[Population subspace learning on a loss plateau]
\label{thm:main}
Fix $r\ge1$, $\mu$, and $0\le\alpha<1$ as above, and tolerances
$0<\delta<1/2$, $0<\epsilon_L\le1$, $0<\epsilon_G\le1/4$.
Choose $m,d,s,h$ and the public checkpoints $N<N_c$ according to
the complete parameter prescription in
Appendix~\ref{app:sw-parameters}, specialized to
$\mu_i=\mu$, $\lambda_i=r^{-1/2}$ and residual $e=0$.
Train \eqref{eq:network} by \eqref{eq:gd} from the stated IID Gaussian
initialization. With probability at least $1-7\delta/8$, the following
hold simultaneously:
\begin{enumerate}
\item \textbf{Loss plateau.} Throughout $0\le n\le N_c$, the original loss obeys
\begin{equation}
 1-\epsilon_L/16\le L_n\le1+\epsilon_L/16,\quad
 L_n\ge1-G_c,\quad
 \frac{\max_{n\le N_c}L_n}{\min_{n\le N_c}L_n}\le1+\epsilon_L/4.
 \label{eq:plateau-guarantee}
\end{equation}
\item \textbf{Subspace recovery.} The full AGOP has a positive rank-$r$ cutoff gap at $0$ and $N$,
and its principal-angle scores satisfy
\begin{equation}
 \begin{split}
 A_{\min,0}\le A_{\mathrm{mean},0}&\le\tfrac14,\qquad
 A_{\min,N}\ge1-\epsilon_G,\qquad
 A_{\mathrm{mean},N}\ge1-\epsilon_G/r,\\
 A_{\min,N}-A_{\min,0}&\ge\tfrac34-\epsilon_G\ge\tfrac12.
 \end{split}
 \label{eq:angle-guarantee}
\end{equation}
\item \textbf{Refit improvement.} For the identical bounded diagnostic \eqref{eq:refit},
\begin{equation}
 \mathcal R(0)-\mathcal R(N)>0.399.
 \label{eq:refit-guarantee}
\end{equation}
\item \textbf{Later loss decrease.} At a finite subsequent checkpoint $N_2>N_c$, the same unchanged
GD trajectory satisfies
\begin{equation}
 L_{N_2}\le1-2G_c,\qquad
 L_n-L_{N_2}\ge G_c\quad(0\le n\le N_c).
 \label{eq:release-guarantee}
\end{equation}
\end{enumerate}
The teacher subspace $\operatorname{span}(U)$ is minimal. The times
$N,N_c$ are fixed before initialization; $N_2$ may depend on the
realized trajectory.
\end{theorem}

\paragraph{Proof outline.}
The complete discrete-time proof appears in
Appendix~\ref{app:proof-sw}. We summarize the four interfaces that
connect acquisition to the stated observable.

\textit{1. Directional acquisition.} Gaussian occupancy supplies initial neurons in suitable
axis and bias regions. A coupled induction on the actual GD updates
shows that these neurons acquire every teacher direction by a public
time. The induction retains the forces generated by all neurons.
A subsequent interval makes their signal dominate the total signed
defect, including adverse heads and components outside $U$.
The small raw scale keeps the full loss within
\eqref{eq:plateau-guarantee} while these relative changes occur.

\textit{2. From gradients to the AGOP.} The proof controls a matrix of second-Hermite coefficients of
the \emph{complete} input gradient. Write $Z=U^TX$,
$e_r=r^{-1/2}(1,\ldots,1)^T$, and
$\xi=(e_r^TZ)Z-e_r$, so
$\mathbb E[\xi\xi^T]=I+e_re_r^T\preceq2I$.
For the symmetric frame $D=\mathbb E[(U^T\nabla f_N)\xi^T]$, the signed-defect estimate
gives a positive lower bound on $\lambda_{\min}(D)$, as well as
a small complete outside-gradient energy. Bessel's inequality yields
\begin{equation}
 U^T\mathsf G_NU\succeq\tfrac12DD^T.
 \label{eq:frame-bessel}
\end{equation}
Minmax and the positive-semidefinite comparison
$P_N\preceq\mathsf G_N/\lambda_r(\mathsf G_N)$ turn these two
estimates into the spectral gap and both angle bounds. This step
retains the AGOP cross terms. Initially, rotational invariance and
almost-sure simplicity of the positive spectrum imply
$\mathbb E[A_{\mathrm{mean},0}]=r/d$, averaging over initialization.

\textit{3. Prediction from the learned features.} The teacher's absence of Hermite degrees below three bounds
its correlation with an initial normalized feature by the cube of
that feature direction's overlap with $U$. A simultaneous Gaussian
bound and the fixed $\ell^1$ cap imply
$\mathcal R(0)\ge1-10^{-4}/8$. The acquired biased profiles admit
a readout satisfying both caps with
$\mathcal R(N)<(\sqrt{3/5}+10^{-4}/64)^2$.
Their difference exceeds $0.399$.

\textit{4. Loss release on the same trajectory.} A full-energy continuation estimate turns the acquired
signal into later loss decrease under the original step size.
The acquisition, initial-refit, and initial-projector events have
failure probabilities at most $3\delta/8$, $\delta/4$, and
$\delta/4$, respectively. A union bound proves the simultaneous
claim without assuming these events are independent.

\paragraph{Broader mixture teachers.}
The full theorem in Appendix~\ref{app:sw-theorem} permits different
measures $\mu_i$, positive coefficients $\lambda_i$ with
$\sum_i\lambda_i^2=1$, and a normalized target
$y=\sum_i\lambda_iq_{\mu_i}(u_i^T x)+e(U^Tx)$.
The residual obeys every explicit $H^1$ bound in the parameter
prescription and may contain interactions within $U$.
Each directional signal $K_i^*$ is defined by \eqref{eq:signal}
with $(r^{-1/2},\mu)$ replaced by $(\lambda_i,\mu_i)$;
the assumption is $\min_iK_i^*\ge(5/6)\max_iK_i^*$.
For a common link this allows a coefficient ratio at most $6/5$.
Neither arbitrary signed mixtures nor arbitrary $H^1$ teachers
satisfy this hypothesis.

\paragraph{Unequal coefficients near an additive cubic.}
A complementary theorem, proved in Appendix~\ref{app:proof-aj},
starts from
\begin{equation}
 y_c(x)=\sum_{i=1}^r a_i h_3(u_i^Tx),\qquad
 a_i>0,\quad\sum_i a_i^2=1,\qquad
 \|y-y_c\|_{H^1}\le\nu/2,
 \label{eq:additive-teacher}
\end{equation}
where $\nu>0$ is the prescribed perturbation tolerance and the unit-norm
teacher $y$ must remain \emph{additive} on $U$.
No ratio bound is imposed on the positive $a_i$; the
smallest coefficient enters the public costs. Under the complete
prescription in Appendix~\ref{app:aj-parameters}, with probability
at least $1-\delta$, its public acquisition checkpoint $N_A$ has minimum-AGOP
gain at least $1/2$ during a high-loss plateau. At that checkpoint,
the projected diagnostic
\begin{equation}
 \mathcal R^P_B(n)=\inf_{\|v\|_2\le B}
 \left\|y-\sum_jv_j
 \sigma_\alpha((P_nW_{j,n})^TX+B_{j,n})\right\|_2^2
 \label{eq:projected-refit}
\end{equation}
improves by more than $1/2$, both for $B=C_\alpha/(7s)$ and
for $B=\infty$. The explicit profile constant $C_\alpha$ and the
much stricter admissible angle tolerance are given in that appendix.
Projection is before activation, biases are retained, and rows are
not renormalized. The same unchanged GD later reduces loss below
every plateau value by at least
$G_A=1/(8192\overline T_A^2)>0$. Here $N_A,\overline T_A$ denote
the quantities $N,\overline T$ in that parameter prescription.
Thus this extension uses a different feature
class and budget from \eqref{eq:refit}.

\Needspace{10\baselineskip}
\section{Population illustrations across teacher links}
\label{sec:experiments}

Twenty-one normalized teachers share a ReLU-student configuration:
$r=2$, $d=16$, $m=32$, $s=10^{-4}$. All parameters train from IID
$N(0,s^2/d)$ entries, with fifty predetermined seeds per teacher and
a 20,000-update population-GD budget. These are illustrative parameters, outside
the theorem's sufficient prescriptions. Labeled followups use $d=64$, $s=0.01$
(Appendix~\ref{app:followup-protocol}).

\paragraph{Full teachers and population evaluation.}
For each of the thirteen scalar raw links $q$, we retain its actual mean
and set
\begin{equation}
 y(X)=\frac{\sum_{i=1}^r q(X_i)}
 {\sqrt r\sqrt{\mathbb E[q(Z)^2]+(r-1)(\mathbb E[q(Z)])^2}},
 \qquad Z\sim N(0,1).
 \label{eq:experiment-normalization}
\end{equation}
Six product targets and two additive SiLU controls complete the family
(Appendix~\ref{app:experiments}); the interactions are exploratory
$H^1(\gamma_2)$ examples outside the plateau theorems. Gradients use
analytic student moments and validated deterministic teacher quadrature,
without a finite training set. Here population refers to expectation
over the input distribution. Gaussian weights specify initialization
only; no Gaussian law is assumed for the parameters after training
begins. The reported ReLU loss is the full MSE
$\mathbb E[(f(X)-y(X))^2]$. Its unit baseline comes from teacher
normalization, not from dividing each trajectory by its own initial loss.

\paragraph{What the averaged curves measure.}
Figure~\ref{fig:main-trajectories} shows three teachers; the appendix
gives the others. Means and 10th--90th percentiles retain all fifty seeds,
including bounds for unresolved AGOP scores. The shared display expands
early time without seed-specific alignment or smoothing.

Each seed's loss-only plateau is the maximal initial prefix with MSE in
$[0.95,1.05]$ and running maximum/minimum at most $1.05$, checked at every
update. Shading marks the common prefix; training runs to the fixed horizon.

\paragraph{Prediction and geometry are separate diagnostics.}
Refitting uses the unprojected, normalized bank in \eqref{eq:refit},
with fixed caps $\|v\|_2\le32$, $\|v\|_1\le64\sqrt2$ and no intercept;
it never enters training. Separate ensemble means do not establish
simultaneous alignment and prediction improvement for every initialization.
Table~\ref{tab:ensemble-means} reports continuous endpoint gains and
later original loss. These numerical illustrations do not test the
theorem's sufficient parameter assumptions.

Neuron angles and axis-snapped refits probe basis-dependent specialization:
\emph{monosemantic} means concentration near teacher axes;
\emph{polysemantic} means mixing their coordinates. Snapping projects each
teacher-plane component onto its nearest frame axis, preserving bias,
outside-plane component and normalization (Appendix~\ref{app:frame-diagnostics}).
A small snapping penalty alone does not prove specialization.
The quadratic-gate target remains at squared $L^2$ distance at least
$1/4$ from every orthonormal two-axis additive class
(Proposition~\ref{prop:orthogonal-frame-floor}).

\paragraph{Optimization.}
Runs start at $h=0.05$ (raw step $\eta=mh/2=0.8$).
Armijo backtracking halves rejected steps and carries the accepted step
forward (Appendix~\ref{app:experiments}); the theorems use fixed steps.

\subsection{SwiGLU: feature learning through a smooth gate}
\label{sec:swiglu-main}

SwiGLU's smooth gate in \eqref{eq:swiglu-network} permits directional
learning while its nonconstant output remains small.

\begin{theorem}[SwiGLU plateau and leading AGOP alignment]
\label{swb:thm:alignment}
Fix finite $m\ge1$, $1\le r<d$, and $\kappa,\beta_0,s_0,h>0$.
Let $y=F(U^TX)$ with $F\in L^2(\gamma_r)$, $U^TU=I_r$, and
$\operatorname{Var}(y)=1$. Assume that $F$ has a nonzero Hermite
component of total degree one, two, or three. Put $k=3$ if degree one or
two is nonzero, and $k=4$ otherwise; higher degrees are unrestricted.
Use the SwiGLU model, profiled intercept, Gaussian initialization and
updates of Section~\ref{sec:setup}, with the same Gaussian draw
$\vartheta$ coupled across $\varepsilon$.

There is an initialization event $\mathcal E$ with
$\mathbb P(\mathcal E)\ge1-2^{-m}$ such that, almost surely on
$\mathcal E$, for every $\delta\in(0,1)$ there are
$\tau_1,C>0$ and $\varepsilon_0=\varepsilon_0(\vartheta,\delta,h)>0$
for which every $0<\varepsilon<\varepsilon_0$ satisfies, at
$n_1=\lfloor\tau_1/(h\varepsilon^{k-2})\rfloor$,
\begin{equation}
 \max_{0\le n\le n_1}|L_n-L_0|\le C\varepsilon^k,
 \quad L_0=1+O_\vartheta(\varepsilon^k),
 \quad A_{{\rm top},n_1}\ge1-\delta.
 \label{eq:swiglu-main-guarantee}
\end{equation}
For gradient flow the same conclusions hold at
$t_1=\tau_1\varepsilon^{-(k-2)}$, with the maximum replaced by
the supremum over $0\le t\le t_1$.
The constants $\tau_1,C$ depend on $(\vartheta,\delta)$;
the initialization-dependent cutoff gives no explicit finite-scale
confidence bound.
\end{theorem}

Theorem~\ref{swb:thm:alignment} controls geometry. To also quantify
prediction, we restrict to rank-one teachers with cubic signal and no
linear or quadratic Hermite component. At a specified width, the next
result gives alignment and unrestricted-refit gains at the same
checkpoint during the plateau.

\begin{corollary}[Joint SwiGLU alignment and refit gains]
\label{cor:swiglu-joint-main}
In Theorem~\ref{swb:thm:alignment}, take $r=1$, $d\ge2$ and
$m=D=(d+1)(d+2)/2$. Let $y=F(u^TX)$ have unit variance, vanishing
first and second Hermite coefficients, and
$a_3=\mathbb E[F(Z)h_3(Z)]\ne0$, $Z\sim N(0,1)$.
Write $\mathcal R_{\rm free}$ for minimum population MSE after unrestricted
refitting of the frozen SwiGLU features and an intercept
as defined in \eqref{swr:eq:refit}. Fix $q_{\rm A},q_{\rm R}\in(0,1)$ and
$\delta_{\rm A},\delta_{\rm R}>0$ with $q_{\rm A}+\delta_{\rm A}<1$.
For flow or any fixed-step GD, at a common endpoint $*$ the joint event
\[
 A_{{\rm top},*}-A_{{\rm top},0}\ge1-\delta_{\rm A}-q_{\rm A},\qquad
 \mathcal R_{\rm free}(0)-\mathcal R_{\rm free}(*)
 \ge a_3^2(1-q_{\rm R})-\delta_{\rm R}
\]
along with the entire-prefix loss bound $\max|L-L_0|=O(\varepsilon^4)$
has limiting inferior probability, as $\varepsilon\downarrow0$, at least
\[
 \max\!\left\{0,\Pr(B_d\le q_{\rm A})-2^{-D}
       -\frac{3}{q_{\rm R}d(d+2)}\right\},
 \qquad B_d\sim\operatorname{Beta}(1/2,(d-1)/2).
\]
The endpoint time scales as $\varepsilon^{-2}$; its rescaled time,
constants and initialization cutoff depend on the Gaussian draw.
The rescaled time is fixed before $\varepsilon$ shrinks.
\end{corollary}

A zero bound is vacuous; high probability requires suitable dimensions
and thresholds.
For pure $h_3$, $d=16$ and $m=153$, the gains can be at least $0.5$
and $0.7$, with asymptotic probability lower bound approximately $0.9519$
(Corollary~\ref{swj:cor:joint}). No practical initialization cutoff is
supplied; refit coefficients may diverge as $\varepsilon\to0$.
Later decrease of the trained loss remains open.

The two guarantees in Theorem~\ref{swb:thm:alignment} have different
time quantifiers: the loss bound holds throughout the prefix, whereas
the alignment bound is asserted at its endpoint. Thus the selected
checkpoint lies inside an interval of uniformly small loss movement.
The exponent $k$ records the degree of the leading student--teacher
interaction. A nonzero linear or quadratic Hermite component gives
$k=3$, an endpoint time proportional to $\varepsilon^{-1}$, and
$O(\varepsilon^3)$ loss variation. When degrees one and two vanish
and cubic signal remains, $k=4$ gives the scales $\varepsilon^{-2}$ and
$O(\varepsilon^4)$. These are guaranteed observation windows, rather
than formulas for the eventual plateau exit.

For $r>1$, a large $A_{\rm top}$ places the leading AGOP eigenspace
close to the teacher subspace; it does not establish recovery of every
teacher direction. The joint corollary adds the rank-one cubic-signal
and width assumptions under which a simultaneous refit improvement
is proved.
The numerical comparison in Figure~\ref{fig:swiglu-main} uses
$d=16$, $m=64$ and $\varepsilon=.05$, whereas the stated quantitative
corollary example uses $m=153$. The plotted finite-initialization trajectories
therefore illustrate the separation of alignment, refit and trained
loss at another width, without inheriting the example's asymptotic
probability bound.

\paragraph{Proof overview.}
Appendix~\ref{swb:sec:alignment} expands the smooth gate near initialization.
After rescaling parameters and time, the leading dynamics become
independent homogeneous gradient flows for the neurons; the network's
feedback enters at higher order. On the escape event, the earliest
comparator escape is almost surely unique. The growing neuron's spatial components
orthogonal to the teacher subspace remain fixed, while its contribution
within that subspace dominates the leading AGOP. A fixed rescaled time
before escape therefore gives the required alignment. Uniform tracking
transfers this conclusion to population gradient flow and fixed-step GD,
and the loss expansion controls the whole preceding interval.

For Corollary~\ref{cor:swiglu-joint-main},
Appendix~\ref{swr:sec:refit} uses the quadratic and cubic terms of the
feature expansion. Near comparator escape, at the stated width, the nonwinning
neurons' leading quadratic polynomials and the intercept almost surely
span all polynomials of degree at most two. Cancellation using the actual rescaled parameters removes
the quadratic contribution; normalization and a linear correction then
give a feature combination approximating $h_3$ along the teacher direction.
This controls the attained refit error. Corollary~\ref{swj:cor:joint}
compares it with initialization, chooses a common endpoint for prediction
and alignment before shrinking $\varepsilon$, and combines the escape
and initial-headroom probability bounds without assuming independence.
These steps yield the simultaneous gains stated above.

\begin{figure}[t]
\centering
\includegraphics[width=\textwidth]{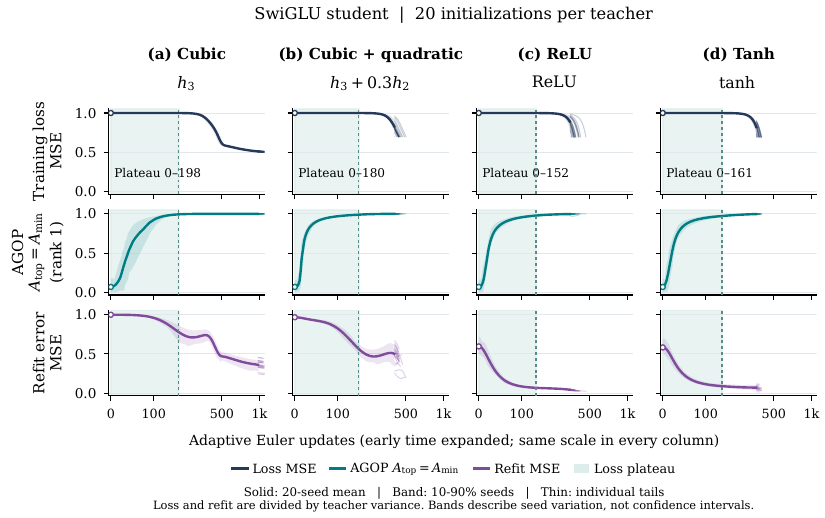}
\caption{\textbf{SwiGLU students: twenty initializations per teacher.}
All seeds 0--19; $d=16$, $m=64$, $\varepsilon=.05$.
Loss and numerical unrestricted-refit MSE are divided by teacher variance;
rank-one AGOP alignment is $A_{\rm top}=A_{\min}$.
Navy, teal and purple consistently denote loss, AGOP alignment and refit MSE.
Solid means and light 10th--90th percentile bands use all twenty seeds
on their shared support; later individual tails remain visible.
Shading marks the shared loss-only prefix $|\ell-\ell(0)|\le10^{-3}$.
The axis counts adaptive Euler updates, with early updates expanded.
Diagnostic interpolation is for display only; bands show seed variation,
not confidence intervals or integration-error bounds. Crosses, if present,
mark the last finite point of censored or failed trajectories.}
\label{fig:swiglu-main}
\end{figure}

\clearpage
\begin{figure}[p]
\centering
\includegraphics[width=\textwidth]{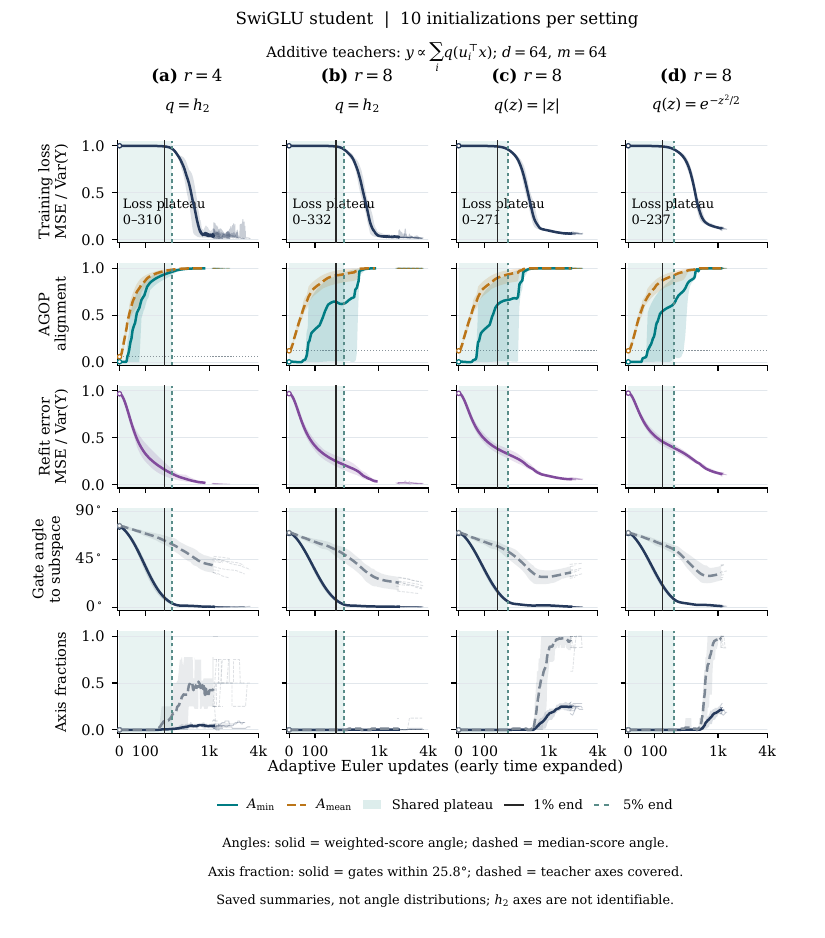}
\caption{\textbf{SwiGLU subspace learning and gate geometry across ranks and links.}
Ten initializations per setting, $d=64$, $m=64$;
additive quadratic, absolute-value and Gaussian-bump teachers vary
rank and link. Loss and numerical refit risk are divided by teacher variance.
The first three rows retain the loss, AGOP and refit conventions;
the dotted $r/d$ line references isotropic $A_{\rm mean}$.
The fourth row shows effective gate angles obtained from weighted
and median squared-cosine scores; these are saved scalar summaries,
not per-neuron angle distributions. The last row reports the fraction
of gates within $25.8^\circ$ of a teacher axis and the fraction of
axes covered by such gates. Quadratic-teacher axes remain unidentifiable.
Shading uses a common 5\% loss-ratio prefix, with a 1\% comparison
boundary; this is broader than the $|\ell_n-\ell_0|\le10^{-3}$ prefix
retained in Figure~\ref{fig:swiglu-main}.
All ten seeds contribute to means and 10th--90th percentile bands;
individual tails and diagnostic gaps remain visible.
Gate entry and axis specialization are distinct observations;
neither alone establishes a prediction gain.
Appendix~\ref{app:highrank} gives definitions, numerical diagnostics,
threshold sensitivity and rank-16 limitations.}
\label{fig:highrank-swiglu-main}
\end{figure}

\clearpage
\section{Discussion}
\label{sec:discussion}

\paragraph{What a plateau can conceal.}
Theorem~\ref{thm:main} links recovery of every teacher direction by the
leading AGOP subspace, an improved bounded readout, and a later decrease
of the original predictor's loss on one trajectory. The first two
conclusions hold while loss remains close to the zero predictor's unit
loss. A plateau controls variation over a finite interval; it need not
be exactly flat or converge to a positive error floor.

Small initialization permits feature geometry to change while the
output remains small. Scaling all ReLU or leaky-ReLU heads, weights
and biases by one positive factor scales the output quadratically.
Geometry and output amplitude therefore need not evolve on the same
scale. This is a dynamical mechanism, not a normalization that hides
loss decrease. GD uses a fixed step throughout, subject to the theorem's
restrictive initialization and step prescriptions.

\paragraph{Geometry and predictive information.}
The minimum principal-angle score tests the weakest teacher direction;
a high leading-direction score alone may represent only one direction
of a higher-rank teacher. The full AGOP also depends on the current
readout, so its motion cannot by itself demonstrate that the hidden
feature bank has improved. The same-budget refit addresses this separate
question: both checkpoints use the same feature normalization, coefficient
constraints and target. Its improvement measures additional predictive
information accessible under that budget. Refitting never changes the
heads used by GD. The complementary unequal-coefficient theorem instead
uses projected features and a different budget, and the SwiGLU
corollary uses an unrestricted refit; these conclusions should not be
read as interchangeable.

\paragraph{Relation to early alignment and lazy training.}
Early directional learning and silent alignment already show that
representations can change before loss falls appreciably
\citep{maennel2018quantizes,atanasov2022silent}. Our contribution is the
joint all-direction, bounded-refit and later-loss guarantee for the
specified nonlinear population-GD system. Lazy training concerns the
accuracy of linearization around initialization and the behavior of the
tangent kernel, rather than the speed of loss reduction
\citep{chizat2019lazy}. Neither a flat loss nor a moving AGOP determines
whether a trajectory is lazy. Comparing the learned hidden bank with its
initial version is informative, but does not establish a separation from
the full initial tangent-feature model. Such a comparison would require
additional analysis or a separate diagnostic.

\paragraph{Teacher regularity and neuron mixing.}
Gaussian Sobolev regularity provides a common language for the links,
while the positive results use additional structure. The main ReLU
family contains $h_3$ and positive mixtures of Gaussian-damped cubics
with infinite Hermite expansions. Its extension allows carefully
controlled changes in directional signals and a small Sobolev residual;
the complementary theorem permits unequal positive cubic coefficients
and a small additive perturbation. These are explicit families within
$H^1$, rather than a guarantee for every Sobolev link. Extra Hermite
components can alter directional forces, and target energy in components
not yet acquired still matters for prediction.

The interaction teachers $g(S)T/\|g\|_2$, where $S=u_1^TX$,
$T=u_2^TX$, and $g$ is SiLU, ReLU, GELU, $\tanh$, or $h_2$, supply
further Gaussian-Sobolev examples in the experiments. They lie outside
the ReLU plateau theorems. Their neuron-angle and axis-snapping
diagnostics illustrate why recovering a subspace need not require
specialization to its coordinate axes. For $g=h_2$,
Proposition~\ref{prop:orthogonal-frame-floor} gives MSE at least $1/4$
when all teacher-plane neuron components lie on one shared orthonormal
pair of axes. This approximation obstruction does not itself prove
which neuron configurations training selects.

\paragraph{Why study a gate?}
SwiGLU exposes a complementary route because each feature multiplies a
smooth gate by a trainable value response. Both factors can change
direction while their product and the resulting network output remain
small. Theorem~\ref{swb:thm:alignment} formalizes leading-direction
learning for a broader low-degree Hermite signal class, at fixed finite
width and dimension as initialization vanishes. Its profiled intercept,
initialization-dependent cutoff and leading-direction conclusion differ
from the ReLU result. All-direction recovery and later trained-loss
release remain open in this setting. The restricted rank-one corollary
adds refit improvement, but its unconstrained readout can grow without
bound as initialization shrinks.

\paragraph{Limits and next questions.}
The quantitative ReLU prescriptions are sufficient existence conditions,
not practical width or dimension bounds. Population gradients average
over the input distribution; finite-sample guarantees require further
work. Experiments at accessible scales illustrate mechanisms rather
than verify all theorem hypotheses. Higher-rank evaluation must retain
the weakest-direction score and check alignment and prediction on the
same runs and time windows. For smooth ungated students such as SiLU and GELU, trainable biases
can expose cubic correlations and collective responses can carry
directional information at small output scale. Transferring that
intuition into the same all-direction and bounded-refit guarantee
requires control of the coupled dynamics. This is a route for further
work, not a theorem asserted here. Controlling broader Hermite mixtures, sharpening parameter costs, and
establishing persistence after acquisition are substantive remaining
problems.

\section{Conclusion}
\label{sec:conclusion}
Loss alone can miss substantial progress in a representation. Tracking
loss, all-direction AGOP geometry and a controlled refit separately makes
that progress measurable and distinguishes an acquired subspace from a
predictor that already uses it effectively. Our results establish this
separation in explicit population regimes and identify the additional
conditions needed to connect feature acquisition to prediction.

\clearpage
\begingroup
  \edef\MainEndPage{\number\numexpr\value{page}-1\relax}
  \renewcommand{\thepage}{\MainEndPage}
  \label{main:end}
  \typeout{PUBLIC-PAGE-CHECK: main text ends on p. \MainEndPage}
\endgroup

\section*{AI use}
Generative AI tools assisted with problem formulation and hypothesis
development, mathematical modeling and theorem formulation, proof
development and writing, experimental design, implementation and debugging,
numerical diagnostics, interpretation of results, literature discovery,
and manuscript and figure preparation. The authors are responsible for the
correctness, originality, and final content of this work. AI-assisted
reviews and numerical checks do not constitute formal verification.

\bibliographystyle{iclr2027_conference}
\bibliography{references}

\begin{thebibliography}{21}
\providecommand{\natexlab}[1]{#1}
\providecommand{\url}[1]{\texttt{#1}}
\expandafter\ifx\csname urlstyle\endcsname\relax
  \providecommand{\doi}[1]{doi: #1}\else
  \providecommand{\doi}{doi: \begingroup \urlstyle{rm}\Url}\fi

\bibitem[Ainsworth \& Shin(2021)Ainsworth and Shin]{ainsworth2021plateau}
Mark Ainsworth and Yeonjong Shin.
\newblock Plateau phenomenon in gradient descent training of {ReLU} networks:
  Explanation, quantification, and avoidance.
\newblock \emph{SIAM Journal on Scientific Computing}, 43\penalty0
  (5):\penalty0 A3438--A3468, 2021.
\newblock \doi{10.1137/20M1353010}.
\newblock URL \url{https://epubs.siam.org/doi/10.1137/20M1353010}.

\bibitem[Atanasov et~al.(2022)Atanasov, Bordelon, and
  Pehlevan]{atanasov2022silent}
Alexander Atanasov, Blake Bordelon, and Cengiz Pehlevan.
\newblock Neural networks as kernel learners: The silent alignment effect.
\newblock In \emph{International Conference on Learning Representations}, 2022.
\newblock URL \url{https://openreview.net/forum?id=1NvflqAdoom}.

\bibitem[Bantzis et~al.(2026)Bantzis, Simon, and Jacot]{bantzis2026saddle}
Ioannis Bantzis, James~B. Simon, and Arthur Jacot.
\newblock Saddle-to-saddle dynamics in deep {ReLU} networks: Low-rank bias in
  the first saddle escape.
\newblock In \emph{International Conference on Learning Representations}, 2026.
\newblock URL \url{https://openreview.net/forum?id=B4zcoLvjw0}.

\bibitem[Barak et~al.(2022)Barak, Edelman, Goel, Kakade, Malach, and
  Zhang]{barak2022hidden}
Boaz Barak, Benjamin~L. Edelman, Surbhi Goel, Sham Kakade, Eran Malach, and
  Cyril Zhang.
\newblock Hidden progress in deep learning: {SGD} learns parities near the
  computational limit.
\newblock In \emph{Advances in Neural Information Processing Systems},
  volume~35, pp.\  21750--21764, 2022.
\newblock \doi{10.52202/068431-1581}.
\newblock URL
  \url{https://proceedings.neurips.cc/paper_files/paper/2022/hash/884baf65392170763b27c914087bde01-Abstract-Conference.html}.

\bibitem[Ben~Arous et~al.(2021)Ben~Arous, Gheissari, and
  Jagannath]{benarous2021online}
G{\'e}rard Ben~Arous, Reza Gheissari, and Aukosh Jagannath.
\newblock Online stochastic gradient descent on non-convex losses from
  high-dimensional inference.
\newblock \emph{Journal of Machine Learning Research}, 22\penalty0
  (106):\penalty0 1--51, 2021.
\newblock URL \url{https://jmlr.org/papers/v22/20-1288.html}.

\bibitem[Bietti et~al.(2025)Bietti, Bruna, and
  Pillaud-Vivien]{bietti2025gradient}
Alberto Bietti, Joan Bruna, and Loucas Pillaud-Vivien.
\newblock On learning {Gaussian} multi-index models with gradient flow part
  {I}: General properties and two-timescale learning.
\newblock \emph{Communications on Pure and Applied Mathematics}, 78\penalty0
  (12):\penalty0 2354--2435, 2025.
\newblock \doi{10.1002/cpa.70006}.
\newblock URL \url{https://doi.org/10.1002/cpa.70006}.

\bibitem[Boursier \& Flammarion(2025)Boursier and
  Flammarion]{boursier2025early}
Etienne Boursier and Nicolas Flammarion.
\newblock Early alignment in two-layer networks training is a two-edged sword.
\newblock \emph{Journal of Machine Learning Research}, 26\penalty0
  (183):\penalty0 1--75, 2025.
\newblock URL \url{https://jmlr.org/papers/v26/24-1523.html}.

\bibitem[Boursier et~al.(2022)Boursier, Pillaud-Vivien, and
  Flammarion]{boursier2022orthogonal}
Etienne Boursier, Loucas Pillaud-Vivien, and Nicolas Flammarion.
\newblock Gradient flow dynamics of shallow {ReLU} networks for square loss and
  orthogonal inputs.
\newblock In \emph{Advances in Neural Information Processing Systems},
  volume~35, pp.\  20105--20118, 2022.
\newblock \doi{10.52202/068431-1462}.
\newblock URL
  \url{https://proceedings.neurips.cc/paper_files/paper/2022/hash/7eeb9af3eb1f48e29c05e8dd3342b286-Abstract-Conference.html}.

\bibitem[Chizat et~al.(2019)Chizat, Oyallon, and Bach]{chizat2019lazy}
L{\'e}na{\"i}c Chizat, Edouard Oyallon, and Francis Bach.
\newblock On lazy training in differentiable programming.
\newblock In \emph{Advances in Neural Information Processing Systems},
  volume~32, 2019.
\newblock URL
  \url{https://proceedings.neurips.cc/paper_files/paper/2019/hash/ae614c557843b1df326cb29c57225459-Abstract.html}.

\bibitem[Constantine et~al.(2014)Constantine, Dow, and
  Wang]{constantine2014active}
Paul~G. Constantine, Eric Dow, and Qiqi Wang.
\newblock Active subspace methods in theory and practice: Applications to
  kriging surfaces.
\newblock \emph{SIAM Journal on Scientific Computing}, 36\penalty0
  (4):\penalty0 A1500--A1524, 2014.
\newblock \doi{10.1137/130916138}.
\newblock URL \url{https://doi.org/10.1137/130916138}.

\bibitem[Damian et~al.(2022)Damian, Lee, and
  Soltanolkotabi]{damian2022representations}
Alexandru Damian, Jason Lee, and Mahdi Soltanolkotabi.
\newblock Neural networks can learn representations with gradient descent.
\newblock In \emph{Proceedings of Thirty Fifth Conference on Learning Theory},
  volume 178 of \emph{Proceedings of Machine Learning Research}, pp.\
  5413--5452. PMLR, 2022.
\newblock URL \url{https://proceedings.mlr.press/v178/damian22a.html}.

\bibitem[Dandi et~al.(2024)Dandi, Krzakala, Loureiro, Pesce, and
  Stephan]{dandi2024giant}
Yatin Dandi, Florent Krzakala, Bruno Loureiro, Luca Pesce, and Ludovic Stephan.
\newblock How two-layer neural networks learn, one (giant) step at a time.
\newblock \emph{Journal of Machine Learning Research}, 25\penalty0
  (349):\penalty0 1--65, 2024.
\newblock URL \url{https://www.jmlr.org/papers/v25/23-1543.html}.

\bibitem[Kunin et~al.(2025)Kunin, Marchetti, Chen, Karkada, Simon, DeWeese,
  Ganguli, and Miolane]{kunin2025alternating}
Daniel Kunin, Giovanni~Luca Marchetti, Feng Chen, Dhruva Karkada, James~B.
  Simon, Michael~R. DeWeese, Surya Ganguli, and Nina Miolane.
\newblock Alternating gradient flows: A theory of feature learning in two-layer
  neural networks.
\newblock In \emph{Advances in Neural Information Processing Systems},
  volume~38, 2025.
\newblock \doi{10.52202/085713-0156}.
\newblock URL
  \url{https://papers.nips.cc/paper_files/paper/2025/hash/06cbd2e81dfbd3bb4cb0abce95b32584-Abstract-Conference.html}.

\bibitem[Maennel et~al.(2018)Maennel, Bousquet, and
  Gelly]{maennel2018quantizes}
Hartmut Maennel, Olivier Bousquet, and Sylvain Gelly.
\newblock Gradient descent quantizes {ReLU} network features, 2018.
\newblock URL \url{https://arxiv.org/abs/1803.08367}.

\bibitem[Mallinar et~al.(2025)Mallinar, Beaglehole, Zhu, Radhakrishnan, Pandit,
  and Belkin]{mallinar2025emergence}
Neil~Rohit Mallinar, Daniel Beaglehole, Libin Zhu, Adityanarayanan
  Radhakrishnan, Parthe Pandit, and Mikhail Belkin.
\newblock Emergence in non-neural models: grokking modular arithmetic via
  average gradient outer product.
\newblock In \emph{Proceedings of the 42nd International Conference on Machine
  Learning}, volume 267 of \emph{Proceedings of Machine Learning Research},
  pp.\  42834--42856. PMLR, 2025.
\newblock URL \url{https://proceedings.mlr.press/v267/mallinar25a.html}.

\bibitem[Min et~al.(2024)Min, Mallada, and Vidal]{min2024early}
Hancheng Min, Enrique Mallada, and Rene Vidal.
\newblock Early neuron alignment in two-layer {ReLU} networks with small
  initialization.
\newblock In \emph{International Conference on Learning Representations}, 2024.
\newblock URL
  \url{https://proceedings.iclr.cc/paper_files/paper/2024/hash/0f07eb4d39094358667640e4bd9f2c9d-Abstract-Conference.html}.

\bibitem[Radhakrishnan et~al.(2024)Radhakrishnan, Beaglehole, Pandit, and
  Belkin]{radhakrishnan2024mechanism}
Adityanarayanan Radhakrishnan, Daniel Beaglehole, Parthe Pandit, and Mikhail
  Belkin.
\newblock Mechanism for feature learning in neural networks and
  backpropagation-free machine learning models.
\newblock \emph{Science}, 383\penalty0 (6690):\penalty0 1461--1467, 2024.
\newblock \doi{10.1126/science.adi5639}.
\newblock URL \url{https://doi.org/10.1126/science.adi5639}.

\bibitem[Saxe et~al.(2014)Saxe, McClelland, and Ganguli]{saxe2014exact}
Andrew~M. Saxe, James~L. McClelland, and Surya Ganguli.
\newblock Exact solutions to the nonlinear dynamics of learning in deep linear
  neural networks.
\newblock In \emph{International Conference on Learning Representations}, 2014.
\newblock URL \url{https://arxiv.org/abs/1312.6120}.

\bibitem[{\c{S}}im{\c{s}}ek et~al.(2025){\c{S}}im{\c{s}}ek, Bendjeddou, and
  Hsu]{simsek2025learning}
Berfin {\c{S}}im{\c{s}}ek, Amire Bendjeddou, and Daniel Hsu.
\newblock Learning {Gaussian} multi-index models with gradient flow: Time
  complexity and directional convergence.
\newblock In \emph{Proceedings of The 28th International Conference on
  Artificial Intelligence and Statistics}, volume 258 of \emph{Proceedings of
  Machine Learning Research}, pp.\  4204--4212. PMLR, 2025.
\newblock URL \url{https://proceedings.mlr.press/v258/simsek25a.html}.

\bibitem[Tansley et~al.(2026)Tansley, Massart, and Cartis]{tansley2026neural}
Edward Tansley, Estelle Massart, and Coralia Cartis.
\newblock On the {Neural Feature Ansatz} for deep neural networks.
\newblock In \emph{Proceedings of the 29th International Conference on
  Artificial Intelligence and Statistics}, volume 300 of \emph{Proceedings of
  Machine Learning Research}, pp.\  5077--5085. PMLR, 2026.
\newblock URL \url{https://proceedings.mlr.press/v300/tansley26a.html}.

\bibitem[Yuan et~al.(2025)Yuan, Xu, Kpotufe, and Hsu]{yuan2025efficient}
Gan Yuan, Mingyue Xu, Samory Kpotufe, and Daniel Hsu.
\newblock Efficient estimation of the central mean subspace via smoothed
  gradient outer products.
\newblock \emph{SIAM Journal on Mathematics of Data Science}, 7\penalty0
  (3):\penalty0 1241--1264, 2025.
\newblock \doi{10.1137/23M1626700}.
\newblock URL \url{https://doi.org/10.1137/23M1626700}.

\end{thebibliography}

\clearpage
\begingroup
\setcounter{tocdepth}{1}
\makeatletter
\renewcommand{\l@section}{\@dottedtocline{1}{0em}{2.4em}}
\makeatother
\renewcommand{\contentsname}{Table of contents}
\tableofcontents
\endgroup
\clearpage
\appendix
\section{Numerical protocol and complete teacher comparison}
\label{app:experiments}

\paragraph{Run matrix and reproducible initialization.}
The original baseline comparison has twenty-one teacher functions, one ReLU
student configuration, and the fixed initialization seeds 1--50 for
each teacher: 1,050 trajectories. The sizes are $r=2,d=16,m=32$, with
all raw heads, spatial weights and biases drawn independently from
$N(0,10^{-8}/d)$. The same seed produces the same initial parameter
arrays across teachers; seed outcomes across different teachers are
therefore paired, not independent additional trials. The configuration
is inherited from the earlier illustrations. Seeds 1--3 were previously
inspected; the remaining forty-seven extend that comparison. No seed
is selected or removed because of its trajectory.

The endpoint-gain table below retains this original baseline. Four replacement teacher panels
use the separate followup protocol in Section~\ref{app:followup-protocol}.

All parameters train throughout. Every trajectory has a fixed budget
of 20,000 updates; neither the loss plateau nor feature learning stops
training. The earlier 114-run ReLU/leaky-ReLU/Adam study and the separate
fifteen-run gated extension remain archived under their original
protocols. They are not pooled as extra seeds in these averages.

\paragraph{Computational environment.}
Seeds 1--20 use Python 3.11.4 on macOS and fresh seeds 21--50 use
Python 3.11.10 on Linux, with identical frozen numerical source and
NumPy 1.24.3, SciPy 1.10.1, and Clarabel 0.11.1. Disjoint scheduling
uses six local and twenty-four remote CPU workers, with one numerical
thread per worker. Three fixed 1,000-update portability checks
(additive cubic, SiLU product, and quadratic product at seed 1)
agree with the corresponding local trajectory prefixes to within
$2\times10^{-15}$ in the compared loss, step, path, and parameter
arrays. This check does not assert bitwise equality of all subsequent
trajectories across platforms. Completed transfers retain source hashes;
no partial trajectory is resumed in a different runtime.

\paragraph{Population gradient descent.}
For full MSE $L$ and parameter vector $\theta$, write
$F=-(m/2)\nabla L(\theta)$. The nominal force step is $h=0.05$, or raw
learning rate $\eta=mh/2=0.8$. A trial is accepted if
\begin{equation}
 L(\theta+hF)\le L(\theta)-10^{-4}\frac{2h}{m}\|F\|^2
                +10^{-13}\max\{1,|L(\theta)|\}.
 \label{eq:empirical-armijo}
\end{equation}
Otherwise $h$ is halved; the accepted step persists and cannot increase.
At most forty halvings are allowed per update. This safeguard is used
uniformly from initialization, rather than switched on at an observed
plateau endpoint. A divergent or nonfinite run is a recorded failure,
not an excluded seed. The update axis does not represent a common
physical-time horizon when accepted steps differ. Cumulative raw GD
time, accepted steps and backtracking counts are saved separately.
These variable-step experiments are distinct from the fixed-step
learning theorem.

\paragraph{Complete teacher definitions.}
For $Z\sim N(0,1)$ let $h_k=\operatorname{He}_k/\sqrt{k!}$ and use the
positive-mixture building block $g_v$ defined in the main theorem.
The thirteen scalar raw links are
\[
\begin{gathered}
 h_3,\quad g_{2/3},\quad g_{1/3},\quad (g_1+g_{1/3})/2,\quad g_{1/2},\quad
 0.2g_{0.4}+0.3g_{0.7}+0.5g_1,\\
 h_3+0.1h_5,\quad h_3+0.1\sin,\quad h_2,\quad h_4,\quad z_+,\quad |z|,\quad\sin z.
\end{gathered}
\]
Each gives the full additive target in
\eqref{eq:experiment-normalization}, including its mean and every
Hermite component. The positive-mixture links are in the central
teacher family; the displayed perturbation magnitude $0.1$ is not
certified by the small-neighborhood theorem.
For independent $S,T\sim N(0,1)$, the six product targets are
\begin{equation}
 Y_g=\frac{g(S)T}{\sqrt{\mathbb E[g(S)^2]}},\qquad
 g\in\{\operatorname{SiLU},\ z_+,\ z\Phi(z),\ \tanh,\ h_2,\ z\}.
 \label{eq:empirical-new-teachers}
\end{equation}
Here $z\Phi(z)$ is exact GELU and
$\operatorname{SiLU}(z)=z/(1+e^{-z})$. The two further controls are
\begin{align}
 Y_{\rm raw}&=\frac{\operatorname{SiLU}(S)+\operatorname{SiLU}(T)}
 {\sqrt{2\mathbb E[\operatorname{SiLU}(S)^2]+2(\mathbb E[\operatorname{SiLU}(S)])^2}},\notag\\
 Y_{\rm centered}&=
 \frac{\operatorname{SiLU}(S)+\operatorname{SiLU}(T)-2\mathbb E[\operatorname{SiLU}(S)]}
 {\sqrt{2\operatorname{Var}(\operatorname{SiLU}(S))}}.
 \label{eq:frame-centered-control}
\end{align}
All targets have unit second moment. Products have zero mean but can
retain a linear component in $T$. For nonzero $g\in H^1(\gamma)$,
$\|Y_g\|_{H^1(\gamma_2)}^2=2+\|g'\|_2^2/\|g\|_2^2$.
This ambient regularity is not a plateau guarantee for these
standalone interaction targets. Expectations use analytic student
moments and analytic or deterministic one-dimensional teacher
quadrature, without finite training or test samples.

\paragraph{Per-seed plateaus and common observation times.}
For each seed separately, the plateau ends immediately before the first
update at which MSE leaves $[0.95,1.05]$ or its running maximum/minimum
ratio exceeds $1.05$. Loss is checked at every update. If no exit occurs
by update 20,000, the plateau is right-censored at that horizon.
AGOP and refit do not select the interval. The numerical tables use the
exact endpoint state, while ensemble curves use the common grid
$n=0,100,\ldots,20000$. Initialization, the endpoint and its successor
are retained separately. A first qualifying saved checkpoint is not an
exact hitting time between saved states.

\paragraph{Averaging and uncertainty.}
Every seed has weight $1/50$. Identified values are summarized by arithmetic means and pointwise
empirical 10th--90th percentiles across seeds. Where a score is only
known within a numerical interval, the figure instead propagates lower
and upper bounds on its mean and percentile band.
These are variability bands, not confidence intervals for a population
mean. Loss uses its full every-update record; the other curves join
observed common-grid diagnostics without smoothing or imputing an
unmeasured diagnostic. Neuron-angle heatmaps average the separately
normalized importance distributions of all seeds. They do not pool
unnormalized neuron mass in a way that gives high-norm runs more weight.

Let $E_s$ be the endpoint update of initialization seed $s$.
Pale background marks $[0,\min_s E_s]$, the intersection of the
individual loss-selected plateaus; endpoint quantiles additionally
summarize their variation. The panel abbreviation ``cens.'' counts prefixes that have
not ended by the fixed horizon. A plateau of the average loss would not
establish the same property on every run and is not used. Likewise,
separately high mean alignment and mean refit improvement do not imply
that they occur together within each run.
Table~\ref{tab:ensemble-means} reports continuous endpoint changes and
later original loss, retaining all fifty seeds.

\paragraph{Full AGOP and unresolved spectra.}
We use current trained heads, activation gates and all signed
cross-neuron terms in the population AGOP. Positive scalar
rescaling for eigendecomposition leaves its eigenspaces unchanged.
A rank-$r$ score is numerically resolved only if both
$\lambda_r/\lambda_1>10^{-9}$ and
$(\lambda_r-\lambda_{r+1})/\lambda_1>10^{-9}$.
This threshold is a reporting convention, not a uniform error theorem.
For averaging, an unresolved score contributes its possible interval
$[0,1]$: its mean and percentile uncertainty are propagated, and the
resolved fraction is reported. It is never set to zero, treated as a
proved alignment, or silently removed from a fifty-seed denominator.
The hidden-weight energy fraction
$A_{\rm sub}=\sum_j\|UU^TW_j\|^2/\sum_j\|W_j\|^2$
(undefined for a zero denominator) is saved but does not
replace the full-predictor AGOP statistic.

\paragraph{Constrained refitting and its numerical error.}
At every checkpoint, use all unprojected features divided by their
original augmented row norms $\sqrt{\|W_j\|^2+B_j^2}$, with caps
$\|v\|_2\le32$ and $\|v\|_1\le64\sqrt2$ and no extra intercept.
For normalized features $\phi=(\phi_1,\ldots,\phi_m)^T$,
let $H=\mathbb E[\phi(X)\phi(X)^T]$ and $c=\mathbb E[y(X)\phi(X)]$. Minimize
$R(v)=1-2c^Tv+v^THv$ on that fixed feasible set. Warm starts do not change
this set. No ridge is added and no Gram eigenmode is discarded in the
objective. Each feasible candidate gives an upper risk $\overline R$;
a first-order support-function bound, with negative-eigenvalue and
arithmetic allowances, supplies a numerical lower estimate
$\underline R$. Original and snapped refit brackets remain separate
from the between-seed bands. Solver-tolerance misses remain flagged;
these are numerical optimization brackets for the computed moments,
not rigorous interval bounds on integration error.

\paragraph{Per-run outcomes.}
Define the conservative numerical refit improvement
$\widehat\Delta_n=\underline R_0-\overline R_n-10^{-5}$.
The joint reference event requires resolved initial/current spectra,
$A_{\min,n}-A_{\min,0}\ge0.5$ and
$\widehat\Delta_n\ge0.399$ at one checkpoint in the same plateau.
We distinguish the exact endpoint event from an event at any saved
in-prefix checkpoint, and record later original-model loss separately.
Failure of the absolute $0.399$ test is not absence of useful learning
when the initial refit risk is already small. Mean and linear target
components remain in the optimization objective. Their contributions
and nonlinear residual errors are retained in the machine-readable
results; large initial means are not removed to create a high-loss
claim.

\begin{table}[!htbp]
\centering
\caption{Baseline results for all 21 teachers and all 50 declared seeds per teacher. Endpoint gains use each seed's exact loss-prefix endpoint. Brackets are mean lower and upper bounds, not confidence intervals; unresolved AGOPs contribute $[0,1]$ before taking gain differences. Refits retain numerical lower and feasible upper risks. Final MSE is mean $\pm$ seed standard deviation.}
\label{tab:ensemble-means}
\begin{tabular}{@{}p{112pt}rrr@{}}
\toprule
Teacher & \shortstack{Mean\\$\Delta A_{\min}$} & \shortstack{Mean refit\\decrease} & Final MSE \\
\midrule
$|z|$ & $[-0.013,0.727]$ & $-0.021$ & $0.147\,\pm\,0.106$ \\
$st$ & $0.972$ & $0.769$ & $0.003\,\pm\,0.001$ \\
$g_{1/2}$ & $[0.952,0.972]$ & $0.697$ & $0.099\,\pm\,0.053$ \\
$g_{1/3}$ & $[0.912,0.972]$ & $0.473$ & $0.173\,\pm\,0.058$ \\
$g_{2/3}$ & $0.972$ & $0.783$ & $0.032\,\pm\,0.003$ \\
$\mathrm{GELU}(s)t$ & $0.972$ & $0.240$ & $0.017\,\pm\,0.005$ \\
$h_2$ & $0.972$ & $0.468$ & $0.005\,\pm\,0.001$ \\
$h_3$ & $0.972$ & $0.601$ & $0.050\,\pm\,0.014$ \\
$h_3+0.1h_5$ & $0.972$ & $0.545$ & $0.066\,\pm\,0.018$ \\
$h_3+0.1\sin$ & $0.972$ & $0.259$ & $0.190\,\pm\,0.080$ \\
$h_4$ & $0.972$ & $0.133$ & $0.198\,\pm\,0.069$ \\
$g_1+g_{1/3}$ & $[0.952,0.972]$ & $0.656$ & $0.089\,\pm\,0.058$ \\
$h_2(s)t$ & $[0.952,0.972]$ & $0.553$ & $0.080\,\pm\,0.037$ \\
ReLU & $[-0.028,0.972]$ & $0.110$ & $0.000\,\pm\,0.000$ \\
$\mathrm{ReLU}(s)t$ & $0.972$ & $0.224$ & $0.004\,\pm\,0.000$ \\
Raw additive SiLU & $[-0.028,0.972]$ & $0.141$ & $0.001\,\pm\,0.008$ \\
Centered additive SiLU & $[-0.028,0.972]$ & $0.193$ & $0.001\,\pm\,0.001$ \\
$\mathrm{SiLU}(s)t$ & $0.972$ & $0.307$ & $0.023\,\pm\,0.019$ \\
$\sin z$ & $[-0.028,0.972]$ & $0.209$ & $0.061\,\pm\,0.006$ \\
$\tanh(s)t$ & $0.972$ & $0.819$ & $0.002\,\pm\,0.001$ \\
Three-scale mixture & $0.972$ & $0.808$ & $0.012\,\pm\,0.002$ \\
\bottomrule
\end{tabular}
\end{table}

\subsection{Neuron axes and frame interventions}
\label{app:frame-diagnostics}

Here $u_1,u_2$ are the orthonormal teacher directions, and
$A_j,W_j,B_j$ denote a neuron's trained head, spatial vector and bias
at the displayed checkpoint. The two extra rows in the additive/gated
panels distinguish subspace recovery from axis concentration. For a nonzero in-plane row set
\begin{equation}
 \vartheta_j=\operatorname{atan2}(|W_j^Tu_2|,|W_j^Tu_1|)\in[0,90^\circ],
 \qquad \omega_j=\frac{|A_j|\|W_j\|}{\sum_k|A_k|\|W_k\|}.
 \label{eq:frame-angle}
\end{equation}
Five-degree bins contain the sum of these weights; zero in-plane rows
are omitted and the omitted mass is recorded. The heatmap averages
these histograms across seeds. Concentration near zero or ninety
degrees describes teacher-axis alignment; interior angles mix teacher
coordinates. This unsigned projection is not itself a prediction
metric. The labels \emph{monosemantic} and \emph{polysemantic} refer to
this chosen basis and observed time, not a permanent or semantic
classification of every neuron.

For an orthonormal frame rotated within the teacher plane, project each
row's in-plane component onto its closest unoriented frame axis.
Preserve its bias, outside-plane component and \emph{original}
augmented denominator. Refit with exactly the same coefficient caps as
for the original bank. Figure~\ref{fig:main-trajectories} compares original
features with teacher-axis and $45^\circ$ snapping in its fourth row.
The appendix's five-row sheets additionally show the best of eighteen
frames at $0,5,\ldots,85^\circ$. The two fixed frames are evaluated
on the common hundred-update grid; the eighteen-frame comparison is
evaluated at $0,1000,\ldots,20000$. For the best-grid numerical interval
we minimize lower and feasible risks separately at each seed and
checkpoint before averaging. The selected frame can therefore differ
between seeds and times; it is not a single shared rotation. In all cases, markers denote
evaluated states and connecting lines only guide the eye.
For original and snapped brackets the gap interval is
\begin{equation}
 [\underline R_F-\overline R,\;\overline R_F-\underline R].
 \label{eq:frame-gap}
\end{equation}
A small gap means that the intervention preserves refit quality; it
does not prove that the original neurons already specialized. Negative
gaps can occur because original and snapped spans are not nested.
A tested rotation grid is not an optimization over all rotations.

\paragraph{An exact obstruction for a fixed orthonormal frame.}
The following elementary population statement separates axis geometry
from subspace recovery. It is an approximation bound, not a learning
dynamics theorem.

\begin{proposition}[Orthogonal-frame floor]
\label{prop:orthogonal-frame-floor}
Let $X\sim N(0,I_d)$, $P=UU^T$, and let $(a_1,a_2)$ be an
orthonormal basis of $\operatorname{span}(U)$. Write
$Z_k=a_k^TX$ and let $Y$ be a square-integrable function of
$U^TX$. Suppose every row of a finite ridge network satisfies
$PW_j\in\operatorname{span}(a_1)$ or
$PW_j\in\operatorname{span}(a_2)$, with arbitrary biases and
outside-plane components, and square-integrable features. Then its
population MSE is at least
\begin{equation}
 \mathcal E_F(Y)=
 \bigl\|Y-\mathbb E[Y\mid Z_1]-\mathbb E[Y\mid Z_2]
                       +\mathbb E[Y]\bigr\|_2^2.
 \label{eq:orthogonal-frame-floor}
\end{equation}
For independent $S=u_1^TX,T=u_2^TX\sim N(0,1)$ and
$Y=g(S)T/\|g\|_2$, the teacher-frame floor is
$\operatorname{Var}(g(S))/\mathbb E[g(S)^2]$. The minimum floor over all
orthonormal frames is $1/4$ when $g=h_2$, and zero when $g(z)=z$.
\end{proposition}

\begin{proof}
Gaussian independence of $PX$ and $(I-P)X$ makes the conditional
expectation of each ridge feature given $U^TX$ a univariate function
of its selected $Z_k$. Thus $\mathbb E[f(X)\mid U^TX]$ is additive
in $Z_1,Z_2$. Conditional Jensen bounds the network's MSE below by
the squared distance of $Y$ from this additive space. Since $Z_1,Z_2$
are independent, their centered univariate spaces are orthogonal and
the residual is exactly \eqref{eq:orthogonal-frame-floor}.
In the teacher frame, the additive projection of $g(S)T/\|g\|_2$
is $(\mathbb E[g(S)])T/\|g\|_2$, giving the stated variance ratio.
For a frame rotated by $\theta$, put $c=\cos\theta$, $s=\sin\theta$.
Hermite orthogonality gives captured additive energy $3c^2s^2$ for
$h_2(S)T$ and $4c^2s^2$ for $ST$. Since
$c^2s^2\le1/4$, their minimum residual energies are $1/4$ and zero,
respectively, both attained at $45^\circ$.
\end{proof}

This obstruction concerns every orthonormal frame, whereas a finite
rotation-grid diagnostic is only a comparison over the tested frames.
Neither statement excludes all nonorthogonal frames. For smooth gates,
a computed $45^\circ$ floor does not establish that this rotation is the
best possible one. The proposition is an approximation bound; it adds
no learning-dynamics theorem for interaction targets.

\subsection{Denser followups for four short-prefix teachers}
\label{app:followup-protocol}

\paragraph{Why a separate followup is needed.}
The raw and centered additive SiLU, additive ReLU and additive sine
teachers have short loss-selected prefixes under the original
$d=16$, $s=10^{-4}$ protocol. Their early behavior is difficult to see
on a 20,000-update axis. A display change alone also leaves a numerical
problem: the original endpoint rank-two spectra are unresolved for
these four cohorts. We therefore retain the original outcomes in
the baseline tables and report a separate experiment with better
resolved spectra and denser observations. These familiar links remain
empirical cases outside the positive-family theorem.

\paragraph{Pilot and disclosed configuration choice.}
The pilot varied $s\in\{10^{-4},0.01,0.1\}$ at $d=64$, $m=32$,
$r=2$, with force step $h=0.05$, seeds 51--53 and 3,000 updates:
36 trajectories over the four teachers. The pilot did not identify
a reliable regime of full-subspace learning during the plateau. For a common, interpretable
comparison we subsequently chose $s=0.01$ for all four teachers;
all twelve pilot endpoints at this scale had resolved rank-two spectra.
This choice was made after inspecting the pilot. All pilot outcomes
are retained.

\paragraph{Fifty-seed followup and unchanged definitions.}
Each selected teacher uses seeds 1--50, 20,000 updates,
$d=64$, $m=32$, $r=2$, and IID $N(0,0.01^2/d)$ entries for every
raw head, weight and bias. These seeds are disjoint from the pilot
but reuse the baseline seed identifiers; the followup is not an
independent replication of the original protocol. All parameters
train. The targets, persistent Armijo rule, plateau test, full-predictor
AGOP, resolution convention and constrained refit are unchanged.
No teacher mean is removed except in the explicitly centered SiLU
control. Computation uses Linux CPUs, one numerical thread per
worker, with the same numerical libraries and frozen population
kernels as the baseline. No finite sample or stochastic gradient is used.

The loss is saved at every update. AGOP and refit diagnostics are
saved every ten updates through 1,000, every fifty through 3,000,
and every hundred thereafter, together with each exact plateau
endpoint and its successor. For each teacher the figure pairs a
linear 0--500-update view with a linear 0--20,000-update view.
Both use all fifty seeds with equal weights and empirical
10th--90th percentile envelopes. Shading ends at
$\min_s E_s$, the earliest individual loss-selected endpoint;
neither alignment nor refit selects the shaded interval.
An expanded early view is not a claim of a longer physical-time plateau.

\paragraph{Interpretation.}
Both $A_{\min}$ and $A_{\rm mean}$ are shown. An increase in the mean
with little increase in the minimum describes partial geometric
learning, not recovery of every teacher direction. Refit gains are
measured separately and may include improved prediction of low-degree
components. A successful refit therefore does not by itself establish
full-subspace recovery. The table below reports continuous alignment
and refit gains, and later original-model loss separately.

All 200 initial states, endpoint states and saved in-prefix spectra are
resolved under the stated numerical convention. Median mean-alignment
and refit gains are substantial across all four teachers, while weakest-direction
gains are smaller. Thus these followups display partial feature and
prediction improvement on the plateau without claiming uniform
recovery of both teacher directions. The 281 unresolved checkpoints
later in training remain visible as numerical uncertainty in the plots;
no refit solver-tolerance misses occurred. Better visibility does not
imply better final optimization: raw SiLU has median final MSE $0.0542$
here, versus $0.000279$ in the original baseline. The reduction in
mean-plus-linear squared prediction error accounts for a median
90--98\% of the original-loss decrease during these plateaus.
Accordingly, the displayed refit gains should not be interpreted as
isolated learning of higher Hermite components.

\begin{table}[htbp]
\centering
\caption{Separate four-teacher followup, fifty seeds each. Endpoint columns report medians of individual-run endpoint changes, not changes at an averaged endpoint. $E$ is the plateau endpoint update (median and range); $\Delta A=A(E)-A(0)$ for each alignment score; $\Delta R=\underline R_0-\overline R_E-10^{-5}$ is the conservative numerical refit gain. Later loss is the original model at update 20,000.}
\label{tab:followup-outcomes}
\textit{Alignment and plateau endpoint}\par
\begin{tabular}{@{}p{108pt}rrr@{}}
\toprule
Teacher & $E$ [range] & $\Delta A_{\min}$ & $\Delta A_{\rm mean}$ \\
\midrule
Raw additive SiLU & 128 [122,137] & 0.222 & 0.576 \\
Centered additive SiLU & 203 [190,227] & 0.000411 & 0.476 \\
Additive ReLU & 113 [108,120] & 0.115 & 0.534 \\
Additive sine & 203.5 [191,219] & 0.0231 & 0.487 \\
\bottomrule
\end{tabular}
\end{table}

\begin{table}[!htbp]
\centering
\textbf{Table~\ref{tab:followup-outcomes} (continued).}\par
\textit{Refit and later loss}\par
\begin{tabular}{@{}p{108pt}rr@{}}
\toprule
Teacher & $\Delta R$ & $L_{20000}$ \\
\midrule
Raw additive SiLU & 0.43 & 0.054 \\
Centered additive SiLU & 0.589 & 0.00056 \\
Additive ReLU & 0.313 & 1.2e-06 \\
Additive sine & 0.607 & 0.062 \\
\bottomrule
\end{tabular}
\end{table}

\clearpage
\begin{figure}[p]
\centering
\includegraphics[width=\linewidth]{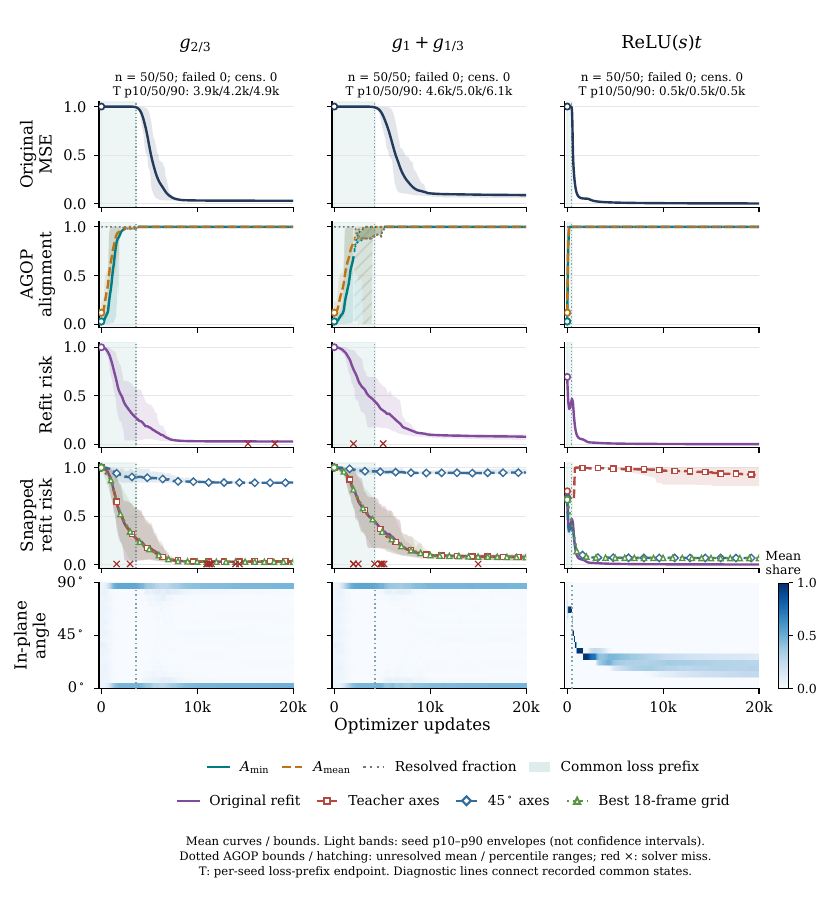}
\caption{\textbf{Additive damped and mixture links, and the ReLU gate.} The two additive-family targets and the ReLU product use five diagnostic rows. Each teacher uses fifty fixed initialization seeds with identical optimization parameters. Means, numerical bounds, percentile bands and loss-prefix shading follow Figure~\ref{fig:main-trajectories}.}
\label{fig:ensemble-frames-1}
\end{figure}
\clearpage
\begin{figure}[p]
\centering
\includegraphics[width=\linewidth]{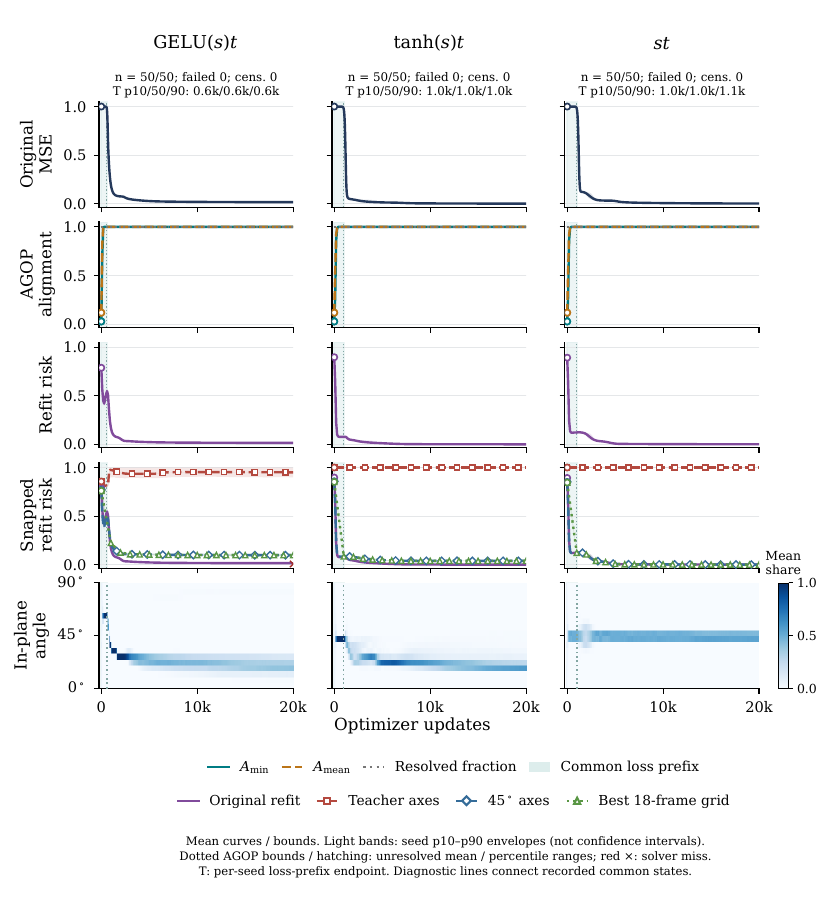}
\caption{\textbf{Other gated teachers.} Exact GELU and tanh products, together with the bilinear rotation control, use the same five rows. Each teacher uses fifty fixed initialization seeds with identical optimization parameters. Means, numerical bounds, percentile bands and loss-prefix shading follow Figure~\ref{fig:main-trajectories}.}
\label{fig:ensemble-frames-2}
\end{figure}
\clearpage
\begin{figure}[p]
\centering
\includegraphics[width=\linewidth]{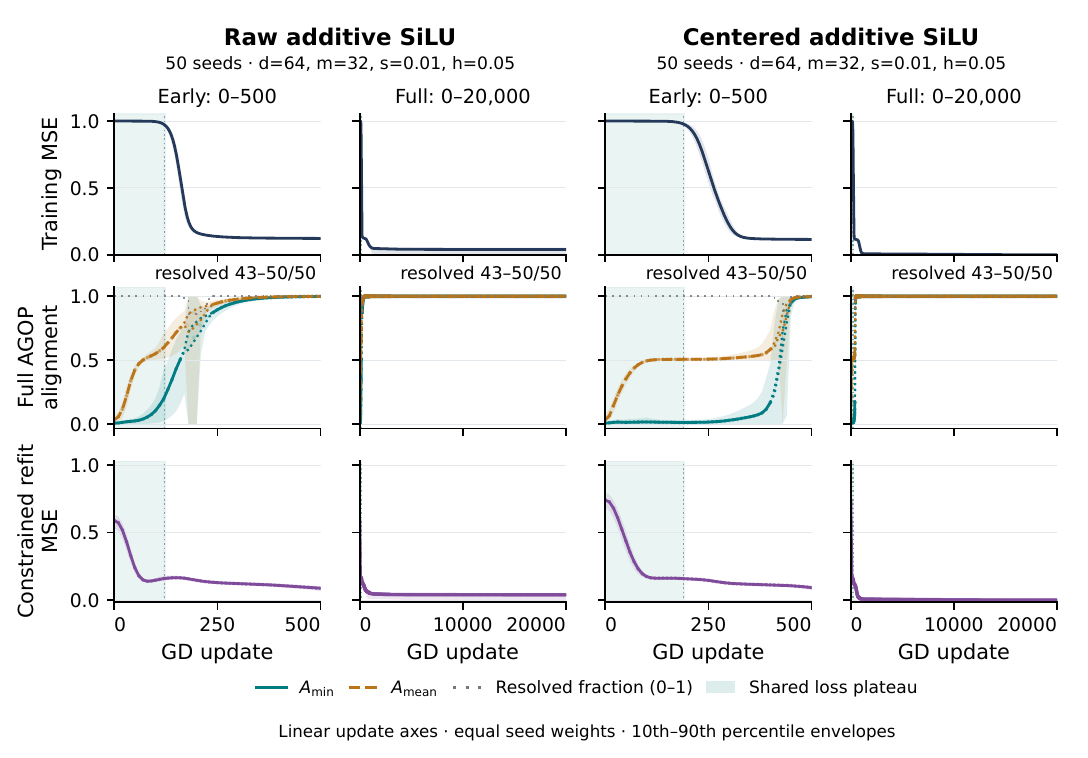}
\caption{\textbf{Raw and centered additive SiLU: denser followups.}
Each teacher has paired linear early (0--500 updates) and full
(0--20,000 updates) views, with fifty seeds at $d=64$, $m=32$,
$s=0.01$. Rows show original loss, minimum and mean full-AGOP
alignment, and constrained refit MSE. Bands are empirical
10th--90th percentiles; the gray trace reports the resolved fraction.
Refit lower/upper numerical bounds remain separate. Shading marks
the loss-prefix intersection across all fifty seeds. The common
configuration was chosen after the pilot did not establish reliable
full-subspace learning, as disclosed in Section~\ref{app:followup-protocol}.
The enlarged early view exposes partial alignment and refit changes;
the averages do not establish full-subspace recovery for every seed.}
\label{fig:ensemble-controls}
\end{figure}
\clearpage
\begin{figure}[p]
\centering
\includegraphics[width=\linewidth]{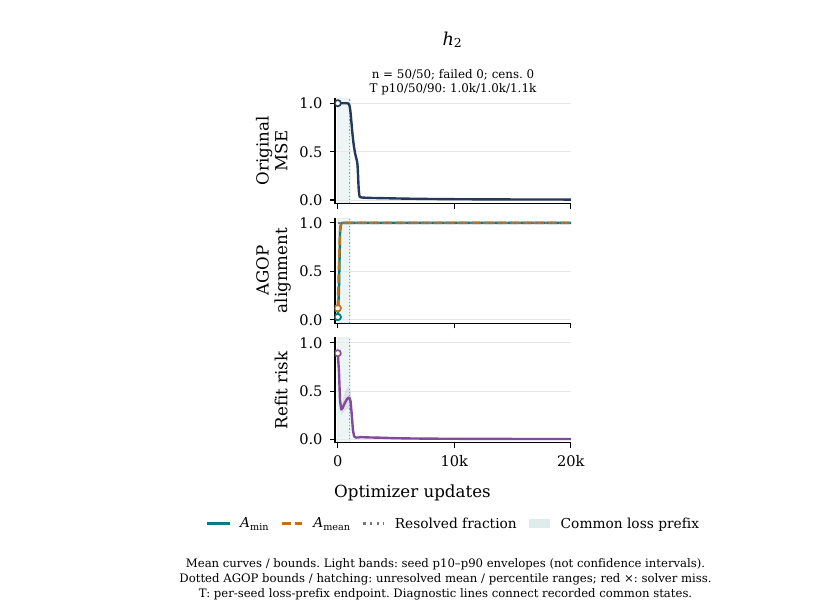}
\caption{\textbf{Quadratic additive teacher.} The three rows display original loss, full-current-head AGOP alignment and same-budget refit MSE. Each teacher uses fifty fixed initialization seeds with identical optimization parameters. Means, numerical bounds, percentile bands and loss-prefix shading follow Figure~\ref{fig:main-trajectories}.}
\label{fig:ensemble-h2}
\end{figure}
\clearpage
\begin{figure}[p]
\centering
\includegraphics[width=\linewidth]{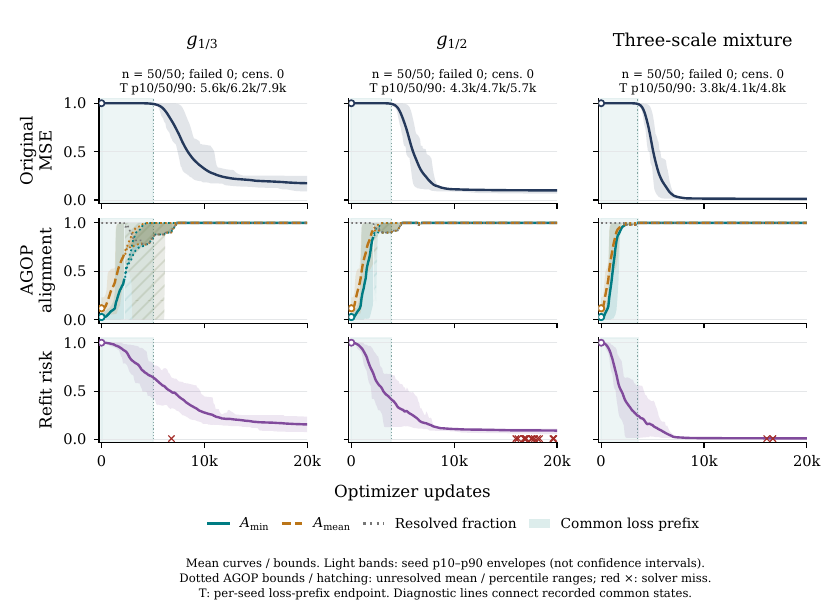}
\caption{\textbf{Additional positive-mixture links.} Damped links and the three-scale mixture use the same three-row format. Each teacher uses fifty fixed initialization seeds with identical optimization parameters. Means, numerical bounds, percentile bands and loss-prefix shading follow Figure~\ref{fig:main-trajectories}.}
\label{fig:ensemble-mixtures}
\end{figure}
\clearpage
\begin{figure}[p]
\centering
\includegraphics[width=\linewidth]{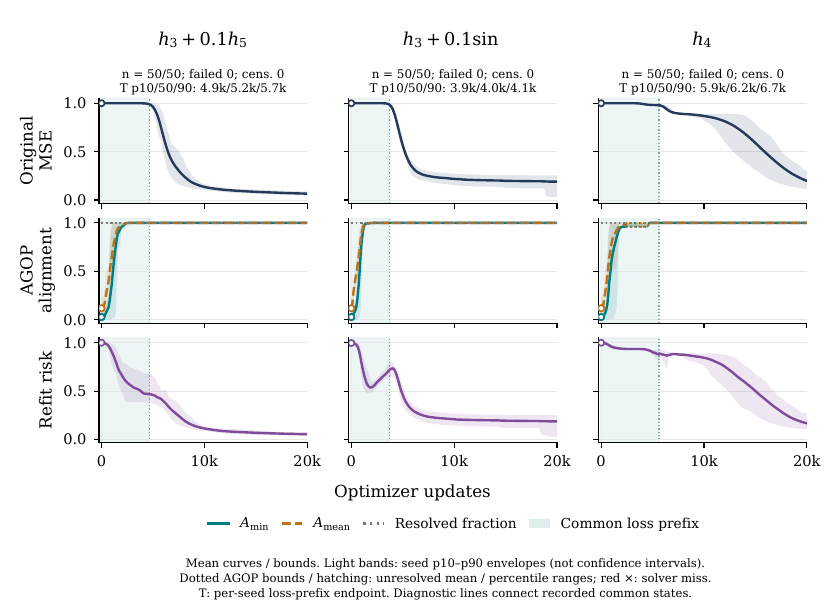}
\caption{\textbf{Perturbed cubic and fourth-Hermite teachers.} These standalone perturbation magnitudes and the fourth Hermite are empirical cases, not claims of the small-neighborhood theorem. Each teacher uses fifty fixed initialization seeds with identical optimization parameters. Means, numerical bounds, percentile bands and loss-prefix shading follow Figure~\ref{fig:main-trajectories}.}
\label{fig:ensemble-perturbations}
\end{figure}
\clearpage
\begin{figure}[p]
\centering
\includegraphics[width=\linewidth]{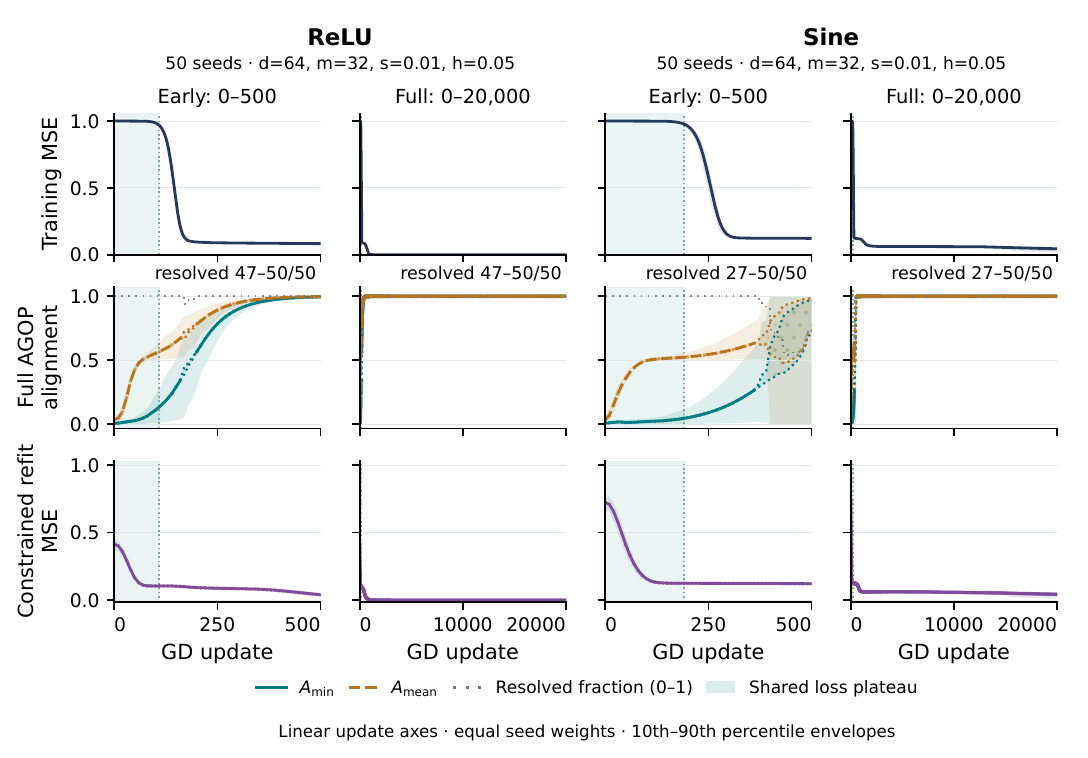}
\caption{\textbf{Additive ReLU and sine: denser followups.}
The paired early and full views use the same fifty-seed configuration,
three rows, linear update axes, numerical uncertainty conventions and
loss-prefix shading as Figure~\ref{fig:ensemble-controls}. Both teacher
means and all Hermite components are retained. Partial geometric
learning and refit gains are distinct from recovery of every teacher
direction; Table~\ref{tab:followup-outcomes} reports continuous endpoint changes.}
\label{fig:ensemble-ordinary}
\end{figure}

\clearpage
\begin{figure}[p]
\centering
\includegraphics[width=\linewidth]{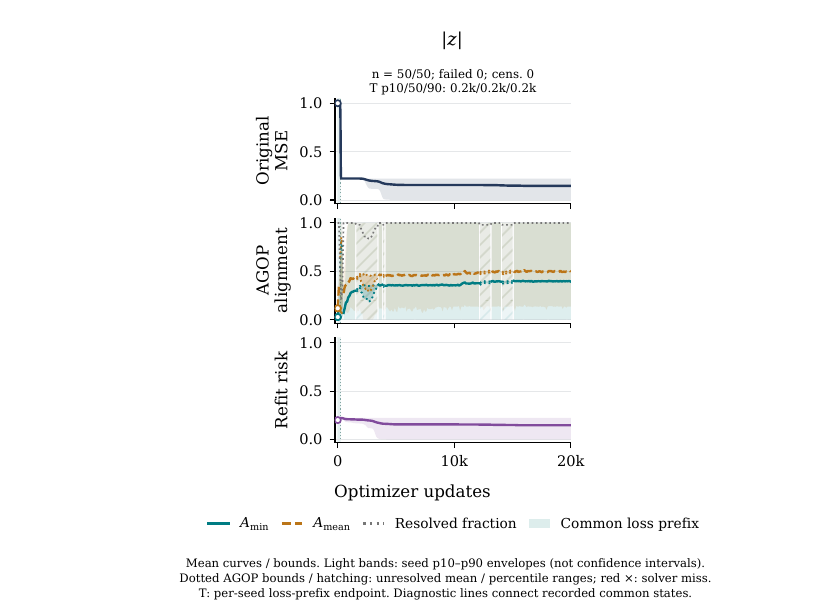}
\caption{\textbf{Absolute-value teacher: original baseline.}
This panel retains the original fifty seeds at $r=2$, $d=16$, $m=32$,
$s=10^{-4}$ and initial force step $h=0.05$; it was not rerun.
The three rows show original loss, full AGOP alignment and constrained
refit MSE over 20,000 updates. Means, percentiles, unresolved spectral
intervals and shared loss-prefix shading follow the original baseline
protocol. It is separated here so the four new followups do not displace
the absolute-value control.}
\label{fig:ensemble-abs}
\end{figure}

\clearpage
\section{Higher-rank comparisons and neuron geometry}
\label{app:highrank}

This separate cohort tests ranks and teacher links beyond the rank-two
ReLU and rank-one SwiGLU illustrations. Its 140 trajectories comprise
60 ReLU and 80 SwiGLU runs, with ten prescribed seeds 9351--9360 per
configuration. The same seed is reused across settings; such outcomes
are paired comparisons, not additional independent repetitions.
Every declared seed is retained. Four rank-16 ReLU runs are earlier
confirmation runs with matching numerical code; the other six extend
that setting. These are empirical illustrations, not additional
higher-rank learning theorems.

\paragraph{Teachers, sizes and optimization.}
Let $X\sim N(0,I_d)$, let $U=(u_1,\ldots,u_r)$ have orthonormal columns,
and write $P_U=UU^T$. All runs use $d=64$ and additive teachers
\begin{equation}
 Y=\frac{\sum_{i=1}^r q(u_i^TX)}
 {\sqrt{r\mathbb E[q(Z)^2]+r(r-1)(\mathbb E[q(Z)])^2}},
 \qquad Z\sim N(0,1),
 \quad q\in\{h_2,\ |\cdot|,\ e^{-(\cdot)^2/2}\}.
 \label{eq:highrank-teachers}
\end{equation}
Thus $\mathbb E[Y^2]=1$; $h_2(z)=(z^2-1)/\sqrt2$ has zero mean
and unit variance. The other teacher means are retained, with an
intercept fitted analytically at every training and refit evaluation.
Write $V_Y=\operatorname{Var}(Y)$ and $\ell_n=L_n/V_Y$ for the
variance-normalized MSE. For the nonzero-mean teachers, the unit initial
baseline is the intercept-only residual error, not the full raw second
moment. At rank eight, $V_Y$ is approximately $0.06660$ for $|z|$
and $0.01897$ for the Gaussian bump.

ReLU uses $m=256$, IID $N(0,s^2/d)$ raw heads, weights and biases
with $s=0.01$, all trained. Its force step is $h=0.5$, corresponding
to raw-MSE gradient step $mh/2=64$; its Armijo safeguard never activates.
Runs stop at force time 1500 (3000 updates) or raw MSE at most $0.01$
for $h_2$, and at most $0.001$ for the absolute-value and Gaussian-bump links.
The bounded refit uses the normalized, unprojected feature bank,
an intercept, and caps $\|v\|_2\le32$, $\|v\|_1\le64\sqrt r$.
This is a separate protocol from the no-intercept rank-two baseline.

SwiGLU uses $m=64$, initialization scale $s=0.1$, and block-adaptive
explicit-Euler population flow, with maximum relative block change
$0.01$, maximum time step 50 and flow-time ceiling 3000.
The head learning-rate multiplier is $0.01$, or $0.001$ in the
labeled comparison. Runs can terminate at normalized loss
$\ell\le0.01$, the flow-time ceiling, a 20,000-update ceiling,
or a wall-time budget of approximately thirty minutes.
The horizontal axis counts updates, not a shared physical-time clock.
Refit is unrestricted numerically: the displayed risk is the largest
actual risk among four pseudoinverse-cutoff solves. This sensitivity
envelope is neither the same comparison class as bounded ReLU refit
nor a certificate of the exact unrestricted oracle.

\paragraph{Loss-selected windows and sensitivity.}
For $q_0\in\{0.01,0.05\}$, let $T_{q_0}$ be the last update before
the first failure of
\begin{equation}
 \frac{\max_{0\le k\le n}\ell_k}{\min_{0\le k\le n}\ell_k}
 \le 1+q_0.
 \label{eq:highrank-window}
\end{equation}
Losses must be positive and finite. ReLU additionally requires
$\ell_k\in[0.95,1.05]$ at every update in the prefix for both thresholds.
This agrees with the raw-MSE convention for the zero-mean quadratic
teacher. Every-update loss histories determine the window independently
of alignment, refit and neuron geometry. Shading marks
$[0,\min_s T_{0.05}^{(s)}]$ across all ten seeds; the tighter comparison
boundary is $\min_s T_{0.01}^{(s)}$. A 5\% maximum/minimum ratio
is not an absolute $0.05$ loss tolerance and does not assert zero slope.
Figure~\ref{fig:swiglu-main} retains the stricter absolute condition
$|\ell_n-\ell_0|\le10^{-3}$ from its separate rank-one protocol.

In the four high-rank SwiGLU main-text columns, the common 5\% windows
end at 310, 332, 271 and 237 updates; their mean losses are approximately
$0.967,0.963,0.962,0.963$ there. Their 1\% comparison windows end at
233, 248, 182 and 149 updates. Applying the stricter absolute
$10^{-3}$ rule to the same data instead gives endpoints 119, 128, 64
and 41. These changes are changes of the reporting threshold, not
changes of trajectories or optimization.

\paragraph{Diagnostic support and interpretation.}
The initial and current numerical screens are applied before interpreting
alignment and refit gains. These checks concern computed moments, not
rigorous integration-error certificates. Separate mean curves do not
establish a same-checkpoint event for every initialization, and the
broader loss window does not establish the same conclusion under the
stricter absolute rule.

Curves use arithmetic means and pointwise 10th--90th percentiles,
with all ten seeds required at each aggregate point. Thin individual
tails remain visible beyond shared support. ReLU scalar curves join
common saved checkpoints. SwiGLU interpolates only between adjacent
valid saved diagnostics for display; missing or rejected diagnostics
are not bridged. Exact window endpoints need not have saved geometry
or refit diagnostics. Interpolated endpoint values create no observed
joint event. Some later SwiGLU trajectories oscillate, and wall stops
truncate others. Saved-data consistency checks do not certify trajectory
convergence or eliminate quadrature and time-discretization error.
The dotted $r/d$ line is the isotropic reference for $A_{\rm mean}$,
not the expected minimum alignment.
The ordered principal-cosine heatmap sorts the eigenvalues of
$U^TP_rU$ at each saved state, where $P_r$ is the top-$r$ AGOP
projector. It describes subspace principal angles, not alignment with
fixed individual teacher axes or the directions of individual neurons.

\paragraph{ReLU angles and importance weighting.}
For spatial weight $W_j$ and raw head $A_j$, define
\begin{align}
 \theta_{U,j}&=\arccos\frac{\|P_UW_j\|}{\|W_j\|},&
 \theta_{{\rm axis},j}&=\arccos\frac{\max_{1\le i\le r}|u_i^TW_j|}
 {\|P_UW_j\|},&
 \omega_j&=\frac{|A_j|\|W_j\|}{\sum_k|A_k|\|W_k\|}.
 \label{eq:highrank-neuron-angles}
\end{align}
The first angle measures entry into $U$; the second measures individual-axis
alignment of the projected weight. Directions are considered up to sign.
The nearest-axis angle lies in $[0,\arccos(1/\sqrt r)]$; at rank two it
folds the original Figure~\ref{fig:main-trajectories} in-plane angle at
$45^\circ$. It is not an AGOP principal angle. Histograms use five-degree
bins and weights normalized separately in each seed and state, before
averaging the ten histograms. All compared saved states have positive
weight denominators and nonzero projections. Grey regions denote absence
of a state shared by all seeds; hatching marks impossible angles.

For equal-coefficient quadratic teachers,
\begin{equation}
 \frac1{\sqrt r}\sum_i h_2(u_i^TX)
 =\frac{\|P_UX\|^2-r}{\sqrt{2r}}.
 \label{eq:highrank-quadratic-invariance}
\end{equation}
Rotating the axes inside $U$ leaves the teacher unchanged. Their
nearest-axis statistics are therefore symmetry controls, not evidence
for an intrinsic preference for or against polysemanticity. The
nonquadratic links give a more informative teacher-coordinate comparison:
subspace entry can precede concentration near individual axes. This
does not prove that mixing is necessary, optimal or permanent.

\paragraph{SwiGLU gate and value summaries.}
Write $a_j,w_j,z_j$ for the head, spatial gate weight and spatial value
weight in the model of Section~\ref{sec:setup}; biases are excluded from
the following angular diagnostics. For $c_j=\|P_Uw_j\|^2/\|w_j\|^2$,
the plotted effective angles are
\begin{equation}
 \Theta_{\rm weighted}=\arccos\sqrt{\frac{\sum_j\rho_jc_j}{\sum_j\rho_j}},
 \qquad \rho_j=|a_j|\|w_j\|^2\|z_j\|,
 \qquad \Theta_{\rm median}=\arccos\sqrt{\operatorname{median}_j c_j}.
 \label{eq:highrank-gate-angles}
\end{equation}
The weights $\rho_j$ are a degree-three contribution proxy.
$\Theta_{\rm weighted}$ is not a weighted mean angle; for even $m=64$,
transforming the median squared cosine need not equal the usual median
of individual angles. Transformations are applied to each seed's saved
scores before aggregation; variability bands concern seeds, not neurons.
The neuron-summary appendix panels also show a dotted purple value-weight
curve,
\begin{equation}
 \Theta_{\rm value}=\arccos\sqrt{\operatorname{median}_j
 \frac{\|P_Uz_j\|^2}{\|z_j\|^2}},
 \label{eq:highrank-value-angle}
\end{equation}
computed from each seed's saved median squared-cosine score before
averaging. The same even-width median qualification applies.
The main-text geometry rows show gates only.
Only these scalar summaries and threshold counts were retained, so
per-neuron angle histograms cannot be recovered for SwiGLU.

A gate is counted near an axis if
$\max_i|u_i^Tw_j|/\|w_j\|>0.9$, equivalently its angle to that axis is
less than $\arccos(0.9)\simeq25.8^\circ$. This uses the full gate norm,
unlike the projected ReLU nearest-axis angle. We plot both the unweighted
fraction of such gates and the fraction of teacher axes with at least
one such gate. These are gate-only diagnostics, not whole-neuron semantic
classifications. Few near-axis gates alone do not demonstrate mixing:
the gates could still lie outside $U$. Read them jointly with subspace
entry at the same time and retain the quadratic-axis caveat.

\clearpage
\begin{figure}[p]
\centering
\includegraphics[width=\textwidth]{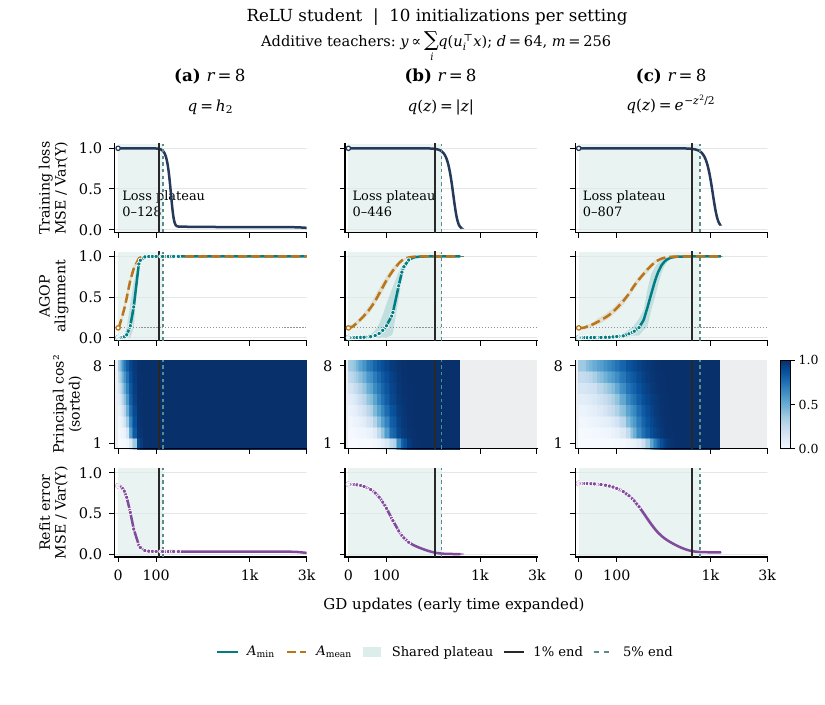}
\caption{\textbf{ReLU teacher-link comparison at rank eight.}
Ten seeds each for quadratic, absolute-value and Gaussian-bump teachers,
$d=64,m=256$. Loss and bounded refit are in units of $V_Y$; target means
are retained and fitted by the intercept. The 5\% shared plateaus and
1\% boundaries use \eqref{eq:highrank-window}. Means, bands and
support conventions are those defined above.}
\label{fig:highrank-relu-teachers}
\end{figure}
\clearpage
\begin{figure}[p]
\centering
\includegraphics[width=\textwidth]{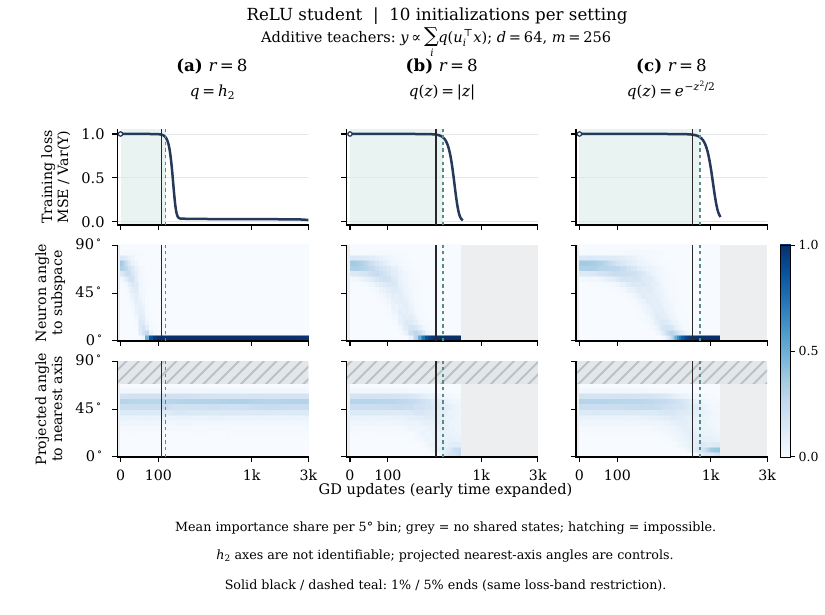}
\caption{\textbf{ReLU subspace entry and teacher-axis specialization.}
The same rank-eight runs as Figure~\ref{fig:highrank-relu-teachers}.
Importance-weighted angle histograms distinguish entry into the teacher
subspace from alignment with a single projected teacher axis.
For the nonquadratic links, axis concentration develops later than
subspace entry; the quadratic column is a rotation-invariant control.
Five-degree histograms average the ten separately normalized seed
distributions. Grey means no shared saved state, and hatching marks
angles above the nearest-axis geometric limit.}
\label{fig:highrank-relu-neurons-teachers}
\end{figure}
\clearpage
\begin{figure}[p]
\centering
\includegraphics[width=\textwidth]{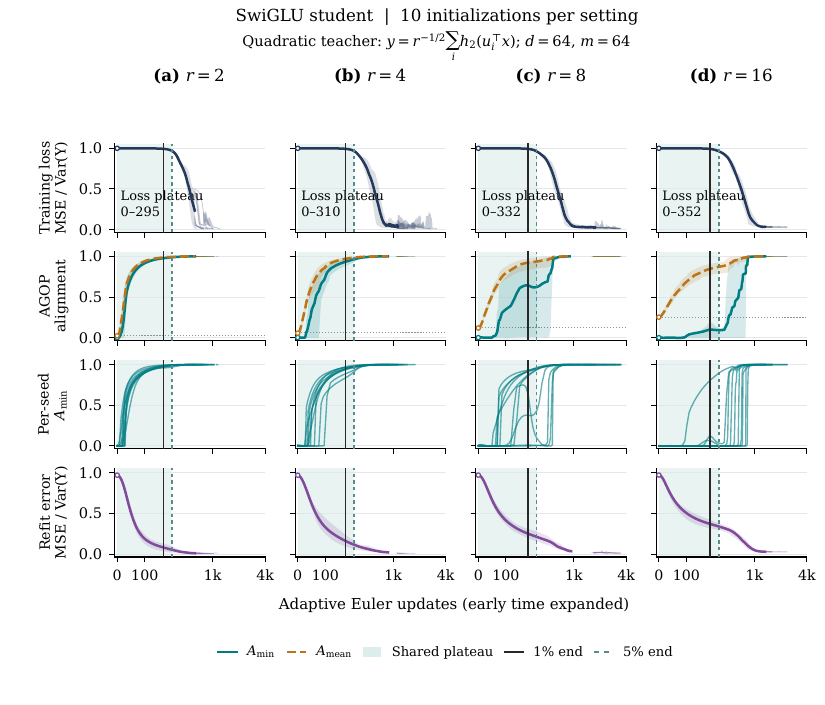}
\caption{\textbf{SwiGLU rank sweep, including the weak rank-sixteen case.}
Quadratic teachers at ranks $2,4,8,16$, with $d=64,m=64$,
head-rate multiplier $0.01$ and ten seeds each.
The rank-sixteen column does not establish robust
all-direction recovery during the plateau. All seeds and later
diagnostic gaps remain visible; large mean alignment does not replace
the weakest-direction and same-checkpoint requirements.}
\label{fig:highrank-swiglu-ranks}
\end{figure}
\clearpage
\begin{figure}[p]
\centering
\includegraphics[width=\textwidth]{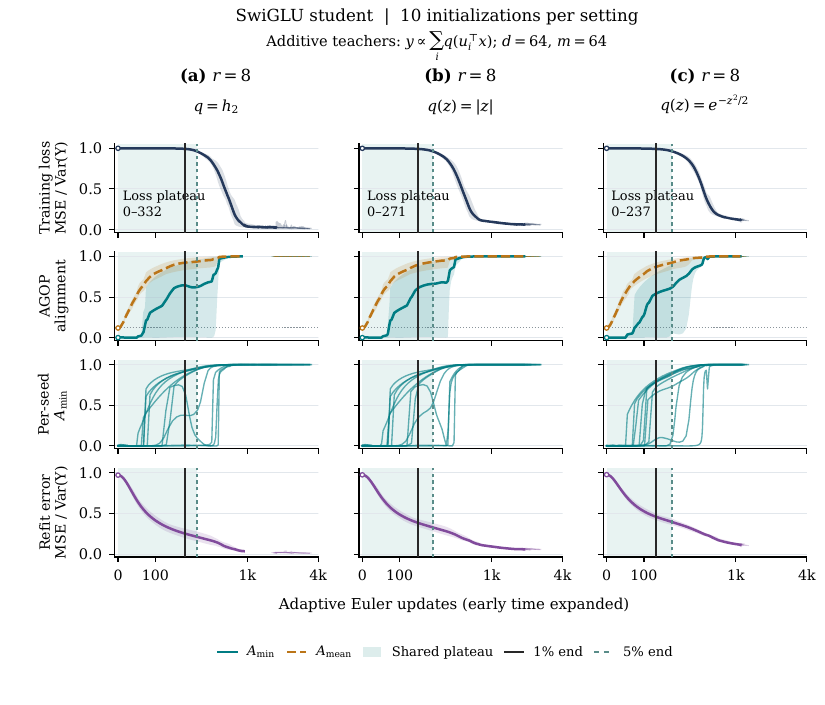}
\caption{\textbf{SwiGLU teacher-link comparison at rank eight.}
Quadratic, absolute-value and Gaussian-bump links use the same
ten-seed protocol with head-rate multiplier $0.01$.
Loss and numerical unrestricted-refit risks are divided by $V_Y$.
These are empirical trajectories under declared loss tolerances,
not guarantees for every link or initialization.}
\label{fig:highrank-swiglu-teachers}
\end{figure}
\clearpage
\begin{figure}[p]
\centering
\includegraphics[width=\textwidth]{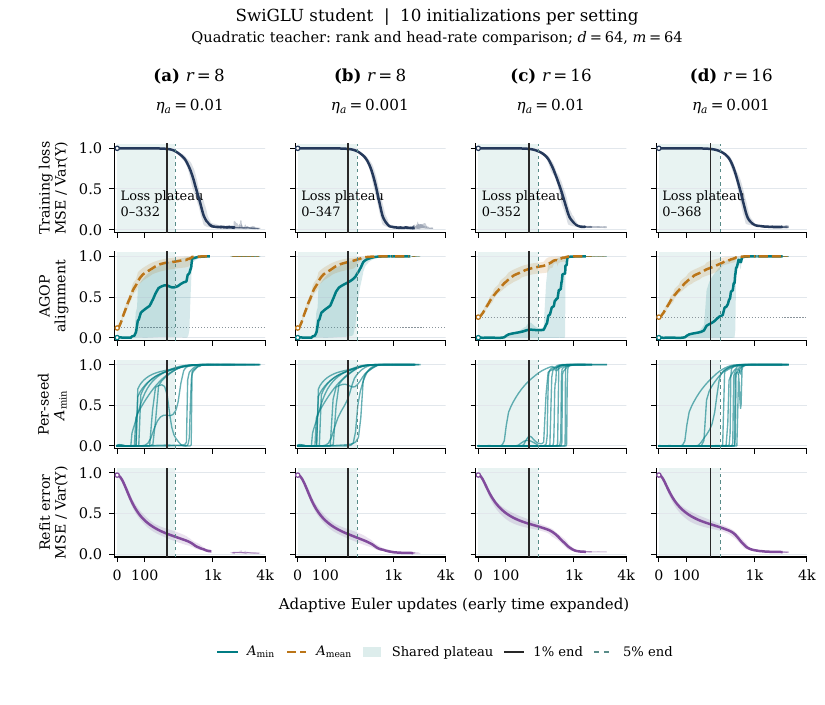}
\caption{\textbf{Slower SwiGLU heads help some runs without resolving the rank limit.}
Quadratic teachers at ranks eight and sixteen compare head-rate
multipliers $0.01$ and $0.001$, with ten paired seeds per setting.
The curves compare how head-update speed affects loss and alignment;
the rank-sixteen setting remains a limitation.
The settings remain distinct cohorts; outcomes are not pooled or
selected by performance.}
\label{fig:highrank-swiglu-headrate}
\end{figure}
\clearpage
\begin{figure}[p]
\centering
\includegraphics[width=\textwidth]{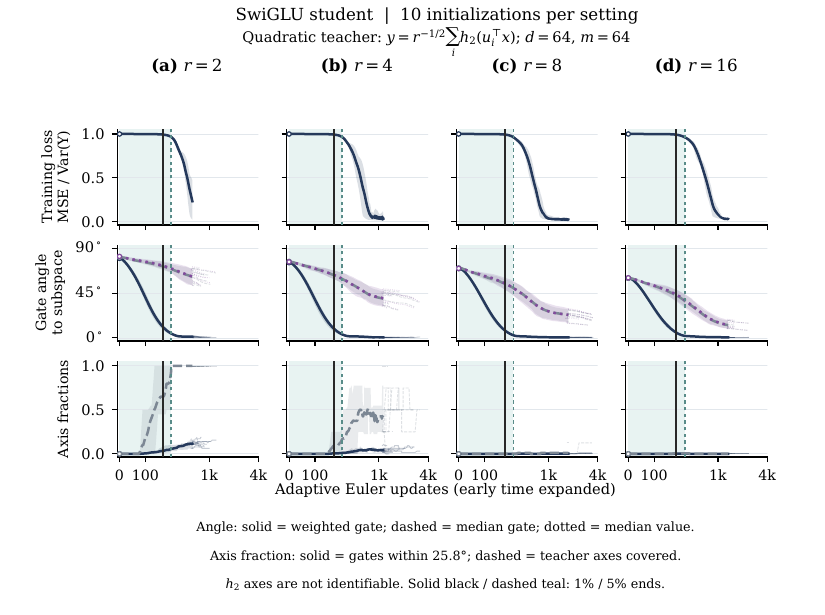}
\caption{\textbf{SwiGLU gate geometry across ranks.}
The same standard-head rank sweep as Figure~\ref{fig:highrank-swiglu-ranks}.
Gate-summary angles from \eqref{eq:highrank-gate-angles}, together
with the dotted value-weight angle in \eqref{eq:highrank-value-angle},
can improve
even where the weakest AGOP direction is not recovered during the
plateau. They describe different quantities, and weighted gate
alignment alone is not an all-direction recovery result.
The quadratic teacher has no identifiable internal axes.}
\label{fig:highrank-swiglu-neurons-ranks}
\end{figure}
\clearpage
\begin{figure}[p]
\centering
\includegraphics[width=\textwidth]{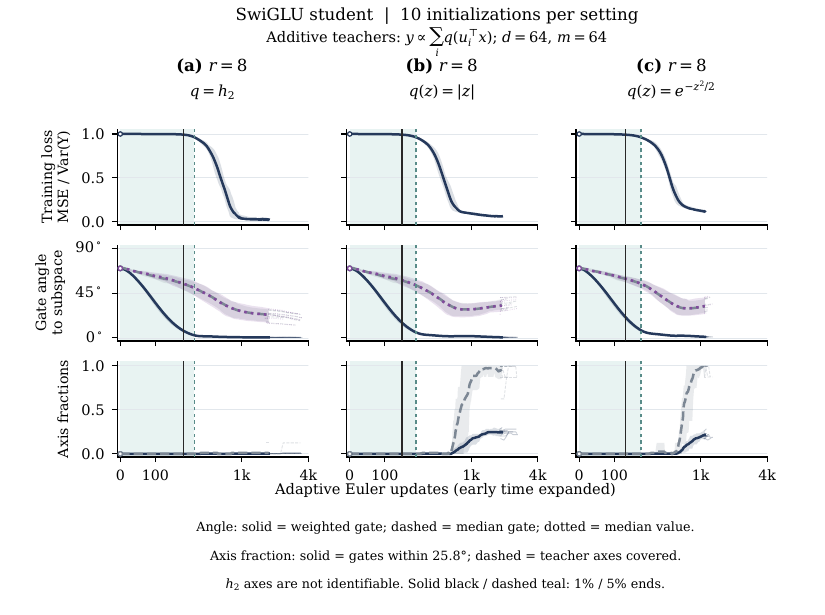}
\caption{\textbf{SwiGLU gate entry and axis specialization for different links.}
The same rank-eight teacher comparison as Figure~\ref{fig:highrank-swiglu-teachers}.
Effective gate and value angles and thresholded gate-axis counts use saved
scalar summaries; no per-neuron angle distributions are inferred.
For the nonquadratic teachers, the fraction of near-axis gates rises
later than the decline in the effective subspace angle.
The weighted angle and unweighted gate counts must be interpreted
together and relative to the stated loss window; they do not classify
every neuron or establish that coordinate mixing is necessary.}
\label{fig:highrank-swiglu-neurons-teachers}
\end{figure}
\clearpage

\clearpage
\clearpage
\section{Teacher families and complete parameter prescriptions}
\label{app:parameters}
\subsection{Positive mixtures of Gaussian third derivatives}
\label{app:sw-parameters}
Here $X\sim\gamma_d=N(0,I_d)$, $U=(u_1,\ldots,u_r)$,
$P_U=UU^T$ and $P_{U^\perp}=I_d-P_U$. Function norms are Gaussian:
$\|g\|_2^2=\mathbb E[g^2]$ and
$\|g\|_{H^1}^2=\mathbb E[g^2+\|\nabla g\|^2]$.
Write $h_k=\operatorname{He}_k/\sqrt{k!}$ for the normalized
probabilists' Hermite polynomial, and $\Phi$ for the standard normal
CDF. Expectations without an indicated law are over the displayed
Gaussian input; expectations of initialization statistics are over the
original parameter draw.
Let $J=1-\alpha>0$, $U^TU=I_r$, and let $\mu_i$ be probability
measures on $[1/3,1]$. Throughout this appendix $\varphi_v$ denotes
the density of $N(0,v)$ and $\varphi=\varphi_1$. Define
\begin{align}
 g_v(t)&=v^{-7/2}(t^3-3vt)e^{-(v^{-1}-1)t^2/2},&
 g_i&=\int g_v\,d\mu_i(v),&N_i&=\|g_i\|_2,\notag\\
 y_c(X)&=\sum_{i=1}^r\lambda_i g_i(u_i^TX)/N_i,&
 \lambda_i&>0,&\sum_i\lambda_i^2&=1.
 \label{eq:app-mixture}
\end{align}
The target is $y=y_c+e=y(U^TX)$ with $\|y\|_2=1$ and
$\|e\|_{H^1}\le\zeta$. The residual may contain interactions on $U$.
The central common-link theorem takes $\mu_i=\mu$,
$\lambda_i=r^{-1/2}$ and $e=0$. For the more general theorem require
\begin{equation}
 K_i^*=\max_{B\ge0}\frac{J\lambda_i}{2N_i}
 \sqrt{\frac B{1+B}}\int v^{-3/2}\varphi(\sqrt{B/v})\,d\mu_i(v),
 \qquad K_*:=\max_iK_i^*,\qquad \min_iK_i^*\ge\frac56K_*.
 \label{eq:app-signals}
\end{equation}
Here and below, all quantities are public choices made before the
independent Gaussian draw. The following list prints the entire
prescription, including the conditions needed for later loss release.
There is no further implicit width or dimension condition.

\paragraph{Structural constants and sizes.}
Fix $0<\delta<1/2$, $0<\epsilon_L\le1$ and
$0<\epsilon_G\le1/4$. Set
\begin{gather}
 \lambda_*:=\min_i\lambda_i,\quad \chi_0=10^{-4},\quad
 e_*=\frac{J\chi_0}{4096\sqrt r},\quad
 \varepsilon_{\rm BM}=e_*/8,\quad
 \kappa_{\rm BM}=J\lambda_*/1024,\notag\\
 \vartheta=\min\left\{\lambda_*/1000,e_*/16,
             \sqrt{\kappa_{\rm BM}\varepsilon_{\rm BM}/6400}\right\},
 \quad
 R=\max\left\{8,16/e_*,
             \sqrt{25600/(\kappa_{\rm BM}\varepsilon_{\rm BM})}\right\}.
 \label{eq:sw-structure}
\end{gather}
With $p_I=\Phi(1)-\Phi(1/2)$ and $p_G=\Phi(2)-\Phi(1)$, put
\begin{equation}
 p_{\rm rect}=p_Gp_I^2
 \begin{cases}1,&r=1,\\
 [2\Phi(\vartheta/(4\sqrt{r-1}))-1]^{r-1},&r>1,
 \end{cases}
 \qquad m\ge\max\{r,p_{\rm rect}^{-1}\log(16r/\delta)\}.
 \label{eq:sw-width}
\end{equation}
Choose integer $d>r$ satisfying all of
\begin{gather}
 \ell=\log(64m/\delta),\quad B_1=64\sqrt r/J,\quad
 \ell_I=\log(8m/\delta),\quad Q_I=r+2\sqrt{r\ell_I}+2\ell_I,
 \notag\\
 d\ge\max\{r,64,128\ell,16r/\delta,16\ell_I,
                    2Q_I(32B_1/\chi_0)^{2/3}\}.
 \label{eq:sw-dimension}
\end{gather}
These conditions allow compatible $m\ge d$. Once $d$ is fixed set
$x_0=d^{-1/2}$, $k=J\lambda_*/(4096d^2)$, and
$Q_*=k/(10\sqrt d)$.

\paragraph{Acquisition and spectral clocks.}
Define
\begin{gather}
 T_1=\kappa_{\rm BM}^{-1}\log(4/\varepsilon_{\rm BM}),\quad
 T_0=4R/k+T_1+1,\quad S_0=5me^{4T_0},\quad v_*=4R^2,
 \notag\\
 c_G=\frac{J\varphi(1)\sqrt{\epsilon_G}}{128\sqrt r},\quad
 \Gamma=\frac{3S_0/4+400m}{v_*},\quad
 T_G=\frac{100}{K_*}\log\frac{2\Gamma}{c_G},\notag\\
 T=T_0+T_G+1,\quad S=S_0e^{4(T_G+1)},\quad M=2\sqrt S.
 \label{eq:sw-clock}
\end{gather}
Choose one positive mesh $h$ and one positive force allowance $\nu$
obeying \emph{every} cutoff in
\begin{align}
 h&\le\min\left\{2^{-12},\frac{k}{512M},\frac1{4M^2T},
 \frac{\kappa_{\rm BM}\varepsilon_{\rm BM}}{3200},
 \frac{K_*}{10000},\frac1{1000r},10^{-4},
 \frac{c_Gv_*}{20000TS}\right\},\label{eq:sw-mesh}\\
 \nu&\le\min\left\{\begin{gathered}\frac{Q_*}4,\frac1{400T},
 \frac{x_0}{4MT},\frac{\vartheta\lambda_*kx_0^2}{2048M^3},
 \frac{k}{512M},\frac{2k}{3M^2},\\ \frac1{80M^2T},
 \frac{\kappa_{\rm BM}\varepsilon_{\rm BM}}{320},
 \frac{K_*}{10000},\sqrt{\frac{K_*}{1000T}},
 \frac{c_Gv_*}{80TS}\end{gathered}\right\}.\label{eq:sw-force}
\end{align}
Set
\begin{gather}
 N_* =\lceil4R/(hk)\rceil+\lceil T_1/h\rceil,\qquad
 N=N_*+\lceil T_G/h\rceil,\notag\\
 t_G=(N-N_*)h,\qquad v_G=v_*e^{3.32K_*t_G}.
 \label{eq:sw-checkpoints}
\end{gather}
In particular $N_*h\le T_0$, $T_G\le t_G\le T_G+h$ and $Nh\le T$.

\paragraph{Continuation, residual and original Gaussian scale.}
After choosing the same $h$, define
\begin{gather}
 \gamma=K_*/2,\quad
 k_c=\left\lceil\frac2{h\gamma}\log\frac{32S}{\gamma v_*}\right\rceil,
 \quad T_c=h+\frac2\gamma\log\frac{32S}{\gamma v_*},\notag\\
 S_c=Se^{4T_c},\quad N_c=N+k_c,\quad
 R_c^2=\gamma m/64,\quad G_c=\gamma^2/8192=K_*^2/32768,\notag\\
 0<\nu_c\le\min\{\gamma/16,\sqrt{\gamma/(8T_c)}\}.
 \label{eq:sw-release-parameters}
\end{gather}
The allowable residual and raw scale satisfy
\begin{gather}
 0\le\zeta\le\min\left\{\nu/2,\chi_0/128,\lambda_*/100,
                       \chi_0/(32B_1),\nu_c/2\right\},
 \label{eq:sw-residual}\\
 0<s^2\le\min\left\{\frac{m\nu}{S},
 \frac{m\epsilon_L}{20M^2},\frac{m\nu_c}{S_c},
 \frac{R_c^2}{4S_c},\frac{mG_c}{S_c},
 \frac{m\epsilon_L}{20S_c}\right\}.
 \label{eq:sw-scale}
\end{gather}
All raw coordinates are independently $N(0,s^2/d)$ and the single
Euclidean step on full MSE is $\eta=mh/2$. Smaller positive $h,\nu,s$
are allowed, with the downstream quantities recomputed in the printed
order. No parameter is chosen from an observed favorable trajectory.
The very small scale and mesh are part of the theorem.

\subsection{Additive cubic neighborhood: complete parameter package}\label{app:aj-parameters}
Let $X\sim N(0,I_d)$, let $U=(u_1,\ldots,u_r)$ have orthonormal
columns, and let
\[
 y_c=\sum_{i=1}^r a_i h_3(u_i^TX),\qquad
 a_i>0,\quad\sum_i a_i^2=1,\quad
 a_0=\min_i a_i,\quad a_{\max}=\max_i a_i.
 \tag{AJ1}
\]
In this subsection the local $N,\overline T,G_{\rm rel}$ are the main-text $N_A,\overline T_A,G_A$, respectively.
Here $h_3(t)=(t^3-3t)/\sqrt6$, $\gamma_d=N(0,I_d)$,
$\|g\|_2^2=\mathbb E[g(X)^2]$, and
$\|g\|_{H^1}^2=\mathbb E[g(X)^2+\|\nabla g(X)\|^2]$.
Nonzero coefficient signs can be absorbed into the $u_i$.
For $0\le\alpha<1$, train
\[
 f_n(X)=\frac1m\sum_{j=1}^m A_{j,n}
       \sigma_\alpha(W_{j,n}^TX+B_{j,n}),\qquad
 \sigma_\alpha(t)=\alpha t+(1-\alpha)t_+,\qquad
 L_n=\mathbb E[(y-f_n)^2]
 \tag{AJ2}
\]
by simultaneous Euclidean GD in every raw coordinate at one fixed
step $\eta=mh/2$. All initial raw coordinates, including heads and
biases, are independent $N(0,s^2/d)$. There is no screening, resetting,
freezing, sign selection by the algorithm, or refitting during training.

Fix $0<\delta<1/2$ and $0<\delta_L\le1$, and put
$\delta_0=\delta/2$ and $\zeta=\delta/2$.
The following scalar constants depend only on the indicated public
parameters, not on a realized initialization:
\begin{gather}
 b_*^2=(\sqrt5-1)/2,\quad p=2\Phi(b_*)-1,\quad
 v_g=p-2b_*\varphi(b_*)+b_*^2(1-p)-p^2,\notag\\
 t_g=-2b_*\varphi(b_*)/\sqrt6,\quad K_*=t_g^2/v_g>2/3,\quad
 D_\alpha=\sqrt{(1-\alpha)^{-2}+p^2/(1+\alpha)^2},\notag\\
 C_\alpha=|t_g|D_\alpha/v_g,\qquad
 D_g=|t_g|\max\{p,1-p\}/v_g,\notag\\
 \xi=\min\{(16\sqrt r)^{-1},(128C_\alpha\sqrt r)^{-1}\},\quad
 \vartheta=\min\{a_0/8,\xi/32\},\quad R_D=64/\xi.
 \tag{AJ3}
\end{gather}
Here $\varphi,\Phi$ are the standard normal density and distribution
function. Set
\[
 p_C=[\Phi(2)-\Phi(1)][\Phi(2)-\Phi(1/2)]^2
       [2\Phi(\vartheta/(4\sqrt{r-1}))-1]^{r-1},
 \tag{AJ4}
\]
omitting the final factor when $r=1$. Choose integer sizes satisfying
\begin{gather}
 m\ge\max\{4r,\log(16r/\delta_0)/p_C\},\quad
 \ell=\log(96m/\delta_0),\quad H=r+2\sqrt{r\ell}+2\ell,\notag\\
 d\ge\max\{256,128\ell,16H,
                (r+4)(128/\delta_0)^{1/3},16r/\delta\}.
 \tag{AJ5}
\end{gather}
In particular $r<d$ and $m\ge r$. There is no $m<d$ requirement.
The restriction on width is stronger than $m\ge r$: it pays for all
four axis/bias categories under the original Gaussian draw.

Define the original-coordinate constants
\begin{gather}
 g_0=\frac{\zeta a_0}{4mr^2\sqrt d},\quad
 b_0=a_{\max}\sqrt{H/d},\quad
 D_a=\sum_i a_i^{-2},\quad L_a=(a_{\max}^2D_a)^{1/3},\notag\\
 Z_*=40+640\sqrt{D_a}
                  \{b_0^2/g_0+(1+L_a)b_0\},\quad
 \Gamma=3Z_*+182.
 \tag{AJ6}
\end{gather}
Choose any
$0<\varepsilon\le\min\{1/4,(768D_g^2)^{-1}\}$ and then a public
radius $R$ with
\[
 R\ge\max\left\{R_D,16\sqrt m,
       \sqrt{m\Gamma/18},
       \sqrt{\frac{m\sqrt{1522/\varepsilon}}
                         {18(1-\alpha)\varphi(1)}}\right\}.
 \tag{AJ7}
\]
Only after this choice set
\begin{gather}
 c=(1-\alpha)/\sqrt6,\quad k=c\varphi(1)a_0/(64d^2),\quad
 T=1+8R/k,\quad\overline T=T+1,\notag\\
 M^2=20m e^{4\overline T},\quad \overline S=4\log(8M),\quad
 0<h\le\min\{1/512,\xi/1024,(16\overline T)^{-1}\},\quad
 N=\lceil T/h\rceil .
 \tag{AJ8}
\end{gather}
Choose $\nu>0$ at most the minimum of
\[
 \frac1{64},\quad \frac{k}{10^4M^2},\quad
 \frac1{1024\overline T^2},\quad
 \frac{k}{64M^2\overline S\sqrt d},\quad
 \frac{k\vartheta a_0}{1024M^3d},\quad
 \frac{k\xi}{12800M},\quad
 \frac{m}{1024\overline T M^2},\quad
 \frac{g_0e^{-4\overline T}}{10^5\overline T M}.
 \tag{AJ9}
\]
The full target is any additive function on $U^TX$ with
\[
 \|y\|_2=1,\qquad y\in H^1(\gamma_d),\qquad
 \|y-y_c\|_{H^1}\le\nu/2.
 \tag{AJ10}
\]
Its actual mean, linear component and higher tails remain in all
gradients, losses and comparisons. Since $\nu/2<a_0$, its minimal
index subspace is $U$, by the cubic tensor argument in Appendix~\ref{app:teacher-rank}. Thus $r$ here is a true teacher rank, not only a representation.
Finally put
\begin{gather}
 a_{\rm rel}=1/(64m\overline T),\quad
 R_{\rm rel}^2=a_{\rm rel}m^2,\quad
 G_{\rm rel}=1/(8192\overline T^2),\notag\\
 0<s^2\le\min\left\{\frac{m\nu}{M^2},
  \frac{m\delta_L}{50M^2},
  \frac{R_{\rm rel}^2}{4M^2},
  \frac{mG_{\rm rel}}{M^2}\right\}.
 \tag{AJ11}
\end{gather}
This order of choices is acyclic: $Z_*$ is determined before $R$,
and neither $Z_*$ nor $R$ depends on the ensuing clock or raw scale.
All quantities are finite and positive at every allowed finite rank
and coefficient vector. No practical scale is asserted.

\section{Exact update and minimal teacher rank}
\label{app:teacher-rank}
Here $X\sim N(0,I_d)$,
$f=m^{-1}\sum_j A_j\sigma_\alpha(W_j^TX+B_j)$,
$\sigma_\alpha(t)=\alpha t+(1-\alpha)t_+$, and a superscript $+$
denotes the next simultaneous iterate. The teacher axes are the
columns of $U$, and $h_3(t)=(t^3-3t)/\sqrt6$.
The coefficient families $a_i$ and $\lambda_i/N_i$ are those in
AJ1 and \eqref{eq:app-mixture}, respectively.
For $z_j=W_j^TX+B_j$, the simultaneous full-MSE step $\eta=mh/2$ is
\begin{align}
 A_j^+&=A_j+h\mathbb E[(y-f)\sigma_\alpha(z_j)],\notag\\
 W_j^+&=W_j+hA_j\mathbb E[(y-f)\sigma'_\alpha(z_j)X],\notag\\
 B_j^+&=B_j+hA_j\mathbb E[(y-f)\sigma'_\alpha(z_j)].
 \label{eq:app-exact-update}
\end{align}
Every right-hand side uses the same current state. At a nonzero spatial
row the kink has Gaussian probability zero. At a pure-bias row with
nonzero bias the derivative is the corresponding affine derivative.
The identically zero joint row is absorbing and may use any fixed
subgradient convention. Initial nonzero balanced hidden states remain
nonzero throughout the controlled update windows, by the compatible
balance identity and step bounds proved below.

For a unit vector $u$, the normalized cubic $h_3(u^TX)$ corresponds
isometrically to the symmetric third-chaos tensor $u^{\otimes3}$.
The central cubic tensor is $\sum_i c_i u_i^{\otimes3}$, with
$c_i=a_i$ in AJ and $c_i=\sqrt6\lambda_i/N_i$ in the mixture theorem.
Its mode-one matricization has singular values $c_i$, because both
$u_i$ and $u_i\otimes u_i$ are orthonormal families. Orthogonal chaos
projection and matricization give an operator perturbation at most
$\|y-y_c\|_2$. In AJ this is less than $a_0$; in SW it is at most
$\lambda_*/100<\sqrt6\lambda_*/9$. Thus the cubic matricization of
the actual target has rank at least $r$. If the target were measurable
on a smaller linear subspace, all modes of every chaos tensor would
lie in that subspace, a contradiction. Since the target depends on
$U^TX$ by hypothesis, its minimal index subspace is exactly $U$.

For $Z\sim N(0,1)$, the mixture identity $g_v\varphi=-\varphi_v'''$ implies
\[
 \mathbb E[g_v(Z)e^{tZ-t^2/2}]=t^3e^{-(1-v)t^2/2}.
\]
The resulting coefficients are printed in MP2. Their squared sums and
Hermite-degree-weighted squared sums converge uniformly for
$v\in[1/3,1]$. Indeed their generating function is
$K(t)=3(2+3t)(1-t)^{-7/2}$ at $t=(1-v)^2\le4/9$;
the weighted sum is $2tK'(t)+3K(t)$. Minkowski's inequality therefore
justifies integration against arbitrary probability measures,
including non-atomic ones, in both $L^2$ and $H^1$.
It gives $\sqrt6\le N_i<9$ and
$\|g_i/N_i\|_{H^1}<13$. The coefficient at every odd degree
$2j+3$ is nonzero whenever $\mu_i([1/3,1))>0$. No truncation enters
the updates or the proofs.

\subsection{Common original-law event assembly}
\label{app:probability-assembly}
Write $[x]_+=\max\{0,x\}$. All events in the following rule use one
probability space; in our applications this is the original Gaussian draw.

\begin{lemma}[Common-event and headroom rule]
\label{lem:common-event}
Let $E_1,\ldots,E_k$ be measurable events with
$\Pr(E_i^c)\le p_i$, and put $E=\bigcap_i E_i$.
If the desired conclusions hold simultaneously on $E$, their probability
is at least $[1-\sum_i p_i]_+$.
For a coupled scale $\varepsilon>0$, suppose instead that $c>0$ is a
measurable random cutoff on $E$, and that all desired conclusions hold
on $E\cap\{\varepsilon<c\}\cap H_\varepsilon$, where
$\Pr(H_\varepsilon^c)\le q$ at each scale. Then their probability is at least
\[
 [1-\textstyle\sum_i p_i-q-r_\varepsilon]_+,
 \qquad r_\varepsilon=\Pr(E\cap\{\varepsilon\ge c\})\longrightarrow0.
\]
Consequently its limit inferior is at least $[1-\sum_i p_i-q]_+$.
When conclusions are claimed at a common endpoint, the premise must
supply that common endpoint on the same trajectory.
\end{lemma}
\begin{proof}
The complement of the displayed intersection is contained in the union
of the listed failure events. The union bound proves both inequalities,
without independence. On $E$, $c>0$, so
$\mathbf1_E\mathbf1_{\{\varepsilon\ge c\}}\to0$; bounded convergence proves
the limit. Probability is nonnegative, giving the positive parts.
\end{proof}
The events $H_\varepsilon$ need not stabilize along a coupled draw.
A random cutoff supplies no public numerical finite-scale confidence bound.

\section{Full mixture theorem and proof}
\label{app:proof-sw}
Use the teacher and parameter choices of
Appendix~\ref{app:sw-parameters}, with $J=1-\alpha$ and $U=(u_1,\ldots,u_r)$.
The full raw student is $f_n=m^{-1}\sum_j A_{j,n}
\sigma_\alpha(W_{j,n}^TX+B_{j,n})$, with
$\sigma_\alpha(t)=\alpha t+Jt_+$ and $X\sim N(0,I_d)$.
For clarity, $L_n=\mathbb E[(y-f_n)^2]$,
$\mathsf G_n=\mathbb E[\nabla_x f_n\nabla_x f_n^T]$, and $P_n$ denotes its
leading rank-$r$ spectral projector at the theorem's checkpoints. The alignment scores are
$A_{\min,n}=\lambda_{\min}(U^TP_nU)$ and
$A_{\rm mean,n}=r^{-1}\operatorname{tr}(U^TP_nU)$.
The diagnostic below uses
$\phi_{j,n}=\sigma_\alpha(W_{j,n}^TX+B_{j,n})/
\sqrt{\|W_{j,n}\|^2+B_{j,n}^2}$, with a zero augmented row assigned
zero feature: $\mathcal R(n)$ is the infimum of
$\|y-\sum_jv_j\phi_{j,n}\|_2^2$ under the two stated coefficient budgets.
\begin{theorem}[Comparable mixture signals]
\label{app:sw-theorem}
Under all conditions in Appendix~\ref{app:sw-parameters}, let
$\mathcal R(n)$ be the normalized, unprojected, full-bank diagnostic
with $\ell^2$ budget $32/J$ and $\ell^1$ budget $64\sqrt r/J$.
With probability at least $1-7\delta/8$ over the original Gaussian draw,
the following conclusions hold simultaneously. Both leading rank-$r$
AGOP projectors $P_0,P_N$ have positive cutoff gaps, and
\begin{gather}
 A_{\min,0}\le A_{\rm mean,0}\le1/4,\qquad
 A_{\min,N}\ge1-\epsilon_G,\qquad
 A_{\rm mean,N}\ge1-\epsilon_G/r,\label{eq:full-sw-angle}\\
 \lambda_r(\mathsf G_N)\ge
 a_G:=\tfrac12[J\varphi(1)s^2v_G/(64m\sqrt r)]^2>0,
 \qquad \lambda_{r+1}(\mathsf G_N)\le\epsilon_Ga_G,
 \label{eq:full-sw-spectrum}\\
 \mathcal R(0)\ge1-\chi_0/8,\qquad
 \mathcal R(N)<(\sqrt{3/5}+\chi_0/64)^2,
 \qquad\mathcal R(0)-\mathcal R(N)>0.399.
 \label{eq:full-sw-risk}
\end{gather}
For the entire prefix $0\le n\le N_c$,
\begin{equation}
 1-\epsilon_L/16\le L_n\le1+\epsilon_L/16,\qquad
 L_n\ge1-G_c,\qquad
 \frac{\max_{n\le N_c}L_n}{\min_{n\le N_c}L_n}\le1+\epsilon_L/4.
 \label{eq:full-sw-loss}
\end{equation}
There is a finite later $N_2>N_c$ with
$L_{N_2}\le1-2G_c$, and hence $L_n-L_{N_2}\ge G_c$ for every
$n\le N_c$, under the same original GD update. The physical clocks obey
\begin{gather}
 \frac m2(4R/k+T_1+T_G)\le\eta N\le mT/2,\qquad
 \eta(N-N_*)\ge mT_G/2,\qquad
 \eta N_c\le(m/2)(T+T_c).
 \label{eq:full-sw-time}
\end{gather}
One may take $N_2$ to be the first subsequent hit of joint raw radius
$R_c$. If $\mathfrak R_{N_c}^2=\sum_j(A_j^2+\|W_j\|^2+B_j^2)$
at $N_c$ and $a_{\rm sig}=\gamma/(64m)$, then
\begin{equation}
 \frac m4\log\frac{R_c}{\mathfrak R_{N_c}}
 \le\eta(N_2-N_c)\le
 \eta+a_{\rm sig}^{-1}\log\frac{R_c}{\mathfrak R_{N_c}}.
 \label{eq:full-sw-motion}
\end{equation}
No angle or diagnostic persistence after $N$ is required or asserted.
\end{theorem}

\begin{proof}
The complete proof is organized into the lemmas of
Appendices~\ref{app:kernels}--\ref{app:general-release}. We give the
join here, including the numerical constants that determine the
spectral conclusion.

\emph{Original event and actual profile acquisition.}
Appendix~\ref{app:mixture-acquisition} proves that the Gaussian category
occupancy and simultaneous row-norm event fails with probability at
most $3\delta/8$. On that event its compatible-force induction controls
all original rows through $N$, with $\sum_jV_j\le S$ and
$\|E_j\|\le\nu$. Four distinct original rows per teacher axis retain
calibrated profiles throughout $[N_*,N]$. The notation is
$(A,W,B)=s(q,w,b)$, $\rho=\|w\|$,
$V=q^2+\rho^2+b^2$, $C=\mathbb E[y_c\sigma_\alpha(w^TX+b)]$,
and $Q=qC/V$. The selected rows satisfy
\begin{equation}
 \rho>2R,\quad |q|/\rho>19/20,\quad
 1/16\le b^2/\rho^2\le3/4,\quad
 \|(w/\rho,b/\rho)-(\epsilon u_i,\tau\sqrt{B_i^*})\|\le e_*.
 \label{eq:sw-profiles}
\end{equation}
Their sign is $\operatorname{sign}(q)=-\epsilon\tau$.
Appendix~\ref{app:mixture-defect} proves, using the optimized signal
condition, that each selected $V_j\ge v_G$ at $N$. It simultaneously
controls every unselected or adversely signed row by
\begin{equation}
 \Xi_N:=\sum_j|q_j|
 \left(\|P_{U^\perp}w_j\|^2+
     \sum_i\min\{-\operatorname{sign}(q_jb_j)u_i^Tw_j,0\}^2
 \right)^{1/2}\le c_Gv_G.
 \label{eq:sw-full-defect}
\end{equation}
At $q_jb_j=0$ either orientation is allowed. This is an all-row
estimate for the actual coupled simultaneous candidates, including
all entry and exit steps.

\emph{Full current-head AGOP.}
Use $e_r=r^{-1/2}(1,\ldots,1)^T$ and the tests $\xi$, oriented rows
$x_j$, and weights $\omega_j$ of BF21--24 at $N$. They give $C_\xi=\mathbb E[\xi\xi^T]\preceq2I$ and
\begin{equation}
 D=\mathbb E[(U^T\nabla f_N)\xi^T]
   =\sum_j\omega_j(e_r^Tx_j)x_jx_j^T,
 \label{eq:sw-complete-frame}
\end{equation}
including every current head, with the zero-row conventions of BF22.
Its complete negative part is bounded by $(J\varphi(1)s^2/m)\Xi_N$.
Choose one category per axis. By \eqref{eq:sw-profiles},
$\|x_j-e_i\|\le e_*$, $|q_j|\rho_j\ge V_j/4$,
and $|z_j|\varphi(z_j)\ge\varphi(1)/4$.
Consequently $\omega_j\ge J\varphi(1)s^2v_G/(16m)$ and
\begin{equation}
 \lambda_{\min}(D)\ge\frac{J\varphi(1)s^2}{m}
 \left(\frac{v_G}{32\sqrt r}-\Xi_N\right)
 \ge\frac{J\varphi(1)s^2v_G}{64m\sqrt r}.
 \label{eq:sw-frame-floor}
\end{equation}
Here $(r^{-1/2}-e_*)(1-\sqrt r e_*)^2\ge1/(2\sqrt r)$ and
$c_G\le1/(64\sqrt r)$. Every other positive frame summand is
positive semidefinite and is retained.
The complete outside-trace estimate BF27 gives
\begin{equation}
 \operatorname{tr}(P_{U^\perp}\mathsf G_NP_{U^\perp})
 \le\beta_G:=(s^4/m^2)\Xi_N^2,
 \qquad \beta_G/a_G\le\epsilon_G/2.
 \label{eq:sw-outside-floor}
\end{equation}
Lemma~\ref{lem:frame-agop}, with $H=\nabla f_N$, $c=2$ and
$\gamma=J\varphi(1)s^2v_G/(64m\sqrt r)$, now proves
\eqref{eq:full-sw-spectrum} and the endpoint bounds in
\eqref{eq:full-sw-angle}, with a genuine positive cutoff gap.

\emph{Initial scores and the same diagnostic.}
Appendix~\ref{app:initialization} proves almost-sure simplicity and
rotational invariance, so $\mathbb E[A_{\rm mean,0}]=r/d$ and
$\Pr(A_{\rm mean,0}>1/4)\le4r/d\le\delta/4$.
The central mixture has Hermite rank at least three. Apply
Lemma~\ref{lem:initial-bank} with $k_0=3$,
$\delta_I=\delta/4$ and the unchanged $B_1=64\sqrt r/J$.
Equations~\eqref{eq:sw-dimension}--\eqref{eq:sw-residual} give
$(2Q_I/d)^{3/2}+\zeta\le\chi_0/(16B_1)$, so
$\mathcal R(0)\ge1-\chi_0/8$ outside failure probability $\delta/4$.
Appendix~\ref{app:mixture-prediction} supplies a legal terminal witness
under both original caps, retaining the full infinite Hermite tail.
Its $L^2$ error is below $\sqrt{3/5}+C_1e_*+\zeta$,
where $C_1<12\sqrt r/J$, $C_1e_*<\chi_0/128$ and
$\zeta\le\chi_0/128$. Lemma~\ref{lem:bounded-refit-comparison}
therefore proves \eqref{eq:full-sw-risk}.

\emph{Loss band and continuation.}
At $N$ a selected actual row has $Q>K_*/2$ and $V>v_*$, and
all rows are balanced with complete prefix energy at most $S$.
Apply Appendix~\ref{app:mixture-release} with
\[
 (\gamma,U,v,\nu,\overline U,M)
       =(K_*/2,S,v_*,\nu_c,S_c,N_c).
\]
Here the local $M$ in that appendix is its suffix endpoint;
$0<\gamma\le1/2$ follows from the unit central-kernel bound (MR8).
Equation~\eqref{eq:sw-release-parameters} is exactly (MR12)'s suffix
clock and force allowance, while
\eqref{eq:sw-residual}--\eqref{eq:sw-scale} supply its residual and
scale cutoffs. Moreover $h\le2^{-12}<1/1024$.
Thus (MR19)--(MR22) verify the common full-teacher entry rule
\ref{ag:averaged-release}, with $R_{\rm sig}=R_c$ and
$G_{\rm sig}=G_c$. Equations (MR23)--(MR24) give the later loss
decrease and \eqref{eq:full-sw-motion} at the unchanged full-MSE rate.
Throughout the prefix through $N_c$,
\[
 \|f_n\|_2\le\frac{s^2S_c}{2m}
            \le\min\{G_c/2,\epsilon_L/40\}.
\]
Loss expansion proves the band and baseline
\eqref{eq:full-sw-loss}; its ratio follows from
$(1+\epsilon_L/16)/(1-\epsilon_L/16)\le1+\epsilon_L/4$.
The integer ceiling definitions give \eqref{eq:full-sw-time}.
The acquisition, initial-projector and initial-refit events have
failure bounds $3\delta/8,\delta/4,\delta/4$, respectively. Every
conclusion above holds on their intersection for the same run, with
geometry and refit evaluated at the same public $N$.
Lemma~\ref{lem:common-event} gives probability at least $1-7\delta/8$.
\end{proof}

\section{The additive cubic-neighborhood extension}\label{app:proof-aj}
All parameters, the full teacher $y$ and the raw student $f_n$ are
those of AJ1--11 in Appendix~\ref{app:aj-parameters}.
Here $(A_j,W_j,B_j)=s(q_j,w_j,b_j)$,
$\theta_j=(w_j,b_j)$, $\rho_j=\|w_j\|$, and
$V_n=\sum_j(q_{j,n}^2+\|\theta_{j,n}\|^2)$ is complete normalized
joint energy. Write $P_U=UU^T$, $P_{U^\perp}=I_d-P_U$ and
$\|g\|_2^2=\mathbb E[g(X)^2]$ for $X\sim N(0,I_d)$.
Define the complete spatial-energy statistic and full student AGOP by
\[
 A_{{\rm sub},n}=\frac{\sum_j\|UU^TW_{j,n}\|^2}
                         {\sum_j\|W_{j,n}\|^2},\qquad
 \mathsf G_n=\mathbb E[\nabla_x f_n\nabla_x f_n^T].
 \tag{AJ12}
\]
At $n=0,N$ let $P_n$ be its genuine leading rank-$r$ projector,
and set
\[
 A_{\min,n}=\lambda_{\min}(U^TP_nU),\qquad
 A_{\rm mean,n}=r^{-1}\operatorname{tr}(U^TP_nU).
 \tag{AJ13}
\]
The theorem proves that these projectors are well defined; it does
not complete a deficient eigenspace arbitrarily.

\begin{theorem}[Simultaneous actual AGOP acquisition during a high-loss window]\label{app:aj-theorem}
Under AJ1--11, one event of the original initialization, of probability
at least $1-\delta$, has all of the following properties.

Throughout $0\le n\le N$, the full original loss satisfies
\[
 1-\delta_L/20\le L_n\le1+\delta_L/20,\qquad
 L_n\ge1-G_{\rm rel},\qquad
 \frac{\max_{n\le N}L_n}{\min_{n\le N}L_n}
                    \le1+\delta_L/4 .
 \tag{AJ14}
\]
At initialization $A_{{\rm sub},0}\le1/4$ and
$A_{\min,0}\le A_{\rm mean,0}\le1/4$. At the common public
endpoint,
\[
 A_{{\rm sub},N}>.9,\qquad
 A_{\min,N}\ge1-\varepsilon,\qquad
 A_{\rm mean,N}\ge1-\varepsilon/r .
 \tag{AJ15}
\]
In particular the complete spatial-energy improvement exceeds $.65$,
and the minimum principal-angle improvement is at least
$3/4-\varepsilon\ge1/2$.
With
\[
 a_* =324(1-\alpha)^2\varphi(1)^2 s^4R^4/m^2,
 \tag{AJ16}
\]
the full endpoint spectrum obeys
\[
 \lambda_r(\mathsf G_N)\ge a_*,\quad
 \lambda_{r+1}(\mathsf G_N)<1522s^4\le\varepsilon a_*,\quad
 \lambda_r-\lambda_{r+1}\ge(1-\varepsilon)a_*>0.
 \tag{AJ17}
\]

Fix any public feature multiplier $\kappa>0$. At both times compare
the identical population diagnostic class
\begin{gather}
 \psi^P_{j,n}(X)=\kappa\sigma_\alpha((P_nW_{j,n})^TX+B_{j,n}),
 \qquad B_\kappa=C_\alpha/(7\kappa s),\notag\\
 \mathcal R^P_{B_\kappa}(n)=
 \inf_{\|v\|_2\le B_\kappa}
         \left\|y-\sum_jv_j\psi^P_{j,n}\right\|_2^2.
 \tag{AJ18}
\end{gather}
Projection precedes the activation, and no projected weight is
renormalized. The same definition with no coefficient constraint is
$\mathcal R^P_\infty$. Then
\begin{gather}
 \mathcal R^P_{B_\kappa}(0),\ \mathcal R^P_\infty(0)
                                 \ge1-9/256,\notag\\
 \mathcal R^P_{B_\kappa}(N),\ \mathcal R^P_\infty(N)
                          \le(1/\sqrt3+11/128)^2,\notag\\
 \mathcal R^P_{B_\kappa}(0)-\mathcal R^P_{B_\kappa}(N),\quad
 \mathcal R^P_\infty(0)-\mathcal R^P_\infty(N)
 \ \ge\ 1-9/256-(1/\sqrt3+11/128)^2>1/2.
 \tag{AJ19}
\end{gather}
These are squared population prediction-risk comparisons, not
classification-accuracy or finite-sample guarantees. Refit is a
diagnostic and is not an update of the trained heads. Its budget
scales as $1/s$ and is identical at the two checkpoints.

The actual optimization clock satisfies
\[
 mT/2\le\eta N\le m\overline T/2.
 \tag{AJ20}
\]
There is an acquired dictionary time $J\le N$ with
$\eta J\ge(m/8)\log(1024/5)$; complete $A_{\rm sub}>.9$
persists on $J\le n\le N$. The minimum-AGOP claim is at the
public endpoint $N$; no earlier hitting-time or whole-suffix
minimum-angle assertion is made.
The unchanged GD run later reaches $L\le1-2G_{\rm rel}$,
a decrease at least $G_{\rm rel}$ from every loss in the plateau.
This later original-loss bound may be very small and is distinct
from the constant diagnostic gain in AJ19.
\end{theorem}

\begin{proof}
Use the category and all-row Gaussian event of Appendix~\ref{app:cubic-maturation}
with confidence parameter $\delta_0$. Its coupon failure is at most
$\delta_0/4$ and its simultaneous row-tail failure at most $7\delta_0/96$.
The initial projected-risk bound PJ21 with $\delta_I=\delta_0/4$
costs at most another $\delta_0/4$. AJ5 gives
$q_{r,d}\le((r+4)/d)^3\le\delta_0/128$.
Call their intersection $E_{\rm base}$; its failure is at most
$\delta_0/4+7\delta_0/96+\delta_0/4=55\delta_0/96$.
The Gaussian coupon and all-row estimates, actual coupled force induction,
profile maturation and complete spatial-energy bounds use no $m<d$
assumption. Only the initial raw-span comparator is replaced, by PJ21.
The unchanged induction AC17--27 uses the public $R,T,M$ and AJ8--11,
so AJ7--11 give all its actual trajectory conclusions on $E_{\rm base}$. In particular all $4r$ original category
representatives mature at $N$, and the complete outside bound CW19
holds with trace and operator norm below $1522s^4$.

Intersect with the original weighted-coordinate gap event FA6--8.
Its additional failure is below $\zeta/2$, with no conditioning or
independence substitution. FA4 is exactly the last cutoff in AJ9 and
the last step cutoff in AJ8. Thus the full forced-anchor proof gives
FA24--28 for every original row through $N$.

Write the full signed second-Hermite gradient frame at $N$ as
\[
\begin{gathered}
\mathsf D_{ki}=\mathbb E[(u_k^T\nabla f_N)h_2(u_i^TX)]
       =\sum_j t_j v_{jk}v_{ji}^2,\\
 v_j=U^Tw_j/\rho_j,\quad
 t_j=-\frac{(1-\alpha)s^2}{\sqrt2m}
                        q_j\rho_j z_j\varphi(z_j).
\end{gathered}
 \tag{AJ21}
\]
To verify this identity directly, put $Z=w_j^TX/\rho_j$,
$z_j=b_j/\rho_j$, and $h_2(t)=(t^2-1)/\sqrt2$.
Gaussian regression gives
\[
\begin{gathered}
 \mathbb E[h_2(u_i^TX)\mid Z]=v_{ji}^2h_2(Z),\\
 \mathbb E[\sigma'_\alpha(\rho_jZ+b_j)h_2(Z)]
 =\frac{1-\alpha}{\sqrt2}\int_{-z_j}^{\infty}(t^2-1)\varphi(t)\,dt
 =-\frac{1-\alpha}{\sqrt2}z_j\varphi(z_j).
\end{gathered}
\]
The first equality follows by writing $u_i^TX=v_{ji}Z+G_\perp$
with independent centered Gaussian $G_\perp$ of variance $1-v_{ji}^2$;
the last follows from $(t\varphi(t))'=(1-t^2)\varphi(t)$.
Multiplying by the actual gradient row $(s^2/m)q_jw_j$ and summing
proves AJ21, equivalently the full frame
\eqref{eq:cubic-signed-frame}. All current trained heads and all
self/cross effects remain present.
FA28 decomposes this matrix as a nonnegative diagonal matrix plus
an error whose operator norm is at most
$(1-\alpha)\varphi(1)s^2\Gamma$. A favorable row's diagonal axis is
its original oriented weighted extremum. For each category row this
is its assigned axis: its original selected coordinate has weighted
magnitude at least $a_i/\sqrt d\ge a_0/\sqrt d$, whereas each
competitor has magnitude at most
$\vartheta/(4\sqrt{r-1}\sqrt d)<a_0/\sqrt d$.
The rank-one case has no competitors. Coordinate order is retained
by FA12, and the category's favorable phase and orientation persist
by AC. Consequently none of the chosen diagonal mass is lost in
the complete FA decomposition.

At $N$, each of the four chosen rows per axis satisfies
$|q_j|>6R$, $|q_j|\le\sqrt2\rho_j$, and
$1/2<|z_j|<1$. These are the acquired AC conclusions used in the calibrated category profiles.
Each contributes at least
$9(1-\alpha)\varphi(1)s^2R^2/m$ to its assigned diagonal. Therefore
\begin{gather}
 \lambda_{\min}(\mathsf D_+)
       \ge36(1-\alpha)\varphi(1)s^2R^2/m,\notag\\
 \sigma_{\min}(\mathsf D)
 \ge(1-\alpha)\varphi(1)s^2(36R^2/m-\Gamma)
 \ge18(1-\alpha)\varphi(1)s^2R^2/m.
 \tag{AJ22}
\end{gather}
The last inequality is the third radius condition in AJ7. Extra
favorable rows only add nonnegative diagonal mass; their errors and
every nonfavorable row have already been charged by FA28.
Thus this is a bound for the complete student, not its chosen part.

Apply Lemma~\ref{lem:frame-agop} with $H=\nabla f_N$,
$\xi_i=h_2(u_i^TX)$, $C=I_r$, $c=1$, and
$\gamma=18(1-\alpha)\varphi(1)s^2R^2/m$. AJ22 gives the complete
frame floor $\mathsf D\mathsf D^T\succeq\gamma^2I_r$, with $\gamma^2=a_*$.
CW19 bounds the complete outside trace strictly below $1522s^4$,
and the last radius condition in AJ7 gives $1522s^4\le\varepsilon a_*$.
The lemma therefore proves AJ17 and
\[
 1-A_{\min,N}
   =\|P_{U^\perp}P_NP_{U^\perp}\|\le\varepsilon,\qquad
 1-A_{\rm mean,N}
   =r^{-1}\operatorname{tr}(P_{U^\perp}P_N)
                                      \le\varepsilon/r.
 \tag{AJ23}
\]
The displayed equalities are the equal-rank principal-angle identities.

At initialization IA gives an almost surely unique positive top-$r$
projector with Haar orientation, hence
$\mathbb E[A_{\rm mean,0}]=r/d$. Markov and AJ5 give
\[
 \mathbb P(A_{\rm mean,0}>1/4)\le4r/d\le\delta/4.
 \tag{AJ24}
\]
Intersect $E_{\rm base}$, the FA gap event, and this initial-headroom
event. Lemma~\ref{lem:common-event}, with $\delta_0=\zeta=\delta/2$,
bounds the failure by
$55\delta_0/96+\zeta/2+\delta/4=151\delta/192<\delta$.
All subsequent claims use this same intersection and the same public
endpoint $N$. AC supplies the initial and final complete $A_{\rm sub}$
claims, the full loss band and the clocks.

For prediction, AJ15 and the choice of $\varepsilon$ imply
PJ22 with $\mu=1-\varepsilon$. PJ9 constructs a legal readout in
exactly AJ18 using coefficients proportional to $(s\rho_j)^{-1}$;
these cancel the learned radii when spatial projection precedes
activation. PJ14--17 bound the projection distortion by
$D_g\sqrt{\varepsilon(1+(1-\varepsilon)^{-1})}\le1/16$.
The normalized-profile and full-teacher costs are $1/64$ and $1/128$.
PJ21 and AJ5 supply the initial projected oracle lower bound on the
same event. This proves AJ19, with no changed comparator or actual
head reset. Finally, verify the common entry conditions for original-loss release.
The complete energy satisfies $V_0<5m$, whereas one acquired row
already gives $V_N>36R^2\ge9216m>2V_0$.
The endpoint force bound is included in AC's actual induction, so
NC20--23 apply to the full target and give
\[
 Q_{\rm raw}(N):=\frac{\mathbb E[yf_N]}{s^2V_N}
       \ge\frac1{16m\overline T}=4a_{\rm rel}.
\]
All original rows remain balanced, $0<V_N\le M^2$, and
$\mathcal H=\|y\|_{H^1}\le2+\nu/2<3$.
Use Corollary~\ref{ag:averaged-release} with checkpoint $N$,
complete-prefix ceiling $U=M^2$, and
$(a,R_c,G_c)=(a_{\rm rel},R_{\rm rel},G_{\rm rel})$.
Its scalar conditions follow directly from AJ8 and AJ11:
\[
 \begin{gathered}
 a_{\rm rel}\le\frac1{4m},\quad
 h\le\frac1{512}<\frac1{32(1+(1-\alpha)\mathcal H)},\\
 R_{\rm rel}^2=a_{\rm rel}m^2,\quad
 G_{\rm rel}=\frac{a_{\rm rel}R_{\rm rel}^2}{2},\quad
 s^2M^2\le\min\{R_{\rm rel}^2/4,mG_{\rm rel}\}.
 \end{gathered}
\]
Thus the unchanged full-MSE update at rate $\eta=mh/2$ reaches a
later checkpoint with $L\le1-2G_{\rm rel}$, a decrease of at least
$G_{\rm rel}$ from every prefix checkpoint. This uses the complete
trained bank and the full teacher, with no additional random event.
For the remaining loss-band bounds, AJ11 gives
$\|f_n\|_2\le s^2M^2/(2m)\le\delta_L/100$ throughout the prefix.
Expanding $L_n=1-2\langle y,f_n\rangle+\|f_n\|_2^2$ gives the
band and ratio in AJ14; the corollary supplies its baseline
$L_n\ge1-G_{\rm rel}$.

\end{proof}

\section{Gaussian kernels and compatible simultaneous updates}\label{app:kernels}

\subsection{Conventions and the normalized additive teacher}

Let $X\sim\gamma_d=N(0,I_d)$, $0\le\alpha<1$,
$\sigma(t)=\alpha t+(1-\alpha)t_+$, and $\widetilde X=(X,1)$.
Here $\|g\|_2$ is its Gaussian $L^2$ norm; $\varphi$ and $\Phi$
are the standard normal density and CDF. Derivatives of the teacher
are with respect to $X$, while $\nabla C_y$ and $D^2C_y$ differentiate
the augmented parameter $\theta$.
Write $\theta=(w,b)$ and
\[
 C_y(\theta)=\mathbb E[y(X)\sigma(w^TX+b)],\quad
 A=\|y\|_2,\quad S=\mathbb E[\|\nabla y\|^2],\quad
 H_y=(A^2+S)^{1/2}.
 \tag{HK1}
\]
The weak Gaussian Sobolev space here is the closure of polynomials in
the norm $H_y$ (equivalently, the usual weak-derivative space). All norms
of parameter derivatives below are Euclidean operator norms.

For orthonormal $u_1,\ldots,u_r$ and real $c_i$, let
$y=\sum_i c_i g_i(u_i^TX)$ with $g_i\in H^1(\gamma_1)$. Use
$h_k=\operatorname{He}_k/\sqrt{k!}$ and
$g_i=\sum_{k\ge0}\gamma_{ik}h_k$. Exactly,
\begin{align}
 \mu&=\mathbb E[y]=\sum_i c_i\gamma_{i0},&
 \ell&=\mathbb E[yX]=\sum_i c_i\gamma_{i1}u_i,\nonumber\\
 A^2&=\mu^2+\sum_i c_i^2\sum_{k\ge1}\gamma_{ik}^2,&
 S&=\sum_i c_i^2\sum_{k\ge1}k\gamma_{ik}^2.
 \tag{HK2}
\end{align}
Full population normalization means $A=1$ in these formulas. In
particular, a common link with $c_i=r^{-1/2}$ has unnormalized energy
$r\gamma_0^2+\sum_{k\ge1}\gamma_k^2$, not $\|g\|_2^2$.

An optional sharper directional size is
\[
 \Lambda_y^2=\|\mathbb E[\nabla y\nabla y^T]\|_{\rm op}\le S.
 \tag{HK3}
\]
For the additive teacher its matrix on $U=\operatorname{span}(u_i)$ is
$\operatorname{diag}\{c_i^2\operatorname{Var}(g_i')\}+\ell\ell^T$;
thus $\Lambda_y^2\le\max_i c_i^2\operatorname{Var}(g_i')+\|\ell\|^2$.
These expressions contain no extraneous factor of $r$ or $d$.

\subsection{A Gaussian Hilbert-space trace bound}

\begin{lemma}[One-dimensional weighted trace]
\label{lem:hkr-trace}
Let $f\in H^1(\gamma_1;\mathcal H)$ for a real Hilbert space
$\mathcal H$, with $a=\|f\|_{L^2(\gamma;\mathcal H)}$ and
$b=\|f'\|_{L^2(\gamma;\mathcal H)}$. Its continuous local representative
satisfies, at every real $t$,
\[
 \varphi(t)\|f(t)\|_{\mathcal H}^2
 \le a\sqrt{b^2+a^2/2}
 \le (a^2+b^2)/\sqrt2.
 \tag{HK4}
\]
Consequently, for a unit $v\in\mathbb R^d$, the canonical $L^2$ trace
of any $y\in H^1(\gamma_d)$ obeys
\[
 \|y\vert_{v^TX=t}\|_{L^2(\gamma_{v^\perp})}
 \le \varphi(t)^{-1/2} A^{1/2}
       (\|\partial_vy\|_2^2+A^2/2)^{1/4}
 \le2^{-1/4}\varphi(t)^{-1/2}H_y.
 \tag{HK5}
\]
\end{lemma}
\begin{proof}
First take smooth $f$ with sufficient decay and set
$h(t)=\sqrt{\varphi(t)}f(t)$. Gaussian integration by parts gives
\[
 \|h\|_{L^2(dt)}=a,\qquad
 \|h'\|_{L^2(dt)}^2
 =b^2+a^2/2-\|tf\|_{L^2(\gamma)}^2/4
 \le b^2+a^2/2.
\]
There is no unproved multiplication hypothesis: integration by parts
also gives $\|tf\|_2^2=a^2+2\langle tf,f'\rangle$, hence
$\|tf\|_2\le b+\sqrt{b^2+a^2}$. Density extends this multiplier estimate
and the displayed identity to $H^1(\gamma;\mathcal H)$.
For an $H^1(\mathbb R;\mathcal H)$ function, integrate the derivative
of $\|h\|^2$ on the two sides of $t$ and average the results. This gives
$\|h(t)\|^2\le\int\|h(s)\|\|h'(s)\|ds\le\|h\|_2\|h'\|_2$.
The final inequality in (HK4) follows from
$a^2(b^2+a^2/2)\le(a^2+b^2)^2/2$.
Rotate Gaussian space so its first coordinate is $v^TX$ and take
$\mathcal H=L^2(\gamma_{v^\perp})$ to obtain (HK5).
\end{proof}

\subsection{Uniform full-parameter curvature, including pure-bias rows}

\begin{lemma}[Dimension-free $H^1$ kink regularity]
\label{thm:hkr-curvature}
For every $y\in H^1(\gamma_d)$, $C_y$ is positively homogeneous of
degree one and belongs to $C^2(\mathbb R^{d+1}\setminus\{0\})$.
For all $\theta\ne0$,
\begin{gather}
 |C_y(\theta)|\le A\|\theta\|,\qquad
 \|\nabla C_y(\theta)\|\le A,\qquad
 \theta^T\nabla C_y(\theta)=C_y(\theta),\qquad
 D^2C_y(\theta)\theta=0,\tag{HK6}\\
 \|\theta\|\,\|D^2C_y(\theta)\|
 \le B_y:=4(1-\alpha)H_y.\tag{HK7}
\end{gather}
One may instead use the sharper constant
\[
 B_y^{\rm sharp}=K_0(1-\alpha)A^{1/2}
                  (\Lambda_y^2+A^2/2)^{1/4},\qquad
 K_0=\sqrt3(2\pi)^{-1/4}6^{3/2}e^{-5/4}<4.61.
 \tag{HK8}
\]
When $\rho=\|w\|>0$, $v=w/\rho$, and $z=b/\rho$, the exact Hessian is
\[
 D^2C_y(\theta)
  =(1-\alpha)\frac{\varphi(z)}\rho
    \mathbb E[y(X)\widetilde X\widetilde X^T\mid v^TX=-z],
 \tag{HK9}
\]
where the conditional expectation is the canonical trace in (HK5).
At $w=0,b\ne0$ its values are exactly
\[
 C_y(0,b)=\mu\sigma(b),\qquad
 \nabla C_y(0,b)=\sigma'(b)(\ell,\mu),\qquad
 D^2C_y(0,b)=0.
 \tag{HK10}
\]
In particular, no uniform lower bound on spatial radii, no bound on
$|b|/\|w\|$, and no exclusion of aligned teacher directions is needed.
\end{lemma}
\begin{proof}
For polynomial $y$, differentiate the Gaussian half-space integral to
obtain (HK9). For any unit augmented parameter vector $a$, conditional
on $v^TX=-z$, the scalar $a^T\widetilde X$ is Gaussian with mean $m$
and variance $s^2$ satisfying $m^2+s^2\le1+z^2$. Consequently
\[
 \mathbb E[(a^T\widetilde X)^4\mid v^TX=-z]
 =m^4+6m^2s^2+3s^4\le3(1+z^2)^2.
\]
Testing the symmetric matrix (HK9) on $a$, using (HK5), and noting
$\|\theta\|/\rho=\sqrt{1+z^2}$, gives
\[
 \|\theta\|\,\|D^2C_y(\theta)\|
 \le (1-\alpha)\sqrt3(1+z^2)^{3/2}\sqrt{\varphi(z)}
       A^{1/2}(\|\partial_vy\|_2^2+A^2/2)^{1/4}.
 \tag{HK11}
\]
The maximum of $(1+z^2)^{3/2}e^{-z^2/4}$ occurs at $z^2=5$.
This proves (HK8); (HK4) gives the simpler constant
$2^{-1/4}K_0<3.88<4$ in (HK7).

For polynomial $y$, the same formula and its derivatives have a
polynomial factor in $z$ times $\varphi(z)/\rho$. Thus they extend at
$w=0,b\ne0$ with zero Hessian and the affine values in (HK10).
Now approximate arbitrary $y$ by its multivariate Hermite polynomials
$y_N\to y$ in $H^1$. Cauchy--Schwarz and $|\sigma(t)|\le|t|$,
$|\sigma'|\le1$ give, for any such difference $e$,
\[
 |C_e(\theta)|\le\|e\|_2\|\theta\|,\qquad
 \|\nabla C_e(\theta)\|\le\|e\|_2.
 \tag{HK12}
\]
Indeed $\mathbb E[(a^T\widetilde X)^2]=1$ for unit $a$. Together with
(HK7), applied to polynomial differences, these estimates make
$C_{y_N}$, their gradients, and their Hessians uniformly Cauchy on
each compact subset of $\theta\ne0$. Their limit is $C_y$ and is
$C^2$ there, with all the asserted bounds and values. This also proves
continuity through changes of hyperplane direction and through the
pure-bias rows; pointwise weak derivatives of the original link never
need to be evaluated. At each fixed hyperplane, (HK5) makes the traces
converge in $L^2$, proving (HK9) for $y$. Homogeneity gives the two
Euler identities in (HK6).
\end{proof}

The linear dependence on a general unnormalized $H_y$ is natural by
scaling. With $A=1$, (HK8) improves the large-derivative dependence to
$O(\Lambda_y^{1/2})$. This power cannot be decreased uniformly:
take a nonnegative smooth bump $g_\epsilon(t)$ supported on
$[-\epsilon,\epsilon]$, of value $\asymp\epsilon^{-1/2}$ at zero,
and normalize its Gaussian $L^2$ norm to one. Then
$\|g_\epsilon'\|_2\asymp\epsilon^{-1}$, whereas at $\theta=(u,0)$
the bias-bias Hessian is
$(1-\alpha)\varphi(0)g_\epsilon(0)\asymp\epsilon^{-1/2}$.
In particular no bound depending only on $\|y\|_2=1$ is possible.

\subsection{Full uncentered Hermite kernel and meaning of differentiation}

For $Z\sim\gamma_1$ define $\eta_k(z)=\mathbb E[h_k(Z)\sigma(Z+z)]$.
Gaussian integration by parts, in the distributional sense for the
kink, gives
\begin{align}
 \eta_0(z)&=\alpha z+(1-\alpha)\{\varphi(z)+z\Phi(z)\},\nonumber\\
 \eta_1(z)&=\alpha+(1-\alpha)\Phi(z),\nonumber\\
 \eta_k(z)&=(1-\alpha)\varphi(z)
       \frac{h_{k-2}(-z)}{\sqrt{k(k-1)}}\quad(k\ge2),\tag{HK13}\\
 \eta_0'&=\eta_1,\qquad
 \eta_k'=(1-\alpha)\varphi(z)h_{k-1}(-z)/\sqrt k\quad(k\ge1),
 \nonumber\\
 \eta_k''&=(1-\alpha)\varphi(z)h_k(-z)\quad(k\ge0).
 \nonumber
\end{align}
For $a_i=u_i^Tv$ the exact full correlation is
\[
 C_y(w,b)=\rho F(v,z),\quad
 F(v,z)=\mu\eta_0(z)+\eta_1(z)\ell^Tv
       +\sum_i\sum_{k\ge2}c_i\gamma_{ik}a_i^k\eta_k(z).
 \tag{HK14}
\]
The series is absolutely convergent at every $(v,z)$, including
$a_i=\pm1$, by Cauchy--Schwarz and Parseval for
$\sigma(Z+z)$. It is the original correlation, not a truncated or
centered replacement. Finite Hermite truncations of (HK14), composed
with the actual $\theta$ coordinates, converge together with all
parameter derivatives through order two uniformly on compact subsets
of $\theta\ne0$, by Theorem~\ref{thm:hkr-curvature}. This is a rigorous
termwise-differentiation meaning even at aligned directions.

For first derivatives one may also use the ambient expression
\[
 J(v,z)=\mathbb E[\nabla y(X)\sigma'(v^TX+z)]
   =\sum_i c_i u_i\sum_{k\ge1}k\gamma_{ik}a_i^{k-1}\eta_k(z),
 \qquad D(v,z)=v^TJ(v,z).
 \tag{HK15}
\]
Here $\|J\|\le\sqrt S$ and $|D|\le\|\partial_v y\|_2$.
The $i$th scalar sum is the covariance derivative
of $\mathbb E[g_i(T)\sigma(Z+z)]$, where
$\operatorname{Corr}(T,Z)=a_i$. In the open interval this derivative
is $\mathbb E[g_i'(T)\sigma'(Z+z)]$ by Gaussian integration by parts;
its magnitude is at most $\|g_i'\|_2$. Approximation in $H^1$ extends
this formula and uniform convergence to $a_i=\pm1$.

Weak Gaussian integration by parts also gives the exact formulas
\[
 H(v,z)=(1-\alpha)\varphi(z)\mathbb E[y\mid v^TX=-z],\quad
 \nabla_w C_y=J+Hv,\quad
 \partial_b C_y=F_z=\mathbb E[y\sigma'(v^TX+z)].
 \tag{HK16}
\]
They can alternatively be proved first for polynomials and then by
(HK5), (HK12), and $H^1$ convergence. In spherical coordinates,
\[
 \nabla_w C_y=v(F-zF_z)+(I-vv^T)J,
 \qquad F-zF_z=D+H.
 \tag{HK17}
\]
No claim is made that the artificial off-sphere partial derivatives
$\partial_{a_i}^2F$ or $\partial_{a_i}\partial_zF$ exist at
$|a_i|=1$ under $H^1$ alone. Their apparent singularities can cancel
in the actual $\theta$ derivatives. For example, for a compactly
supported link agreeing near zero with $|t|^\beta$, $1/2<\beta<1$,
the link is $H^1$ but $g'$ has no finite trace at zero; the aligned
mixed covariance/shift derivative asks for that nonexistent trace.
The uniform actual-parameter conclusion above does not make that
extra assertion.

For completeness an individual additive summand has an explicit trace
without second derivatives of $g_i$. Put $t=-z$, $a=u_i^Tv$,
$s=\sqrt{1-a^2}$, $e=(u_i-av)/s$ when $s>0$, and
$\xi=(tv,1)$, $\bar e=(e,0)$, $\bar P=\operatorname{diag}(I-vv^T,0)$.
With an independent standard normal $G$, define
\[
 M_0=\mathbb E[g_i(at+sG)],\quad
 M_1=\mathbb E[Gg_i(at+sG)],\quad
 M_2=\mathbb E[(G^2-1)g_i(at+sG)].
\]
Then its conditional matrix in (HK9) equals
\[
 M_0(\bar P+\xi\xi^T)
 +M_1(\bar e\xi^T+\xi\bar e^T)+M_2\bar e\bar e^T.
 \tag{HK18}
\]
For $s>0$, weak integration by parts also gives
$M_1=s\mathbb E[g_i'(at+sG)]$ and
$M_2=s\mathbb E[Gg_i'(at+sG)]$.
At $s=0$ the matrix is
$g_i(at)(\bar P+\xi\xi^T)$ using the continuous one-dimensional
$H^1$ representative. Summing (HK18) with coefficients $c_i$ gives
the full additive trace, including means and all chaos orders.

\subsection{Interfaces for coupled Euclidean population GD}

Use actual raw coordinates $x=(q_j,\theta_j)_{j=1}^m$ and the fixed
normalization
\[
 f_x=\kappa\sum_{j=1}^m q_j\sigma(\theta_j^T\widetilde X),\quad
 \ell(x)=\tfrac12\mathbb E[(y-f_x)^2],\quad
 x^+=x-\eta\nabla\ell(x).
 \tag{HK19}
\]
The full MSE is $L=2\ell$ and its corresponding raw rate is $\eta/2$.
No coordinates or row interactions are removed. Let
$R=\|x\|$, $F_x=\mathbb E[yf_x]$, $V_x=\|f_x\|_2^2$, $Q=F_x/R^2$.
\begin{lemma}[Raw coupled-bank regularity interface]
\label{lem:hkr-coupled}
Suppose $A=1$ and $R>0$. Delete any joint rows
$(q_j,\theta_j)=(0,0)$ and differentiate in the remaining raw
coordinates, where every $\theta_j\ne0$. Then
\begin{gather}
 |F_x|\le\kappa R^2/2,\quad \|\nabla F_x\|\le\kappa R,
 \quad\|f_x\|_2\le\kappa R^2/2,\quad
 V_x\le\kappa^2R^4/4,\quad\|\nabla V_x\|\le\kappa^2R^3,
 \nonumber\\ \|\nabla Q\|\le\kappa/R.\tag{HK20}
\end{gather}
At points with $|q_j|\le\beta\|\theta_j\|$ for every retained row,
where $\beta\ge0$,
\[
 \|D^2F_x\|\le\kappa(1+\beta B_y),\qquad
 \|D^2Q\|\le\kappa(8+\beta B_y)/R^2.
 \tag{HK21}
\]
These Hessians act on the full retained raw coordinates, including
head, spatial and bias variables.
\end{lemma}
\begin{proof}
Theorem~\ref{thm:hkr-curvature},
$\sum_j|q_j|\|\theta_j\|\le R^2/2$, and
$\|Df_x\|_{\rm op}\le\kappa R$ give the first five bounds.
Euler's identity $x^T\nabla F_x=2F_x$ gives
$R^4\|\nabla Q\|^2=\|\nabla F_x\|^2-4F_x^2/R^2$.
Each Hessian block of $F_x$ is
$\kappa\left(\begin{smallmatrix}0&\nabla C_y^T\\
 \nabla C_y&q_jD^2C_y\end{smallmatrix}\right)$,
of norm at most $\kappa(1+\beta B_y)$.
In differentiating $F_xR^{-2}$, the mixed product term costs at most
$4\kappa/R^2$ and $F_xD^2(R^{-2})$ at most $3\kappa/R^2$.
This proves (HK21).
\end{proof}
In particular, the $\beta=2$ interpolation segments have quotient
constant $8+2B_y$. Lemma~\ref{ag:release} applies this interface;
its entry conditions and same-step continuation proof appear in
Appendix~\ref{app:general-release}.

The whole student Gram also has sufficient regularity. If
$G(\theta,\psi)=\mathbb E[\sigma(\theta^T\widetilde X)
\sigma(\psi^T\widetilde X)]$, then away from zero rows
\[
 \|D_\theta^2G\|\le4\sqrt2(1-\alpha)\|\psi\|/\|\theta\|,
 \qquad \|D_\theta D_\psi G\|\le1.
 \tag{HK22}
\]
The first bound uses (HK7) with the Sobolev teacher
$\sigma(\psi^T\widetilde X)$, whose $H^1$ norm is at most
$\sqrt2\|\psi\|$. The second uses the exact mixed matrix
$\mathbb E[\sigma'(\theta^T\widetilde X)
\sigma'(\psi^T\widetilde X)\widetilde X\widetilde X^T]$ and
Cauchy--Schwarz on unit directions. These derivatives are jointly
continuous, including coincident nonzero hyperplanes: for the pure
second derivative use (HK7) and continuity of the feature in $H^1$;
for the mixed derivative use almost-everywhere convergence and
Gaussian domination.

For example, on the same balanced region,
\[
 \|D^2V_x\|\le
   \{3+4\sqrt2\beta(1-\alpha)\}\kappa^2R^2.
 \tag{HK23}
\]
To verify this, write the second variation as
$2\|Df_x[h]\|_2^2+2\langle f_x,D^2f_x[h,h]\rangle$ using the Gram
derivatives. The first term is at most $2\kappa^2R^2\|h\|^2$.
The head/hidden cross term is at most
$2\kappa\|f_x\|_2\|h\|^2\le\kappa^2R^2\|h\|^2$.
The remaining hidden Hessians are at most
$2\kappa\beta\{4(1-\alpha)\|f_x\|_{H^1}\}\|h\|^2$;
here $\|f_x\|_{H^1}\le\kappa R^2/\sqrt2$.
This proves (HK23) without declaring the feature itself twice
differentiable as an $L^2$-valued map.

Rows with $(q_j,\theta_j)=(0,0)$ stay zero under the conventional
actual update and may be omitted. No $C^2$ claim at $\theta=0$ with
nonzero head is made. The usual discrete row balance
$\|\theta_j\|^2-q_j^2\ge0$, when transported by the actual algorithm,
excludes that singular case; Gaussian initialization gives nonzero
hidden rows almost surely. Neither regularity nor a positive reached
$Q$ implies acquisition of all teacher directions.

\section{Original Gaussian initialization and prediction baselines}\label{app:initialization}

\subsection{Population model and genuine initial eigendirections}
Function norms and gate expectations use the displayed Gaussian input;
expectations of initial alignment statistics use the parameter draw.
Write $P_U=UU^T$ for the teacher-space projector when $U$ is an
orthonormal axis matrix, and $h_k=\operatorname{He}_k/\sqrt{k!}$ for
normalized probabilists' Hermite polynomials.
Let $d\ge2$, $m\ge1$, $0\le\alpha<1$, $J=1-\alpha$, and
\[
 f(X)=\frac1m\sum_{j=1}^m a_j\sigma_\alpha(w_j^TX+b_j),
 \qquad \sigma_\alpha(t)=\alpha t+Jt_+,
 \qquad X\sim N(0,I_d).
 \tag{IA1}
\]
Use spatial rows $W\in\mathbb R^{m\times d}$ and define
\[
 \begin{gathered}
 d_j(X)=\sigma'_\alpha(w_j^TX+b_j),\quad
 K_{jl}=\mathbb E[d_jd_l],\quad D_a=\operatorname{diag}(a_j),\\
 G_s=\mathbb E[\nabla_Xf\nabla_Xf^T]
       =\frac1{m^2}W^TD_aKD_aW.
 \end{gathered}
 \tag{IA2}
\]
The initialization has IID spatial rows $w_j\sim N(0,s_w^2I_d)$
with one common $s_w>0$, independent Gaussian biases, and independent
nondegenerate Gaussian heads. The positive spatial, bias and head
block scales may differ; the spatial scale is common to all rows. All heads, biases and rows are those of the original draw;
no teacher-based screening or initialization rejection is used.

\begin{lemma}[Initial population spectral rank and simplicity]
With probability one, $G_s(0)$ has rank $p=\min(m,d)$ and its
$p$ positive eigenvalues are pairwise distinct. Consequently, for
any $1\le q\le p$ with $q<d$, its leading rank-$q$ projector
$\widehat P_{q,0}$ is uniquely defined and has a positive boundary
gap. This is qualitative almost-sure nondegeneracy, not a lower
bound on the size of that gap.
\end{lemma}
\begin{proof}
Almost surely $W$ has rank $p$, each head is nonzero, each spatial
row is nonzero, and its affine kink hyperplanes are pairwise distinct.
For such a hidden bank the gate Gram $K$ is positive definite. Indeed,
if $v^TKv=0$, the piecewise constant function $\sum_jv_jd_j(X)$
is zero almost everywhere. For each hyperplane choose a point not on
any other hyperplane. On a sufficiently small ball about that point,
the two open sides differ in just the corresponding gate. Both have
positive Gaussian measure, so their constants must both be zero.
Their difference is $Jv_j$, hence $v_j=0$ for every $j$.
Equation IA2 then gives rank $G_s=p$.

Condition on this entire hidden/bias bank, and choose an orthonormal
basis of its $p$-dimensional spatial span. The restriction of $G_s$
to that span is a symmetric $p\times p$ matrix polynomial in the
head coordinates. The discriminant of its characteristic polynomial
is therefore a polynomial in those coordinates. It is not identically
zero, as the following explicit existence argument shows.

Choose $p$ linearly independent spatial rows, indexed by $I$, and
temporarily set the other heads to zero only to evaluate that
polynomial. In the chosen span basis let $S$ be the invertible
$p\times p$ selected spatial matrix. The restricted AGOP is
$m^{-2}S^TD K_{II}DS$. There are fixed $0<c_-\le c_+<\infty$
such that its ordered eigenvalues satisfy
\[
 c_- a_{(i)}^2\le\lambda_i(m^{-2}S^TD K_{II}DS)
                  \le c_+a_{(i)}^2\quad(1\le i\le p),
 \tag{IA3}
\]
where the selected head magnitudes are sorted decreasingly.
For example, use
\[
\begin{aligned}
c_-&=m^{-2}\lambda_{\min}(K_{II})s_{\min}(S)^2,\\
c_+&=m^{-2}\lambda_{\max}(K_{II})s_{\max}(S)^2.
\end{aligned}
\]
These inequalities follow from matrix order and the singular-value
bounds for multiplication by the fixed invertible $S$.
Choose successive squared head magnitudes with ratio greater than
$2c_+/c_-$. The intervals in IA3 are disjoint, so all eigenvalues
are positive and distinct at that polynomial evaluation.

A nonzero real polynomial vanishes on a Lebesgue-null set. The
conditional Gaussian head law has a density, so repeated positive
eigenvalues have probability zero. The temporary zero-head choice
was a polynomial witness, not an initialization used by the theorem.
Integrating the conditional statement proves the result. For $q<p$
the boundary eigenvalues are distinct and positive; for $q=p<d$
the boundary is the positive last eigenvalue versus zero.
\end{proof}

\subsection{Haar law, mean/min statistics and initial headroom}
For an orthogonal $R$, the transformation $W\mapsto WR^T$ leaves
the original spatial Gaussian law unchanged. The gate Gram is
unchanged, and IA2 transforms as $G_s\mapsto RG_sR^T$.
Thus the initial AGOP law is rotationally invariant. Its ordered
eigenvalues are unchanged by this transformation, while its unique
leading projector transforms equivariantly.

It follows that $\widehat P_{q,0}$ is a Haar rank-$q$ projector,
including conditionally on its eigenvalues. To justify uniqueness
of this invariant law directly, average any bounded function of a
rank-$q$ projector over an independent uniform orthogonal rotation.
That average is the same for every starting rank-$q$ projector.
Rotational invariance makes the original expectation equal to this
average. The identical argument with an arbitrary bounded function
of the eigenvalues proves the conditional statement.

Fix the teacher's orthonormal basis $U\in\mathbb R^{d\times r}$,
with $1\le r\le m$ and $r<d$, and use $q=r$ for the
direction metric. Set
\[
 B_0=U^T\widehat P_{r,0}U,\qquad
 A_{\min,0}=\lambda_{\min}(B_0),\qquad
 A_{\rm mean,0}=\operatorname{tr}(B_0)/r.
 \tag{IA4}
\]
Then the exact initial mean and variance are
\[
 \mathbb E[A_{\rm mean,0}]=r/d,\qquad
 \operatorname{Var}(A_{\rm mean,0})
       =\frac{2(d-r)^2}{d^2(d-1)(d+2)},\qquad
 0\le A_{\min,0}\le A_{\rm mean,0}.
 \tag{IA5}
\]
For completeness, in coordinates adapted to $U$, each diagonal
entry of the Haar projector has beta law with parameters
$r/2,(d-r)/2$, hence variance
$2r(d-r)/[d^2(d+2)]$. Since its diagonal sum is exactly $r$,
exchangeability makes the covariance of two different diagonal
entries equal to minus that variance divided by $d-1$.
Summing over the first $r$ coordinates and dividing by $r^2$
proves IA5.

For any $0<\delta_0<1$ and $a>r/d$, Chebyshev therefore gives
\[
 \Pr\{A_{\rm mean,0}>a\}
 \le\frac{2(d-r)^2}
 {d^2(d-1)(d+2)(a-r/d)^2}.
 \tag{IA6}
\]
In particular, the explicit size regime
\[
 m\ge r,\qquad
 d\ge\max\{8r,\sqrt{128/\delta_0}\}
 \tag{IA7}
\]
ensures both $A_{\rm mean,0}\le1/4$ and
$A_{\min,0}\le1/4$ with probability at least $1-\delta_0$.
This gives initial geometric headroom for a half-unit improvement;
it does not prove that such an improvement is attained.
No upper restriction on width occurs in IA5--7.

The complete weight-energy statistic remains distinct:
\[
 A_{\rm sub}(W_0)=\frac{\|W_0P_U\|_F^2}{\|W_0\|_F^2}
 \sim\operatorname{Beta}(mr/2,m(d-r)/2).
 \tag{IA8}
\]
This follows by adding the independent squared spatial Gaussian
coordinates inside and outside $U$. Its mean is also $r/d$, but
its distribution and later dynamics are not those of IA4.
Any desired simultaneous initial score event must be intersected
explicitly; no independence between IA4 and IA8 is asserted.

\subsection{Hermite correlation and bounded readouts}
Let $X\sim N(0,I_d)$, $U^TU=I_r$, and $1\le r<d$.
Suppose $F=F(U^TX)$ has unit $L^2$ norm and is orthogonal to all
Gaussian polynomials of total degree below an integer $k_0\ge1$.
Write $y=F+e$, $\|y\|_2=1$, $\|e\|_2\le\zeta$.
Only $L^2$ regularity is required.

\begin{lemma}[Hermite correlation]\label{lem:initial-hermite}
For any unit $n$, put $t=\|U^Tn\|$. Every square-integrable function
$\psi=\psi(n^TX)$ satisfies
$|\langle y,\psi\rangle|\le\|\psi\|_2(t^{k_0}+\zeta)$.
In particular, for $0\le\alpha<1$ and $w\ne0$, set $n=w/\|w\|$ and
\begin{gather}
 \phi_{w,b}=\frac{\sigma_\alpha(w^TX+b)}{\sqrt{\|w\|^2+b^2}},\notag\\
 \|\phi_{w,b}\|_2\le1,\qquad
 |\langle y,\phi_{w,b}\rangle|\le t^{k_0}+\zeta.\tag{BW1}
\end{gather}
A pure-bias feature satisfies BW1 with $t=0$; a zero augmented row
is assigned the zero feature.
\end{lemma}
\begin{proof}
For $t>0$, put $Z=(UU^Tn/t)^TX$ and
$H(Z)=\mathbb E[F\mid Z]=\sum_{\ell\ge k_0}a_\ell h_\ell(Z)$.
The degree assumption removes lower coefficients and conditional
contraction gives $\sum a_\ell^2\le1$.
Since $n^TX=tZ+\sqrt{1-t^2}G$ with $G$ independent of $U^TX$,
Gaussian regression of the Hermite generating function gives
\[
 \|\mathbb E[F\mid n^TX]\|_2^2
 =\sum_{\ell\ge k_0}a_\ell^2t^{2\ell}\le t^{2k_0}.\tag{BW2}
\]
Indeed $\mathbb E[h_\ell(Z)\mid n^TX]=t^\ell h_\ell(n^TX)$;
finite sums extend by $L^2$ contraction. At $t=0$, independence
and $\mathbb E[F]=0$ give zero instead. Cauchy--Schwarz proves the
claim, including the residual. Finally
$|\sigma_\alpha(z)|\le|z|$ gives $\|\phi_{w,b}\|_2\le1$;
constant and zero features obey the stated conventions.
\end{proof}

\begin{lemma}[Same-budget comparison]\label{lem:bounded-refit-comparison}
Let $\|y\|_2=1$ and let $\{\phi_{j,t}\}_{j=1}^m$, $t=0,1$, be
square-integrable feature banks with the same nonempty coefficient
set $\mathcal V\subseteq\{v:\|v\|_1\le B_1\}$.
Write $\mathcal R(t)=\inf_{v\in\mathcal V}\|y-\sum_jv_j\phi_{j,t}\|_2^2$.
If $\max_j|\langle y,\phi_{j,0}\rangle|\le\epsilon$, then
\[
 \mathcal R(0)\ge1-2B_1\epsilon.\tag{BW7}
\]
If a legal $v_1\in\mathcal V$ also has endpoint error at most $\rho$
in $L^2$, then $\mathcal R(0)-\mathcal R(1)\ge1-2B_1\epsilon-\rho^2$.
Any additional coefficient cap must be checked for that witness.
\end{lemma}
\begin{proof}
For every $v\in\mathcal V$, expansion of the square gives
$\|y-\sum_jv_j\phi_{j,0}\|_2^2
\ge1-2\sum_jv_j\langle y,\phi_{j,0}\rangle\ge1-2B_1\epsilon$.
Take the infimum and use $v_1$ for the endpoint upper bound.
No feature independence or Gram conditioning is required.
\end{proof}

\subsection{The original Gaussian bank}
Let $m\ge1$ and let the original spatial rows be independent
$w_j\sim N(0,s_w^2I_d)$ with common public $s_w>0$.
Keep the declared bias and training-head laws, and fix a public
$B_1>0$ and the common nonempty feasible set
$\mathcal V\subseteq\{v:\|v\|_1\le B_1\}$. Define
\[
 \mathcal R(0)=\inf_{v\in\mathcal V}
       \left\|y-\sum_jv_j\phi_{w_j,b_j}\right\|_2^2.\tag{BW3}
\]
All existing $\ell^2$ caps are retained; no intercept or projection
is added. For $0<\delta_I<1$, set
\[
 \ell_I=\log(2m/\delta_I),\quad
 Q_I=r+2\sqrt{r\ell_I}+2\ell_I,\qquad d\ge16\ell_I.\tag{BW4}
\]
\begin{lemma}[Initial normalized-bank obstruction]\label{lem:initial-bank}
With probability at least $1-\delta_I$ under the original draw,
\[
 \max_j\|U^T(w_j/\|w_j\|)\|^2\le2Q_I/d,\qquad
 \mathcal R(0)\ge1-2B_1\{(2Q_I/d)^{k_0/2}+\zeta\}.\tag{BW5}
\]
There is no width ceiling; the bound includes adaptive feasible readouts.
\end{lemma}
\begin{proof}
For $Z_k\sim\chi_k^2$, exponential Markov gives
\[
 \Pr\{Z_k>k+2\sqrt{kx}+2x\}\le e^{-x},\quad
 \Pr\{Z_k<k-2\sqrt{kx}\}\le e^{-x}.\tag{BW6}
\]
Indeed the centered log MGFs are bounded by $ku^2/(1-2u)$ for the
upper tail and $ku^2$ for the lower tail; take, respectively,
$u=\sqrt{x}/(\sqrt{k}+2\sqrt{x})$ and $u=\sqrt{x/k}$.
For a standardized row $g$, apply BW6 to $\|U^Tg\|^2$ and
$\|g\|^2$. At $x=\ell_I$ their bounds are $Q_I$ and $d/2$.
Union over both tails and all rows costs $2me^{-\ell_I}=\delta_I$,
without assuming numerator/denominator independence.
On this event Lemmas~\ref{lem:initial-hermite} and
\ref{lem:bounded-refit-comparison} give BW5 simultaneously.
\end{proof}

\section{Mixture profile acquisition}\label{app:mixture-acquisition}

All public quantities are those of Appendix~\ref{app:sw-parameters}.
Here $X\sim N(0,I_d)$, $P_U=UU^T$, $P_{U^\perp}=I_d-P_U$,
$\varphi_v$ is the $N(0,v)$ density, $\varphi=\varphi_1$, and
$\Phi$ is the standard normal CDF. The $h_k$ below are normalized
probabilists' Hermite polynomials.
The equation families BM, BA and BAJ all refer to calculations printed
in this section. Write $\mathcal V_n=\sum_{j=1}^m V_{j,n}$ for the
complete normalized row energy.
Within the base acquisition calculation write
$\varepsilon=\varepsilon_{\rm BM}$ and $\kappa=\kappa_{\rm BM}$.
The acquisition endpoint is $n_0+n_1=N_*$. All candidate and
cumulative estimates below use the final public envelope and horizon
$(M,T)$ from the outset, with one mesh and one force allowance throughout.
A single induction then gives acquisition at $N_*$ and persistence
through $N$; no second continuation argument is needed.

\subsection{Full links and unchanged actual dynamics}
Let $J=1-\alpha>0$, $0\le\alpha<1$, and
$\sigma_\alpha(t)=\alpha t+Jt_+$. For probability measures $\mu_i$ on
$[1/3,1]$ define
\begin{gather}
 g_i(t)\varphi(t)=-\int\varphi_v'''(t)d\mu_i(v),\quad
 N_i=\|g_i\|_2,\quad \bar g_i=g_i/N_i,\nonumber\\
 y_c=\sum_{i=1}^r\lambda_i \bar g_i(u_i^TX),\quad
 U^TU=I_r,\quad\lambda_i>0,\quad\sum_i\lambda_i^2=1.
 \tag{BM1}
\end{gather}
Signs can be absorbed into the axes because these full links are odd.
The exact H1-convergent series is
\[
\begin{gathered}
g_i=\sum_{j\ge0}\frac{(-1)^j\sqrt{(2j+3)!}}{2^jj!}
       \left[\int(1-v)^j d\mu_i(v)\right]h_{2j+3},\\
 \sqrt6\le N_i<9,\quad\|y_c\|_2=1,\quad\|y_c\|_{H^1}<13.
\end{gathered}
 \tag{BM2}
\]
Indeed integration by parts gives the generating function
$t^3\int e^{-(1-v)t^2/2}d\mu_i(v)$; the squared coefficients and
their degree weights are dominated by the summable sequence at $v=1/3$.
For $a=1-v$, the squared L2 and derivative norms before normalization are
$3(2+3a^2)/(1-a^2)^{7/2}<80$ and
$(18+69a^2+18a^4)/(1-a^2)^{9/2}<800$.
The coefficient of $h_3$ is $\sqrt6$, proving BM2.
If $\mu_i(v<1)>0$, every displayed coefficient is nonzero. At
$\mu_i=\delta_{1/3}$ the normalized link is at L2 distance greater
than one from $h_3$. Thus this is not a tiny cubic neighborhood.

The full teacher may be any function of $U^TX$ satisfying
\[
 y=y_c+e,\qquad \|y\|_2=1,\qquad \|e\|_{H^1}\le\zeta.
 \tag{BM3}
\]
This optional small residual retains its actual mean, linear part and
every tail coefficient. The central links intrinsically have zero mean
and first moment; no centering operation is performed. This partial
branch does not cover arbitrary order-one mean or linear components.
Train the original averaged network and full MSE by simultaneous raw GD:
\[
 f=m^{-1}\sum_j a_j\sigma_\alpha(w_j^TX+b_j),\qquad
 L=\mathbb E[(y-f)^2],\qquad \eta=mh/2.
 \tag{BM4}
\]
Every initial raw entry is independent $N(0,s^2/d)$. For analysis only,
write $(a,w,b)=s(q,\theta)$, $\theta=(w,b)$, suppressing decorations;
put $\rho=\|w\|$, $V=q^2+\|\theta\|^2$. The exact recurrence is
\[
 q^+=q+h(C+\theta^TE),\quad
 \theta^+=\theta+hq(\nabla C+E),\quad
 C=\mathbb E[y_c\sigma_\alpha(w^TX+b)],
 \tag{BM5}
\]
\[
 E_j=\mathbb E[(e-s^2\widehat f)\sigma'_j(X,1)],\quad
 \widehat f=m^{-1}\sum_lq_l\sigma_\alpha(\theta_l^T(X,1)),\quad
 \|E_j\|\le\zeta+\frac{s^2}{2m}\sum_lV_l.
 \tag{BM6}
\]
Here $\sigma'_j=\sigma'_\alpha(\theta_j^T(X,1))$ is a scalar gate;
its following factor $(X,1)$ is the augmented input vector.
The same adaptive vector generates both errors. Gaussian
Cauchy--Schwarz gives the last bound. The full uncentered H1 kernel
and BM2 give
\[
 |C|\le\|\theta\|,\quad\|\nabla C\|\le1,\quad
 \|\theta\|\|D^2C\|\le52,\quad\theta^T\nabla C=C.
 \tag{BM7}
\]
This holds away from the augmented origin, including pure-bias states.
The regularity and compatible-update bounds are proved in Appendix~\ref{app:kernels}.

Let $B_i^*$ be the unique solution
\[
 B_i^*(1+B_i^*)m_i(B_i^*)=1,\quad
 m_i(B)=\frac{\int v^{-5/2}e^{-B/(2v)}d\mu_i(v)}{\int v^{-3/2}e^{-B/(2v)}d\mu_i(v)},\quad
 \frac{\sqrt{7/3}-1}{2}\le B_i^*\le\frac{\sqrt5-1}{2}.
 \tag{BM16}
\]
Put $b_i^*=\sqrt{B_i^*}$ and take the actual integer checkpoint
\[
 n_0=\lceil4R/(hk)\rceil,\quad n_1=\lceil T_1/h\rceil,
 \quad N_*=n_0+n_1,\qquad N_*h\le T_0.
 \tag{BM17}
\]

\subsection{Exact selected-cone dynamics, including candidate exits}
For $x_i=u_i^Tw$, $\rho>0$, put
$D_{iv}=\rho^2-(1-v)x_i^2$, $B=b^2/\rho^2$, $B_{iv}=b^2/D_{iv}$.
Three integrations by parts give the full uncentered correlation
\[
 C=\sum_i\int C_{iv}d\mu_i(v),\qquad
 C_{iv}=-\frac{J\lambda_i}{N_i}b x_i^3
                    D_{iv}^{-3/2}\varphi(b/\sqrt{D_{iv}}).
 \tag{BM23}
\]
Write $H_{iv}=qC_{iv}$, $H=qC$. On $B\le3/4$ one has
$\rho^2/3\le D_{iv}\le\rho^2$ and $B_{iv}\le9/4$.
For a selected axis $i$ fix $\epsilon=-\operatorname{sign}(qb)$ and set
\begin{gather}
 x=\epsilon x_i>0,\quad u^2=\sum_{l\ne i}x_l^2,\quad
 O=\|P_{U^\perp}w\|,\quad u\le\vartheta x,\nonumber\\
 H_i=\int H_{iv}d\mu_i>0,\quad
 H_{\rm other}=\sum_{l\ne i}\int|H_{lv}|d\mu_l.
 \tag{BM24}
\end{gather}
Persistence of the head and bias signs is proved below. Uniformly,
\[
 \sum_l\int|C_{lv}|d\mu_l\le J|b|,\quad
 |C_i|\ge\frac{J\lambda_i|b|x^3}{72\rho^3},\quad
 H_{\rm other}\le\eta_iH_i,\qquad
 \eta_i=\frac{72\vartheta^3}{\lambda_i}<1/100.
 \tag{BM25}
\]
For the upper bound use $N_l\ge\sqrt6$, $D_{lv}\ge\rho^2/3$,
$3\sqrt3\varphi(0)/\sqrt6<1$, and
$\sum_l\lambda_l|x_l|^3\le\rho^3$.
For the lower bound use $N_i<9$, $\varphi(3/2)>1/8$.
The same upper estimate on the other coordinates has numerator $u^3$.

The pure teacher spatial candidate has a common radial multiplier
and nonnegative quadratic boosts in $y_l=\epsilon x_l$:
\[
 y_l^{\,t}=A y_l+d_l y_l^2,\quad
 P_{U^\perp}w^{\,t}=A P_{U^\perp}w,\quad
 A=1+\sum_l\int \frac{h(B_{lv}-3)H_{lv}}{D_{lv}}d\mu_l,
 \tag{BM26}
\]
\[
 d_l=\frac{hpJ|b|\lambda_l}{N_l}
 \int D_{lv}^{-3/2}\varphi(b/\sqrt{D_{lv}})
 \left[3+\frac{(1-v)x_l^2(3-B_{lv})}{D_{lv}}\right]d\mu_l,
 \qquad p=|q|.
\]
The bracket is between $3$ and $9$. As
$p\le\|\theta\|\le\sqrt{7/4}\rho$, BM25 gives
\[
 |A-1|\le11h,\qquad
 d_i\ge\frac{hpJ|b|\lambda_i}{24\rho^3},\qquad
 \|(d_ly_l^2)_{l\ne i}\|\le\frac{9hpJ|b|u^2}{\rho^3}.
 \tag{BM27}
\]
Thus $A>0$, and $u\le\vartheta x$ implies the candidate inequality
\[
 u^t-\vartheta x^t
 \le-\frac{\vartheta hpJ|b|\lambda_i x^2}{48\rho^3}.
 \tag{BM28}
\]
The force displacement changes its left side by at most $2hp\nu$.
Since $|C|\ge k$ implies $J|b|\ge k$, \eqref{eq:sw-force} pays that displacement
whenever $x\ge x_0/2$, $\rho\le M$. This proves the actual candidate
cone constraint, not just an inward derivative at its boundary.

The own-component contribution to the oriented $x$ increment is
\[
 \int hH_{iv}\left[\frac{3}{x}-\frac{vx(3-B_{iv})}{D_{iv}}\right]d\mu_i
 =\int hH_{iv}\frac{3(u^2+O^2)+B_{iv}vx^2}{xD_{iv}}d\mu_i
 \ge\frac{3hH_i u^2}{x\rho^2}.
 \tag{BM29}
\]
The adverse radial contribution of the other axes has magnitude at
most $9hH_{\rm other}x/\rho^2$. Relative to the displayed lower
bound, it is at most $216\vartheta/\lambda_i<1$ when $u>0$;
when $u=0$ it vanishes. Hence $x^t\ge x$. The cumulative actual
decrease is at most $\nu MT\le x_0/4$. Starting with $x\ge x_0$
therefore closes $x\ge x_0/2$ at every candidate state.

For the upper bias barrier let $F=b^2-(3/4)\rho^2$. Its derivative is
\[
 DF[q\nabla C]=2\sum_l\int[1-(7/4)B_{lv}]H_{lv}d\mu_l.
 \tag{BM30}
\]
If $2/3\le B\le3/4$, this is at most
$2[-1/6+3\eta_i]H_i\le-H_i/4\le-pk/8$.
The pure quadratic remainder and full force displacement together
are at most $hpM(4h+2\nu)<hpk/8$ by \eqref{eq:sw-mesh}--\eqref{eq:sw-force}.
If $B<2/3$, the initial gap is $F\le-\rho^2/12$, whereas the
absolute actual one-step change is at most $8h\rho^2$.
Thus both cases, and every transition between them, give $B^+\le3/4$.

The radial multiplier further satisfies
\[
 A\le1-\frac{hH}{3\rho^2},\qquad
 O^+\le(1-t/3)O+tM^2\nu/k,\qquad t=hH/\rho^2\ge0.
 \tag{BM31}
\]
Indeed its own term is at most $-(3/4)hH_i/\rho^2$ and its adverse
terms sum to at most $9h\eta_iH_i/\rho^2$.
Use $H\le(1+\eta_i)H_i$, $A>0$, and $|C|\ge k$.
Then $\nu\le2k/(3M^2)$ preserves $O\le2$.
No analogous cone or favorable sign has been assumed for other rows.

\subsection{Favorable growth and a lower bias barrier}
We record the generic full-H1 quotient estimate explicitly:
\[
 Q=\frac{qC}{V},\qquad Q^+\ge Q-h\nu^2,
 \qquad |q|\le\|\theta\|,\quad \|E\|\le\nu\le1,
 \quad h\le2^{-12}.
 \tag{BM32}
\]
To verify it, put $z=(q,\theta)$, $F_H=\nabla(qC)$,
$p_E=(\theta^TE,qE)$, $a=z^Tp_E/V$,
$p_\perp=p_E-az$, $G=F_H-2Qz$.
Homogeneity makes the new quotient
$Q(z+\tau(G+p_\perp))$, where $\tau=h/[1+h(2Q+a)]$.
This is a tangent retraction. The joining segment has norm at least
$\sqrt V$ and head/hidden ratio at most two. Differentiation with
BM7 gives $\|D^2Q\|\le128/V$ there: the Hessian-of-$H$ term
contributes at most $105/V$ and the remaining quotient terms at most
$9/V$. Also $\|p_\perp\|\le\nu\sqrt V$ and
$\gamma=\|G\|/\sqrt V\le1$. Taylor's theorem yields
\[
 Q^+-Q\ge\tau(\gamma^2-\nu\gamma)
       -64\tau^2(\gamma+\nu)^2
 \ge\tau(\gamma^2/4-3\nu^2/4)\ge-h\nu^2.
\]
Balance persists exactly, since homogeneity also gives
\[
 \|\theta^+\|^2-(q^+)^2
 \ge(1-h^2\|\nabla C+E\|^2)(\|\theta\|^2-q^2),
 \qquad V^+\le(1+2h)^2V.
\]
These estimates apply to every original row, independent of signs.

For a selected seed with $p_0\ge1/(2\sqrt d)$, $|C_0|\ge k$,
$H_0>0$, $V_0\le5$, BM32 gives
$Q_n\ge Q_*-T\nu^2\ge1599Q_*/1600$.
Orient $C$ by the initial head sign and call it $K$.
Euler's identity gives $K/\|\theta\|\ge2Q_n\ge Q_*$.
Thus $\nu\le Q_*/4$ is at most one quarter of both
$K/\|\theta\|$ and $\|\nabla K\|$.
The head sign persists and $p^+\ge p+3hK/4$.
Taylor's theorem with BM7 gives for every $0\le u\le1$
\[
 K(\theta+uhp(\nabla K+\widetilde E))
 \ge K(\theta)+(u/2)hp\|\nabla K\|^2,
 \qquad p^+\ge p+3hk/4.
 \tag{BM33}
\]
Here $\widetilde E$ includes the fixed head orientation.
The first-order inner product is at least
$(3/4)uhp\|\nabla K\|^2$, while the curvature remainder is at most
$82u^2h^2p\|\nabla K\|^2$.
Both $\|\theta\|^2$ and $V$ increase. Since BM23 vanishes at
$b=0$ or $w=0$, BM33 preserves the original bias sign along each
entire hidden step. BM28--29 preserve the chosen spatial orientation.
This supplies the signs used above without assuming a future phase.
The argument is a simultaneous induction on signs, $K\ge k$, the
cone and the two barriers: every inequality uses only its pre-state
conditions and has just been established for its complete candidate.

Set $W=(\rho^2-8b^2)_+$. When $B\le1/6$, the first-order change
of the untruncated quadratic is
$2h\sum_l\int(9B_{lv}-8)H_{lv}d\mu_l\le0$:
the own coefficient is at most $-7/2$, the other absolute coefficients
at most $8$. When $B>1/6$, its initial negative gap is at least
$\rho^2/3$; the possible positive first-order change is at most
$30h\rho^2$. Indeed $\sum_l\int|H_{lv}|\le pJ|b|
\le(6/5)\rho^2$ and $|9B_{lv}-8|\le49/4$.
The pure quadratic remainder is at most $h^2M^2$, and the force
cost is at most $20h\nu M^2$ (the quadratic matrix has norm eight).
The positive part therefore obeys, including all branch crossings,
\[
 W_n\le W_0+TM^2(h+20\nu)<5.
 \tag{BM34}
\]
Also ${\cal D}=\|\theta\|^2-p^2$ satisfies
${\cal D}^+-{\cal D}\le4h^2M^2$, so ${\cal D}_n\le6$.
Consequently, at every state with $\rho\ge8$,
\[
 1/16\le B\le3/4,\qquad
 p/\rho\ge\sqrt{1-6/\rho^2}>0.95.
 \tag{BM35}
\]
No monotone spatial radius has been assumed. At $n_0$ BM33 gives
$p>3R$; at all later states balance and the upper bias barrier give
$\rho\ge p/\sqrt{7/4}>2R$. This is a persistent late-radius bound.

\subsection{Actual convergence to the link-dependent bias}
In BM16, $m_i(B)=\mathbb E_B[1/v]$ under probability weights
proportional to $v^{-3/2}e^{-B/(2v)}d\mu_i$. Thus $1\le m_i\le3$
and $m_i'=-\operatorname{Var}_B(1/v)/2\in[-1/2,0]$.
On $0\le B\le3/4$, the derivative of $B(1+B)m_i(B)$ is between
one and nine. At the endpoints stated in BM16 the function is,
respectively, at most and at least one. This proves existence,
uniqueness and the stated bounds for $B_i^*$.

At a late selected row write
$e_a=1-x^2/\rho^2\le\vartheta^2+4/R^2$.
The own dimensionless variance in BM23 is
$d_v=v+(1-v)e_a$ rather than $v$.
Let $m_{i,e_a}(B)$ be the resulting weighted mean of $1/d_v$.
Differentiation with respect to $e_a$ gives a direct term of magnitude
at most nine, and a covariance term at most 48: use $1/d_v\le3$
and
$|\partial_{e_a}\log(d_v^{-3/2}e^{-B/(2d_v)})|\le8$.
Hence, uniformly on the corridor,
\[
 |m_{i,e_a}(B)-m_i(B)|\le60e_a.
 \tag{BM36}
\]
With $t_i=hH_i/\rho^2$, BM25 and BM35 give
\[
 \kappa h\le t_i\le(6/5)h.
 \tag{BM37}
\]
For the lower bound use $p/\rho>0.95$, $|b|/\rho\ge1/4$ and
$x/\rho\ge0.9$ in BM25; the coefficient is greater than
$J\lambda_i/1024$. For the upper bound use
$H_i\le\sum_l\int|H_{lv}|\le pJ|b|$.

Expand the actual quotient $B^+=(b^+)^2/\|w^+\|^2$:
\[
 B^+-B=\frac{2h}{\rho^2}
    \sum_l\int[1-(1+B)B_{lv}]H_{lv}d\mu_l+{\cal E},
 \qquad |{\cal E}|\le200h^2+20h\nu.
 \tag{BM38}
\]
Here is an explicit remainder check. Both
$\|\Delta w\|/\rho$ and $|\Delta b|/\rho$ are at most $3h$.
The denominator divided by $\rho^2$ differs from one by at most
$9h$. In the numerator for $B^+-B$, the first-order teacher term
has absolute value at most $8h$, the force term at most $5h\nu$,
and the quadratic term at most $13h^2$. Division by $1-9h$
proves the conservative remainder in BM38.
For this comparison, refine the uniform BM25 bound using the
actual transverse coordinates. Since $u/x\le\vartheta$ and
$u^2/\rho^2\le e_a$, the same numerator estimate gives
\[
 \frac{H_{\rm other}}{H_i}
 \le\frac{72}{\lambda_i}\left(\frac ux\right)^3
 \le\frac{72\vartheta}{\lambda_i}\frac{e_a}{1-e_a}
 \le0.09e_a.
\]
The last inequality uses $x/\rho\ge0.9$ and
$\vartheta\le\lambda_i/1000$; it also holds when $u=0$.
Together with BM36, this changes the teacher term from
$2t_i[1-(1+B)Bm_i(B)]$ by at most
$2t_i(80e_a+3H_{\rm other}/H_i)\le400he_a$,
using BM37.
The ideal bracket equals $-a(B)(B-B_i^*)$ with $1\le a(B)\le9$.
Its candidate coefficient has no sign reversal, since
$2t_i a(B)<1$. Therefore
\[
 |B^+-B_i^*|\le(1-\kappa h)|B-B_i^*|
             +h[400(\vartheta^2+4/R^2)+200h+20\nu].
 \tag{BM39}
\]
By \eqref{eq:sw-structure}--\eqref{eq:sw-force} the last bracket is at most $\kappa\varepsilon/4$.
After $n_1h\ge T_1$, the initial error, at most one, is at most
$\varepsilon/4$, and the accumulated error at most $\varepsilon/4$.
Thus $|B_{N_*}-B_i^*|\le\varepsilon$.
The original bias sign persists, and
$|\sqrt B-\sqrt{B_i^*}|\le2|B-B_i^*|$.
Also
$\|w/\rho-\epsilon u_i\|\le\sqrt2\vartheta+3/R$.
These inequalities and \eqref{eq:sw-structure} prove the calibrated profile assertion with room to spare.
This proof includes the actual integer endpoint, all intermediate
bias transitions, and a radius bound acquired before the late phase.

\subsection{Original Gaussian event and full coupled-force closure}
Write the normalized initial row as $(q,w,b)=(Z,G,B_0)/\sqrt d$.
For each $(i,\epsilon,\tau)$ use the event
\[
 \epsilon G_i\in[1,2],\quad
 |G_l|\le\vartheta/(4\sqrt{r-1})\ (l\ne i,\ l\le r),\quad
 \tau B_0\in[1/2,1],\quad -\epsilon\tau Z\in[1/2,1],
 \tag{BM40}
\]
omitting the other-axis restriction for $r=1$.
It has probability exactly $p_{\rm rect}$ and places no restriction
on the remaining $d-r$ coordinates. The union bound over $4r$
categories shows that every category occurs among the original $m$
rows except with probability $\delta/4$. Distinct categories need
distinct rows. This is a property of the unscreened draw; no row is
filtered from the actual training network.
Separately, Gaussian and chi-square tails give, for all original rows,
\[
 1/2\le\rho_{j,0}\le2,\qquad
 |Z_j|,|B_{0,j}|\le\sqrt{2\ell},
 \tag{BM41}
\]
except with probability at most
$m(4e^{-\ell}+2e^{-d/8})\le6m e^{-\ell}<\delta/8$.
On BM41, $V_{j,0}<5$, $q_{j,0}^2\le\|\theta_{j,0}\|^2$ and
$O_{j,0}\le2$. Each designated row additionally has
$x\ge x_0$, $u\le\vartheta x_0/4$, $B\le4/d\le1/16$,
$p_0\ge1/(2\sqrt d)$, and
$|C_0|\ge J\lambda_i/(2048d^2)\ge k$ by BM25.
Its signs in BM40 give $H_0>0$.

\subsection{One full-horizon induction and its persistent profiles}

On the original event BM40--41, of failure at most $3\delta/8$,
all rows start balanced with $V_{j,0}<5$. At a pre-state with force
at most $\nu$, BM32 gives $V_j^+\le(1+2h)^2V_j$, while BM6 gives
\[
 \|E_j\|\le\zeta+\frac{s^2}{2m}\mathcal V_n.
 \tag{BA18}
\]
Induction therefore supplies, including each endpoint candidate,
\[
 \mathcal V_n\le5me^{4nh},\qquad
 \mathcal V_{N_*}\le S_0,\qquad
 \mathcal V_n\le S_0e^{4(n-N_*)h}\le S\quad(N_*\le n\le N).
 \tag{BA19}
\]
Indeed $N_*h\le T_0$ and $(N-N_*)h\le T_G+1$.
The residual and scale cuts recompute BA18 at the candidate as
$\|E_j\|\le\nu/2+s^2S/(2m)\le\nu$, closing the induction.
Consequently the complete interacting bank obeys
\[
 V_j^+\le(1+2h)^2V_j,\qquad
 \|E_{j,n}\|\le\nu,\qquad \sum_jV_{j,n}\le S=M^2/4
 \quad(n\le N).
 \tag{BAJ17}
\]
These are public envelopes, not conditions imposed on the realized
trajectory. No signed row has been discarded.

The selected-row induction BM28--35 now applies on this entire
horizon. Its invariant is the conjunction of the original head/bias
signs, $K\ge k$, $Q\ge1599Q_*/1600$, $x\ge x_0/2$,
$u\le\vartheta x$, $B\le3/4$ and $O\le2$.
Every update was proved for the full candidate, including repeated
bias transitions. Its payments are, respectively, the quotient debit
$T\nu^2\le Q_*/1600$, the cone cut
$\nu\le\vartheta\lambda_*kx_0^2/(2048M^3)$, coordinate debit
$\nu MT\le x_0/4$, upper-bias cuts $h,\nu\le k/(512M)$,
and outside cut $\nu\le2k/(3M^2)$.
The two cumulative quadratics, already proved in BM34--35, give
\begin{gather}
 W_n\le W_0+TM^2(h+20\nu)<5,\qquad
 TM^2h\le1/4,\quad20\nu TM^2\le1/4,\notag\\
 0\le\|\theta_n\|^2-q_n^2
 \le\|\theta_0\|^2-q_0^2+4hTM^2\le6.
 \tag{BA20}
\end{gather}
The second estimate uses only compatibility, with upper debit
$4h^2V_j$ per step, and thus holds for every original row:
\[
 0\le D_j:=\|\theta_j\|^2-q_j^2\le6\quad(n\le N).
 \tag{BAJ19}
\]
Since $p^+\ge p+3hk/4$, at $n_0$ one has $p>3R$ and thereafter
$\rho>2R$. BA20 then gives $p/\rho>19/20$ and $B\ge1/16$,
without assuming monotonicity of $\rho$.

For every late candidate BM36--39, including its actual
angular-error-dependent contamination bound, give
\[
 |B^+-B_i^*|\le(1-\kappa_{\rm BM}h)|B-B_i^*|
       +h\{400(\vartheta^2+4/R^2)+200h+20\nu\}.
 \tag{BA21}
\]
The braces are at most $\kappa_{\rm BM}\varepsilon/4$.
The discrete geometric sum yields, for every $n\ge N_*$,
$|B_n-B_i^*|\le e^{-\kappa_{\rm BM}(n-n_0)h}+\varepsilon/4
\le\varepsilon/2\le\varepsilon$.
Together with BM39's bounds on the angular and square-root errors,
$\sqrt2\vartheta+3/R+2\varepsilon<e_*$ proves, for the same four
original categories of every axis and every $N_*\le n\le N$,
\begin{gather}
 \rho_j>2R,\quad |q_j|/\rho_j>19/20,\quad
 1/16\le b_j^2/\rho_j^2\le3/4,\notag\\
 \|(w_j/\rho_j,b_j/\rho_j)
        -(\epsilon u_i,\tau\sqrt{B_i^*})\|\le e_*,\qquad
 \operatorname{sign}(q_j)=-\epsilon\tau.
 \tag{BAJ18}
\end{gather}
No comparability of the optimized directional signals was used here.
Finally $\|f_n\|_2\le s^2M^2/(8m)\le\epsilon_L/160$ gives the
original-loss bound, and $\sum_jq_j^2\le M^2/4$ gives the raw head
bound. Each of the $4r$ distinct rows starts with spatial norm at
most $2s$ and has late norm above $2sR$; reverse triangle inequality
gives the same stated spatial-motion bound.

\section{Control of every signed row for mixture teachers}\label{app:mixture-defect}

\subsection{Full model and the sharp static ceiling}
Use $X\sim N(0,I_d)$, $\sigma_\alpha(t)=\alpha t+(1-\alpha)t_+$,
$(A_j,W_j,B_j)=s(q_j,w_j,b_j)$ and $\theta_j=(w_j,b_j)$.
For a normalized row, $C(\theta)=\mathbb E[y_c\sigma_\alpha(w^TX+b)]$;
$P_U=UU^T$ and $P_{U^\perp}=I_d-P_U$.
The densities $\varphi_v$ and $\varphi$ are those of $N(0,v)$ and
$N(0,1)$, respectively.

Use precisely the full teacher, network, original Gaussian law and
simultaneous update \eqref{eq:app-mixture} and BM4:
\begin{gather}
 g_v(t)\varphi(t)=-\varphi_v'''(t),\quad
 g_i=\int_{[1/3,1]}g_v\,d\mu_i(v),\quad N_i=\|g_i\|_2,\quad
 \notag\\ y_c=\sum_i\frac{\lambda_i}{N_i}g_i(u_i^TX),\quad
 y=y_c+e,\quad \|y\|_2=1 .
 \tag{SB1}
\end{gather}
Here each \(\mu_i\) is an arbitrary probability measure,
\(\lambda_i>0\), \(\sum_i\lambda_i^2=1\), and \(U^TU=I_r\).
No Hermite coefficient is removed.  In particular
\(\sqrt6\le N_i<9\) and \(\|y_c\|_{H^1}<13\).
The full \(H^1\) residual retains its mean, linear and interaction
terms.  For \(J=1-\alpha>0\), let
\[
 K_i(B)=\frac{J\lambda_i}{2N_i}\sqrt{\frac B{1+B}}
       \int v^{-3/2}\varphi(\sqrt{B/v})\,d\mu_i(v),\qquad
 K_i^*=\max_{B\ge0}K_i(B),\qquad K_*=\max_iK_i^* .
 \tag{SB2}
\]
For each fixed axis $i$, $\mathbb E_B[\cdot]$ denotes expectation
under the probability measure proportional to
$v^{-3/2}e^{-B/(2v)}d\mu_i(v)$; the axis index is suppressed.
The maxima occur at the unique mixture roots
\(B_i^*(1+B_i^*)\mathbb E_{B_i^*}[1/v]=1\);
thus \(B_i^*\in[(\sqrt{7/3}-1)/2,(\sqrt5-1)/2]\).
The proof following (BM35) in Appendix~\ref{app:mixture-acquisition}
establishes uniqueness in this interval. There is no other global
maximum: the logarithmic derivative of the axis objective is
\(1/[2B(1+B)]-\mathbb E_B[1/v]/2<0\) for \(B>\beta\);
the objective vanishes at zero and infinity.
In particular \(J/(1000\sqrt r)\le K_*\le1/2\).

Write \(p=|q|\), \(\rho=\|w\|\), \(V=p^2+\rho^2+b^2\),
\(H=qC\), \(Q=H/V\), and \(B=b^2/\rho^2\).  At zero spatial
radius \(C=0\), including pure-bias rows.  Put
\[
 \beta=\frac{\sqrt5-1}{2},\qquad
 \Theta(a)=
 \left\{\frac{a/(1+a)}{\beta/(1+\beta)}
                         e^{-(a-\beta)}\right\}^{1/2}.
 \tag{SB3}
\]
For every \(a\ge\beta\), every row, every head ratio and every
collection of measures and coefficients,
\[
 {\displaystyle \quad B\ge a\ \Longrightarrow\ |Q|\le\Theta(a)K_*.\quad}
 \tag{SB4}
\]
This ceiling is sharp uniformly over the family: equality holds
for a single cubic axis \(\mu=\delta_1\), \(B=a\), and
\(p^2=(1+a)\rho^2\).

Here is a full proof that includes signed spatial coordinates.
Let \(s_i=(u_i^Tw/\rho)^2\), \(d=1-(1-v)s_i\).
The head factor is at most \(1/2\).  The squared ratio of the
absolute \(i,v\) contribution to \(s_i\) times its axis contribution
at \(\beta\) is
\[
 \frac{B/(1+B)}{\beta/(1+\beta)}s_i
       (v/d)^3\exp(-B/d+\beta/v).
 \tag{SB5}
\]
Since \(v\le d\), the function \(3\log v+\beta/v\) is increasing
on \([1/3,1]\).  Consequently the last two factors are at most
\(\exp[-(B-\beta)/d]\le\exp[-(B-\beta)]\).
Sum the resulting unsquared positive bounds, using
\(\sum_i s_i\le1\) and \(K_i(\beta)\le K_i^*\).
Finally \(\Theta\) is decreasing for \(B\ge\beta\), because
\((\log\Theta^2)'=1/[B(1+B)]-1\).  This proves SB4 without
a parity cancellation or a favorable row assumption.

In particular
\[
\begin{gathered}
\Theta(3/4)=0.991615\ldots,\qquad
 \Theta(29/20)=0.821163\ldots<411/500,\\
 \Theta(3/2)=0.806393\ldots .
\end{gathered}
 \tag{SB6}
\]
Only the strict rational bound in the middle is used below.
It can be certified without numerical optimization: first write
\(\beta/(1+\beta)=1-\beta\), enclose
\(\beta\) between \(618033/10^6\) and \(618035/10^6\), substitute
the upper endpoint in both increasing factors of SB3, and lower
bound \(e^{831965/10^6}\) by its Taylor polynomial of degree eight.
All remaining comparisons are rational.

\subsection{A bounded bias weight for every signed component}

For \(B\ge0\), define the continuous function
\[
 c(B)=
 \begin{cases}
 1/5,&0\le B\le1,\\[2pt]
 \displaystyle
 \frac{(1+B)^{-1}-9/2+6B}{2\{3B(1+B)-1\}},&1\le B\le3/2,\\[6pt]
 c(3/2),&B\ge3/2,
 \end{cases}
 \qquad F(B)=\exp\left(\int_0^B c(u)\,du\right).
 \tag{SB7}
\]
On \([0,3/2]\), \(1/5\le c\le1/4\), \(c\) is
Lipschitz with constant two, and
\[
 1\le F(B)\le e^{3/8}<3/2.
 \tag{SB8}
\]
For example, the numerator of \(c-1/5\), after multiplication
by the positive denominator and \(10(1+B)\), is
\((B-1)(-12B^2+24B+31)\).
For \(c\le1/4\), the relevant numerator is
\((1+B)^{-1}-4+(9/2)B-(3/2)B^2
 \le1/2-5/8<0\).
Differentiating the displayed quotient bounds \(|c'|\) by two;
its numerator is at most \(49/10\), its derivative at most six,
its denominator at least ten, and the latter's derivative at most
24.  The constant continuation preserves this Lipschitz bound.

For \(qb\ne0\) and \(\rho>0\), let
\[
 \tau=-\operatorname{sign}(qb),\quad
 E=\|P_{U^\perp}w\|^2+\sum_i\min\{\tau u_i^Tw,0\}^2,\quad
 X=p\sqrt E,\qquad Z=F(B)X .
 \tag{SB9}
\]
Orientation is analytical; no parameter is reflected.
The unweighted \(X\) alone is used for pure-bias or unmarked
rows; \(Z\) is needed only on marked states, where \(Q>0\)
already ensures \(\rho>0\).
Fix \(11/10\le a<3/2\), put \(\Delta=3/2-a\), and assume
\[
 B\le a,\qquad
 \frac{1}{1+B}\le k:=\frac{\rho^2}{p^2}
       \le\frac{1+\varepsilon}{1+B},\qquad
 0<\varepsilon\le\min\{1/15,\Delta\}.
 \tag{SB10}
\]
For each full kernel component put \(t=1/d\in[1,3]\).
The two decisive coefficients are
\begin{align}
 L&=3t-Bt^2-k-2c(B)+2c(B)B(1+B)t,\notag\\
 R&=k+Bt+2c(B)-2c(B)B(1+B)t.
 \tag{SB11}
\end{align}
They satisfy, uniformly over all mixtures,
\[
 {\displaystyle \quad L\ge\Delta,\qquad R\ge\Delta,\qquad
              L+R=(3+B)t-Bt^2.\quad}
 \tag{SB12}
\]
For \(0\le B\le1\), \(L\) is concave in \(t\) and its
two endpoint values exceed one; \(R\ge k+2/5\ge9/10\).
Explicitly those endpoint values of \(L\) are
\(13/5-3B/5+2B^2/5-k\) and
\(43/5-39B/5+6B^2/5-k\).
The first exceeds one by \(k\le16/15\); the second decreases
to a value at least \(2-8/15\).
For \(1\le B\le a\), put \(k_0=(1+B)^{-1}\).
The choice SB7 gives at \(t=3\)
\[
 L=(9-6B)/2-(k-k_0)\ge(5/2)\Delta,\qquad
 R=(9-6B)/2+(k-k_0)\ge3\Delta.
 \tag{SB13}
\]
At \(t=1\), \(L>1\) and
\(R\ge2/5+2/5-3/8=17/40>\Delta\).
Concavity of \(L\) and affinity of \(R\) finish SB12.
This uses \(\Delta\le2/5\), the reason for the harmless lower
restriction \(a\ge11/10\).

To explain the new weight, use the exact components of BM5
\[
 H_{iv}=-\frac{J\lambda_iq}{N_i}c_i^3bD_{iv}^{-3/2}
                  \varphi(b/\sqrt{D_{iv}}),\quad
 D_{iv}=\rho^2-(1-v)c_i^2,\quad c_i=u_i^Tw .
 \tag{SB14}
\]
Their signs equal \(\operatorname{sign}(\tau c_i)\).
The common radial spatial coefficient is
\(-\sum\int H_{iv}(3t-Bt^2)/\rho^2\).
The coordinate boost magnitude is
\[
 k_{iv}^{\rm boost}
   =\frac{|H_{iv}|}{|c_i|}\{(3+B)t-Bt^2\},\qquad
 \dot B=\frac2{\rho^2}\sum_i\int
             H_{iv}\{1-B(1+B)t\}\,d\mu_i.
 \tag{SB15}
\]
Dots denote the vector field of the exact simultaneous teacher
Euler candidate, not a replacement flow.
Form a vector \(z\) of \(F(B)pw_\perp\) and the negative
coordinates \(F(B)p\tau c_i\).  It has norm \(Z\).
For a positive component the first-order vector field is
\(-H_{iv}Lz/\rho^2\).  For a negative component,
writing \(A_{iv}=|H_{iv}|\), its field is
\[
 A_{iv}Lz/\rho^2+F(B)p k_{iv}^{\rm boost}e_i,\qquad
 -\langle z,\dot z_{iv}\rangle
 =A_{iv}F(B)^2p^2
       \{R+L(1-E/\rho^2)\}
 \ge \Delta\, A_{iv} Z^2/\rho^2 .
 \tag{SB16}
\]
The negative coordinate itself pays its potentially expanding
radial contribution.  Summing all signed components proves
\[
 -\langle z,\dot z\rangle\ge\Delta H Z^2/\rho^2
 \quad\hbox{when }H>0.
 \tag{SB17}
\]

\subsection{The actual candidate, including every error and crossing}
Here $E_{\rm force}$ is the augmented compatible vector $E_j$ of
(BM6); it is distinct from the scalar signed-spatial energy $E$ in (SB9).

Assume SB10, \(Q\ge\kappa>0\), \(\|E_{\rm force}\|\le\nu\le
\min\{1,\kappa\}\), and \(0<h\le10^{-4}\).  In the actual
compatible simultaneous update BM4, the candidate keeps the
nonzero signs of \(q,b\), and
\[
 {\displaystyle \quad
 Z^+\le(1-h\Delta H/\rho^2)Z+
                    (5000h^2+20h\nu)V .
 \quad}
 \tag{SB18}
\]
Only the pre-state must be in SB10; the candidate can cross its
bias or energy boundary.  The function in SB7 is defined beyond
the corridor precisely for this assertion.

Here are explicit finite-step remainder bounds.  At \(B\le3/2\),
upper balance gives
\(\sum_i\int|H_{iv}|\le2\rho^2\),
\(|3t-Bt^2|\le9\), and
\(0\le(3+B)t-Bt^2\le9\).
Also \(|H_{iv}|/c_i^2\le2\).
Thus the common radial coefficient has magnitude at most 18,
each coordinate boost divided by its absolute coordinate is at
most 18, and all pure coordinate multipliers are at least
\(1-36h>0\).  Zero coordinates stay zero.
The pure head grows, since \(H>0\).
The pure bias multiplier is at least \(1-6h\), by its exact form
\(1+hH/b^2-h\sum\int H_{iv}/D_{iv}\).
The full kernel bounds \(|C|\le|b|\), \(|C|\le\|\theta\|\)
give \(|b|\ge\kappa p\).
The actual bias error is at most \(h\nu p\); the head error
is at most \(h\nu\|\theta\|\).
Consequently neither sign changes, including at the candidate.

Parameterize the straight pure hidden segment by \(u\in[0,h]\).
Because \(p/\rho<2\) and \(\|\nabla C\|\le1\), its spatial
radius is at least \((1-2h)\rho\) and its squared bias ratio
is below two.  Direct quotient differentiation yields
\[
 |B'(u)|\le16,\qquad |B''(u)|\le200.
 \tag{SB19}
\]
For detail, put \(N=(b+u\dot b)^2\) and
\(D=\|w+u\dot w\|^2\).  The bounds
\[
 N''/D,D''/D\le9,\quad |D'/D|\le5,\quad
 N/D\le2,\quad |N'/D|\le6
\]
give \(|B'|\le16\) and
\(|B''|\le9+60+18+100<200\).
Since \(|c'|\le2\), the scalar
\[
 s(u)=(1+u C_{\rm oriented}/p)\,F(B(u))/F(B(0))
 \tag{SB20}
\]
has \(|s'|\le10\) and \(|s''|\le1300\) on this segment.
Indeed
\[
 C_{\rm oriented}/p\le\sqrt{1+\varepsilon}<2,\qquad
 F(B(u))/F(B(0))\le e^{4h}<1.001;
\]
the differentiated bounds give \(|s'|<7\) and
\(|s''|<1300\).  Here the deliberately looser 10 is used.
The bounds also hold across \(B=1,3/2\) in the Lipschitz
derivative sense.

Because pure coordinate signs persist, the exact weighted defect
vector equals
\[
 s(h)\{z+hF(B)p\dot w_{\rm defect}\}.
\]
Its difference from \(z+h\dot z\) is at most \(1000h^2V\).
Moreover \(\|\dot z\|\le42Z\): the head, bias-weight, radial
and negative-coordinate terms are bounded respectively by
\(2Z,4Z,18Z,18Z\).
Using SB17 and \(\sqrt{1+x}\le1+x/2\) therefore gives the
pure part of SB18 with an error below \(2000h^2V\).
At \(Z=0\), the pure defect stays zero, so no division by zero
is involved in that argument.

Finally connect the pure and actual candidates by a straight
segment.  Throughout it \(\rho>\rho_{\rm old}/2\),
\(B<2\), and \(p\le2p_{\rm old}\).
The map \(F(B)p\,\operatorname{dist}(w,\mathrm{cone}_\tau)\)
is locally Lipschitz.  Its head derivative is at most \(2\rho\);
its hidden derivative norm is at most
\[
 p\{F+\rho |F'|\|\nabla B\|\}\le5p,
 \qquad \|\nabla B\|=2\sqrt{B+B^2}/\rho .
 \tag{SB21}
\]
The compatible displacements \(h\nu\|\theta_{\rm old}\|\)
and \(h\nu p_{\rm old}\) thus cost less than \(20h\nu V\).
This proves SB18 for all coordinate sign crossings caused by
the actual force.  The teacher and full student/residual terms
have not been separated into different dynamics.

\subsection{A discrete entry envelope for the complete bank}

We now prove the acquisition and switching bounds used by the full-bank
estimate. The actual candidate induction in
Appendix~\ref{app:mixture-acquisition}, with the public
\(M,T\), gives all four original categories per axis throughout
\(N_*\le n\le N\), with their same calibrated errors and
\(\rho>2R\).  Every row obeys
\(\sum_jV_{j,n}\le S\), the complete compatible force cap
\(\|E_{j,n}\|\le\nu\), and
\[
 0\le D_j:=\|\theta_j\|^2-q_j^2\le6 .
 \tag{SB27}
\]
These estimates are (BAJ17)--(BAJ19) in that appendix; their proofs
do not use the comparable-signal condition \eqref{eq:app-signals}.
They include each candidate and
recomputed endpoint force.

At a selected profile, balance and BAJ19 give
$t^2:=q^2/\|\theta\|^2\ge1-6/(4R^2)\ge125/128$, hence
$2t/(1+t^2)\ge1-10^{-4}$. Normalizing the actual and ideal augmented
hidden profiles costs at most $2e_*$; the central correlation on the
unit augmented sphere is $1$-Lipschitz by HK6. The balanced ideal
profile has normalized signal exactly $K_i^*$. It follows that
\[
 Q_{j,N_*}\ge(1-10^{-4})K_i^*-e_*
 \ge\{.9999(5/6)-1/40960\}K_*>.833K_* .
 \tag{SB28}
\]
This follows from \(q^2/\|\theta\|^2\ge125/128\) and
\(2t/(1+t^2)\ge1-10^{-4}\), not an assumed balanced head.
The exact quotient debit \(Q^+\ge Q-h\nu^2\) and \eqref{eq:sw-mesh}--\eqref{eq:sw-force}
keep every selected \(Q_j>.832K_*\).
The exact energy lower multiplier
\(V^+\ge[1+h(4Q-2\nu)]V\) then gives
\[
 V_j(N_*+l)\ge v_*e^{3.32K_*lh}.
 \tag{SB29}
\]
Indeed its multiplier is at least \(1+3.3278hK_*\), and
\(hK_*\le1/40000\) pays the logarithmic quadratic debit.

Mark each original row at its first suffix checkpoint
\(Q_j\ge.826K_*\).  Labels do not affect training.
Thereafter \(Q_j\ge.825K_*\), so SB4--6 imply \(B_j<29/20\).
Its head and bias signs persist: upper balance alone, the
component bounds used in SB18, and \(\nu\le\kappa=.825K_*\)
give the same sign margins.  Before marking the exact energy
upper bound gives, including the first marking candidate,
\[
 \log(V_j^+/V_j)
 \le4hQ_j+2h\nu+4h^2<3.31hK_* .
 \tag{SB30}
\]
Every marked row has increasing \(V_j\).  Once \(V_j\ge480\),
SB27 implies
\[
 \frac{q_j^2}{\|\theta_j\|^2}
  =\frac{V_j-D_j}{V_j+D_j}
  \ge\frac{474}{486}>\frac{20}{21}.
 \tag{SB31}
\]
Hence SB10 holds with \(a=29/20\) and
\(\varepsilon=\Delta=1/20\).
SB18 applies at every ensuing pre-state, including candidate
exits, and its optional contraction can be dropped. Write
$\omega=5000h^2+20h\nu$ and $b_0=3.31K_*$.
For suffix indices $l\ge0$ (thus $V_{j,l}=V_{j,N_*+l}$), define
\[
 A_{j,l}:=(3V_{j,0}/4+400)e^{b_0lh}
                  +\omega\sum_{q=0}^{l-1}V_{j,q}.
\]
We claim $X_{j,l}\le A_{j,l}$ for every row; after marking and
reaching $480$, the stronger invariant is $Z_{j,l}\le A_{j,l}$.
By SB8, $X\le Z\le3X/2$ on marked states. This is an entry-aware discrete estimate:
\begin{itemize}
\item Until marking, including its first candidate, SB30 gives
$V_{j,l}\le V_{j,0}e^{b_0lh}$ and $X_{j,l}\le V_{j,l}/2$.
A row marked directly at energy at least $480$ therefore enters
with $Z\le3V/4\le A_{j,l}$; this also covers marking at $l=0$.
\item A marked row below $480$ has $Z\le3V/4<360$.
Its first candidate at or above $480$ has
$V^+\le480(1+2h)^2<481$ and $Z^+<400\le A_{j,l+1}$.
\item Once active, increasing energy and persistent marking keep
SB18 applicable, so $Z_{j,l+1}\le Z_{j,l}+\omega V_{j,l}$.
But $A_{j,l+1}\ge A_{j,l}+\omega V_{j,l}$, closing the invariant.
\end{itemize}
Never-marked and zero rows are covered by the first case; a row
that remains marked below $480$ is covered by the second.
Thus no marking-time or threshold-crossing contribution is omitted.
Summing $X\le A$, using $\sum_jV_{j,0}\le S_0$ and the full-bank
bound $\sum_jV_{j,q}\le S$, gives
\begin{gather}
 \Xi(N_*+l):=\sum_jX_j(N_*+l)\notag\\
 \le(3S_0/4+400m)e^{3.31K_*lh}
                      +(5000h+20\nu)TS,\notag\\
 {\displaystyle \quad \Xi_N/v_G
 \le\Gamma e^{-.01K_*t_G}
                  +(5000h+20\nu)TS/v_*\le c_G.\quad}
 \tag{SB32}
\end{gather}
Before marking \(X_j\) uses its instantaneous orientation;
always \(X_j\le V_j/2\).  Thus no initial or later sign event
has been assumed for an adverse row.

\section{The complete symmetric derivative frame}\label{app:mixture-frame}
Here $X\sim N(0,I_d)$, $U=(u_1,\ldots,u_r)$, $P_U=UU^T$,
$P_{U^\perp}=I_d-P_U$, $J=1-\alpha$, and
$\sigma_\alpha(t)=\alpha t+Jt_+$. Use normalized rows
$(A_j,W_j,B_j)=s(q_j,w_j,b_j)$, $\rho_j=\|w_j\|$, and the complete
student $f=(s^2/m)\sum_jq_j\sigma_\alpha(w_j^TX+b_j)$.
The density $\varphi$ and scalar $G$ below are standard normal;
$\nabla_x f$ differentiates the input.

For any orientation $\tau\in\{-1,1\}$ define
$\mathcal E_\tau(w)=\|P_{U^\perp}w\|^2+
\sum_i\min\{\tau u_i^Tw,0\}^2$ and
$\Xi=\sum_j|q_j|\sqrt{\mathcal E_{\tau_j}(w_j)}$,
where $\tau_j=-\operatorname{sign}(q_jb_j)$ when $q_jb_j\ne0$
and either orientation otherwise. This is the same defect as SB9.

\subsection{A symmetric signed frame from fixed quadratic tests}

This subsection is deterministic at any one actual bank, with no
corridor hypothesis. Put $Z=U^TX$, $e=r^{-1/2}(1,\ldots,1)^T$,
and use the $r$ centered Gaussian quadratic tests
\[
 \xi_i(X)=(e^TZ)Z_i-(e)_i,\qquad
 T:=\mathbb E[\xi\xi^T]=I_r+ee^T,\qquad I_r\preceq T\preceq2I_r.
 \tag{BF21}
\]
In the profile comparisons below $e_i$ denotes the $i$th coordinate
basis vector, whereas $(e)_i$ is the scalar component in (BF21).
For each row with $\rho_j>0$ define
\[
 x_j=\tau_jU^Tw_j/\rho_j,\quad z_j=b_j/\rho_j,\qquad
 \omega_j=\frac{Js^2}{m}|q_j\rho_jz_j|\varphi(z_j)\ge0.
 \tag{BF22}
\]
If $q_jb_j=0$, set $\omega_j=0$ and choose either orientation.
If $\rho_j=0$, the spatial gradient contribution is zero and all
formulas below assign the row zero contribution.
The \emph{complete} current-head frame is exactly symmetric:
\[
 {\displaystyle \quad
 D:=\mathbb E[(U^T\nabla_x f)\xi^T]
       =\sum_{j=1}^m\omega_j(e^Tx_j)x_jx_j^T.
 \quad}
 \tag{BF23}
\]
Indeed Gaussian regression onto $w_j^TX/\rho_j$ and
$\mathbb E[\mathbf1_{G>-z}(G^2-1)]=-z\varphi(z)$ give
\[
 \mathbb E[\sigma'_\alpha(w_j^TX+b_j)\xi_i]
 =-Jz_j\varphi(z_j)(U^Tw_j/\rho_j)_i
                             e^T(U^Tw_j/\rho_j).
\]
Multiplication by the actual gradient row
$(s^2/m)q_jw_j$ proves BF23; the constant derivative $\alpha$
has zero correlation with each centered quadratic test. It remains
present in the full AGOP and in its outside estimate below.

Split BF23 by the sign of $e^Tx_j$ as $D=D_+-D_-$, where both
$D_+,D_-$ are positive semidefinite. Every signed original row is
in exactly one sum, with zero summands harmless. Then
\[
 {\displaystyle \qquad
 \|D_-\|_{\rm op}\le\frac{J\varphi(1)s^2}{m}\,\Xi,
 \qquad\Xi=\sum_j |q_j|\sqrt{\mathcal E_{\tau_j}(w_j)}.
 \qquad}
 \tag{BF24}
\]
To see this, $(-e^Tx)_+\le\|x_-\|$, $\|x\|\le1$, and
$|z|\varphi(z)\le\varphi(1)$. Thus each negative summand has
operator norm at most
$(J\varphi(1)s^2/m)|q_j|\rho_j\|(x_j)_-\|$,
which is bounded by its BF24 debit. No angular separation, common
head sign or axis assignment of the other rows is needed.

For an explicit chosen-profile interface, suppose there are $g$
disjoint groups of $r$ actual rows, one row for every $i$ in each
group, obeying
\[
 \|x_{\ell i}-e_i\|\le\epsilon,\quad 0\le\epsilon<r^{-1/2},\qquad
 \omega_{\ell i}\ge\omega_0>0,\quad1\le\ell\le g.
 \tag{BF25}
\]
If the scalar on the right is positive, then
\[
 \lambda_{\min}(D)
 \ge\Lambda:=
 g\omega_0(r^{-1/2}-\epsilon)(1-\sqrt r\epsilon)^2
                         -\frac{J\varphi(1)s^2}{m}\Xi>0.
 \tag{BF26}
\]
For each group the column matrix of its $x_{\ell i}$ differs from
$I_r$ in operator norm by at most $\sqrt r\epsilon$. All its
$e^Tx_{\ell i}$ are at least $r^{-1/2}-\epsilon$.
Consequently its positive semidefinite contribution has least
eigenvalue at least
$\omega_0(r^{-1/2}-\epsilon)(1-\sqrt r\epsilon)^2$.
All other positive summands only increase $D_+$, and BF24 pays all
negative ones. This proves BF26 for the full bank.

\subsection{Full AGOP, outside gradient, and a conditional spectral join}

\begin{lemma}[Frame-to-full-AGOP comparison]\label{lem:frame-agop}
Let $U\in\mathbb R^{d\times r}$ have orthonormal columns, $1\le r\le d$,
and let $H$ be a square-integrable $\mathbb R^d$-valued random vector. A centered test
vector $\xi$ on the same probability space has finite dimension and satisfies $0\prec C:=\mathbb E[\xi\xi^T]\preceq cI$.
Put $\mathsf G=\mathbb E[HH^T]$, $D=\mathbb E[(U^TH)\xi^T]$ and
$P_{U^\perp}=I-UU^T$. Suppose, for $\gamma>0$,
\[
 DD^T\succeq\gamma^2I_r,\qquad
 \operatorname{tr}(P_{U^\perp}\mathsf G)\le\beta,\qquad a:=\gamma^2/c.
\]
Then $U^T\mathsf GU\succeq DC^{-1}D^T\succeq aI_r$ and
$\lambda_r(\mathsf G)\ge a$; for $r<d$, $\lambda_{r+1}(\mathsf G)\le\beta$.
If $\beta<a$, the top-$r$ projector $P$ is unique and
\[
 P\preceq\mathsf G/a,\qquad
 1-\lambda_{\min}(U^TPU)\le\beta/a,\qquad
 1-r^{-1}\operatorname{tr}(U^TPU)\le\beta/(ra).
\]
If $r<d$ and $\beta<a$, the cutoff gap is at least $a-\beta>0$.
For $r=d$, $P=I_d$ and both angle deficits vanish without the condition $\beta<a$.
\end{lemma}
\begin{proof}
Expanding $\mathbb E[(U^TH-DC^{-1}\xi)(U^TH-DC^{-1}\xi)^T]\succeq0$
and using $C^{-1}\succeq c^{-1}I$ gives the inside bound.
Minmax gives $\lambda_r\ge a$ and, when
$r<d$, $\lambda_{r+1}\le\|P_{U^\perp}\mathsf GP_{U^\perp}\|
\le\operatorname{tr}(P_{U^\perp}\mathsf G)\le\beta$.
The gap (or $r=d$) makes $P$ unique; spectral decomposition gives
$P\preceq\mathsf G/a$. Finally $0\preceq I_r-U^TPU$ and
\[
 \operatorname{tr}(I_r-U^TPU)
 =\operatorname{tr}(P_{U^\perp}P)\le\beta/a.
\]
Its largest eigenvalue is at most its trace, proving both angle bounds.
All comparisons use the full $\mathsf G$, including its cross blocks.
\end{proof}

Let $\mathsf G=\mathbb E[\nabla f\nabla f^T]$ for the actual
current heads. The exact outside derivative is
\[
 P_{U^\perp}\nabla f=\frac{s^2}{m}
       \sum_jq_jP_{U^\perp}w_j\sigma'_\alpha(w_j^TX+b_j).
\]
The $L^2$ triangle inequality, retaining all cross terms, gives
\[
 {\displaystyle \quad
 \operatorname{tr}(P_{U^\perp}\mathsf GP_{U^\perp})
 \le\beta:=\frac{s^4}{m^2}\Xi^2.
 \quad}
 \tag{BF27}
\]
This is the complete outside \emph{trace}, retaining every trained row.
Apply Lemma~\ref{lem:frame-agop} with $H=\nabla f$, $C=T$ from BF21,
$c=2$ and $\gamma=\Lambda$ from BF26. Its inside comparison is
\[
 U^T\mathsf GU\succeq DT^{-1}D^T\succeq\tfrac12DD^T.
 \tag{BF28}
\]
Thus $\lambda_r(\mathsf G)\ge a:=\Lambda^2/2$. If $r<d$ and
$\beta<a$, the same lemma gives $\lambda_{r+1}(\mathsf G)\le\beta$ and
\[
 \lambda_r-\lambda_{r+1}\ge a-\beta>0,\qquad
 \lambda_{\min}(U^TPU)\ge1-\beta/a,\qquad
 r^{-1}\operatorname{tr}(U^TPU)\ge1-\beta/(ra).
 \tag{BF29}
\]
For $r=d$, $P=I_d$ and the angle assertions are automatic.

\section{A bounded readout for the full mixture links}\label{app:mixture-prediction}

\subsection{The full mixture and its exact scalar projection}
Let $\varphi_v$ be the $N(0,v)$ density and $\varphi=\varphi_1$.
Here $Z\sim N(0,1)$, $\Phi$ is its CDF,
$h_k=\operatorname{He}_k/\sqrt{k!}$, and scalar inner products and
norms are in $L^2(N(0,1))$. The activation is
$\sigma_\alpha(t)=\alpha t+(1-\alpha)t_+$, $0\le\alpha<1$, and
$J=1-\alpha>0$.
For any probability measure $\mu$ on $[1/3,1]$, define
\begin{gather}
 g_v\varphi=-\varphi_v''',\qquad
 g_\mu=\int g_v\,d\mu(v),\quad N_\mu=\|g_\mu\|_2,
 \quad y_\mu=g_\mu/N_\mu,\notag\\
 g_v(t)=v^{-7/2}(t^3-3vt)e^{-(v^{-1}-1)t^2/2}.
 \tag{MP1}
\end{gather}
These are the full links in \eqref{eq:app-mixture}. They are odd, have zero first Gaussian
moment, and belong to Gaussian $H^1$. For completeness their coefficients
and exact norm kernel are
\begin{gather}
 \langle g_\mu,h_{2j+3}\rangle
 =\frac{(-1)^j\sqrt{(2j+3)!}}{2^jj!}
                  \int(1-v)^j\,d\mu(v),\notag\\
 K(t)=\frac{3(2+3t)}{(1-t)^{7/2}},\qquad
 N_\mu^2=\iint K((1-v)(1-w))\,d\mu(v)d\mu(w),\notag\\
 N_v^2=\frac{3\{2+3(1-v)^2\}}{[v(2-v)]^{7/2}}.
 \tag{MP2}
\end{gather}
The generating function
$\mathbb E[g_v(Z)e^{tZ-t^2/2}]=t^3e^{-(1-v)t^2/2}$ gives the
coefficients. Squaring and summing the absolutely convergent coefficient
series gives MP2. The same series weighted by Hermite degree is summable
uniformly on this interval, so the full $H^1$ claim and integration in
$\mu$ are justified. No truncation is used.

For $B=b^2>0$, set
\begin{gather}
 p_b=2\Phi(b)-1,\qquad
 q_b(z)=\operatorname{clip}(z,-b,b)-p_bz,\notag\\
 V_b=\|q_b\|_2^2
 =p_b-2b\varphi(b)+B(1-p_b)-p_b^2,\notag\\
 L_\mu(b)=2b\int v^{-3/2}\varphi(b/\sqrt v)\,d\mu(v)>0.
 \tag{MP3}
\end{gather}
The function $q_b$ is odd and orthogonal to $Z$. Integrating MP1 by
parts against the clipped linear function gives the exact full-link
identity
\[
 \langle g_\mu,q_b\rangle=-L_\mu(b),\qquad
 E_b(y_\mu)=\frac{L_\mu(b)^2}{N_\mu^2V_b}.
 \tag{MP4}
\]
Here $E_b$ is the full nonconstant projection energy in the four-profile
span $\operatorname{span}\{\sigma_\alpha(\epsilon z+\tau b):
\epsilon,\tau\in\{-1,1\}\}$. Directly, that span is generated by
$1,Z,q_b$ and the centered even function
$\max(|Z|,b)-\mathbb E[\max(|Z|,b)]$. Oddness and zero first moment
leave only $q_b$ for MP1. In particular the unexplained infinite tail
has exactly squared norm $1-E_b(y_\mu)$; it is not dropped.

\paragraph{Uniform full-link prediction lemma.}
For every such $\mu$ and every
\[
 B\in[B_-,B_+],\qquad
 B_-=(\sqrt{7/3}-1)/2,\quad B_+=(\sqrt5-1)/2,
 \tag{MP5}
\]
one has $E_b(y_\mu)>2/5$. This includes every link-dependent
stationary bias supplied by the mixture acquisition interface, without
using its stationarity equation.

\paragraph{Proof.}
Write $L_v=2bv^{-3/2}\varphi(b/\sqrt v)$. Minkowski's inequality and
positivity imply
\[
 \frac{\int L_v\,d\mu}{N_\mu}
 \ge\frac{\int L_v\,d\mu}{\int N_v\,d\mu}
 \ge\inf_{v\in[1/3,1]}\frac{L_v}{N_v}.
 \tag{MP6}
\]
For fixed $B$, the logarithm of $L_v^2/N_v^2$, up to a constant, is
\[
 f_B(v)=\tfrac12\log v+\tfrac72\log(2-v)-B/v
              -\log\{2+3(1-v)^2\}.
\]
Its second derivative is
\[
 f_B''(v)=-\frac1{2v^2}-\frac7{2(2-v)^2}-\frac{2B}{v^3}
 +\frac{-12+18(1-v)^2}{\{2+3(1-v)^2\}^2}<0.
 \tag{MP7}
\]
Thus its minimum over the interval is at one of the two endpoints.
The explicit rational scalar certificate proved below gives
\begin{gather}
 39/1000<V_b<13/250,\qquad p_b<57/100,\notag\\
 B e^{-B}>1/5,\qquad B e^{-3B}>967/10000,\qquad
 N_{1/3}^2<79,\qquad \pi<22/7.
 \tag{MP8}
\end{gather}
Since $N_1^2=6$, the two endpoint ratios obey
\begin{gather}
 \frac{L_1^2}{N_1^2V_b}
 =\frac{Be^{-B}}{3\pi V_b}>\frac{175}{429}>\frac25,\notag\\
 \frac{L_{1/3}^2}{N_{1/3}^2V_b}
 =\frac{54Be^{-3B}}{\pi N_{1/3}^2V_b}
 >\frac{182763}{451880}>\frac25.
 \tag{MP9}
\end{gather}
Equations MP6--9 prove the claim for every positive measure, including
non-atomic measures. They use the full norm MP2 and full correlation MP4.

\subsection{A finite exact-rational certificate for MP8}
The following gives a finite rational certificate. For $n\ge0$ define
\[
 T_n(x)=\sum_{j=0}^n\frac{(-x)^j}{j!},\qquad
 I_n(b)=\sum_{j=0}^n\frac{(-1)^jb^{2j+1}}{2^jj!(2j+1)}.
\]
Taylor's theorem, followed by integration, gives
$T_{2k+1}(x)\le e^{-x}\le T_{2k}(x)$ for $x\ge0$ and
$I_{2k+1}(b)\le\int_0^b e^{-t^2/2}dt\le I_{2k}(b)$.
Put
\[
 c_-=3989422804/10^{10},\qquad c_+=3989422805/10^{10}.
\]
The alternating series for
$\pi=16\arctan(1/5)-4\arctan(1/239)$, through the first sixteen
terms and with the next term as remainder bound, verifies
$c_-<(2\pi)^{-1/2}<c_+$ by squaring positive rational endpoints.
Define rational enclosures
\[
 P_-(b)=2c_-I_9(b),\quad P_+(b)=2c_+I_{10}(b),\quad
 F_-(b)=c_-T_9(b^2/2),\quad F_+(b)=c_+T_{10}(b^2/2).
\]
They bound $p_b$ and $\varphi(b)$ on $[.51,.79]$.
Since $p_b$ increases and both $p_b/b$ and $\varphi(b)$ decrease,
\[
 \frac{dV_b}{dB}=1-p_b-2(p_b/b)\varphi(b)
 \ge 1-P_+((k+1)/100)
       -\frac{200}{k}P_+(k/100)F_+(k/100)>3/1000
 \tag{MP10}
\]
whenever $b\in[k/100,(k+1)/100]$, for each of the twenty-eight
integers $k=51,\ldots,78$. The last inequality is checked by
substitution of the displayed rational polynomials. Therefore $V_b$
is increasing on $[.51,.79]$. At a rational $b<1$ use
\begin{gather}
 b^2+(1-b^2)P_-(b)-P_+(b)^2-2bF_+(b)\le V_b,\notag\\
 V_b\le b^2+(1-b^2)P_+(b)-P_-(b)^2-2bF_-(b).
 \tag{MP11}
\end{gather}
At $b=.51$ the lower expression exceeds $39/1000$, and at
$b=.787$ the upper expression is less than $13/250$; also
$P_+(.787)<57/100$. These are three further rational checks.

The exact MP5 interval is contained in
$[B_l,B_u]=[2637/10000,6181/10000]$, itself contained in
$(.51^2,.787^2)$. This follows by substituting the endpoints into
the increasing polynomial $B+B^2$. On this interval $Be^{-B}$
is increasing and $Be^{-3B}$ has only one critical point, a maximum.
The three rational checks
\[
 B_lT_{13}(B_l)>1/5,\qquad
 B_lT_{13}(3B_l)>967/10000,\qquad
 B_uT_{13}(3B_u)>967/10000
\]
prove the exponential parts of MP8. Finally
$N_{1/3}^2=10(9/5)^{7/2}<79$ follows from
$100(9/5)^7<79^2$. This proves every scalar inequality used above
by a finite, reproducible exact calculation.

\subsection{One exact four-profile identity}
\begin{lemma}[Four biased profiles]\label{lem:four-profile-identity}
Let $0\le\alpha<1$, $J=1-\alpha$, $b>0$ and $0\le p\le1$. Set
\[
 c_{\epsilon,\tau}(p)=\frac\epsilon2
       \left(\frac\tau J-\frac p{1+\alpha}\right),
 \qquad (\epsilon,\tau)\in\{-1,1\}^2.
\]
Then, for every $z\in\mathbb R$,
\begin{gather*}
 \sum_{\epsilon,\tau}c_{\epsilon,\tau}(p)
       \sigma_\alpha(\epsilon z+\tau b)
       =\operatorname{clip}(z,-b,b)-pz,\\
 \|c(p)\|_2^2=J^{-2}+p^2/(1+\alpha)^2,
 \qquad \|c(p)\|_1=2/J.
\end{gather*}
\end{lemma}
\begin{proof}
Write $\sigma_\alpha(t)=((1+\alpha)t+J|t|)/2$.
Since $|z+b|-|z-b|=2\operatorname{clip}(z,-b,b)$,
\[
 D_\tau(z):=\sigma_\alpha(z+\tau b)-\sigma_\alpha(-z+\tau b)
 =(1+\alpha)z+\tau J\operatorname{clip}(z,-b,b).
\]
Multiplying by $(\tau/J-p/(1+\alpha))/2$ and summing over $\tau$
proves the representation. Squaring the four coefficients cancels
the mixed terms and gives the $\ell^2$ identity. Finally
$0\le p/(1+\alpha)\le J^{-1}$, so summing their absolute values
gives $\|c(p)\|_1=(J^{-1}-p/(1+\alpha))
 +(J^{-1}+p/(1+\alpha))=2/J$.
\end{proof}

\subsection{Normalized, unprojected mixture application}
Let $X\sim N(0,I_d)$ and $U=(u_1,\ldots,u_r)$ have orthonormal columns, choose arbitrary
probability measures $\mu_i$ on $[1/3,1]$, and define the full central
teacher
\[
 y_0(X)=\sum_{i=1}^r a_i y_{\mu_i}(u_i^TX),\qquad
 \sum_i a_i^2=1.
 \tag{MP12}
\]
Signs are unrestricted. This teacher is unit norm, centered, and has
Hermite rank at least three. If every $a_i\ne0$, its true minimal
linear index space is $\operatorname{span}(U)$: the nonzero degree-three
tensor has contraction range equal to that space. Zero coefficients
may instead be omitted from the definition of rank.
Allow a full actual unit target
\[
 y=y_0+e_y,\qquad \|y\|_2=1,\quad
 \|e_y\|_{H^1(\gamma_d)}\le\zeta.
 \tag{MP13}
\]
Every mean, lower Hermite term and interaction in $e_y$ is retained and
paid below. No stronger index identification for the perturbed target
is inferred without an additional same-index assumption.

Choose any $b_i^2$ in MP5 and put
\begin{gather}
 t_i=-\frac{L_{\mu_i}(b_i)}{N_{\mu_i}V_{b_i}},\qquad
 F_0(X)=\sum_i a_it_iq_{b_i}(u_i^TX),\notag\\
 E=\|F_0\|_2^2=\sum_i a_i^2E_{b_i}(y_{\mu_i})>2/5,\qquad
 \|y_0-F_0\|_2^2=1-E<3/5.
 \tag{MP14}
\end{gather}
Independence of the Gaussian coordinates makes the scalar projection
energies add exactly. Applying Lemma~\ref{lem:four-profile-identity}
with $p=p_{b_i}$ gives the exact representation of $F_0$ with
\[
 d_{i,\epsilon,\tau}=a_it_i c_{\epsilon,\tau}(p_{b_i}),\qquad
 F_0(X)=\sum_{i,\epsilon,\tau}d_{i,\epsilon,\tau}
            \sigma_\alpha(\epsilon u_i^TX+\tau b_i).
 \tag{MP15}
\]
Set $C_2=\|d\|_2$, $C_1=\|d\|_1$. Cauchy--Schwarz gives
$|t_i|\le V_{b_i}^{-1/2}$, so MP8 proves
\begin{gather}
 C_2^2=\sum_i a_i^2t_i^2
       \left(J^{-2}+\frac{p_{b_i}^2}{(1+\alpha)^2}\right)
 \le\frac{1000}{39J^2}\left(1+\frac{57^2}{100^2}\right)
 <\frac{36}{J^2},\notag\\
 C_1=\frac2J\sum_i|a_it_i|\le2\sqrt r C_2<12\sqrt r/J.
 \tag{MP16}
\end{gather}
Thus rank and leak are explicitly charged; there is no uniform bound
as $\alpha\uparrow1$.

Suppose at one actual common checkpoint $N$ there are $4r$ distinct
original rows with $\rho_j=\|W_{j,N}\|>0$ and
\[
 \left\|\left(\frac{W_{j,N}}{\rho_j},\frac{B_{j,N}}{\rho_j}\right)
            -(\epsilon u_i,\tau b_i)\right\|\le\varepsilon,
 \qquad 0\le\varepsilon\le1/100.
 \tag{MP17}
\]
The assertion MP17 must come from the full actual coupled trajectory;
it is not implied by initial category counts or by aggregate alignment.
Define, at both endpoints and using every original row,
\begin{gather}
 H^{\rm aug}_{j,n}=\sqrt{\|W_{j,n}\|^2+B_{j,n}^2},\qquad
 \phi_{j,n}(X)=\frac{\sigma_\alpha(W_{j,n}^TX+B_{j,n})}{H^{\rm aug}_{j,n}},\notag\\
 \mathcal R(n)=\inf_{\substack{\|v\|_2\le32/J\\
                    \|v\|_1\le64\sqrt r/J}}
             \left\|y-\sum_{j=1}^mv_j\phi_{j,n}\right\|_2^2.
 \tag{MP18}
\end{gather}
A zero augmented row is assigned the zero feature. The normalization,
identity spatial projection, row set, and both budgets agree exactly
at $n=0,N$. Actual training heads remain untouched.

On the selected rows take
$v_j=d_{i,\epsilon,\tau}\sqrt{1+(B_{j,N}/\rho_j)^2}$, and put
zero diagnostic coefficients on other rows. Since
$|B_{j,N}/\rho_j|\le.787+.01$ and
$\sqrt{1+(.787+.01)^2}<1.3$, MP16 shows that both MP18 budgets hold.
Positive homogeneity gives the exact radius-canceling identity
\[
 v_j\phi_{j,N}=d_{i,\epsilon,\tau}
        \sigma_\alpha((W_{j,N}/\rho_j)^TX+B_{j,N}/\rho_j).
\]
The activation is one-Lipschitz. Gaussian affine second moments
therefore show that its sum $F_N$ obeys
\[
 \|F_N-F_0\|_2\le C_1\varepsilon
 \le\frac{12\sqrt r}{J}\varepsilon=:\Delta,
 \qquad
 \mathcal R(N)\le(\sqrt{1-E}+\zeta+\Delta)^2.
 \tag{MP19}
\]
This is an original-loss diagnostic retaining the full target MP13.
The entire MP18 class has $L^2$ norm, Lipschitz norm and affine
variation cost at most $64\sqrt r/J$. The witness has
$\|F_N\|_2\le\sqrt E+\Delta$ and Lipschitz norm at most $C_1$.
None of these budgets depend on initialization scale or learned radii.
No feature-Gram conditioning claim is required or implied.

\section{Acquisition of full-bank signal before loss release}\label{app:mixture-release}

Use the teacher in \eqref{eq:app-mixture}, writing $a_i=\lambda_i$.
Its full central kernel $C$ is defined in (BM5) and evaluated in (BM23)
of Appendix~\ref{app:mixture-acquisition}, and its normalized
$H^1$ norm is less than $13$. The residual is the full $e=y-y_c$.
This section proves a deterministic implication from one actual signal
row. In Theorem~\ref{app:sw-theorem} that premise is supplied by the
preceding acquisition proof, with $(\gamma,U,v)=(K_*/2,S,v_*)$.

\subsection{Actual simultaneous update and two complete-row inequalities}
Here $X\sim N(0,I_d)$, $0\le\alpha<1$, $\theta_j=(w_j,b_j)$,
$\sigma_\alpha(t)=\alpha t+(1-\alpha)t_+$, and
$\sigma'_{\alpha,j}=\sigma'_\alpha(\theta_j^T(X,1))$ is the scalar gate.
The scalar $U$ used in the energy premise (MR11) is a public energy
ceiling; teacher axes remain $u_i$.

Train the averaged network with the one unchanged full-MSE step
\[
 \begin{gathered}
 f_n=m^{-1}\sum_j A_{j,n}\sigma_\alpha(W_{j,n}^TX+B_{j,n}),
 \\
 (A_j,W_j,B_j)=s(q_j,w_j,b_j),\quad
 L_n=\|y-f_n\|_2^2,\quad \eta=mh/2,
 \quad 0<h\le1/1024.
 \end{gathered}
 \tag{MR6}
\]
These are analytic normalized coordinates only. The exact row update is
\[
 q^+=q+h(C+\theta^TE),\quad
 \theta^+=\theta+hq(T+E),\quad T=\nabla C,
 \quad E_j=\mathbb E[(e-f_n)\sigma'_{\alpha,j}(X,1)].
 \tag{MR7}
\]
The gate in the last expression multiplies the augmented vector.
Write $z=(q,\theta)$, $H=qC$, $V=\|z\|^2$, $Q=H/V$,
$F=\nabla H=(C,qT)$ and $p=(\theta^TE,qE)$, so $z^+=z+h(F+p)$.
The full-H1 kernel bounds give
\[
 |C|\le\|\theta\|,\quad \|T\|\le1,\quad
 \|D^2C\|\le52/\|\theta\|,\quad
 |Q|\le\tfrac12,\quad \|F\|\le\sqrt V.
 \tag{MR8}
\]
The mean/linear-free central target has $C=T=0$ at $w=0,b\ne0$;
regularity there is included in Appendix~\ref{app:kernels}.
Identically zero joint rows stay zero and can have $Q$ defined as zero.
They are still present in the network and all sums.

\begin{lemma}[Compatible-force row estimates]\label{mr:compatible-force}
Under the kernel bounds (MR8), at any balanced nonzero row
$|q|\le\|\theta\|$, with
$\|E\|\le\nu\le1$, the exact update preserves the balance and
nonzero hidden state. In fact, for $D=\|\theta\|^2-q^2$,
\[
 D^+\ge(1-h^2\|T+E\|^2)D,\quad
 \|\theta^+-\theta\|\le2h\|\theta\|,\quad
 V^+\le(1+2h)^2V.
 \tag{MR9}
\]
The two useful inequalities, valid with every possible row sign, are
\[
 {\displaystyle \quad Q^+\ge Q-h\nu^2,\qquad
 \Delta(V-4hH)\le4hH+3h\nu V.\quad}
 \tag{MR10}
\]
\end{lemma}

\begin{proof}
One has $\|p\|\le\nu\sqrt V$ and $|z^Tp|\le\nu V$.
Along the actual joining segment the head/hidden ratio is less than two,
so $\|D^2H\|\le105$. Taylor's formula and
$F^Tp\ge-\|F\|^2/4-\nu^2V$ yield
\[
 \Delta H\ge h(3\|F\|^2/4-\nu^2V)
                 -105h^2(\|F\|^2+\nu^2V)
 \ge h\|F\|^2/2-2h\nu^2V.
\]
The exact energy expansion has upper bound
$\Delta V\le4hH+2h\nu V+2h^2\|F\|^2+2h^2\nu^2V$.
Subtracting $4h\Delta H$ leaves at most
$4hH+2h\nu V+10h^2\nu^2V\le4hH+3h\nu V$.

For the normalized inequality put $a=z^Tp/V$,
$F_\perp=F-2Qz$, $p_\perp=p-az$,
$\tau=h/[1+h(2Q+a)]$. Homogeneity gives
$Q(z^+)=Q(z+\tau(F_\perp+p_\perp))$ and
$\tau\le h/(1-2h)\le4h/3$.
The latter segment is tangent to $z$, has norm at least $\sqrt V$,
and its head/hidden ratio is less than two: each block moves by at most
$4\tau\|\theta\|$. The quotient-Hessian bound from the full-H1
calculation is $\|D^2Q\|\le112/V$ along this segment.
With $\beta=\|F_\perp\|/\sqrt V\le1$,
\[
 \Delta Q\ge\tau(\beta^2-\nu\beta)
                  -56\tau^2(\beta+\nu)^2
 \ge\tau(7\beta^2/8-\nu\beta-\nu^2/8)
 \ge-5\tau\nu^2/8\ge-h\nu^2.
\]
Here $112\tau\le112/1022<1/8$ and
$\nu\beta\le(\beta^2+\nu^2)/2$ suffice.
No lower correlation detector or spatial radius appears in MR10.
\end{proof}

\subsection{A public suffix that pays all unmarked energy}

The precise acquisition obligation is this: on an explicitly specified
event of the original draw, an actual integer $N$ obeys
\[
\begin{gathered}
\max_{0\le n\le N}\sum_jV_{j,n}\le U,\qquad
 |q_{j,N}|\le\|\theta_{j,N}\|\quad(j\le m),\\
 V_{j_*,N}\ge v>0,\quad Q_{j_*,N}\ge\gamma,
 \quad 0<\gamma\le\tfrac12,
\end{gathered}
 \tag{MR11}
\]
where $U\ge v$ and $\gamma$ are public constants chosen before the draw.
Nonzero hidden states are also part of entry, or follow from preceding
balanced actual dynamics. The original independent Gaussian law is
$(A_{j,0},W_{j,0},B_{j,0})\sim N(0,s^2/d)$ entrywise; no row is
screened, replaced or sign-flipped. The deterministic implication here
assumes (MR11); Appendix~\ref{app:mixture-acquisition} supplies the
original-law acquisition event used in Theorem~\ref{app:sw-theorem}.

Choose all following parameters before that original draw:
\[
 \begin{gathered}
 k=\left\lceil\frac2{h\gamma}
                      \log\frac{32U}{\gamma v}\right\rceil,
 \quad T=h+\frac2\gamma\log\frac{32U}{\gamma v},\quad
 \overline U=Ue^{4T},\\
 0<\nu\le\min\left\{\frac\gamma{16},
                                \sqrt{\frac\gamma{8T}}\right\},
 \quad \zeta\le\nu/2,
 \quad s^2\le\frac{m\nu}{\overline U}.
 \end{gathered}
 \tag{MR12}
\]
This adds $k$ actual steps; $kh\le T$.
The full force cap $\|E_j\|\le\nu$ holds at every suffix checkpoint,
including its endpoint, by induction. Indeed
\[
 \|E_j\|\le\zeta+\|f\|_2
       \le\zeta+\frac{s^2}{2m}\sum_lV_l,
 \qquad\sum_jV_{j,N+l}\le U(1+2h)^{2l}
                                  \le Ue^{4lh}\le\overline U.
 \tag{MR13}
\]
The initial force follows from $U$ and MR12; MR9 then proves each
candidate energy bound, and recomputing the force closes the induction.
Every head and hidden row remains trained and balanced.

Mark a row at its first suffix checkpoint with $Q\ge\gamma/4$,
including the entry checkpoint. This is only an analytic label.
By MR10 and $\nu^2kh\le\gamma/8$, every marked endpoint has
$Q\ge\gamma/8$. The original signal row obeys throughout the suffix
$Q_{j_*}\ge7\gamma/8\ge3\gamma/4$. Its exact energy increment gives
\[
 V_{j_*}^+\ge[1+h(4Q_{j_*}-2\nu)]V_{j_*}
              \ge(1+23h\gamma/8)V_{j_*},\qquad
 V_{j_*,N+l}\ge v e^{2\gamma lh}.
 \tag{MR14}
\]
For the last inequality, if $x=h\gamma\le1/2048$, then
$\log(1+23x/8)\ge(23x/8)/(1+23x/8)\ge2x$.

A row still unmarked at the final checkpoint has $Q<\gamma/4$
at every preceding suffix checkpoint, regardless of its sign.
Let $Z=V-4hH=(1-4hQ)V$. MR10 gives
\[
 Z^+\le\left(1+h\frac{\gamma+3\nu}{1-h\gamma}\right)Z,
 \qquad
 V_{j,N+l}\le\frac{1+2h}{1-h\gamma}
       \exp\left(\frac{\gamma+3\nu}{1-h\gamma}lh\right)V_{j,N}
 \le2e^{3\gamma lh/2}V_{j,N}.
 \tag{MR15}
\]
Both scalar constants use $h\le1/1024$, $\gamma\le1/2$,
$\nu\le\gamma/16$. Summing every still-unmarked row and comparing
with the genuinely growing signal row proves, at $M=N+k$,
\[
 p_{\rm un}:=\frac{\sum_{j\ \text{unmarked}}V_{j,M}}
                         {\sum_jV_{j,M}}
 \le\frac{2U}{v}e^{-\gamma kh/2}\le\frac\gamma{16}.
 \tag{MR16}
\]
Rows crossing the threshold join the positive marked set; MR15 is
needed only for those that never cross. No assertion that the selected
row dominates other marked rows is made or needed.

All signs now have an explicit payment. Since every unmarked row has
$Q\ge-1/2$, while every marked row has $Q\ge\gamma/8$,
\[
 {\displaystyle \quad
 \frac{\sum_jH_{j,M}}{\sum_jV_{j,M}}
 \ge\frac\gamma8(1-p_{\rm un})-\frac12p_{\rm un}
 \ge\frac\gamma{16}.\quad}
 \tag{MR17}
\]
The last inequality follows from $\gamma\le1/2$ and MR16;
indeed the preceding lower bound is at least $23\gamma/256$.
This is the complete original bank's energy-weighted central signal.

\subsection{Full original loss, unchanged rate, and explicit gain}

Write $\widehat f=m^{-1}\sum_jq_j\sigma_\alpha(\theta_j^T(X,1))$,
so $f=s^2\widehat f$. The compatibility in MR7 implies exactly
\[
 \mathbb E[e\widehat f]
   =m^{-1}\sum_jq_j\theta_j^TE_j+s^2\|\widehat f\|_2^2.
 \tag{MR18}
\]
Thus no retained teacher mean, linear term or other tail is dropped.
At $M$, MR17 and the complete force cap give, for the raw joint radius
$\mathfrak R_M=s\sqrt{\sum_jV_{j,M}}$,
\[
 Q_{\rm raw}(M):=\frac{\mathbb E[yf_M]}{\mathfrak R_M^2}
 \ge\frac1m\left(\frac\gamma{16}-\frac\nu2\right)
 \ge\frac\gamma{32m}.
 \tag{MR19}
\]
Choose publicly
\[
\begin{gathered}
a_{\rm sig}=\frac\gamma{64m},\quad
 R_{\rm sig}^2=a_{\rm sig}m^2=\frac{\gamma m}{64},\quad
 G_{\rm sig}=\frac{a_{\rm sig}R_{\rm sig}^2}{2}
                 =\frac{\gamma^2}{8192},\\
 s^2\le\min\left\{\frac{R_{\rm sig}^2}{4\overline U},
                           \frac{mG_{\rm sig}}{\overline U}\right\}.
\end{gathered}
 \tag{MR20}
\]
These are further original scale restrictions; no modification occurs
at $N$ or $M$. They give $\mathfrak R_M\le R_{\rm sig}/2$ and
\[
 \|f_n\|_2\le G_{\rm sig}/2,\qquad
 L_n\ge1-G_{\rm sig}\qquad(0\le n\le M).
 \tag{MR21}
\]
Apply Corollary~\ref{ag:averaged-release} with complete-prefix
ceiling $\overline U$ and $(a,R_c,G_c)=(a_{\rm sig},R_{\rm sig},G_{\rm sig})$.
The acquisition envelope (MR11), suffix envelope (MR13), and scale
conditions (MR20) verify its energy and radius premises, including
$\mathfrak R_M>0$. All rows remain balanced by (MR9).
The full teacher has $\mathcal H<14$, and the remaining conditions are
\[
 Q_{\rm raw}(M)\ge2a_{\rm sig},\quad
 a_{\rm sig}\le1/(4m),\quad R_{\rm sig}^2=a_{\rm sig}m^2,
 \quad h\le1/1024<[32(1+J\mathcal H)]^{-1}.
 \tag{MR22}
\]
The corollary uses the original full-MSE rate $\eta=mh/2$.
At the finite first later radius hit $K>M$, it yields
\[
 {\displaystyle \quad L_K\le1-2G_{\rm sig},\qquad
 L_n-L_K\ge G_{\rm sig}\quad(0\le n\le M).\quad}
 \tag{MR23}
\]
Its additional full-MSE physical clock, after the explicitly paid
suffix $\eta k=mkh/2$, is
\[
 \frac m4\log\frac{R_{\rm sig}}{\mathfrak R_M}
 \le\eta(K-M)\le\eta+a_{\rm sig}^{-1}
                              \log\frac{R_{\rm sig}}{\mathfrak R_M}.
 \tag{MR24}
\]
Also $R_{\rm sig}\le\mathfrak R_K<\sqrt2R_{\rm sig}$,
the joint displacement from $M$ is at least $R_{\rm sig}/2$, and
\[
 \frac\gamma{64}R_{\rm sig}
       \le\|(A_{j,K})_j\|_2<R_{\rm sig}.
 \tag{MR25}
\]
No retention of a spectral subspace, alignment or refit accuracy is
asserted during either continuation.

\section{Same-step release of the original population risk}\label{app:general-release}

Throughout this section $X\sim\gamma_d=N(0,I_d)$, $y\in H^1(\gamma_d)$,
$\|y\|_2=1$, $\mathcal H=\|y\|_{H^1}$, $0\le\alpha<1$,
$J=1-\alpha>0$, and
$\sigma(t)=\alpha t+Jt_+$. Function norms are Gaussian, with
$\|g\|_{H^1}^2=\mathbb E[g(X)^2+\|\nabla g(X)\|^2]$;
$\varphi$ is the standard normal density. Parameter derivatives
below act on the explicitly displayed raw coordinates.
The full teacher is retained in every expectation.

\subsection{Raw coordinates, full teacher, and exact row balance}

For any finite $m$ and fixed network factor $\kappa>0$ use
\begin{equation}
 f_x(X)=\kappa\sum_{j=1}^m q_j\sigma(\theta_j^T\widetilde X),
 \quad\theta_j=(w_j,b_j),\quad\widetilde X=(X,1),\quad
 x=(q_j,\theta_j)_{j=1}^m.
 \label{ag:model}
\end{equation}
Every displayed coordinate takes the same raw Euclidean step
\begin{equation}
 x^+=x-\eta\nabla\ell(x),\qquad
 \ell(x)=\tfrac12\mathbb E[(y-f_x)^2],\qquad L=2\ell.
 \label{ag:gd}
\end{equation}
The averaged model has $\kappa=1/m$; this factor is not a change of
coordinates. The identical update has full-MSE raw rate $\eta/2$,
and its full-MSE physical clock is $\eta n/2$.

Set $K(\theta)=\mathbb E[y\sigma(\theta^T\widetilde X)]$ and
\begin{equation}
 R=\|x\|,\quad F=\mathbb E[yf_x],\quad V=\|f_x\|_2^2,
 \quad Q=F/R^2\ (R>0),\quad L=1-2F+V.
 \label{ag:potentials}
\end{equation}
$F$ and $V$ have homogeneous degrees two and four respectively;
$V$ is the energy of the whole interacting bank.
With
\[
 g_j=\kappa\mathbb E[(f_x-y)\sigma'(\theta_j^T\widetilde X)
                    \widetilde X],
 \qquad D_j=\|\theta_j\|^2-q_j^2,
\]
the exact coupled update and balance identity are
\begin{align}
 q_j^+&=q_j-\eta\theta_j^Tg_j,&
 \theta_j^+&=\theta_j-\eta q_jg_j,\nonumber\\
 D_j^+&=D_j+\eta^2\{q_j^2\|g_j\|^2-(\theta_j^Tg_j)^2\}
       \ge(1-\eta^2\|g_j\|^2)D_j,&
 \|g_j\|&\le\kappa\sqrt L.
 \label{ag:balance}
\end{align}
These follow by Euler homogeneity and Cauchy--Schwarz, since
$\mathbb E[(h^T\widetilde X)^2]=1$ for a unit augmented vector $h$.
Thus $D_j\ge0$ is transported through any preceding actual loss band
with $\eta\kappa\sqrt{L_n}\le1$.

For ordinary independent blocks
$\theta_{j,0}\sim N(0,s^2I_{d+1})$,
$q_{j,0}\sim N(0,s_q^2)$, $s_q\le s$, all initial balances hold
with probability at least $1-3m e^{-(d+1)/8}$. Indeed each failure
is contained in
$\{|q_j|>s\sqrt{d+1}/2\}\cup
  \{\|\theta_j\|<s\sqrt{d+1}/2\}$;
the two costs are at most $2e^{-(d+1)/8}$ and
$\exp[-(d+1)(\log4-3/4)/2]$.
This is a sufficient event of the original law, not a rejection or
conditioning instruction. More general public scales may use a direct
balance check on their proved Gaussian event.

\subsection{The shared regularity interface}

\begin{lemma}[Full-teacher kink regularity]
\label{ag:regularity}
The function $K$ is $C^2$ on $\theta\ne0$ and satisfies
\begin{equation}
 |K(\theta)|\le\|\theta\|,\qquad
 \|\nabla K(\theta)\|\le1,\qquad
 \|\theta\|\,\|D^2K(\theta)\|\le B,
 \quad B:=4J\mathcal H.
 \label{ag:Kbounds}
\end{equation}
Consequently, globally where derivatives are evaluated,
\begin{gather}
 |F|\le\kappa R^2/2,\quad \|\nabla F\|\le\kappa R,
 \quad\|f_x\|_2\le\kappa R^2/2,\quad
 V\le\kappa^2R^4/4,\quad \|\nabla V\|\le\kappa^2R^3,
 \label{ag:global}\\
 \|\nabla Q\|\le\kappa/R,\qquad
 \|D^2Q\|\le M\kappa/R^2,\quad M:=8+2B,
 \label{ag:Qbounds}
\end{gather}
where the Hessian assertion applies on any region with
$|q_j|\le2\|\theta_j\|$ for every nonzero joint row.
An identically zero joint row is omitted and stays zero in GD.
\end{lemma}

\begin{proof}
Apply Theorem~\ref{thm:hkr-curvature} with $A=1$:
$K=C_y$ and $B=B_y=4J\mathcal H$.
Lemma~\ref{lem:hkr-coupled} gives \eqref{ag:global}--\eqref{ag:Qbounds}
with $\beta=2$. Its domain convention omits fixed zero joint rows.
In particular, the common kernel theorem includes every pure-bias row:
with $\mu=\mathbb E[y]$ and $v_1=\mathbb E[yX]$,
$K(0,b)=\mu\sigma(b)$,
$\nabla K(0,b)=\sigma'(b)(v_1,\mu)$, and
$D^2K(0,b)=0$ for $b\ne0$.
\end{proof}

\subsection{A stronger same-step deterministic release}

\begin{lemma}[Positive balanced H1 entry releases original risk]
\label{ag:release}
At an actual integer checkpoint suppose
\[
 R_0>0,\qquad |q_{j,0}|\le\|\theta_{j,0}\|\ (j\le m),
 \qquad Q(x_0)\ge2a>0.
\]
Choose public constants, before reducing the original hidden scale,
such that
\begin{equation}
 0<a\le\kappa/4,\qquad
 0<\eta\kappa\le\frac1{4M}=\frac1{32(1+J\mathcal H)},\qquad
 0<R_*^2\le\frac a{\kappa^2},\qquad R_0<R_*.
 \label{ag:public}
\end{equation}
Continue the identical simultaneous raw GD. At its finite first
crossing $N$ of $R_*$,
\begin{gather}
 R_*\le R_N<\sqrt2R_*,\qquad
 Q(x_n)\ge29a/16\quad(0\le n\le N),\qquad
 |q_{j,n}|\le\|\theta_{j,n}\|,
 \label{ag:conclusion}\\
 L_N\le1-aR_*^2,\qquad
 1-2\kappa R_*^2\le L_n\le1\quad(0\le n\le N),\qquad
 \frac{\max_{n\le N}L_n}{\min_{n\le N}L_n}
 \le\frac1{1-2\kappa R_*^2},
 \label{ag:loss}\\
 \frac1{2\kappa}\log\frac{R_*}{R_0}
 \le\eta N\le\eta+\frac2a\log\frac{R_*}{R_0}.
 \label{ag:clock}
\end{gather}
If a preceding checkpoint $t$ of this same run has
$L_t\ge1-\epsilon$ with $\epsilon\le aR_*^2/2$, then
\begin{equation}
 L_t-L_N\ge aR_*^2/2.
 \label{ag:gain}
\end{equation}
At the maximal permitted radius this gain is $a^2/(2\kappa^2)$.
There is no head refit, reset, selected sub-bank, or new random event.
\end{lemma}

\begin{proof}
Stop at the first radius hit or failure of $Q\ge a$ or row balance.
At every pre-exit state, \eqref{ag:global} and
$\kappa R^2\le a/\kappa\le1/4$ give
$\|g_j\|<2\kappa$ and $R^+\le(1+2\eta\kappa)R$.
Equation~\eqref{ag:balance} propagates balance; moreover
$\|\theta_j^+-\theta_j\|<2\eta\kappa\|\theta_j\|$, so an active
hidden row cannot vanish. Euler homogeneity gives the exact radial
identity
\begin{equation}
 (R^+)^2-R^2
 =4\eta(F-V)+\eta^2\|\nabla F-\tfrac12\nabla V\|^2
 \ge2\eta aR^2,
 \label{ag:radial}
\end{equation}
since $V/R^2\le\kappa^2R^2/4\le a/4$.

To control correlation, separate the step into radial and tangent
parts. With $P=I-xx^T/R^2$, set
\[
 v=\tfrac12P\nabla V,\quad
 \lambda=1+2\eta(Q-V/R^2)\ge1,\quad t=\eta/\lambda,\quad
 e=t(R^2\nabla Q-v).
\]
Then $x^+=\lambda(x+e)$, $e\perp x$, and
$\|v\|\le\kappa^2R^3/2$. Thus $Q(x^+)=Q(x+e)$ and the tangent
segment has norm at least $R$. Each of its head and hidden blocks
moves by at most
\[
 \eta(2\kappa+\kappa^2R^2)\|\theta_j\|
 \le\tfrac94\eta\kappa\|\theta_j\|
 \le\tfrac9{128}\|\theta_j\|.
\]
Indeed the teacher, student and removed radial parts cost respectively
$\kappa$, $\kappa^2R^2/2$ and
$\kappa+\kappa^2R^2/2$, times $\|\theta_j\|$.
Consequently every active hidden row remains nonzero, and the
head/hidden ratio is at most $(1+9/128)/(1-9/128)<2$.
The Hessian bound \eqref{ag:Qbounds} therefore applies on the whole
segment. Writing $G=\|\nabla Q\|$, Taylor and $M\kappa t\le1/4$ give
\begin{align*}
 Q(x^+)-Q(x)
 &\ge t(R^2G^2-G\|v\|)
       -\frac{M\kappa t^2}{2R^2}(R^2G+\|v\|)^2\\
 &\ge\tfrac34tR^2G^2-tG\|v\|
                       -\tfrac14t\|v\|^2/R^2\\
 &\ge\tfrac14tR^2G^2-\tfrac34t\|v\|^2/R^2
 \ge-\tfrac3{16}\eta\kappa^4R^4.
\end{align*}
The two scalar inequalities used here are
$(u+v)^2\le2u^2+2v^2$ and
$G\|v\|\le R^2G^2/2+\|v\|^2/(2R^2)$.

The radial increment implies
$R_{k+1}^4-R_k^4\ge4\eta aR_k^4$.
Also $\eta\kappa\le1/32$, so every stopped checkpoint, including
an exit candidate, has $R_n^2<2R_*^2$. Hence
\[
 \sum_{k<n}\eta R_k^4
 \le\frac{R_n^4-R_0^4}{4a}<\frac{R_*^4}{a},\qquad
 Q(x_n)\ge2a-\frac{3\kappa^4R_*^4}{16a}
                  \ge\frac{29a}{16}>a.
\]
This excludes angular exit and closes the balance induction. Radial
growth forces a finite radius hit and gives the stated overshoot.
Using $R_{N-1}<R_*$, the radial bounds yield \eqref{ag:clock} from
$\log(1+2\eta a)\ge\eta a$ and
$\log(1+2\eta\kappa)\le2\eta\kappa$.
The fourth-power debit is thus summable independently of the length
of the logarithmic release clock.

Through the hit, $V/R^2\le a/2$ and $Q\ge a$, so
$1-L=2F-V\ge aR^2$; at the hit this is at least $aR_*^2$.
Conversely $L\ge1-\kappa R^2\ge1-2\kappa R_*^2$.
These prove \eqref{ag:loss}, and subtraction of the preceding
checkpoint bound proves \eqref{ag:gain}.
\end{proof}

\paragraph{Joint motion, head cost, and the nonzero-mean distinction.}
The release certifies $\|x_N-x_0\|\ge R_*-R_0$.
Writing $q=(q_j)$, Cauchy--Schwarz gives
$F\le\kappa\|q\|\|\Theta\|_F$, hence
\[
 \|q_N\|\ge\frac{aR_*}{\kappa},\qquad
 \|q_N\|\le R_N/\sqrt2<R_*.
\]
Thus the actual head cost is explicit; it is not an oracle budget.
There is no general lower bound on spatial motion in this release.
For example a nonzero constant target can produce positive $Q$ with
$W=0$ and improve solely through head/bias growth. Such a teacher is
an illustration of the release interface, not an alignment theorem.

If additionally $\mu=\mathbb E[y]=0$, the exact Lipschitz bound
$|K(w,b)-K(0,b)|\le\|w\|$ gives $|K(w,b)|\le\|w\|$, and then
\[
 \|W_N\|_F\ge\frac{aR_*}{\kappa},\qquad
 \|W_N-W_0\|_F\ge\frac{aR_*}{\kappa}-\|W_0\|_F.
\]
For $\mu\ne0$, the retained inequality is only
$|K(w,b)|\le\|w\|+|\mu||b|$; it cannot be simplified by
centering the target or discarding the bias contribution.
Neither form proves a half-unit complete-score gain or all-direction
coverage. The earlier entry window must supply those separately.

\subsection{One entry rule for the averaged network}

\begin{corollary}[Complete-prefix entry to same-step release]
\label{ag:averaged-release}
Train $f_n=m^{-1}\sum_j A_{j,n}\sigma(W_{j,n}^TX+B_{j,n})$
on full MSE at the unchanged raw rate $\eta=mh/2$, and write
$(A_j,W_j,B_j)=s(q_j,w_j,b_j)$ with $s>0$, only for analysis. Set
\[
 \mathcal E_n=\sum_j(q_{j,n}^2+\|w_{j,n}\|^2+b_{j,n}^2),\qquad
 \mathfrak R_n=s\sqrt{\mathcal E_n}.
\]
At an actual integer $M$, suppose all rows are balanced,
$\mathcal E_M>0$, $\max_{0\le n\le M}\mathcal E_n\le U$, and
$\mathbb E[yf_M]/\mathfrak R_M^2\ge2a$.
The public constants are to satisfy
\begin{equation}
 \begin{gathered}
 0<a\le\frac1{4m},\quad
 0<h\le\frac1{32(1+J\mathcal H)},\quad 0<R_c^2\le am^2,\\
 G_c=\frac{aR_c^2}{2},\qquad
 s^2U\le\min\{R_c^2/4,mG_c\}.
 \end{gathered}
 \label{ag:averaged-entry}
\end{equation}
Then $\mathfrak R_M\le R_c/2$ and
$\|f_n\|_2\le G_c/2$, $L_n\ge1-G_c$ throughout $0\le n\le M$.
At the finite first subsequent hit $K>M$ of $R_c$,
\begin{gather}
 L_K\le1-2G_c,\qquad L_n-L_K\ge G_c\quad(0\le n\le M),
 \label{ag:averaged-gain}\\
 \frac m4\log\frac{R_c}{\mathfrak R_M}
 \le\eta(K-M)\le
 \eta+\frac1a\log\frac{R_c}{\mathfrak R_M}.
 \label{ag:averaged-clock}
\end{gather}
The raw radius at $K$ lies in $[R_c,\sqrt2R_c)$; raw joint motion
from $M$ is at least $R_c/2$, and the trained heads obey
$amR_c\le\|(A_{j,K})_j\|_2<R_c$.
No geometric or refit persistence during the continuation is implied.
\end{corollary}

\begin{proof}
The complete-bank norm bound gives
$\|f_n\|_2\le s^2\mathcal E_n/(2m)\le G_c/2$, whence
$L_n\ge1-2\|f_n\|_2\ge1-G_c$.
Apply Lemma~\ref{ag:release} to the raw coordinates at $M$ with
$\kappa=1/m$, $R_*=R_c$ and half-MSE rate $\eta_{1/2}=2\eta=mh$.
Then $\eta_{1/2}\kappa=h$, so \eqref{ag:averaged-entry} verifies
all of \eqref{ag:public}; the update is exactly the same one.
Equations~\eqref{ag:loss}--\eqref{ag:gain} give the conclusions,
with the clock divided by two. The radius, motion and head bounds
are the preceding deterministic bounds with $\kappa=1/m$.
\end{proof}

\section{Cubic row growth and selected-cone invariance}\label{app:cubic-seeds}

Use the teacher and public constants of Appendix~\ref{app:aj-parameters}:
$X\sim N(0,I_d)$, $y_c=\sum_i a_i h_3(u_i^TX)$,
$h_3(t)=(t^3-3t)/\sqrt6$, and $\sigma_\alpha(t)=\alpha t+(1-\alpha)t_+$.
In normalized coordinates $(A,W,B)=s(q,w,b)$, put
$\theta=(w,b)$, $\rho=\|w\|$, $n=w/\rho$, $z=b/\rho$ when $\rho>0$,
$P(n)=\sum_i a_i(u_i^Tn)^3$, and
$C(\theta)=\mathbb E[y_c\sigma_\alpha(w^TX+b)]$.
Write $P_U=UU^T$, $P_{U^\perp}=I_d-P_U$ and let $\varphi$ denote
the standard normal density.

This is a deterministic argument for one category representative from
AC13--16. Put $Z_X=(X,1)$, $v=y-y_c$, $\varsigma=\operatorname{sign}C_0$,
$p=\varsigma q$, $K=\varsigma C$, $B=b^2/\rho^2$ and
$c=(1-\alpha)/\sqrt6$. Its initial conditions are
$|C_0|\ge k$, $q_0C_0>0$, $\rho_0\in[1/2,2]$,
$|q_0|\le\|\theta_0\|$, $B_0\le1/16$ and initial joint norm below $3$.
The temporary radius cap denoted $R$ in the following prefix estimates
is the public envelope $M$ of AJ8; the prescribed force in AJ9 pays
every displayed inequality. The category milestone remains the radius
$R$ in AJ7. In the selected-cone calculation write
$R_B=M$, $k_E=k$, $a_E=a_0$, $b_2=b_*^2$, and
$\overline S=4\log(8M)$. Let $x=s_i u_i^Tw$,
$u=(\sum_{l\ne i}(u_l^Tw)^2)^{1/2}$ and $s_i$ be the category's
original axis sign.

\subsection{Deterministic actual compatible-force theorem}
Let $T_c=\nabla C$. The cubic correlation bounds are
\[
 |C(\theta)|\le\|\theta\|,\qquad
 \|T_c(\theta)\|\le1,\qquad
 \|D^2C(\theta)\|\le8/\|\theta\|.
 \tag{DC11}
\]
The exact normalized original coupled recurrence is
\begin{gather}
 q_j^+=q_j+h(C_j+\theta_j^TE_j),\qquad
 \theta_j^+=\theta_j+hq_j(T_{c,j}+E_j),\nonumber\\
 E_j=\mathbb E[(v-s^2\widehat f)\sigma'_\alpha(\theta_j^TZ_X)Z_X],
 \qquad
 \widehat f=m^{-1}\sum_lq_l\sigma_\alpha(\theta_l^TZ_X).
 \tag{DC12}
\end{gather}
The head force is exactly $\theta_j^TE_j$. Below $E_j$ may be
any adaptive compatible force of norm at most $\nu$; no derivative
of that force or independence of its rows is assumed.

\subsection{Direct phase, bias, and outside-energy proof}
Use $g=\varsigma T_c$, $e=\varsigma E$, so
$p^+=p+h(K+\theta^Te)$ and $\theta^+=\theta+hp(g+e)$.
Initial balance and exact compatibility imply
\[
 D^+\ge(1-h^2\|g+e\|^2)D\ge0,\qquad
 D=\|\theta\|^2-p^2 .
 \tag{DC17}
\]
Stop provisionally at the first failure of $p>0,K\ge k,B=z^2\le1$
or $\rho>0$, and at the radius/common stop. At each pre-step point
$\rho<R$, and including its candidate step one has
$\rho\le(1+2h)R<2R$, since
$\|\theta^+-\theta\|\le2h\|\theta\|$ and, for $B\le1$,
the sharper bound $\|w^+-w\|\le2h\rho$ follows from
$p\le\sqrt2\rho$, $\|T_c\|\le1$, and $\nu\ll1$.
In this enlarged $2R$ tube,
\[
 \nu\le\|g\|/4,\qquad |\theta^Te|\le K/4,
 \tag{DC18}
\]
by Euler's identity $\|\theta\|\|g\|\ge K\ge k$ and AJ9.
The hidden segment stays above $\|\theta\|/2$.
Taylor's formula and DC11 give
\[
 K^+-K\ge hp\|g\|^2(3/4-25h/2)
 \ge\tfrac12hp\|g\|^2,\qquad
 p^+-p\ge3hK/4 .
 \tag{DC19}
\]
Thus $p,K$ remain positive with their original common sign.
Also $\|\theta^+\|^2-\|\theta\|^2\ge3hpK/2\ge0$.

For the exact teacher intermediate spatial step put
$t=hpK/\rho^2$, $v_P=\nabla P(n)/P(n)$ and
$a=1-(3-B)t$. Nonzero $K$ ensures $zP\ne0$. Then
\begin{gather}
 \overline w/\rho=a n+t v_P,\quad
 \overline b/\rho=z+t(1-B)/z,\quad
 0<t\le h\sqrt B\le h,\quad 1-3h\le a\le1,\nonumber\\
 D_w:=\|\overline w\|^2/\rho^2
 =(1+Bt)^2+(\|v_P\|^2-9)t^2\ge1,\nonumber\\
 N_b:=\overline b^2/\rho^2
 =B+2(1-B)t+(1-B)^2t^2/B .
 \tag{DC20}
\end{gather}
This is an identity inside the simultaneous step, not a splitting
of the algorithm. The actual added hidden displacement is $hp e$.

If $B\le.75$, then $N_b\le.75+2h+h^2<.755$ and $D_w\ge1$.
The added displacement has norm at most $2h\rho\nu$, so the
actual bias-to-radius ratio remains below one. If $.75\le B\le1$,
\[
 D_w-N_b=1-B+(4B-2)t+
 (B^2-B+2-B^{-1}+\|v_P\|^2-9)t^2\ge1-B+t.
 \tag{DC21}
\]
The quadratic coefficient is nonnegative in this interval.
Since $\|\overline\theta\|\le2\rho$, the force changes
$\|\overline w\|^2-\overline b^2$ by at most
$4\rho hp\nu+(hp\nu)^2$ in the unfavorable direction.
Divided by $\rho^2t$, this is at most
$16R\nu/k+8hR\nu^2/k<.01$.
This proves the candidate upper-bias barrier. The new spatial radius
is positive since $\rho^+\ge\rho-hp\nu\ge(1-2h\nu)\rho>0$.
The nondecreasing augmented norm then gives
$\rho\ge\|\theta_0\|/\sqrt2\ge1/(2\sqrt2)>1/3$.
All provisional failures have been strictly excluded.

Writing $O=\|P_{U^\perp}w\|$, the exact teacher outside update gives
\[
 O^+\le aO+hp\nu
 \le(1-2t)O+\frac{4R^2\nu}{k}t .
 \tag{DC22}
\]
The coefficients are nonnegative, and its scalar equilibrium is
$2R^2\nu/k\le1/5000$. Since $O_0\le2$, this proves $O_n\le2$.
It controls the reliable row only; other original rows have not
been discarded from any asserted complete-bank score.

\subsection{Intrinsic force sums}
For this actual row put $S_n=\sum_{l<n}t_l$.
The augmented-energy increase above, together with $B\le1$, gives
\[
 \|\theta^+\|^2\ge\|\theta\|^2(1+3t/4),\qquad
 S_n\le\frac{20}{7}\log(4R)<\overline S,
 \quad
 \sum_{l<n}hp_l\le\frac{4R^2}{k}\overline S,\quad
 \sum_{l<n}hp_lK_l\le\frac43\rho_n^2 .
 \tag{DC23}
\]
Here $\|\theta_0\|\ge.5$ and
$\|\theta_n\|\le\sqrt2(1+2h)R<2R$;
$\log(1+3t/4)\ge7t/10$ proves the clock bound.
Every sum includes the candidate exit step.
The pure bias step retains its sign because $bC_b=C(1-B)$ and
$B\le1$. Also $|b|\ge|C|\ge k$, whereas the actual bias perturbation
is at most $hp\nu\le hM\nu<k$. Thus the original bias sign persists,
so the cone orientation below is the category's original orientation.

\subsection{Direct invariant projected cone}
Orient all teacher coordinates by the selected sign $s_i$.
Inside the same step the selected coordinate and the other-coordinate
norm satisfy exactly or respectively
\[
 \overline x=ax+d_c a_i x^2,\qquad
 \overline u\le au+d_c u^2,
 \qquad d_c=\frac{3hp c|z|\varphi(z)}{\rho^2}
           =\frac{3t}{\rho|P(n)|}>0.              \tag{DX27}
\]
The norm inequality uses
$(\sum_{l\ne i}a_l^2(u_l^Tw)^4)^{1/2}\le u^2$.
The compatible force changes $x$ in absolute value by at most
$hp\nu$ and increases $u$ by at most $hp\nu$.

Stop provisionally before failure of $x\ge1/(2\sqrt d)$ or
$u\le\vartheta x$. In that cone the selected cubic term dominates
the others because $a_i x^3>u^3$, and $a_i x$ is the largest positive
weighted coordinate, so
$|P(n)|\le(a_ix/\rho)\|P_Un\|^2\le a_ix/\rho$.
Thus $\overline x\ge x(1+Bt)\ge x$.
Initial $x_0\ge1/\sqrt d$ and DC23 give at every candidate state
\[
 x_n\ge x_0-\nu\sum_{l<n}hp_l
 \ge\frac1{\sqrt d}-\frac{4R_B^2\nu\overline S}{k_E}
 \ge\frac{15}{16\sqrt d}>\frac1{2\sqrt d}.        \tag{DX28}
\]
For the cone boundary, monotonicity of $au+d_cu^2$ for $u\ge0$
gives on the entire closed cone
\[
 u^+-\vartheta x^+
 \le-d_c\vartheta(a_i-\vartheta)x^2
                       +(1+\vartheta)hp\nu.      \tag{DX29}
\]
Here $a_i-\vartheta\ge7a_E/8$, $x^2\ge1/(4d)$ and
$c|z|\varphi(z)=K/(\rho|P|)\ge k_E/\rho$.
At the candidate step the pre-step spatial radius is below $R_B$;
using the weaker $2R_B$ bound gives
\[
 d_c\vartheta(a_i-\vartheta)x^2
 \ge\frac{21}{256}\,\frac{hp k_E\vartheta a_E}{R_B^3d}
 >(1+\vartheta)hp\nu .                            \tag{DX30}
\]
The strict final inequality follows from AJ9. Both candidate
failures are excluded. This proves the direct cone retention the selected-cone assertion
for every original cone representative through the declared horizon $N$.
It is not an assumed favorable sign or a relabeled reference basin.

\subsection{Direct bias calibration and retention after a milestone}
Let
\[
 W=b^2-b_2\rho^2,\qquad b_2+b_2^2=1.
\]
The exact teacher intermediate gives
\[
 \overline W=(1-2(1+b_2)t)W+
      h^2p^2\{T_b^2-b_2\|T_w\|^2\},\qquad T=\nabla C.
                                                               \tag{DX31}
\]
Indeed $bT_b=C(1-B)$ and $w^TT_w=BC$ for a cubic target.
The coefficient in front of $W$ is in $[0,1]$ since $t\le h$.
For the full actual force displacement $hp e$, the quadratic
$W$ changes by at most
$4\rho hp\nu+h^2p^2\nu^2$, using $\|\overline\theta\|\le2\rho$.
Correlation ascent, monotonicity of $p,K$ and balance give
\[
 \sum_{l<n}h^2p_l^2\|T_l\|^2
 \le2h\sum_{l<n}p_l(K_{l+1}-K_l)
 \le2hp_nK_n\le4h\rho_n^2.                       \tag{DX32}
\]
For the force sum use $\rho_l/K_l\le2R_B/k_E$ and DC23:
\[
 \sum_{l<n}4\rho_lhp_l\nu
 \le\frac{32R_B\nu}{3k_E}\rho_n^2.
\]
The sum of $h^2p_l^2\nu^2$ is at most
$8hR_B\nu^2\rho_n^2/(3k_E)$, absorbed by the stated weaker bound.
Since $|W_0|<5$, iterating the absolute-value inequality in DX31
proves the bias bound, with slack in both constants.

At $\rho\ge R_D$, the oriented direction error obeys
\[
 \|w/\rho-s_iu_i\|
 \le\sqrt2\sqrt{\vartheta^2+4/R_D^2}<\xi/4.       \tag{DX33}
\]
Also AJ3, AJ9 and $b_*>1/2$ give
\[
 ||b|/\rho-b_*|
 \le2\{5/R_D^2+8h+100R_B\nu/k_E\}<\xi/4.        \tag{DX34}
\]
The sign of $b$ is the original $s_b$, so the actual full affine
profile error is below $\xi/2$, hence below $\xi$.
By DC20 each pure teacher spatial step has nondecreasing radius;
force can decrease it by at most $hp\nu$. On any subinterval of
the prefix the total decrease is at most the whole DC23 budget,
which is at most $1/(16\sqrt d)$. Thus a row reaching $2R_D$
retains radius at least $2R_D-1/(16\sqrt d)>R_D$ through $N$,
and its cone and bias bounds continue to apply. This is actual
post-milestone retention, not a rule freezing that row.

\section{Cubic normalized correlation and complete-row geometry}\label{app:cubic-rows}

\subsection{Actual model and compatible force}
All calculations use teacher coefficients and axes $(\lambda_i,u_i)$,
corresponding to $(a_i,u_i)$ in AJ1. Coefficient signs and magnitudes
are unrestricted subject to normalization. Function norms are Gaussian,
and $\theta_j=(w_j,b_j)$ denotes a normalized augmented hidden row.
The proof uses three discrete invariants: compatible balance, compensated
normalized correlation, and a maximum of spatial/bias quadratics.
They apply to every original row, regardless of its initial head sign.
Let $X\sim N(0,I_d)$, $Z=(X,1)$, and fix a normalized additive cubic
\[
 y_c=\sum_{i=1}^r\lambda_i h_3(u_i^TX),\qquad
 \sum_i\lambda_i^2=1,\quad u_i^Tu_l=\mathbf1_{i=l},
 \qquad h_3(t)=(t^3-3t)/\sqrt6.
 \tag{NC1}
\]
The coefficient magnitudes need not be equal. For $0\le\alpha<1$
let $\sigma(t)=\alpha t+(1-\alpha)t_+$. Retain the full target
$y=y_c+v$, including all its lower terms and tails, and the original
averaged network
\[
 f=m^{-1}\sum_j a_j\sigma(w_j^TX+b_j),\qquad
 L=\mathbb E[(y-f)^2],\qquad
 (a_j,w_j,b_j)=s(q_j,\theta_j).
 \tag{NC2}
\]
All raw coordinates train simultaneously at the same full-MSE step
$\eta=mh/2$. The one-step estimates NC6--12 hold for
$0<h\le1/128$; the public-horizon conclusions below explicitly
require $h\le1/512$, as does the final AJ theorem. Define
\[
 C(\theta)=\mathbb E[y_c\sigma(\theta^TZ)],\qquad
 H(x)=qC(\theta),\quad x=(q,\theta),\quad F=\nabla H,
\]
\[
 V=\|x\|^2,\qquad Q=H/V,
 \qquad \widehat f=m^{-1}\sum_lq_l\sigma(\theta_l^TZ).
 \tag{NC3}
\]
The letter $V$ here denotes a single row's joint energy; complete-bank
energy below is denoted $\mathcal V$. The exact actual update is
\[
 x_j^+=x_j+h\{F(x_j)+p_j\},\qquad
 p_j=(\theta_j^TE_j,q_jE_j),
\]
\[
 E_j=\mathbb E[(v-s^2\widehat f)\sigma'_jZ].
 \tag{NC4}
\]
Here $\sigma'_j=\sigma'(\theta_j^TZ)$ is the scalar gate.
Every original row, bias and student cross term is present. The
analysis scaling by $s$ changes no optimizer or metric.

The previously Gaussian-H1 bounds for the unit cubic give
\[
 |C|\le\|\theta\|,\quad\|\nabla C\|\le1,
 \quad\|D^2C\|\le8/\|\theta\|,\quad
 F=(C,q\nabla C),\quad\|F(x)\|\le\|x\|,
\]
\[
 H(tx)=t^2H(x),\quad F(tx)=tF(x)\quad(t>0),\qquad
 \|DF(q,\theta)\|\le1+8|q|/\|\theta\|.
 \tag{NC5}
\]
These bounds use the augmented hidden norm, so zero spatial weight
with nonzero bias is allowed. In particular $|Q|\le1/2$.

Assume each original initial row has
$0<\|\theta_{j,0}\|$ and $|q_{j,0}|\le\|\theta_{j,0}\|$.
No head sign or teacher detector lower bound is assumed. If the
compatible force at an update obeys $\|E_j\|\le\nu\le1$, then
balance and nonzero augmented hidden state persist through that
update. Indeed, with $G=\nabla C+E$ and $D=\|\theta\|^2-q^2$,
\[
 D^+=D+h^2\{q^2\|G\|^2-(\theta^TG)^2\}
 \ge(1-h^2\|G\|^2)D\ge0,
\]
\[
 \|\theta^+\|\ge(1-2h)\|\theta\|>0,
 \qquad \|x^+\|\le(1+2h)\|x\|.
 \tag{NC6}
\]
This exact compatible balance is the only row-ratio hypothesis used.

\subsection{Correlation invariant and compatible-force debit}
Put $u=x/\sqrt V$ and define its teacher tangent field
\[
 g(u)=F(u)-2Q u,\qquad \langle u,g\rangle=0,
 \qquad \|g\|\le1,
 \qquad F(u)=2Qu+g.
 \tag{NC7}
\]
Euler's identity for degree two gives the factor $2Q$.
For every actual update satisfying the preceding force and balance
conditions,
\[
 Q^+\ge Q+\frac h2\|g(u)\|^2-2h\nu.
 \tag{NC8}
\]
Consequently $Q_n+2\nu hn$ is nondecreasing on every prefix having
that force bound. With zero force, $Q$ itself is nondecreasing at
the original fixed discrete step.

Here is a full one-step proof. First take a pure teacher step from
the current actual row, solely as an algebraic intermediate. Its
unnormalized unit-row update is
\[
 u+hF(u)=(1+2hQ)(u+t g),\qquad
 t=\frac h{1+2hQ},\qquad
 \frac h{1+h}\le t\le\frac h{1-h}\le\frac1{127}.
\]
Since the starting unit row is balanced, its hidden norm is at least
$1/\sqrt2$ and its head magnitude at most $1/\sqrt2$. On the segment
$u+stg$, $0\le s\le1$, the hidden norm is at least
$1/\sqrt2-1/127>.69$, and the head/hidden ratio is at most
$(1/\sqrt2+1/127)/(1/\sqrt2-1/127)<2$.
Thus $\|D^2H\|=\|DF\|\le17$ on the entire Taylor segment. Exact homogeneity and orthogonality give
\[
 Q_{\rm teach}^+
 =\frac{H(u+tg)}{1+t^2\|g\|^2},\qquad
 H(u+tg)\ge Q+t\|g\|^2-\frac{17}2t^2\|g\|^2.
\]
It follows that
\[
 Q_{\rm teach}^+-Q
 \ge\frac{t-9t^2}{1+t^2\|g\|^2}\|g\|^2
 \ge\frac h2\|g\|^2.
 \tag{NC9}
\]
For the last coarse inequality, use $t\ge3h/4$,
$1-9t\ge3/4$ and $1+t^2\|g\|^2\le9/8$, all implied by
$h\le1/128$.

The normalized compatible perturbation is $r=p/\sqrt V$, with
$\|r\|\le\nu$. For nonzero $a,b$ the exact identity
\[
 \|a-b\|^2=(\|a\|-\|b\|)^2+
       \|a\|\|b\|\|a/\|a\|-b/\|b\|\|^2
\]
applied to $a=u+hF(u)$ and $b=a+hr$ bounds the difference of
their normalized directions by $2h\nu$: their norms are at least
$1-h$ and $1-2h$. Both normalized endpoints are balanced by NC6.
Their chord has norm at most one and hidden norm at least
$1/\sqrt2-2h\nu>0$; there $\|\nabla H\|=\|F\|\le1$. The resulting change in $Q=H(u)$
is at most $2h\nu$, proving NC8. No comparison trajectory or Euler
remainder is being charged.

\subsection{A compensated energy inequality valid at negative correlation}
The same actual update satisfies the stronger potential bound
\[
 \Delta\{\log V-4hQ\}\le4hQ+5h\nu.
 \tag{NC10}
\]
Unlike a bound that replaces $Q$ by its positive part, NC10 retains
the energy decrease of adversely correlated rows.
To prove it, the exact pure-teacher radial/tangent decomposition is
\[
 \log(V_{\rm teach}^+/V)
 =2\log(1+2hQ)+\log(1+t^2\|g\|^2)
 \le4hQ+2h^2\|g\|^2.
\]
This uses $1+2hQ>0$, $\log(1+a)\le a$ for $a>-1$, and
$t^2\le2h^2$. Along the additional force segment the joint norm is
at least $1-2h$, so the change in log joint norm is at most $2h\nu$,
and the change in log energy is at most $4h\nu$. By NC8,
\[
 2h^2\|g\|^2\le4h(Q^+-Q)+8h^2\nu.
\]
Combining the inequalities and using $4+8h\le5$ proves NC10.

On any interval $a\le n$ of such updates, put $t=(n-a)h$.
Telescoping NC10 gives
\[
 \log(V_n/V_a)
 \le4h\sum_{l=a}^{n-1}Q_l+5\nu t+4h(Q_n-Q_a).
 \tag{NC11}
\]
If $Q_l\le K$ for every pre-update index $a\le l<n$, with
$K\ge0$, this yields
\[
 V_n\le V_a\exp\{(4K+5\nu)t+4h(Q_n-Q_a)\}
 \le V_a\exp\{(4K+5\nu)t+4h\}.
 \tag{NC12}
\]
There is no requirement on the terminal $Q_n$. In particular NC12
includes the candidate step which first crosses the threshold.

\subsection{Public-horizon irreversible classification, including all crossings}
Fix a public $\overline T\ge1$ and an integer $N_{\max}\ge0$
with $N_{\max}h\le\overline T$. For AJ take
$\overline T=T+1$ and $N_{\max}=\lceil T/h\rceil$.
Use the explicit mesh and polynomial force budget
\[
 0<h\le1/512,\qquad \kappa=\frac1{100\overline T},\qquad
 0\le\nu\le\frac1{400\overline T^2}
            =\frac\kappa{4\overline T}.
 \tag{NC13}
\]
All following assertions hold on any force-controlled prefix through
an integer $N\le N_{\max}$. A first candidate energy exit may be
that terminal integer, as formalized below.
For each original row define its permanent marking time
\[
 \tau_j=\inf\{0\le n\le N:Q_{j,n}\ge2\kappa\},
 \qquad \inf\varnothing=+\infty.
 \tag{NC14}
\]
Then:
\begin{enumerate}
\item Every low-phase and threshold-crossing state has
\[
 V_{j,n}<\frac87 V_{j,0}
 \quad(0\le n\le\min\{\tau_j,N\}).
 \tag{NC15}
\]
For an initially marked row this statement concerns just $n=0$.
For an unmarked row it covers the entire available horizon.
\item Once marked, the row cannot return to $Q\le\kappa$ on that
horizon; more precisely,
\[
 Q_{j,n}\ge\frac32\kappa
 \quad(\tau_j\le n\le N).
 \tag{NC16}
\]
The actual values can cross $2\kappa$ again. The permanent label,
not exact monotonicity of the forced $Q$, is the irreversible object.
\item At every marked state the actual teacher head and correlation
have the same nonzero sign, and
\[
 \frac{|C_j|}{\|\theta_j\|}\ge2Q_{j,n}\ge3\kappa,
 \qquad \frac{|q_j|}{\|\theta_j\|}\ge Q_{j,n}
                         \ge\frac32\kappa.
 \tag{NC17}
\]
Their common sign persists and $|q_j|$ increases on subsequent
controlled updates. The relative head-force and hidden-gradient
force are bounded by
\[
 \frac{|\theta_j^TE_j|}{|C_j|},\quad
 \frac{\|E_j\|}{\|\nabla C_j\|}
 \le\frac\nu{3\kappa}\le\frac1{12\overline T}.
 \tag{NC18}
\]
\end{enumerate}

To prove NC15 for a positive marking time, all pre-crossing
$Q_{j,l}<2\kappa$. Apply NC12 with $K=2\kappa$, including the
crossing candidate. The exponent is at most
\[
 (8\kappa+5\nu)\overline T+4h
 \le\frac8{100}+\frac5{400}+\frac4{512}<\frac18.
\]
Since $e^{1/8}<8/7$, NC15 follows. NC8 gives at every later state,
including a proposed first exit below $\kappa$,
\[
 Q_{j,n}\ge Q_{j,\tau_j}-2\nu(n-\tau_j)h
 \ge2\kappa-2\nu\overline T\ge\frac32\kappa.
\]
This excludes the candidate exit and proves NC16 without an iteration
over repeated threshold crossings.

For NC17 put $z=|q|/\|\theta\|\in(0,1]$ at a marked state.
The identity $Q=qC/(q^2+\|\theta\|^2)$ gives
$|C|/\|\theta\|=Q(z+z^{-1})\ge2Q$ and, since
$|C|\le\|\theta\|$, gives $Q\le z/(1+z^2)\le z$.
Euler's identity $\theta^T\nabla C=C$ then gives
$\|\nabla C\|\ge|C|/\|\theta\|$, proving NC18.
Finally the head update has increment with its current sign because
$|C|-|\theta^TE|\ge(3\kappa-\nu)\|\theta\|>0$.
Together with NC16 at the next state, this preserves also the
correlation sign. There was no initial detector quantile for any
of these classified rows.

All original rows stay in the bank. For example, at any common time
the entire energy of the still-unmarked rows is at most
$(8/7)\sum_jV_{j,0}$. Marked rows are not presumed aligned:
NC16--18 only provide the positive-correlation and relative-force
interface needed for a separate direct spatial argument.

\subsection{Conditional full-bank correlation from actual energy growth}
There is also a useful endpoint consequence that avoids classifying
which rows supplied the energy. For any interval of $n\ge1$ actual
steps starting at zero, set $t=nh$ and retain the force bound.
NC8 implies $Q_l\le Q_n+2\nu(n-l)h$. Substitution in NC11 yields
the rowwise inequality
\[
 Q_n\ge
 \frac{\log(V_n/V_0)+4hQ_0
       -\nu\{4t(t+h)+5t\}}{4(t+h)}.
 \tag{NC19}
\]
The sum $h\sum_{l=0}^{n-1}2\nu(n-l)h
=\nu t(t+h)$ is exact, so no unreported time-step error occurs.

Define the complete actual quantities
\[
 \mathcal V_n=\sum_jV_{j,n},\qquad
 \mathcal H_n=\sum_jH_{j,n},\qquad
 \overline Q_n=\mathcal H_n/\mathcal V_n.
\]
Multiplying NC19 by $V_{j,n}/\mathcal V_n$, summing every original
row, and using $Q_{j,0}\ge-1/2$ gives
\[
 \overline Q_n\ge
 \frac{\log(\mathcal V_n/\mathcal V_0)-2h
       -\nu\{4t(t+h)+5t\}}{4(t+h)}.
 \tag{NC20}
\]
For completeness, the log-sum step is
\[
 \sum_j\frac{V_{j,n}}{\mathcal V_n}\log\frac{V_{j,n}}{V_{j,0}}
 =\log\frac{\mathcal V_n}{\mathcal V_0}
  +\sum_j p_j\log(p_j/p_{j,0})
 \ge\log\frac{\mathcal V_n}{\mathcal V_0},
\]
where $p_j=V_{j,n}/\mathcal V_n$ and
$p_{j,0}=V_{j,0}/\mathcal V_0$. All these weights are positive
by NC6. This is an energy-weighted identity for the complete bank,
not a selected-row estimate.

Under NC13, if an actual integer $1\le N\le N_{\max}$ has
$\mathcal V_N\ge2\mathcal V_0$, then
\[
 \overline Q_N\ge\frac1{8\overline T}.
 \tag{NC21}
\]
Indeed $t\le\overline T$ and $h\le1/512$ bound the debit
$2h+\nu\{4t(t+h)+5t\}$ by $1/32$.
The numerator of NC20 is then greater than $5/8$, since
$\log2\ge2/3$, and its denominator is at most $5\overline T$.
This proves NC21. The energy-doubling hypothesis is an actual
checkpoint requirement still to be acquired, not a proved exit time
or a replacement definition of subspace learning.

For the full target rather than its central cubic, exact compatibility
gives at the endpoint
\[
 \mathbb E[v\widehat f]
 =m^{-1}\sum_jq_j\theta_j^TE_j+s^2\|\widehat f\|_2^2.
 \tag{NC22}
\]
If the endpoint force itself also obeys $\|E_{j,N}\|\le\nu$, then
with complete raw joint norm $\mathfrak R_N^2=s^2\mathcal V_N$,
\[
 \frac{\mathbb E[yf_N]}{\mathfrak R_N^2}
 \ge\frac{\overline Q_N-\nu/2}{m}
 \ge\frac1{16m\overline T}.
 \tag{NC23}
\]
Here the nonnegative student term in NC22 is retained, and
$\sum_j|q_j|\|\theta_j\|\le\mathcal V_N/2$ bounds the debit.
The last inequality uses NC13 and NC21. The endpoint-force condition
is explicitly included because a first energy-crossing state need
not lie inside an unpadded pre-update tube.

\subsection{Spatial/bias invariant from the same actual recurrence}

Continue with NC1--6 and write $T=\nabla C$,
$U=\operatorname{span}\{u_i\}$, with orthogonal projector $P_U$.
In particular the same compatible vector $E$ produces both the head
and hidden errors in NC4. It may depend on every row and the whole
history; no independence or derivative of $E$ is assumed.
For $\rho=\|w\|>0$, put $n=w/\rho$, $z=b/\rho$, $B=z^2$, and
$P(n)=\sum_i\lambda_i(u_i^Tn)^3$. The exact cubic correlation is
\[
 C(\theta)=-c\,b\varphi(z)P(n),\qquad
 c=(1-\alpha)/\sqrt6,\qquad
 \varphi(z)=(2\pi)^{-1/2}e^{-z^2/2}.
 \tag{OE4}
\]
It extends to a $C^2$ function on the punctured augmented space;
$C=T=0$ at $w=0,b\ne0$. The augmented bounds in NC5 apply, and Euler's identity gives
$\theta^TT=C$, including the pure-bias case. No spatial-radius lower
bound is imposed. In addition to $V,H,Q$ from NC3, define
\[
 \mathcal B=C^2+q^2\|T\|^2,\qquad
 I^2=\|P_Uw\|^2,\quad O^2=\|(I-P_U)w\|^2,
\]
\[
 A=I^2/\rho^2\quad(\rho>0),\qquad
 \mathcal M(\theta)=\max\{O^2,b^2-I^2\}
                  =O^2+(b^2-\rho^2)_+.
 \tag{OE6}
\]
Thus $O^2\le\mathcal M$ and $b^2\le\rho^2+\mathcal M$.

\subsection{Sign-free energy and normalized correlation}

Suppose $\|E\|\le\nu\le1$, $0<h\le1/128$,
$\theta\ne0$, and $|q|\le\|\theta\|$. NC6 preserves balance and
nonzero augmented hidden state; the joining hidden segment stays away
from zero because $\|\theta^+-\theta\|\le2h\|\theta\|$.
The following two sign-free estimates hold:
\[
 H^+-H\ge\frac h2\mathcal B-3h\nu^2\|\theta\|^2,
 \tag{OE8}
\]
\[
 (V-4hH)^+-(V-4hH)\le4hH+3h\nu V.
 \tag{OE9}
\]
For completeness their short derivation follows. Put
$a=C+\theta^TE$ and $G=T+E$. Taylor's theorem and NC5 give
\[
 C^+=C+hqT^TG+\mathcal E,\qquad
 |\mathcal E|\le8h^2q^2\|G\|^2/\|\theta\|.
\]
The first-order terms of
$H^+-H=hCa+hq^2T^TG+h^2qaT^TG+(q+ha)\mathcal E$
are at least
$h(3\mathcal B/4-2\nu^2\|\theta\|^2)$ by Young's inequality.
The absolute values of the last two terms sum to at most
$18h^2(\mathcal B+\nu^2\|\theta\|^2)$: respectively use
$|qaT^TG|\le3(\mathcal B+\nu^2\|\theta\|^2)/2$ and
$|(q+ha)\mathcal E|\le16(1+2h)h^2
(\mathcal B+\nu^2\|\theta\|^2)$.
This proves OE8. The exact expansion
\[
 V^+-V=4hH+4hq\theta^TE+h^2(a^2+q^2\|G\|^2)
 \le4hH+2h\nu V+2h^2\mathcal B+2h^2\nu^2V
\]
minus $4h$ times OE8 proves OE9, since $14h\nu\le1$.

The detector-free debit is already a consequence of NC8:
\[
 Q^+\ge Q-2h\nu. \tag{OE10}
\]
Its $h\le1/128$ scope is justified by the explicit $t\le1/127$
segment check in NC9; no wider step range is inferred merely from the
final theorem's $h\le1/512$ assumption. Zero correlation and zero head
are included.

For completeness, record the quotient bounds.
Writing $x=(q,\theta)$, $F=\nabla H$ and $g_Q=\nabla Q$, homogeneity gives
\[
 g_Q=(F-2Qx)/V,\qquad x^Tg_Q=0,\qquad
 \|g_Q\|\le V^{-1/2}. \tag{OE11}
\]
Indeed $\|F-2Qx\|^2=\|F\|^2-4Q^2V\le V$.
If $|q|\le2\|\theta\|$, the four terms in the quotient Hessian
$D^2(H/V)$ have bounds $17/V$, $4/V$, $1/V$ and $4/V$, so
\[
 \|D^2Q\|\le32/V. \tag{OE12}
\]
These are calculus identities, not a second trajectory comparison.

\subsection{A maximum of quadratics handles every bias transition}

For every balanced row, the specific cubic formula improves the
generic bound $|H|\le V/2$ to
\[
 |H|\le\rho^2/4.
 \tag{OE13}
\]
For $\rho>0$, $|P|\le(\sum_i|u_i^Tn|^6)^{1/2}\le1$ and balance
give
\[
 \frac{|H|}{\rho^2}
 \le c\varphi(\sqrt B)\sqrt{B(1+B)}
 \le c(1+B)\varphi(\sqrt B)
 \le2c\varphi(1)<\frac14.
\]
The last strict inequality follows from $2c\varphi(1)\le
1/\sqrt{3\pi e}<1/4$. The maximum of
$(1+B)e^{-B/2}$ occurs at $B=1$. At $\rho=0$, $C=0$, so OE13
holds there as well.

\paragraph{Favorable one-step theorem.}
If $H\ge0$, then every actual candidate step satisfies
\[
 \mathcal M^+-\mathcal M\le h^2\mathcal B+2h\nu V,
 \tag{OE14}
\]
\[
 (\mathcal M-2hH)^+-(\mathcal M-2hH)\le3h\nu V.
 \tag{OE15}
\]
There is no condition on $B$ and no condition on the next sign of $H$.

To prove OE14 first take the pure hidden candidate
$\theta^{\rm t}=\theta+hqT$. Let
$F_1=O^2$ and $F_2=b^2-I^2$; each is a quadratic form with matrix
operator norm at most one. Differentiating OE4, with no division by
$P$ or $C$, gives
\[
 DF_1(\theta)[qT]=-2(3-B)H(1-A),
\]
\[
 DF_2(\theta)[qT]=2H\{(3-B)A-2-B\}.
 \tag{OE16}
\]
For example, $w^TT_w=BC$, $bT_b=(1-B)C$, and
$w_\perp^TT_{w,\perp}=-(3-B)C(1-A)$ imply both identities.
At $\rho>0$ put $t=hH/\rho^2$. By OE13,
$0\le t\le h/4\le h$. We claim
\[
 F_i(\theta)+hDF_i(\theta)[qT]\le\mathcal M(\theta)
 \quad(i=1,2).
 \tag{OE17}
\]
For $F_1$, the derivative is nonpositive if $B\le3$. If $B>3$,
the other branch exceeds it by $(B-1)\rho^2$, while its positive
first-order change is at most $2t(B-3)\rho^2$; the gap absorbs this
for $h\le1/2$. For $F_2$, the bracket in OE16 is nonpositive when
$B\ge1$: its maximum over $A\in[0,1]$ is $1-2B$ for $B\le3$
and $-2-B$ for $B\ge3$. If $B<1$, the gap from $F_1$ is
$(1-B)\rho^2$ and the bracket is at most $1-2B$. It is nonpositive
for $B\ge1/2$; for $B<1/2$,
$2t(1-2B)\le1-B$ when $h\le1/2$. This proves OE17.
The exact quadratic second-order term of each branch is at most
$h^2q^2\|T\|^2\le h^2\mathcal B$. Taking their maximum yields
\[
 \mathcal M(\theta^{\rm t})\le\mathcal M(\theta)+h^2\mathcal B.
 \tag{OE18}
\]
All possible branch crossings are already included. At $\rho=0$,
$C=T=0$ and OE18 is immediate, so no limiting ratio is needed.

For the force displacement $\delta=hqE$, either quadratic obeys
\[
 F_i(\theta^{\rm t}+\delta)-F_i(\theta^{\rm t})
 \le2\|\theta^{\rm t}\|\|\delta\|+\|\delta\|^2
 \le2h\nu|q|\|\theta\|+(2h^2\nu+h^2\nu^2)q^2
 \le(1+3h)h\nu V\le2h\nu V.
\]
Taking maxima proves OE14. Subtracting $2h$ times OE8 gives
\[
 (\mathcal M-2hH)^+-(\mathcal M-2hH)
 \le(2h\nu+6h^2\nu^2)V\le3h\nu V,
\]
which is OE15. Consequently, on any favorable prefix starting at
an actual state $J$, including its last candidate step,
\[
 \mathcal M_n
 \le\mathcal M_J+2h(H_n-H_J)
                  +3\nu\sum_{l=J}^{n-1}hV_l
 \le\mathcal M_J+\frac h2\rho_n^2
                  +3\nu\sum_{l=J}^{n-1}hV_l.
 \tag{OE19}
\]
This is a bias-energy barrier valid even when the normalized bias
crosses one or three repeatedly, or when a spatial row vanishes.

\subsection{A public polynomial budget joins all original rows}

Fix a public horizon $T_0\ge1$ and parameters
\[
 0<h\le1/512,\qquad \kappa=\frac1{100T_0},\qquad
 0\le\nu\le\frac1{400T_0^2}.
 \tag{OE20}
\]
Suppose every pre-step force up to a considered integer state $n$
obeys NC4 with $\|E\|\le\nu$, and $nh\le T_0$. Start at any
nonzero augmented row with $|q_0|\le\|\theta_0\|$. Let $J$ be the
first integer state with $Q_J\ge2\kappa$, or $+\infty$ if none.
Then
\[
 V_l<\frac87V_0\quad(0\le l\le\min\{J,n\}),\qquad
 Q_l\ge\frac32\kappa>0\quad(J\le l\le n).
 \tag{OE21}
\]
The second statement is vacuous for $J=+\infty$.

This is NC15--16 with $\overline T=T_0$ and $N_{\max}=n$.
In particular NC12 applies to every low or first-crossing candidate,
because only its pre-update correlations must be below $2\kappa$:
\[
 V_l\le V_0\exp\{(8\kappa+5\nu)lh+4h(Q_l-Q_0)\}
       <\frac87V_0\quad(l\le\min\{J,n\}). \tag{OE22}
\]
The exponent is at most $8/100+5/400+4/512<1/8$.
After marking, NC8 pays the entire remaining force debit,
$Q_l\ge2\kappa-2\nu T_0\ge3\kappa/2$, including a proposed exit.
Thus repeated crossings of $2\kappa$ never create a new adverse phase.

Apply OE19 to the marked phase and use
$\mathcal M_J\le V_J<8V_0/7$. Before marking, simply use
$\mathcal M_l\le V_l<8V_0/7$. Every original row therefore obeys
\[
 {\displaystyle
 \mathcal M_n\le\frac87V_0+\frac h2\rho_n^2
                           +3\nu\sum_{l=0}^{n-1}hV_l.}
 \tag{OE23}
\]
In particular
\[
 O_n^2\le\frac87V_0+\frac h2\rho_n^2+3\nu\sum_{l<n}hV_l,
 \qquad
 b_n^2\le(1+h/2)\rho_n^2+\frac87V_0+3\nu\sum_{l<n}hV_l,
 \tag{OE24}
\]
\[
 V_n\le(4+h)\rho_n^2+\frac{16}7V_0+6\nu\sum_{l<n}hV_l.
 \tag{OE25}
\]
The last inequality uses balance and $V\le2(\rho^2+b^2)$.
These conclusions require no row's initial detector to be nonzero;
they retain every initial row, including either initial head sign.
An identically zero initial joint row contributes zero throughout
and can be included by continuity or directly.

\section{Simultaneous maturation of the cubic categories}\label{app:cubic-maturation}

Use Appendix~\ref{app:aj-parameters}. Here
$(A_j,W_j,B_j)=s(q_j,w_j,b_j)$, $\theta_j=(w_j,b_j)$,
$\rho_j=\|w_j\|$, and $V_n=\sum_j(q_{j,n}^2+\|\theta_{j,n}\|^2)$
is the complete normalized joint energy. Write $P_U=UU^T$,
$P_{U^\perp}=I_d-P_U$ and
$A_{{\rm sub},n}=\sum_j\|P_Uw_{j,n}\|^2/\sum_j\|w_{j,n}\|^2$.
The activation is $\sigma_\alpha(t)=\alpha t+(1-\alpha)t_+$,
the cubic constant is $c=(1-\alpha)/\sqrt6$, and $\varphi$ is the
standard normal density. In the category and Gaussian-tail
calculation here, $\delta$ denotes its $\delta_0=\delta_{\rm theorem}/2$.
Write $R_M=R$ and define $J$ to be the last first crossing of $2R_M$
among the $4r$ category representatives. The symbol $J$ in this section
is a checkpoint, while the leak factor is written $1-\alpha$.

\paragraph{Original Gaussian event without a conditional-law substitution.}
Write each initial normalized spatial, bias and head row as
$(G,B,Z)/\sqrt d$, where $G\sim N(0,I_d)$ and $B,Z\sim N(0,1)$
are independent; different rows are independent as well.
For every $i$ and $(s_i,s_b)\in\{-1,1\}^2$ require a row with
\[
 s_iu_i^TG\in[1,2],\quad
 |u_l^TG|\le\vartheta/(4\sqrt{r-1})\ (l\ne i),\quad
 s_bB\in[1/2,2],\quad -s_is_bZ\in[1/2,2].           \tag{AC13}
\]
The other-coordinate restrictions are empty at rank one.
Each specified rectangle has probability exactly $p_C$ under
the original independent Gaussian law. Its failure of occupancy
is at most $e^{-mp_C}$, so all $4r$ categories occur except with
probability $\delta/4$. Distinct categories have disjoint
rectangles, hence there are at least $4r$ distinct rows.

The elementary Gaussian tails give, except with probability
$7\delta/96$, simultaneously for every row
\[
 \rho_{j,0}\in[1/2,2],\quad \|P_UG_j\|^2\le H,\quad
 |B_j|,|Z_j|\le\sqrt{2\ell},\quad
 |q_{j,0}|\le\|\theta_{j,0}\|,\quad V_0<5m.          \tag{AC14}
\]
Here is the tail accounting explicitly. For $Z_k\sim\chi_k^2$,
the Gaussian moment-generating function yields
\[
 \log\mathbb E[e^{u(Z_k-k)}]\le\frac{ku^2}{1-2u}\quad(0<u<1/2),
 \qquad
 \log\mathbb E[e^{-u(Z_k-k)}]\le ku^2\quad(u>0).
\]
Exponential Markov with $u=\sqrt\ell/(\sqrt k+2\sqrt\ell)$
and $u=\sqrt{\ell/k}$, respectively, therefore gives
\[
 \Pr\{Z_k>k+2\sqrt{k\ell}+2\ell\}\le e^{-\ell},\qquad
 \Pr\{Z_k<k-2\sqrt{k\ell}\}\le e^{-\ell}.
\]
These are also the bounds BW6 in Appendix~\ref{app:initialization}.
For a standard normal scalar, its moment-generating function
$\mathbb E[e^{uG}]=e^{u^2/2}$ similarly gives
$\Pr\{|G|>\sqrt{2\ell}\}\le2e^{-\ell}$.
Apply the two chi-square tails to $\|G_j\|^2$, the upper tail to
$\|P_UG_j\|^2$, and the scalar bound to each of $B_j,Z_j$.
Their per-row costs sum to $(2+1+2+2)e^{-\ell}$, so the union over
all rows costs $7me^{-\ell}=7\delta/96$.
Since $d\ge128\ell$, these events give
$1/4<\|G_j\|^2/d<4$ and
$q_{j,0}^2,b_{j,0}^2\le2\ell/d\le1/64$.
Thus balance holds and each joint row energy is below $4+1/32<5$,
proving every assertion in AC14. They also give
$A_{{\rm sub},0}\le4H/d\le1/4$.

For each AC13 row the exact cubic correlation is
\[
 C(\theta)=-cb\varphi(b/\rho)P(w/\rho),\qquad
 P(n)=\sum_i a_i(u_i^Tn)^3,\quad \rho=\|w\|.
                                                               \tag{AC16}
\]
The selected signed cubic contribution in Gaussian coordinates
is at least $a_0$, whereas the magnitude of the remainder is at
most $(\vartheta/4)^3<a_0/2$.
Using AC14 gives
$|C_0|\ge k$, $q_0C_0>0$,
$|q_0|,|b_0|\in[1/(2\sqrt d),2/\sqrt d]$,
$(b_0/\rho_0)^2\le1/16$, selected coordinate at least
$1/\sqrt d$, and initial other-coordinate cone ratio at most
$\vartheta/4$. All signs and seeds are acquired properties of
these original rectangles.

\paragraph{A complete actual force induction through the public time.}
Put $v=y-y_c$ and
$\widehat f=m^{-1}\sum_lq_l\sigma_\alpha(\theta_l^T(X,1))$.
The exact actual recurrence is
\[
 q_j^+=q_j+h(C_j+\theta_j^TE_j),\qquad
 \theta_j^+=\theta_j+hq_j(\nabla C_j+E_j),\qquad
 E_j=\mathbb E[(v-s^2\widehat f)\sigma'_j(X,1)].
                                                               \tag{AC17}
\]
Here $\sigma'_j=\sigma_\alpha'(\theta_j^T(X,1))$ is a scalar;
the following $(X,1)$ in AC17 multiplies it as an augmented vector.
Thus the last expression is a vector. Every self/cross interaction and
every retained teacher term occurs in it. Duality and homogeneity give
\[
 \|\nabla C_j\|\le1,\quad
 \|\widehat f\|_2\le V/(2m),\quad
 \|E_j\|\le\|v\|_2+s^2V/(2m).                        \tag{AC18}
\]
Assume inductively the force is controlled up to a candidate step.
Exact compatibility gives
\[
 \|\theta^+\|^2-(q^+)^2
 \ge(1-h^2\|\nabla C+E\|^2)(\|\theta\|^2-q^2)\ge0,
 \qquad V^+\le(1+2h)^2V.
                                                               \tag{AC19}
\]
Nonzero augmented hidden states persist since
$\|\theta^+\|\ge(1-2h)\|\theta\|$.
Starting from AC14, every candidate $n\le N$ consequently has
\[
 V_n\le(1+2h)^{2n}V_0
       <5m e^{4nh}\le5m e^{4\overline T}=M^2/4.
                                                               \tag{AC20}
\]
At every such candidate AC18 and AJ10--11 give
$\|E_j\|\le\nu/2+\nu/8<\nu$.
The initial step satisfies the same inequality, so induction closes
the full force and radius tube through $N$, including its endpoint.
This is not a stopped proof ending at the old $512m$ energy hit.

\paragraph{Selected-row invariants on the same controlled prefix.}
The deterministic proof DC17--23 and DX27--34 applies directly to
each original category row, with its temporary cap $M$,
$R_B=M$, $k_E=k$, $a_E=a_0$, and $b_2=b_*^2$.
AC13--16 supply every starting hypothesis, and AC20 keeps every actual
candidate inside the cap through $N$. The required force cutoffs are
exactly AJ9: $k/(10^4M^2)$ for phase and bias,
$k/(64M^2\overline S\sqrt d)$ for the accumulated spatial debit,
$k\vartheta a_0/(1024M^3d)$ for the strict cone boundary, and
$k\xi/(12800M)$ for bias calibration. The mesh conditions
$h\le1/512$ and $h\le\xi/1024$ are in AJ8.
Thus those discrete proofs, including their candidate-step arguments,
give the following simultaneous invariants.

Orient the head by $\varsigma=\operatorname{sign}C_0$ and put
$p=\varsigma q$, $D=\varsigma C$, $B=b^2/\rho^2$.
The phase and balance invariant is
\[
\begin{gathered}
p>0,\quad D\ge k,\quad p\le\|\theta\|,\quad
0<B\le1,\quad \rho>1/3,\\
D^+-D\ge\tfrac12hp\|\nabla C\|^2,\qquad
p^+-p\ge3hk/4.
\end{gathered} \tag{AC21}
\]
The outside and accumulated-force invariant is
\[
 O\le2,\qquad
 \sum_{l<n}hp_l\le\frac{4M^2}{k}\overline S,\qquad
 \sum_{l<n}hp_lD_l\le\frac43\rho_n^2,
 \quad O=\|P_{U^\perp}w\|. \tag{AC22}
\]
For $x=s_iu_i^Tw$ and $u=\|P_{U\cap u_i^\perp}w\|$, the cone
invariant and the exact teacher candidate that proves it are
\[
\begin{gathered}
x_n\ge15/(16\sqrt d),\qquad u_n\le\vartheta x_n,\\
\overline x=\bar a x+d_ca_ix^2,\quad
\overline u\le\bar a u+d_cu^2,\quad
\bar a=1-(3-B)t,\quad t=hpD/\rho^2,\\
d_c=3hpc|b/\rho|\varphi(b/\rho)/\rho^2,\qquad
|\Delta_{\rm force}x|,\ \Delta_{\rm force}u\le hp\nu.
\end{gathered} \tag{AC23}
\]
DX29--30 treats the entire closed cone and strictly pays both force
terms. DC23 preserves the original bias sign. The calibrated quadratic
DX31, with its discrete debit DX32 and the force sum AC22, gives
\[
 |B-b_*^2|\le5/\rho^2+8h+100M\nu/k. \tag{AC24}
\]
Consequently, whenever $\rho\ge R_M\ge R_D$, the same actual row has
\[
 \|w/\rho-s_iu_i\|
 \le\sqrt2\sqrt{\vartheta^2+4/R_D^2}<\xi/4,\qquad
 \big||b|/\rho-b_*\big|
 \le2(5/R_D^2+8h+100M\nu/k)<\xi/4. \tag{AC25}
\]
These are simultaneous invariants through $N$, not properties of
independently evolved or frozen category rows.

\paragraph{A common acquired time and post-hit retention.}
Summing AC21 over the entire controlled prefix gives, for every category,
$p_N\ge p_0+3Nhk/4>6R_M$. Since $p_N\le\sqrt2\rho_N$,
its endpoint radius exceeds $3\sqrt2R_M>2R_M$. Hence its first
actual crossing $\tau$ of $2R_M$ exists by $N$.
Every pure teacher spatial candidate has nondecreasing radius by DC20;
the actual radius debit on any subinterval is at most
\[
 \nu\sum hp\le4M^2\nu\overline S/k\le1/(16\sqrt d).
\]
Thus after its first crossing each row retains radius
$2R_M-1/(16\sqrt d)>R_M$ and retains AC25 through $N$.
The last first-crossing time $J=\max\tau$ therefore satisfies $J\le N$,
and all $4r$ distinct category rows are mature simultaneously for
$J\le n\le N$. The stronger endpoint head bound $p_N>6R_M$
will also supply the signed diagonal mass in AJ22.

\paragraph{Every original row in the endpoint score.}
Apply the every-row invariant OE23--25 with $T_0=\overline T$.
AJ8--9 supply $h\le1/512$ and
$\nu\le1/(1024\overline T^2)<1/(400\overline T^2)$.
The NC marking invariant controls every low and marking candidate;
OE's maximum of quadratics controls the subsequent positive phase.
Their sum therefore gives, at \emph{every} integer on this prefix,
\[
 O_n^{\rm bank}\le\frac87V_0+\frac h2 S_n+3I_n,\qquad
 V_n\le(4+h)S_n+\frac{16}7V_0+6I_n,
\]
\[
 S_n=\sum_j\|w_{j,n}\|^2,\quad
 O_n^{\rm bank}=\sum_j\|P_{U^\perp}w_{j,n}\|^2,\quad
 I_n=\nu\sum_{l<n}hV_l\le\nu\overline T M^2\le m/1024.
                                                               \tag{AC26}
\]
These are the sum of OE24--25 with the complete force integral retained;
they do not require $n$ to be the first energy crossing.
For $n\ge J$, the $4r$ distinct retained category representatives
give $V_n\ge S_n\ge4rR_M^2\ge1024m$.
Using $V_0<5m$ in AC26 gives $S_n>.24V_n$ and
\[
 1-A_{{\rm sub},n}
 \le h/2+\frac{(40/7+3/1024)m}{.24V_n}<.025<.1.
                                                               \tag{AC27}
\]
This proves the complete spatial-energy conclusion for the full original denominator.
It also gives $\|W_n-W_0\|_F>8s\sqrt m$ and
$\|W_n-W_0\|_F/\|W_0\|_F>4$ at those integers.
Finally $V_J\ge1024m$ and AC19 give
$1024m\le V_J\le e^{4Jh}V_0<5me^{4Jh}$, hence
$\eta J=(m/2)Jh>(m/8)\log(1024/5)$.
This includes the first candidate hit defining each category's crossing;
the plateau bounds use the same entire prefix AC20.

\section{Complete head-weighted outside control for cubic teachers}\label{app:cubic-outside}

\subsection{Actual run and the stronger force integral already paid by AC}

Use AC's unchanged model, raw step and original draw. Here
$X\sim N(0,I_d)$, $U=(u_1,\ldots,u_r)$, $P_U=UU^T$,
$P_{U^\perp}=I_d-P_U$, and $(A_j,W_j,B_j)=s(q_j,w_j,b_j)$.
The public $R_M=R$, $M$, $k$, $\overline T$ and $\nu$ are those of
AJ7--9; $a_0=\min_i a_i$. In its normalized
coordinates the predictor and exact simultaneous recurrence are
\begin{gather}
 f_n(X)=\frac{s^2}{m}\sum_jq_{j,n}
           \sigma_\alpha(w_{j,n}^TX+b_{j,n}),\quad
 \sigma_\alpha(t)=\alpha t+(1-\alpha)t_+,\quad0\le\alpha<1,
 \quad\eta=mh/2,\notag\\
 q_j^+=q_j+h(C_j+\theta_j^TE_j),\qquad
 \theta_j^+=\theta_j+hq_j(T_j+E_j),\quad
 \theta_j=(w_j,b_j),\quad T_j=\nabla C_j,\notag\\
 C_j=\mathbb E[y_c\sigma_\alpha(\theta_j^T(X,1))],\qquad
 E_j=\mathbb E[(y-y_c-f_n)\sigma'_\alpha(\theta_j^T(X,1))(X,1)].
 \tag{CW1}
\end{gather}
The central teacher is the complete additive cubic in AJ1.  All retained
original means, linear terms and tails, and every student interaction,
are in $E_j$.  Nothing is centered or reset.  Let
\begin{gather}
 V_{j,n}=q_{j,n}^2+\|\theta_{j,n}\|^2,\quad
 \mathcal V_n=\sum_jV_{j,n},\quad
 \rho_{j,n}=\|w_{j,n}\|,\quad
 O_{j,n}=\|P_{U^\perp}w_{j,n}\|,\notag\\
 H_{j,n}=q_{j,n}C_{j,n},\quad Q_{j,n}=H_{j,n}/V_{j,n},\quad
 I_{j,n}=\nu\sum_{l<n}hV_{j,l},\qquad
 I_n=\sum_j I_{j,n}.
 \tag{CW2}
\end{gather}
The notation $O_{j,n}$ is a norm, not a squared energy.

On AC's event of original probability at least $1-\delta$, at all
integers $0\le n\le N$ one already has
\begin{gather}
 V_{j,0}<5,\quad |q_{j,n}|\le\|\theta_{j,n}\|,\quad
 \|T_{j,n}\|\le1,\quad\|E_{j,n}\|\le\nu,\notag\\
 \mathcal V_n<M^2/4,\quad Nh\le\overline T,
 \quad h\le1/512,\quad\nu\le1/64.
 \tag{CW3}
\end{gather}
The individual initial bound follows from AC14:
$\rho_{j,0}\le2$ and $q_{j,0}^2,b_{j,0}^2\le2\ell/d\le1/64$.
In particular $V_{j,0}\le4+1/32<5$; no bound on a typical row
is substituted for a simultaneous bound.

AC's existing cone-force cutoff is much stronger than the coarse
$I_n\le m/1024$ previously used.  Its constants satisfy
\begin{equation}
\begin{gathered}
\overline T=2+8R_M/k,\quad M>4R_M,\quad k\le R_M,\quad
 \vartheta\le a_0/8,\quad a_0\le1,\quad d\ge256,\\
 \qquad\nu\le\frac{k\vartheta a_0}{1024M^3d}.
\end{gathered}
 \tag{CW4}
\end{equation}
Here $k\le R_M$ follows directly from AC's displayed $k<1$ and
$R_M\ge16\sqrt m$.  Hence, without a new restriction,
\begin{equation}
 I_n\le\nu\overline T M^2
 \le\frac{(2k+8R_M)\vartheta a_0}{1024Md}
 <\frac{5\vartheta a_0}{2048d}
 \le\frac5{16384d}\le\frac5{2^{22}}<2^{-19}=:\epsilon_I.
 \tag{CW5}
\end{equation}
The actual radius bound in CW3 would improve this by a further
factor four, which is not needed.  Each $I_{j,n}$ is nonnegative
and at most $I_n$.

\subsection{Permanent marking and a new balance-defect upper bound}

Put $\kappa_0=1/(100\overline T)$ and permanently mark a row at its
first actual state with $Q_j\ge2\kappa_0$.
The NC13--18/OE20--23 estimates apply because AC includes
$\nu\le1/(1024\overline T^2)<1/(400\overline T^2)$.
Every unmarked and marking state has
\begin{equation}
 V_{j,n}<\frac87V_{j,0}<\frac{40}7.
 \quad\text{After marking,}\quad
 Q_{j,n}\ge\frac32\kappa_0>0.
 \tag{CW6}
\end{equation}
The latter assertion holds at all later actual states, not just while
the row stays above the original marking threshold.  In particular
its head and cubic correlation have the same nonzero sign.
NC obtains this from the exact debit $Q^+\ge Q-2h\nu$ and a
compensated low-phase energy estimate that includes the crossing
step.  No post-marking Gaussianity or minimum initial detector is used.

We also use the sign-free OE8 inequality and exact cubic
bound OE13, with $\mathcal B_j=C_j^2+q_j^2\|T_j\|^2$:
\begin{equation}
 H_j^+-H_j\ge\frac h2\mathcal B_j
                  -3h\nu^2\|\theta_j\|^2,
 \qquad |H_j|\le\rho_j^2/4.
 \tag{CW7}
\end{equation}
The latter remains valid at zero spatial weight.  It follows from
$C=-cb\varphi(b/\rho)P(w/\rho)$, $|P|\le1$ and balance, since
$c\varphi(\sqrt B)\sqrt{B(1+B)}<1/4$ for $B\ge0$.
The upper bound on the \emph{balance defect}
$D_j=\|\theta_j\|^2-q_j^2\ge0$ is new here.
Exact compatibility, followed by CW7, gives
\begin{align}
 D_j^+-D_j
 &=h^2\{q_j^2\|T_j+E_j\|^2-(C_j+\theta_j^TE_j)^2\}\notag\\
 &\le2h^2\mathcal B_j+2h^2\nu^2V_j
 \le4h(H_j^+-H_j)+14h^2\nu^2V_j.
 \tag{CW8}
\end{align}
Let $\tau_j$ be its finite marking time.  Sum CW8 from $\tau_j$
through the current state, use $H_{j,\tau_j}>0$ and CW6--7, and
enlarge the nonnegative force integral to start at zero.  Every
marked state therefore satisfies
\begin{equation}
 {\displaystyle \quad
 D_{j,n}\le\frac87V_{j,0}+h\rho_{j,n}^2
                           +14h\nu I_{j,n}.\quad}
 \tag{CW9}
\end{equation}
It includes the marking state itself.  This upper bound is separate
from the preserved nonnegative balance lower bound in CW3.

OE's all-bias maximum-of-quadratics estimate supplies, for every row
at every state, including unmarked rows and every branch crossing,
\begin{gather}
 b_{j,n}^2\le(1+h/2)\rho_{j,n}^2+\frac87V_{j,0}+3I_{j,n},\notag\\
 V_{j,n}\le(4+h)\rho_{j,n}^2+\frac{16}7V_{j,0}+6I_{j,n}.
 \tag{CW10}
\end{gather}
These are precisely OE24--25 with the full actual row force integral.
They require no axis cone or initial head sign.

Every state with $\rho_{j,n}\ge4$ is already permanently marked,
because its energy is at least $16>40/7$.  At such a state, CW9--10
and CW5 imply
\begin{gather}
 \frac{q_j^2}{\rho_j^2}
 \ge1-h-\frac{40/7+14h\nu\epsilon_I}{16}>\frac{16}{25},\notag\\
 B_j:=\frac{b_j^2}{\rho_j^2}
 \le1+h/2+\frac{40/7+3\epsilon_I}{16}<\frac{34}{25}.
 \tag{CW11}
\end{gather}
For explicit conservative arithmetic, $14h\nu\epsilon_I<1/1024$
makes the first lower bound exceed
$1-1/512-5/14-1/16384>16/25$.
The second upper bound is below $1+1/1024+5/14+3/(16\,2^{19})<34/25$.
These are acquired all-row bounds at a fixed absolute entry radius;
they are not assumptions about the unselected rows.

\subsection{The exact head--outside product contracts above that radius}

At an actual state with $\rho\ge4$, omit the row subscript and put
$p=|q|$, $O=\|P_{U^\perp}w\|$, and
$t=hH/\rho^2$.  By CW6--7, $0<t\le h/4$.
Use one pure-teacher step only as an algebraic intermediate from
the current actual state.  The exact cubic identities give
\begin{equation}
 p^{\rm t}=p+h|C|=p\left(1+t\frac{\rho^2}{q^2}\right),\qquad
 w_\perp^{\rm t}=[1-(3-B)t]w_\perp.
 \tag{CW12}
\end{equation}
The outside multiplier is in $[0,1]$, because $B<34/25$ and
$t\le h/4$.  Set $a=\rho^2/q^2$ and $b=3-B$.  CW11 implies
$a\le25/16$, $b\ge41/25$, hence $b-a\ge31/400$.
Their product, with its favorable quadratic term retained, satisfies
\begin{equation}
 p^{\rm t}O^{\rm t}
 =(1+at)(1-bt)pO
 \le\left(1-\frac{31}{400}t\right)pO.
 \tag{CW13}
\end{equation}
No selected-axis geometry appears in this exact outside-
\emph{teacher-subspace} calculation.

The full actual step still includes both compatible errors:
$|q^+-q^{\rm t}|\le h\nu\|\theta\|$ and
$\|w^+-w^{\rm t}\|\le hp\nu$.
Since $O^{\rm t}\le O\le\|\theta\|$ and
$p^{\rm t}\le p+h\|\theta\|$, expanding their product gives
\begin{align}
 |q^+|O^+
 &\le p^{\rm t}O^{\rm t}
   +h\nu\{\|\theta\|O^{\rm t}+pp^{\rm t}\}
   +h^2\nu^2p\|\theta\|\notag\\
 &\le\left(1-\frac{31}{400}t\right)pO
    +h\nu(1+h/2+h\nu/2)V\notag\\
 &\le\left(1-\frac{31}{400}t\right)pO+2h\nu V.
 \tag{CW14}
\end{align}
This estimate is valid for an arbitrary adaptive coupled $E_j$ of
the actual run.  It charges both the hidden force and the head force;
dropping the latter would not justify the weighted estimate.

\subsection{All low states, candidate crossings, and every original row}

At a pre-step state with $\rho_j<4$, CW10 gives
\begin{equation}
 V_j<(4+1/512)16+80/7+6\epsilon_I<76.
 \tag{CW15}
\end{equation}
The exact compatible update has
$V_j^+\le(1+2h)^2V_j$, by balance and $\|T_j+E_j\|\le2$.
Thus its next candidate state, whether or not it crosses radius four
or the marking threshold, has
$V_j^+<(257/256)^2\,76<77$.  In particular
$|q_j^+|O_j^+\le V_j^+/2<39$.
An unmarked or marking state already has $V_j<40/7$ by CW6, so it
cannot be an omitted high-radius adverse state.

Starting from $|q_{j,0}|O_{j,0}\le V_{j,0}/2<5/2$, combine
CW14 at high-radius pre-step states with CW15 at every other
pre-step state.  Induction proves, for every original row and every
integer $0\le n\le N$,
\begin{equation}
 {\displaystyle \quad |q_{j,n}|O_{j,n}\le39+2I_{j,n}<40.\quad}
 \tag{CW16}
\end{equation}
There is no need to assume a single radius crossing.  Every return
below radius four is handled by CW15, and every new exit candidate
has the same bound.  Permanent marking only supplies the positive
phase at high states; no unmarked row is discarded.
The rowwise integrals sum exactly, giving the stronger full-bank
statements
\begin{gather}
 \sum_j |q_{j,n}|O_{j,n}\le39m+2I_n,\notag\\
 {\displaystyle \quad
 \sum_j q_{j,n}^2 O_{j,n}^2
 \le1521m+156I_n+4I_n^2<1522m.\quad}
 \tag{CW17}
\end{gather}
The last inequality uses $I_n<2^{-19}$ and $m\ge1$.
This is a complete-bank head-weighted outside energy estimate, not
an estimate on the $4r$ category representatives alone.

\subsection{Exact full outside AGOP and the remaining inside obligation}

Let $\mathsf G_n=\mathbb E[\nabla_x f_n\nabla_x f_n^T]$ be the
actual raw student AGOP.  Positive homogeneity with $s>0$ gives
the exact derivative
\begin{equation}
 P_{U^\perp}\nabla_x f_n
   =\frac{s^2}{m}\sum_jq_{j,n}P_{U^\perp}w_{j,n}
                  \sigma'_\alpha(\theta_{j,n}^T(X,1)).
 \tag{CW18}
\end{equation}
The $L^2$ triangle inequality and $|\sigma'_\alpha|\le1$ retain
all cross-row terms and yield
\begin{equation}
 {\displaystyle \quad
 \operatorname{tr}(P_{U^\perp}\mathsf G_nP_{U^\perp})
 \le\frac{s^4}{m^2}(39m+2I_n)^2<1522s^4,\qquad
 \|P_{U^\perp}\mathsf G_nP_{U^\perp}\|_{\rm op}<1522s^4.
 \quad}
 \tag{CW19}
\end{equation}
The operator bound follows because this is positive semidefinite.
It is uniform through AC's entire actual horizon, independent of
$M$ and later row-radius growth.  No cancellation was assumed in
obtaining it, and the linear part of leaky ReLU was not removed.

\section{Original-coordinate anchors and the cubic signed frame}\label{app:cubic-anchors}

\subsection{The exact map and ordering without a head-sign assumption}
Here $X\sim N(0,I_d)$, $h_3(t)=(t^3-3t)/\sqrt6$,
$\sigma_\alpha(t)=\alpha t+(1-\alpha)t_+$, and $\varphi$ is the standard
normal density. With $U=(u_1,\ldots,u_r)$, write $P_U=UU^T$ and
$P_{U^\perp}=I_d-P_U$; later uses of $U$ as a span refer to the
same teacher subspace. The normalized student is
$\widehat f=m^{-1}\sum_jq_j\sigma_\alpha(w_j^TX+b_j)$,
$\sigma'_j=\sigma'_{\alpha,j}=\sigma'_\alpha(\theta_j^T(X,1))$
denotes a scalar gate, and $E_w$ is the spatial component of $E$.
Absorb coefficient signs into the orthonormal teacher axes and write
\[
 y_c=\sum_{i=1}^r a_i h_3(u_i^TX),\quad
 a_i>0,\quad\sum_i a_i^2=1,\quad
 c=(1-\alpha)/\sqrt6,\quad a_{\max}=\max_i a_i.
 \tag{WG1}
\]
Use the unchanged compatible recurrence for an original row,
\[
 \begin{gathered}
 q^+=q+h(C+\theta^TE),\quad
 \theta^+=\theta+hq(\nabla C+E),\\
 \theta=(w,b),\quad |q|\le\|\theta\|,
 \quad\|E\|\le\nu\le1,\quad0<h\le1/512.
 \end{gathered}
 \tag{WG2}
\]
In the original raw model $(A,W,B)=s(q,w,b)$ and full-MSE step
$\eta=mh/2$, the actual force is
$E_j=\mathbb E[(y-y_c-s^2\widehat f)\sigma'_j(X,1)]$.
All full-teacher and student terms remain in it. The proofs allow
adaptive forces but do not assert that every allowed force is generated
by one fixed teacher. Let $\rho=\|w\|>0$, $z=b/\rho$, and define
\[
 \begin{gathered}
 F(z)=-cz\varphi(z),\quad
 S(w)=\sum_i a_i(u_i^Tw)^3,\quad
 C=F(z)S(w)/\rho^2,\\
 V=q^2+\rho^2+b^2,\quad H=qC,\quad Q=H/V.
 \end{gathered}
 \tag{WG3}
\]
The weighted spatial coordinates $x_i=a_i u_i^Tw$ obey exactly
\[
 \begin{split}
 x_i^+&=A x_i+D x_i^2+\eta_i,\qquad
 \eta_i=hq a_i u_i^TE_w,\\
 A&=1+(z^2-3)hH/\rho^2,\qquad
 D=3hqF(z)/\rho^2,\qquad
 P_{U^\perp}w^+=A P_{U^\perp}w+hqP_{U^\perp}E_w.
 \end{split}
 \tag{WG4}
\]
Indeed differentiation of $F(z)S(w)/\rho^2$, using
$zF'(z)/F(z)=1-z^2$ away from zero and continuity at zero,
gives the displayed radial and coordinate terms.

The following bounds hold with either sign of $H$:
\[
 |A-1|\le h,\qquad |D|\max_i|x_i|\le h,
 \qquad A+Dx_i\ge1-2h,
 \qquad A+D(x_i+x_l)\ge1-3h.
 \tag{WG5}
\]
For details, $|S|/\rho^3\le1$, $|q|/\rho\le\sqrt{1+z^2}$,
and $|x_i|/\rho\le a_{\max}\le1$. With $B=z^2$,
\[
 \frac{|A-1|}{h}
 \le c(B+3)\sqrt{B(1+B)}\varphi(\sqrt B)
 \le\frac{(8+4\sqrt5)e^{-\sqrt5/2}}{\sqrt{12\pi}}<1,
\]
\[
 \frac{|D|\max_i|x_i|}{h}
 \le3c\sqrt{B(1+B)}\varphi(\sqrt B)
 \le6c\varphi(1)<1.
 \tag{WG6}
\]
The first maximum follows by differentiating
$(B+1)(B+3)e^{-B/2}$; its maximum is at $B=\sqrt5$.
The second uses the maximum of $(B+1)\varphi(\sqrt B)$ at $B=1$.

Consequently the pure map ($E=0$) preserves every coordinate sign
and every strict weighted-coordinate order, even in an adverse head
phase. It preserves both original extrema, the weighted maximum and
minimum. For a forced trajectory there is the exact identity
\[
 \Delta_{il,n}=G_{il,n}\left\{\Delta_{il,0}
       +\sum_{v=0}^{n-1}\frac{\eta_{i,v}-\eta_{l,v}}
                                  {G_{il,v+1}}\right\},\quad
 G_{il,n}=\prod_{v<n}[A_v+D_v(x_{i,v}+x_{l,v})]>0,
 \tag{WG7}
\]
where $\Delta_{il}=x_i-x_l$. Thus the absolute sum of the displayed
force debits being less than $|\Delta_{il,0}|$ preserves its sign.
The same formula for a single coordinate uses factors $A_v+D_vx_{i,v}$
and forces $\eta_{i,v}$. These are actual-history identities, including
the final candidate step, not reference-trajectory approximations.
At a zero spatial row the pure cubic spatial gradient is zero; the
ratio statements below begin only after $\rho>0$.

\subsection{Original-coordinate anchors and the complete signed frame}

\paragraph{Inherited actual run and explicit additional costs.}
Use the complete normalized cubic teacher, full retained target,
initial law and simultaneous raw GD of Appendix~\ref{app:aj-parameters}:
\[
\begin{gathered}
y_c=\sum_{i=1}^r a_i h_3(u_i^TX),\quad
 a_i>0,\quad \sum_i a_i^2=1,\\
 a_0=\min_i a_i,\quad a_{\max}=\max_i a_i,\quad
 U=\operatorname{span}\{u_i\}.
\end{gathered}
 \tag{FA1}
\]
All original raw coordinates, including heads and biases, remain
independent $N(0,s^2/d)$, and $(A_j,W_j,B_j)=s(q_j,w_j,b_j)$.
The step is one fixed common raw step $\eta=mh/2$ throughout.
For each original row the exact actual recurrence is
\[
 \begin{gathered}
 q^+=q+h(C+\theta^TE),\qquad
 \theta^+=\theta+hq(\nabla C+E),\qquad \theta=(w,b),\\
 C=\mathbb E[y_c\sigma_\alpha(\theta^T(X,1))],\qquad
 E_j=\mathbb E[(y-y_c-f)\sigma'_{\alpha,j}(X,1)] .
 \end{gathered}
 \tag{FA2}
\]
Here $f=s^2m^{-1}\sum_jq_j\sigma_\alpha(w_j^TX+b_j)$ is the
complete current student. The force includes every student
self/cross interaction and all retained teacher terms. It is
not a frozen-Gram, teacher-only or reference update.

Fix AC's original confidence parameter $\delta_{\rm AC}$ and an
additional $0<\zeta<1$. Its existing Gaussian quantities $H,d,m$
and coefficients define
\[
 \begin{gathered}
 \tau=\frac{\zeta a_0}{4mr^2},\qquad
 g_0=\frac{\tau}{\sqrt d},\qquad
 b_0=a_{\max}\sqrt{H/d},\qquad
 D_a=\sum_i a_i^{-2},\qquad
 L_a=(a_{\max}^2D_a)^{1/3},\\
 Z_*=
 40+640\sqrt{D_a}\left\{\frac{b_0^2}{g_0}
                                  +(1+L_a)b_0\right\}.
 \end{gathered}
 \tag{FA3}
\]
These are public deterministic constants. In particular
$g_0\le1/4$, since $d,m,r\ge1$ and $a_0,\zeta\le1$.
One may first enlarge the public milestone $R_M$ to any value
at least $\max(R_D,16\sqrt m)$, then define AC's $T,\overline T,M$
and all its subsequent budgets from that chosen value. The AC
and CW proofs use only this lower bound on $R_M$, and their
displayed definitions of $T,\overline T,M$; their conclusions
therefore continue to hold with this consistent replacement.
In addition to all the AC cutoffs require
\[
 0<h\le \frac1{16\overline T},\qquad
 0<\nu\le
 \frac{g_0e^{-4\overline T}}{10^5\overline T M}.
 \tag{FA4}
\]
Here $\nu$ is decreased before specifying the retained neighborhood
$\|y-y_c\|_{H^1}\le\nu/2$ and the original scale $s$ in AJ10--11.
The actual AC force closure still gives $\|E_j\|\le\nu$ on
$0\le n\le N=\lceil T/h\rceil$; no force is removed from FA2.
The complete run satisfies
\[
 \begin{gathered}
 Nh\le\overline T,\quad h\le1/512,\quad \nu\le1,\quad
 |q_j|\le\|\theta_j\|,\quad
 \mathcal V_n:=\sum_j(q_j^2+\|\theta_j\|^2)<M^2/4,\\
 |q_{j,n}|\,O_{j,n}<40,\qquad
 O_{j,n}:=\|P_{U^\perp}w_{j,n}\|.
 \end{gathered}
 \tag{FA5}
\]
The last inequality is the complete-row conclusion CW16.
Decreasing $h$ and $\nu$ preserves all its hypotheses.

\paragraph{The simultaneous original-coordinate event.}
Write the original projected Gaussian coordinates as
$u_i^Tw_{j,0}=G_{ji}/\sqrt d$ and set
$x_{ji}=a_i u_i^Tw_j$. For any one coordinate and distinct pair,
Gaussian density bounds give
\[
 \mathbb P(|x_{ji,0}|\le g_0)\le\tau/a_0,\qquad
 \mathbb P(|x_{ji,0}-x_{jl,0}|\le g_0)
                 \le\tau/(a_0\sqrt\pi).
 \tag{FA6}
\]
Indeed the second Gaussian has standard deviation
$\sqrt{a_i^2+a_l^2}/\sqrt d\ge\sqrt2a_0/\sqrt d$.
Union bounding the $mr$ coordinates and the
$mr(r-1)/2$ unordered gaps shows that their total failure
probability is at most
\[
 \frac{\zeta}{4r}
       +\frac{\zeta(r-1)}{8r\sqrt\pi}<\frac\zeta2.
 \tag{FA7}
\]
Intersect this event with the unchanged AC event. On that
intersection, simultaneously for all rows,
\[
 |x_{ji,0}|>g_0,\quad
 |x_{ji,0}-x_{jl,0}|>g_0\ (i\ne l),\quad
 |x_{ji,0}|\le b_0,\quad
 O_{j,0}^2\ge\frac14-\frac Hd\ge\frac3{16}.
 \tag{FA8}
\]
The upper bound and outside lower bound use AC's existing
$\|P_Uw_{j,0}\|^2\le H/d$, $\rho_{j,0}\ge1/2$ and $d\ge16H$.
No later conditional Gaussian distribution is used. The
intersection has original probability at least
$1-\delta_{\rm AC}-\zeta/2$; for total confidence $1-\delta$,
one can take $\delta_{\rm AC}=\delta/2$ and $\zeta=\delta$.
These are events in the original law, not sampling filters.

\paragraph{Exact scalar dynamics through every adverse phase.}
Fix one original row and omit its row subscript. Put
$\rho=\|w\|$, $z=b/\rho$, $F(z)=-(1-\alpha)z\varphi(z)/\sqrt6$,
and $\mathsf H=qC$. As long as $\rho>0$, WG4--5 give the
exact simultaneous coordinate and outside recurrences
\[
 \begin{gathered}
 x_i^+=A x_i+D x_i^2+\eta_i,\qquad
 \eta_i=hq a_i u_i^TE_w,\qquad |\eta_i|\le hM\nu,\\
 A=1+(z^2-3)h\mathsf H/\rho^2,\qquad
 D=3hqF(z)/\rho^2,\qquad
 w_\perp^+=A w_\perp+hqE_\perp,\\
 |A-1|\le h,\quad |D|\max_i|x_i|\le h,\quad
 A+Dx_i\ge1-2h,\quad
 A+D(x_i+x_l)\ge1-3h .
 \end{gathered}
 \tag{FA9}
\]
These bounds are sign-free. In particular they hold before any
favorable phase or permanent marking. All scalar multipliers
in FA9 are positive. Define the actual-history linear anchor
\[
 \Lambda_0=1,\qquad \Lambda_n=\prod_{v<n}A_v.
 \quad\text{Then}\quad
 \Lambda_n\ge e^{-2nh}\ge e^{-2\overline T}.
 \tag{FA10}
\]
The bound follows from
$\log(1-h)\ge-h/(1-h)\ge-2h$. Similarly every product of
coordinate or pair-gap factors in FA9 is at least
$e^{-4nh}$, since
$\log(1-3h)\ge-3h/(1-3h)\ge-4h$.

For a coordinate or difference, write the exact variation formula
along the actual history, as in WG7:
\[
 \begin{aligned}
 x_{i,n}&=P_{i,n}
   \left(x_{i,0}+\sum_{v<n}\frac{\eta_{i,v}}{P_{i,v+1}}\right),
 \\
 P_{i,n}&=\prod_{v<n}(A_v+D_vx_{i,v}),\\
 x_{i,n}-x_{l,n}&=P_{il,n}
   \left(x_{i,0}-x_{l,0}
       +\sum_{v<n}\frac{\eta_{i,v}-\eta_{l,v}}{P_{il,v+1}}\right),
 \\
 P_{il,n}&=\prod_{v<n}[A_v+D_v(x_{i,v}+x_{l,v})].
 \end{aligned}
 \tag{FA11}
\]
Each total absolute debit inside parentheses is at most
$2\overline T M\nu e^{4\overline T}\le2g_0/10^5<g_0/2$.
Consequently every original sign and strict order is preserved
through every candidate endpoint, with the uniform bounds
\[
 |x_{i,n}|\ge\frac{g_0}{2}e^{-4\overline T},\qquad
 |x_{i,n}-x_{l,n}|\ge\frac{g_0}{2}e^{-4\overline T}.
 \tag{FA12}
\]
No transported favorable-entry gap is assumed.

For completeness these statements are not circular at $\rho=0$.
Let $e_0=w_{\perp,0}/O_0$, which exists by FA8. Up to any
first candidate exit one has the exact projection identity
\[
 \langle w_{\perp,n},e_0\rangle
 =\Lambda_n\left\{
 O_0+\sum_{v<n}
       \frac{h q_v\langle E_{\perp,v},e_0\rangle}{\Lambda_{v+1}}
                  \right\}.
 \tag{FA13}
\]
The absolute sum is at most
$\overline T M\nu e^{2\overline T}
 \le g_0e^{-2\overline T}/10^5<1/8<O_0/2$.
Thus induction closes $\rho_n\ge O_n>0$ and proves
\[
 O_n\ge \frac{O_0}{2}\Lambda_n,\qquad
 |q_n|\Lambda_n<\frac{80}{O_0}<320.
 \tag{FA14}
\]
The second bound uses the full actual weighted outside estimate
in FA5, including its compatible head-force payment.

\paragraph{A discrete pair anchor with only a second-order pure debit.}
For any distinct coordinates define the positive quantity
\[
 K_{il,n}=\frac{|x_{i,n}x_{l,n}|}
                  {|x_{i,n}-x_{l,n}|\Lambda_n}.
 \tag{FA15}
\]
Its denominators never vanish by FA10--12.
From the current actual pre-state form just the algebraic
teacher intermediate
$x_i^{\rm t}=x_i(A+Dx_i)$ and
$\Lambda^+=A\Lambda$; this does not define another trajectory.
Factoring the difference gives exactly
\[
 \frac{K_{il}^{\rm t}}{K_{il}}
 =\frac{(A+Dx_i)(A+Dx_l)}
              {A[A+D(x_i+x_l)]}
 =1+\frac{D^2x_i x_l}{A[A+D(x_i+x_l)]}.
 \tag{FA16}
\]
All factors are positive. If the pair has opposite signs, this
ratio is at most one. If it has the same sign, it is at most
$1+3h^2$: the numerator of the last fraction is at most $h^2$
and its denominator is at least $(1-h)(1-3h)>1/3$.
Both statements allow either sign of $D$, so include all
adverse head/bias phases and their reversals.

The additive force debit has no hidden second trajectory or
uncontrolled Hessian term. At both endpoints of the segment
from $(x_i^{\rm t},x_l^{\rm t})$ to $(x_i^+,x_l^+)$, every
coordinate and their difference has the same sign as originally.
By FA9 and FA12 their absolute values throughout this straight
segment are at least
$(g_0/3)e^{-4\overline T}$. Along the segment,
\[
 d\log\frac{|x_i x_l|}{|x_i-x_l|}
 =\frac{dx_i}{x_i}+\frac{dx_l}{x_l}
                        -\frac{dx_i-dx_l}{x_i-x_l}.
 \tag{FA17}
\]
Integrating this exact derivative, and using
$|\eta_i|,|\eta_l|\le hM\nu$, yields
\[
 |\log K_{il}^+-\log K_{il}^{\rm t}|
 \le\frac{12hM\nu e^{4\overline T}}{g_0}.
 \tag{FA18}
\]
The anchor is the same $A\Lambda$ for the intermediate and
actual endpoint, so it creates no extra term here.
Iteration of FA16--18 proves
\[
 \begin{split}
 K_{il,n}
 &\le K_{il,0}
    \exp\!\left(3nh^2+
             \frac{12\overline T M\nu e^{4\overline T}}{g_0}\right)
 <2K_{il,0},\\
 K_{il,n}
 &\le K_{il,0}
    \exp\!\left(
             \frac{12\overline T M\nu e^{4\overline T}}{g_0}\right)
 <2K_{il,0}\qquad (x_{i,0}x_{l,0}<0).
 \end{split}
 \tag{FA19}
\]
Indeed $3nh^2\le3/16$ and the remaining exponent is at most
$12/10^5$; their sum is below $\log2$.
For an initial same-sign pair $K_{il,0}\le b_0^2/g_0$.
For an initial opposite-sign pair,
$K_{il,0}=|x_{i,0}x_{l,0}|/(|x_{i,0}|+|x_{l,0}|)\le b_0$.
The force-free result is exact up to the favorable $O(h^2)$
discrete defect; the exponential cost appears solely in paying
the actual force against original absolute coordinate/gap floors.

\paragraph{Any favorable actual row has an original extremal axis.}
Consider any one time $n\le N$ at which $\mathsf H_n=q_nC_n>0$.
There is no assumption on the row's past sign history. Define
$\epsilon=\operatorname{sign}(q_nF(z_n))$ and $y_i=\epsilon x_{i,n}$.
Because $C=F\sum_i x_i^3/(a_i^2\rho^2)$,
\[
 \sum_i\frac{y_i^3}{a_i^2}>0.
 \tag{FA20}
\]
In particular the oriented maximum $p=y_*>0$ exists and is
unique. Its label is the original maximum of $\epsilon x_{i,0}$
by FA12. The orientation $\epsilon$ is only an analysis choice
at this actual endpoint; all pair bounds already hold for both
orientations and every time, so it need not have been fixed earlier.

If $0<y_i<p$ is a positive competitor, then
\[
 y_i\le\frac{p y_i}{p-y_i}
       =\Lambda_n K_{*i,n}
       \le\frac{2b_0^2}{g_0}\Lambda_n.
 \tag{FA21}
\]
If there are negative coordinates, let $-\ell<0$ be the most
negative one, with axis label $l$. Positivity in FA20 implies
\[
 \frac{\ell^3}{a_l^2}
 <\sum_{y_i>0}\frac{y_i^3}{a_i^2}
 \le p^3D_a,\qquad \ell\le L_a p .
 \tag{FA22}
\]
The opposite-sign anchor consequently gives
\[
 \ell=
 \left(1+\frac{\ell}{p}\right)
       \frac{p\ell}{p+\ell}
 \le (1+L_a)\Lambda_n K_{*l,n}
 \le 2(1+L_a)b_0\Lambda_n.
 \tag{FA23}
\]
Every negative competitor has magnitude at most $\ell$.
This argument covers an arbitrary mixed-sign history.
If all oriented coordinates are positive only FA21 is needed;
an all-negative orientation is excluded by FA20. If $r=1$
there are no competitors and the inside transverse norm is zero.

Combining FA21--23, summing squared spatial components and
using FA14 proves the uniform complete-row conclusion
\[
 \begin{gathered}
 |q_n|\|P_{U\cap u_*^\perp}w_n\|
 \le 640\sqrt{D_a}
          \left\{\frac{b_0^2}{g_0}+(1+L_a)b_0\right\},\\
 {\displaystyle \quad
 |q_n|\|P_{u_*^\perp}w_n\|\le Z_*
 \quad(q_nC_n>0,\ 0\le n\le N).
 \quad}
 \end{gathered}
 \tag{FA24}
\]
The outside term is at most $40$ by FA5 and is combined by
the triangle inequality. The result is uniform over all original
rows, without selecting a favorable subbank. Its deterministic
constant is independent of the radius envelope $M$ and of the
clock $T$. In particular
$b_0^2/g_0=a_{\max}^2H/(\tau\sqrt d)$ is polynomial in the
original Gaussian parameters; no later relative gap occurs.

\paragraph{Every signed frame contribution is retained.}
Let $v_j=U^Tw_j/\rho_j$ and put
\[
 t_j=-\frac{(1-\alpha)s^2}{\sqrt2m}
                q_j\rho_j z_j\varphi(z_j),\qquad
 \mathsf D_n=\sum_j t_j v_j(v_j^{\odot2})^T.
 \tag{FA25}\label{eq:cubic-signed-frame}
\]
This is the complete signed inside gradient-Hermite frame in the cubic derivative frame,
not a Gram matrix made from selected rows. For each favorable
row, choose its endpoint winner from FA20 and note
$t_j\epsilon_j=|t_j|$ and $\epsilon_j u_*^Tw_j>0$.
The hemisphere normalization inequality and a rank-one
telescoping expansion give
\[
 \begin{split}
 \|w_j/\rho_j-\epsilon_j u_*\|
 &\le \sqrt2\,\|P_{u_*^\perp}w_j\|/\rho_j,\\
 \|t_jv_j(v_j^{\odot2})^T-|t_j|e_*e_*^T\|_{\rm op}
 &\le3|t_j|\|w_j/\rho_j-\epsilon_j u_*\|\\
 &\le\frac{3(1-\alpha)\varphi(1)s^2}{m} Z_* .
 \end{split}
 \tag{FA26}
\]
Here $\|v_j\|\le1$ and the squared-coordinate map is
$2$-Lipschitz on the unit ball; the last line uses
$\sup_z|z|\varphi(z)=\varphi(1)$.
Thus every favorable row contributes a nonnegative diagonal
mass plus an explicitly bounded error. No sign of a head
has been chosen or changed by the algorithm.

For any row with $\mathsf H_j\le0$, CW6 implies that it is
unmarked, so $V_j<40/7$. The coarser uniform bound
$V_j<512$ therefore suffices. For each such row,
\[
 \|t_jv_j(v_j^{\odot2})^T\|_{\rm op}
 \le |t_j|
 \le \frac{(1-\alpha)\varphi(1)s^2}{\sqrt2m}|q_j|\rho_j
 <\frac{256(1-\alpha)\varphi(1)s^2}{\sqrt2m}.
 \tag{FA27}
\]
This also covers $\mathsf H_j=0$: such a row may have a nonzero
frame, and it has not been silently discarded.
Summing all $m$ row bounds gives the exact full-bank
decomposition
\[
 \begin{gathered}
 \mathsf D_n=\mathsf D_{+,n}+\mathsf E_n,\qquad
 \mathsf D_{+,n}=
       \sum_{\mathsf H_{j,n}>0}|t_{j,n}|
                            e_{*,j}e_{*,j}^T\succeq0,\\
 {\displaystyle \qquad
 \|\mathsf E_n\|_{\rm op}
 \le(1-\alpha)\varphi(1)s^2(3Z_*+182),\qquad0\le n\le N .
 \qquad}
 \end{gathered}
 \tag{FA28}
\]
The diagonal may have repeated winner axes; coverage of every
axis must come from the separately acquired AC categories.
Together with their positive diagonal lower bounds and CW's
complete outside AGOP estimate, FA28 is a paid interface for a
full spectral theorem. Such a theorem must still state its
public radius choice, actual eigenvalue floor, gap and angle
comparison; none is inferred merely from a count of winners here.

\section{The identical projected-bank diagnostic}\label{app:cubic-prediction}

\subsection{Exact inputs and the diagnostic being compared}
Use the normalized additive cubic and full target of AC. Here
$X\sim N(0,I_d)$, $h_3(t)=(t^3-3t)/\sqrt6$,
$\sigma_\alpha(t)=\alpha t+(1-\alpha)t_+$ with $0\le\alpha<1$, and $\varphi,\Phi$ are the
standard normal density and CDF. The imported $R_M=R$, $\xi$, $b_*$
and $\nu$ are the public constants in AJ3 and AJ7--9;
$\|g\|_2$ denotes its Gaussian $L^2$ norm. Thus
\[
 y_c=\sum_{i=1}^r a_i h_3(u_i^TX),\quad a_i>0,\quad
 \sum_i a_i^2=1,\quad U^TU=I_r,\quad
 \|y\|_2=1,\quad\|y-y_c\|_{H^1}\le\nu/2.
 \tag{PJ1}
\]
Here $U$ denotes the matrix of teacher axes and $P_U=UU^T$.
All nonzero coefficient signs may be absorbed into those axes.
The target includes its actual mean, linear part and all higher
tails, as in AC; the current result does not extend AC's acquired
dictionary to an unrestricted $H^1$ teacher.
At the same fixed public integer $N$ in AJ8, choose for analysis
one of the acquired representatives for each
$(i,\epsilon,\tau)\in\{1,\ldots,r\}\times\{-1,1\}^2$.
There are $4r$ distinct original rows. For each selected row,
\begin{gather}
 (A_j,W_j,B_j)=s(q_j,w_j,b_j),\quad
 \rho_j=\|w_j\|\ge R_M>7,\quad n_j=w_j/\rho_j,\quad z_j=b_j/\rho_j,
 \nonumber\\
 \|(n_j,z_j)-(\epsilon u_i,\tau b_*)\|\le\xi.
 \tag{PJ2}
\end{gather}
These are acquired conclusions of AC21--25; all unchosen rows remain
in training, in the diagnostic bank and in the AGOP.

Let $G_t=\mathbb E[\nabla f_t\nabla f_t^T]$ be the AGOP of the
actual full original student at $t=0,N$. Assume $r<d$ and
\begin{gather}
 \lambda_r(G_N)>\lambda_{r+1}(G_N)\ge0,\qquad
 \lambda_{\min}(U^TP_NU)\ge\mu>0,\nonumber\\
 P_N=\hbox{the leading rank-$r$ projector of }G_N.
 \tag{PJ3}
\end{gather}
Thus the $r$ modes are nonzero and the boundary is genuine.
At initialization the same projector $P_0$ exists and is unique
almost surely by IA, for every $m\ge r$, including $m\ge d$.

Fix one public common multiplier $\kappa>0$ and define at both times
\begin{gather}
 \psi^P_{j,t}(X)=\kappa\sigma_\alpha((P_tW_{j,t})^TX+B_{j,t}),
 \qquad
 B_\kappa=\frac{C_\alpha}{7\kappa s},\nonumber\\
 \mathcal R^P_{B_\kappa}(t)=
 \inf_{\|v\|_2\le B_\kappa}
      \left\|y-\sum_{j=1}^m v_j\psi^P_{j,t}\right\|_2^2,
 \qquad
 \mathcal R^P_\infty(t)=
 \inf_{v\in\mathbb R^m}
      \left\|y-\sum_jv_j\psi^P_{j,t}\right\|_2^2.
 \tag{PJ4}
\end{gather}
There is no re-normalization of $P_tW_{j,t}$ after projection.
The choice $\kappa=1$ is AC's raw-feature convention.
No diagnostic coefficient changes the trained heads.

\subsection{Raw, projected cubic application}
Recall the exact AJ3 constants
\begin{gather}
 b_*^2=(\sqrt5-1)/2,\quad p=2\Phi(b_*)-1,\quad
 g(z)=\operatorname{clip}(z,-b_*,b_*)-pz,\nonumber\\
 v_g=p-2b_*\varphi(b_*)+b_*^2(1-p)-p^2,\quad
 t_g=-2b_*\varphi(b_*)/\sqrt6,\quad K_*=t_g^2/v_g>2/3,\nonumber\\
 D_\alpha=\sqrt{(1-\alpha)^{-2}+p^2/(1+\alpha)^2},\qquad
 C_\alpha=|t_g|D_\alpha/v_g.
 \tag{PJ5}
\end{gather}
Apply Lemma~\ref{lem:four-profile-identity} with $b=b_*$ and
$p=2\Phi(b_*)-1$, so its clipped odd function is exactly $g$. Set
\[
 d_{i,\epsilon,\tau}=(t_g/v_g)a_i c_{\epsilon,\tau}(p).
 \tag{PJ6}
\]
The lemma's exact coefficient norms and $\sum_i a_i^2=1$ give
\[
 \|d\|_2=C_\alpha,\qquad
 \|d\|_1=\frac{2|t_g|}{v_g(1-\alpha)}\sum_i a_i
                                      \le2C_\alpha\sqrt r.
 \tag{PJ7}
\]
The ideal readout is therefore
\[
 F(X)=\sum_{i,\epsilon,\tau}d_{i,\epsilon,\tau}
       \sigma_\alpha(\epsilon u_i^TX+\tau b_*)
      =\frac{t_g}{v_g}\sum_i a_i g(u_i^TX),\qquad
 \|y_c-F\|_2^2=1-K_*.
 \tag{PJ8}
\]
The last equality uses that $g$ is odd, its variance is $v_g$,
and $\langle h_3,g\rangle=t_g$. Independence across teacher axes
and $\sum a_i^2=1$ then prove it for the full cubic.

On each selected row use the coefficient
\[
 v_j=\frac{d_{i,\epsilon,\tau}}{\kappa s\rho_j},
 \quad v_j=0\text{ on all other rows},\qquad
 \|v\|_2\le\frac{C_\alpha}{\kappa sR_M}<B_\kappa.
 \tag{PJ9}
\]
Positive homogeneity gives the exact equality
\[
 v_j\psi^P_{j,N}(X)
 =d_{i,\epsilon,\tau}\sigma_\alpha((P_Nn_j)^TX+z_j).
 \tag{PJ10}
\]
Thus neither the largest raw radius nor an $M$-sized full-bank norm
enters the error or the budget. The denominator is the original
radius $\rho_j$, not the radius after projection. Since the
activation is 1-Lipschitz and $P_N$ is an orthogonal projector,
PJ2, PJ7 and the Gaussian identity
$\mathbb E[(v^TX+b)^2]=\|v\|^2+b^2$ imply
\[
 \left\|\sum_jv_j\psi^P_{j,N}-F(P_NX)\right\|_2
 \le\|d\|_1\xi\le2C_\alpha\sqrt r\,\xi\le1/64.
 \tag{PJ11}
\]
Biases are unchanged throughout this calculation. Zeros on unchosen
diagnostic coefficients are a legal witness in the original full
bank, not a training subbank or a different before-time comparator.

An elementary projection estimate, useful without the structure of
$g$, is
\[
 \|F(P_NX)-F(X)\|_2
 \le2C_\alpha\sqrt r\sqrt{1-\mu}.
 \tag{PJ12}
\]
Indeed $\|(I-P_N)u_i\|^2\le1-\mu$ in every teacher direction,
and one can sum the individual activation errors using PJ7.
It already suffices to impose
$1-\mu\le(1024C_\alpha^2r)^{-1}$ to make PJ12 at most $1/16$.

\subsection{A dimension-free distortion bound for the structured witness}
The following refinement avoids the rank cost in PJ12.
Put
\[
 L_g=\max\{p,1-p\},\qquad D_g=|t_g|L_g/v_g,
 \qquad Q=U^TP_NU,\quad \mu I\preceq Q\preceq I.
 \tag{PJ13}
\]
Here $L_g$ is a global Lipschitz constant of $g$, and $D_g$ is
independent of $r$ and $\alpha$. Define independent Gaussian vectors
$Z=U^TP_NX$, $W=U^T(I-P_N)X$, with covariances $Q$ and $I-Q$.
For $H(z)=\sum_i a_i g(z_i)$, conditional Gaussian Poincare yields
\[
 \mathbb E[\operatorname{Var}(H(Z+W)\mid Z)]
 \le\mathbb E[\nabla H(Z+W)^T(I-Q)\nabla H(Z+W)]
 \le(1-\mu)L_g^2,
 \tag{PJ14}
\]
because $\|\nabla H\|^2\le L_g^2\sum_i a_i^2=L_g^2$.
This includes singular $I-Q$ by restriction to its Gaussian support.
The Gaussian Poincare inequality itself follows by expanding a
function of standard Gaussian coordinates into Hermites: its
variance sums the nonconstant squared coefficients, whereas its
derivative energy weights each such coefficient by its degree.

To bound the conditional mean error, let $q_i=Q_{ii}\ge\mu$ and
\[
 h_i(z)=\mathbb E_T[g(z+\sqrt{1-q_i}\,T)]-g(z),
 \quad T\sim N(0,1),\quad
 \|h_i(Z_i)\|_2^2\le L_g^2(1-q_i).
 \tag{PJ15}
\]
The last bound is Jensen followed by the Lipschitz inequality.
Each $h_i$ is odd, so $\mathbb E[h_i(Z_i)]=0$.
For $V_i=Z_i/\sqrt{q_i}$, the correlation matrix of the standardized
coordinates of $Z$ is
$R=\operatorname{diag}(q_i)^{-1/2}Q
       \operatorname{diag}(q_i)^{-1/2}$, with $\|R\|\le1/\mu$.
For any centered square-integrable functions of these coordinates,
their Hermite expansions give
\[
 \mathbb E\left[\left(\sum_i a_i h_i(Z_i)\right)^2\right]
 \le\|R\|\sum_i a_i^2\|h_i(Z_i)\|_2^2
 \le\frac{1-\mu}{\mu}L_g^2.
 \tag{PJ16}
\]
Here is the matrix detail. At degree $k\ge1$, the covariance is
the quadratic form of the coefficient vector against $R^{\odot k}$,
the entrywise $k$th power of $R$.
The identity for two coordinates follows by comparing coefficients
in $\mathbb E[e^{tV_i-t^2/2}e^{zV_j-z^2/2}]=e^{R_{ij}tz}$.
For a correlation matrix $R$, Schur multiplication preserves order
and sends the identity to itself. Thus
$\|R^{\odot k}\|\le\|R\|$ by induction from
$R^{\odot(k-1)}\preceq\|R\|I$.
Sum the degree-wise bound using Parseval. Centering is valid here
because these \emph{analytical error functions} are odd; no target
or data-centering operation is imposed.

The noise about the conditional mean is orthogonal to its
conditional mean error, and
$\mathbb E[H(Z+W)\mid Z]-H(Z)=\sum_i a_i h_i(Z_i)$.
Combining PJ14--16 therefore gives the exact sufficient bound
\[
 {\displaystyle \quad
 \|F(X)-F(P_NX)\|_2
 \le D_g\sqrt{(1-\mu)(1+1/\mu)}.
 \quad}                                                \tag{PJ17}
\]
The argument holds for the observed $P_N$, however it depends on
the full trained bank and teacher. No later Haar law or independence
between $P_N$ and the learned rows is used.

Let
\[
 \Delta(\mu)=\min\{2C_\alpha\sqrt{r(1-\mu)},
                       D_g\sqrt{(1-\mu)(1+1/\mu)}\}.
 \tag{PJ18}
\]
The full endpoint risks then satisfy
\[
 \mathcal R^P_{B_\kappa}(N),\ \mathcal R^P_\infty(N)
 \le\left(\sqrt{1-K_*}+3/128+\Delta(\mu)\right)^2.
 \tag{PJ19}
\]
The costs are the ideal cubic error, PJ11, the full original
perturbation $\|y-y_c\|_2\le\nu/2\le1/128$, and PJ17 or PJ12.
No selected-head replacement of the actual network is used to
define $P_N$.

\subsection{Initial top-r baseline without a width ceiling}
For ordinary unscreened IID Gaussian raw spatial rows, independent
nondegenerate Gaussian heads and biases, IA proves that $P_0$ is
a Haar rank-$r$ projector whenever $m\ge r$ and $r<d$.
Its genuine nonzero boundary is almost surely simple, also when
$m\ge d$. For the fixed normalized cubic in PJ1,
\[
 \mathbb E_{\rm init}\left[
  \|\mathbb E_X[y_c\mid P_0X]\|_2^2\right]
   =q_{r,d}:=\frac{r(r+2)(r+4)}{d(d+2)(d+4)}.
 \tag{PJ20}
\]
For a fixed projector the projected Hermite inner product is
$(u_i^TP_0u_j)^3$; off-diagonal Haar averages vanish by a
coordinate reflection, and diagonal averages are the third beta
moment. This proves PJ20 without independence of the initialized
projector and the initial feature bank.
Fix $0<\delta_I<1$ and write $e=y-y_c$, $m_y=\mathbb E[y]$, and
$e^\circ=e-\mathbb E[e]$. By Markov, full conditional-expectation
contraction, and the identity $\|e\|_2^2=(\mathbb E[e])^2+\|e-\mathbb E[e]\|_2^2$, with probability at
least $1-\delta_I$ under the original draw,
\begin{gather}
 \|\mathbb E[y\mid P_0X]\|_2^2
 \le m_y^2+
       \left(\sqrt{q_{r,d}/\delta_I}+\|e^\circ\|_2\right)^2
 \le\left(\sqrt{q_{r,d}/\delta_I}+\nu/2\right)^2,\nonumber\\
 \mathcal R^P_{B_\kappa}(0),\ \mathcal R^P_\infty(0)
 \ge1-\left(\sqrt{q_{r,d}/\delta_I}+\nu/2\right)^2.
 \tag{PJ21}
\end{gather}
Every initial projected feature, with its retained bias, is
measurable in $P_0X$; the lower bound applies even to the larger
class of all measurable functions of that input. The same full
target and the same complete projected bank occur at both times.
The full teacher mean was charged, not deleted. The rank in PJ20
is $r$, not the width $m$.

Consequently if $q_{r,d}\le\delta_I/32$ and $\nu/2\le1/128$,
the initial risk is at least $1-9/256$. Either explicit condition
\[
 \mu\ge1-\frac1{1024C_\alpha^2r},
 \qquad\hbox{or}\qquad
 \mu\ge\mu_*:=\max\left\{\frac12,
                         1-\frac1{768D_g^2}\right\}
 \tag{PJ22}
\]
makes $\Delta(\mu)\le1/16$. The second sufficient minimum is
independent of rank, width, dimension and leak parameter. The two
same-budget before/after gains, including the unrestricted one,
therefore obey
\begin{gather}
 \mathcal R^P_{B_\kappa}(0)-\mathcal R^P_{B_\kappa}(N)
   \ge G_{\rm proj}>1/2,\qquad
 \mathcal R^P_\infty(0)-\mathcal R^P_\infty(N)
   \ge G_{\rm proj}>1/2,\nonumber\\
 G_{\rm proj}:=1-9/256-(1/\sqrt3+11/128)^2.
 \tag{PJ23}
\end{gather}
For an entirely rational check of positivity, replace
$1/\sqrt3$ by $7/12$ in this lower bound; the result still
exceeds $1/2$. Approximate scalar values are
$D_g=2.066824$, $\mu_*=0.99969519$, and
$G_{\rm proj}=0.52489308$.
These evaluations are of the displayed constants, not training
experiments. The certified conditions are PJ5 and PJ22--23.
A simple numerical sufficient condition is $\mu\ge0.9997$.
For a finite check without relying on rounded evaluations, use
$0.78615<b_*<0.78616$ and $0.39894<\varphi(0)<0.39895$.
The alternating Taylor sums of orders seven and eight for
$e^{-x}$ and for $\int_0^b e^{-t^2/2}\,dt$ then give
$0.56821<p<0.56824$ and
$0.29288<\varphi(b_*)<0.29291$.
Substitute these rational intervals in PJ5 and use
$\sqrt6>2.44948$: they give $v_g>0.05160647$ and
$D_g<2.08$. Finally
$3(2.08)^2(0.0003)=0.00389376<1/256$,
which verifies PJ22's second sufficient bound at $0.9997$.

The same rational intervals certify the strict inequality $K_*>2/3$
used in AJ3 and PJ5. Indeed, writing their decimal endpoints as exact
rationals gives
\[
\begin{aligned}
 0<v_g
 &<0.56824-2(0.78615)(0.29288)\\
 &\qquad +(0.78616)^2(1-0.56821)-(0.56821)^2
 <\frac{5177}{100000},\\
 t_g^2
 &=\frac23 b_*^2\varphi(b_*)^2
 >\frac23(0.78615)^2(0.29288)^2
 >\frac{3532}{100000}.
\end{aligned}
\]
Positivity of $v_g$ also follows directly because it is the variance
of the nonconstant function $g$. Since $3\cdot3532=10596>
10354=2\cdot5177$, these bounds imply $3t_g^2>2v_g$ and hence
$K_*>2/3$. This argument uses only the displayed rational bounds.

\clearpage
\section{SwiGLU students: leading-direction learning during a plateau}
\label{app:swiglu-students}

This appendix gives a complementary result for a gated student with trainable
inner weights and trainable output heads. Its geometric conclusion concerns
the leading population-AGOP direction. For a teacher subspace of dimension
larger than one, this is weaker than the minimum top-$r$ alignment used in
the main theorems. The results below do not extend those theorems to all
Gaussian Sobolev teachers or replace their initialization restrictions.

The useful distinction is between three statements: leading-direction
alignment on a long, nearly constant-loss interval; approximation by an
unrestricted refitted readout; and a later decrease of the original trained
loss. The first two are treated separately below. The numerical examples
also exhibit later loss decreases, but a general theorem joining that
release to the small-initialization trajectory remains open.

\subsection{Which teacher links are covered?}
\label{swb:sec:teacher-scope}
The alignment theorem requires a nonzero Hermite component of degree one,
two, or three.  The examples below make this condition concrete.  Each
displayed link is centered and divided by its standard deviation when used
as a teacher; this preserves the listed nonzero coefficients and the value
of $k$.  Here $Z,S,T$ are independent standard Gaussians,
$h_1(z)=z$, $h_2(z)=(z^2-1)/\sqrt2$, and
$h_3(z)=(z^3-3z)/\sqrt6$.  Coefficient values in the table are for the
displayed link before normalization.
In general $h_j=\operatorname{He}_j/\sqrt{j!}$ is the normalized
probabilists' Hermite polynomial, and $H_{\ge4}$ denotes a Gaussian
$L^2$ function supported in Hermite degrees at least four.
Expectations in this table are over the displayed Gaussian arguments.
The index $k$ is the degree of the leading homogeneous student--teacher
interaction: $k=3$ if degree one or two is present, and $k=4$ if
degree three is the first nonzero degree, as formalized in
\eqref{swb:eq:lead}.

\begingroup

\renewcommand{\arraystretch}{1.15}
\begin{center}
\begin{tabular}{@{}p{0.33\linewidth}p{0.50\linewidth}c@{}}
\hline
Teacher link $F$ & A nonzero Hermite coefficient & $k$\\
\hline
$h_1(z)$ or $h_2(z)$
 & $\mathbb E[F(Z)h_i(Z)]=1$ for $F=h_i$ & $3$\\
$h_3(z)$
 & $\mathbb E[F(Z)h_3(Z)]=1$; lower degrees vanish & $4$\\
ReLU or leaky ReLU,
$F(z)=\max(z,\lambda z)$,
$0\le\lambda\le1$
 & $\mathbb E[F(Z)h_1(Z)]=(1+\lambda)/2$ & $3$\\
$|z|$
 & $\mathbb E[F(Z)h_2(Z)]=1/\sqrt\pi$ & $3$\\
Softplus, SiLU, or GELU
 & $\mathbb E[F(Z)h_1(Z)]=1/2$ & $3$\\
$\sin z$ or $\tanh z$
 & $\mathbb E[Z\sin Z]=e^{-1/2}$;
   $\mathbb E[Z\tanh Z]>0$ & $3$\\
$a h_3(z)+H_{\ge4}(z)$,
$a\ne0$, $H_{\ge4}\in L^2$
 & Cubic coefficient $a$; lower degrees vanish & $4$\\
Pure $h_j(z)$, $j\ge4$
 & No coefficient in degrees one through three & Excluded\\
\hline
$\operatorname{SiLU}(s)t$ or $\operatorname{ReLU}(s)t$
 & $\mathbb E[F(S,T)ST]=1/2$ & $3$\\
$st$ or $\tanh(s)t$
 & $\mathbb E[F(S,T)ST]=1$ or
   $\mathbb E[S\tanh S]>0$ & $3$\\
$h_2(s)t$
 & $\mathbb E[F(S,T)h_2(S)T]=1$;
   lower degrees vanish & $4$\\
\hline
\end{tabular}
\end{center}
\endgroup

These checks use Hermite orthogonality, independence, parity, and the
Gaussian characteristic function.
Softplus $\log(1+e^z)$, SiLU $z/(1+e^{-z})$, and GELU
$z\,\mathbb P(Z\le z)$ all have odd part $z/2$.
Arbitrary square-integrable additions supported in degrees at least four
leave the classification unchanged.  A surviving degree-one or degree-two
component always sets $k=3$, even when a cubic component is also present.

The rank-one refit theorem is narrower: it requires a scalar link with
zero degree-one and degree-two coefficients and a nonzero cubic coefficient.
Thus it applies to $a h_3+H_{\ge4}$, but does not by itself apply to the
bivariate $h_2(s)t$ example.  For mixtures and interaction teachers alike,
the alignment theorem concerns the leading AGOP direction in the teacher
subspace; it does not assert recovery of every teacher direction.

\providecommand{\swbE}{\mathbb{E}}
\providecommand{\swbP}{\mathbb{P}}
\providecommand{\swbCov}{\operatorname{Cov}}
\providecommand{\swbVar}{\operatorname{Var}}
\providecommand{\swbAtop}{A_{\mathrm{top}}}

\subsection{Setup and leading-direction theorem}
\label{swb:sec:alignment}

This result concerns a SwiGLU \emph{student} whose head, gates, values, and
inner biases all follow the same population gradient update.  Only the output
intercept is minimized exactly.  At every fixed width, a small initialization
produces a long interval with vanishing loss variation, on which the leading
AGOP direction enters the teacher subspace.  The conclusion concerns every
vector in the leading eigenspace; it does not assert recovery of all teacher
directions when their number exceeds one.

\paragraph{Setting and clocks.}
Let $x\sim N(0,I_d)$, $y=F(U^\top x)$, $U^\top U=I_r$, $1\le r<d$,
$F\in L^2(\gamma_r)$, and $\swbVar(y)=1$.  Put $\widetilde x=(x,1)$ and
\begin{equation}
 \widetilde f_\theta(x)=\frac{\kappa}{m}\sum_{j=1}^m
 a_j S(p_j^\top\widetilde x)(v_j^\top\widetilde x),\qquad
 S(s)=\frac{s}{1+e^{-s}},\qquad \kappa>0.
 \label{swb:eq:student}
\end{equation}
Write $\theta_j=(a_j,p_j,v_j)$, $p_j=(w_j,b_j)$, $v_j=(z_j,c_j)$, and
$g_j(x)=S(p_j^\top\widetilde x)(v_j^\top\widetilde x)$.
Here $d$ is the input dimension, $m$ the student width, $r$ the teacher
rank, and $U\in\mathbb R^{d\times r}$ an orthonormal basis of the teacher
subspace. The measure $\gamma_s=N(0,I_s)$ is standard Gaussian measure
in dimension $s$. The scalar $a_j$ is the output head, $w_j,z_j\in\mathbb R^d$
are the gate and value weights, and $b_j,c_j$ their biases.
The fixed positive $\kappa$ is the SwiGLU output scale, distinct from
the leaky-ReLU slope in the main text. Expectations, covariances, and
variances in this subsection are over $x$ unless specified; probabilities
are over the Gaussian initialization. Vector norms are Euclidean,
matrix norms are operator norms unless marked $\mathrm F$ for Frobenius
norm, and function norms $\|\cdot\|_2$ use Gaussian $L^2$.
The fitted intercept is $b_0(\theta)=\swbE[y-\widetilde f_\theta]$, so that
\begin{equation}
 L(\theta)=\swbVar(y-\widetilde f_\theta),\qquad
 \mathcal F_j(\theta)=-m\nabla_{\theta_j}L
 =2\kappa\swbCov\left(y-\widetilde f_\theta,
              \nabla_{\theta_j}(a_jg_j)\right).
 \label{swb:eq:dynamics}
\end{equation}
We consider either $\dot\theta_j=\mathcal F_j(\theta)$ or
$\theta_j^{n+1}=\theta_j^n+h\mathcal F_j(\theta^n)$, with a fixed $h>0$.
Thus a raw GD learning rate $\eta$ corresponds to $h=\eta/m$ in this
mean-field clock.  No step size is changed as the initialization scale shrinks.
Initialize $\theta_j(0)=\varepsilon\vartheta_j$, where independently
\begin{equation}
 \vartheta_j\sim N(0,\beta_0^2)\otimes
 N(0,s_0^2I_{d+1})\otimes N(0,s_0^2I_{d+1}),
 \qquad \beta_0,s_0>0.
 \label{swb:eq:init}
\end{equation}
All dimensions, scales, and the teacher are fixed before $\varepsilon\downarrow0$.
The Gaussian draw $\vartheta=(\vartheta_1,\ldots,\vartheta_m)$ is coupled
across initialization scales.  Set $P=UU^\top$ and
\begin{equation}
 M(\theta)=\swbE[\nabla_x\widetilde f_\theta
                         \nabla_x\widetilde f_\theta^\top],\qquad
 \swbAtop(\theta)=\min_{\substack{\|e\|=1\\
                  M(\theta)e=\lambda_{\max}(M(\theta))e}}\|Pe\|^2.
 \label{swb:eq:alignment}
\end{equation}
The minimum resolves any multiplicity of the largest eigenvalue.

\paragraph{The required teacher signal.}
Define
\[
 g_1=\swbE[yx],\quad H_2=\swbE[y(xx^\top-I_d)],\quad
 T_3=\swbE[y\operatorname{He}_3(x)],
\]
where $(\operatorname{He}_3(x))_{ijk}
=x_ix_jx_k-x_i\delta_{jk}-x_j\delta_{ik}-x_k\delta_{ij}$.
Here $\delta_{ij}$ is the Kronecker delta, and tensor contractions mean
$T_3[a,b,c]=\sum_{i,j,\ell}(T_3)_{ij\ell}a_i b_j c_\ell$.
We use the Frobenius norm for tensors when bounding these contractions.
These expectations exist by Cauchy--Schwarz.  Gaussian independence in
$U\oplus U^\perp$ shows that $g_1,H_2,T_3$ are supported on $U$.
Assume that at least one of these three tensors is nonzero, and define
\begin{equation}
\begin{array}{c|c|c}
\text{signal}& k&T_{\mathrm{lead}}(p,v)\\ \hline
(g_1,H_2)\ne(0,0)&3&
 \tfrac12\{w^\top H_2z+g_1^\top(cw+bz)\}\\[2pt]
g_1=H_2=0,\ T_3\ne0&4&\tfrac14 T_3[w,w,z].
\end{array}
\label{swb:eq:lead}
\end{equation}
There is no smallness condition on higher Hermite components of $F$.
In particular, the theorem is not restricted to a cubic polynomial teacher,
but it does not cover a teacher having no signal in degrees one through three.

\paragraph{A decoupled comparator.}
For one neuron put $\Psi(a,p,v)=2\kappa aT_{\mathrm{lead}}(p,v)$.
This is a nonzero homogeneous polynomial of degree $k$, odd in $a$.
Let $X_j(\tau)$ solve
\begin{equation}
 X_j'=\nabla\Psi(X_j),\qquad X_j(0)=\vartheta_j.
 \label{swb:eq:comparator}
\end{equation}
Let $\tau_j$ be its maximal forward existence time and
$\tau_*=\min_j\tau_j$.  A finite $\tau_j$ means that the comparator
norm diverges there.  Define the escape event $\mathcal E=\{\tau_*<\infty\}$.

\newtheorem*{swbalignmentrestatement}{Theorem~\ref{swb:thm:alignment} (comparator formulation)}
\begin{swbalignmentrestatement}
Fix the setting above and a GD mean-field step $h>0$.  Then
$\swbP(\mathcal E)\ge 1-2^{-m}$, and almost surely on $\mathcal E$ the
earliest escaping neuron $W$ is unique.  On this same probability-one subset,
for every $\delta\in(0,1)$ there are random constants
\[
 0<\tau_1=\tau_1(\vartheta,\delta)<\tau_*,\qquad
 C=C(\vartheta,\delta)<\infty,\qquad
 \varepsilon_0=\varepsilon_0(\vartheta,\delta,h)>0
\]
such that, for every $0<\varepsilon<\varepsilon_0$, the following hold for
both gradient flow and this fixed-step GD.  For flow take
$t_1=\tau_1\varepsilon^{-(k-2)}$; for GD take
$n_1=\lfloor\tau_1/(h\varepsilon^{k-2})\rfloor$.
\begin{enumerate}
\item The loss stays on a plateau:
\[
 \sup_{0\le t\le t_1}|L(\theta(t))-L(\theta(0))|\le C\varepsilon^k,
 \qquad
 \max_{0\le n\le n_1}|L(\theta^n)-L(\theta^0)|\le C\varepsilon^k.
\]
Its initial value has the signed expansion
\begin{equation}
 L(\theta(0))=1-\frac{\varepsilon^k}{m}\sum_j\Psi(\vartheta_j)
                  +O_\vartheta(\varepsilon^{k+1})
             =1+O_\vartheta(\varepsilon^k).
 \label{swb:eq:initial-loss}
\end{equation}
\item At the respective endpoints,
$\swbAtop(\theta(t_1))\ge1-\delta$ and
$\swbAtop(\theta^{n_1})\ge1-\delta$.
\item At either endpoint the winner obeys
\[
 \big(\|(I-P)w_W\|^2+\|(I-P)z_W\|^2\big)^{1/2}
 \le\delta\big(\|Pw_W\|^2+\|Pz_W\|^2\big)^{1/2}.
\]
When $k=4$, the stronger separate inequalities
$\|(I-P)w_W\|\le\delta\|Pw_W\|$ and
$\|(I-P)z_W\|\le\delta\|Pz_W\|$ hold.
\end{enumerate}
The time $\tau_1$ can be chosen arbitrarily close to $\tau_*$, but is fixed
before taking $\varepsilon\downarrow0$.  The cutoff $\varepsilon_0$ is
initialization dependent; no explicit finite-scale confidence bound is asserted.
\end{swbalignmentrestatement}

For $k=4$ the plateau has duration proportional to $\varepsilon^{-2}$ and
loss variation $O(\varepsilon^4)$.  For $k=3$ these orders are
$\varepsilon^{-1}$ and $O(\varepsilon^3)$.  The statement gives an interval
inside the initial plateau, not its actual exit time or a later fixed-size
loss decrease.  It also makes no claim about stability of a fixed GD step
after the small-parameter phase.

\begin{lemma}[Escape and the unique first winner]
\label{swb:lem:escape}
For a homogeneous gradient flow $X'=\nabla\Psi(X)$ of degree $k>2$,
let $\tau_b$ denote its maximal forward existence time.
If $\Psi(X(0))>0$, then
\begin{equation}
 \tau_b\le\frac{\|X(0)\|^2}{k(k-2)\Psi(X(0))},\qquad
 \Psi(X(\tau))\ge\Psi(X(0))
             \left(\frac{\|X(\tau)\|}{\|X(0)\|}\right)^k.
 \label{swb:eq:escape-bound}
\end{equation}
In the present initialization, $\swbP(\Psi(\vartheta_j)>0)=1/2$.
The law of $\tau_j$ has no \emph{finite} atoms; it may have an atom at
$\infty$.  Consequently $\swbP(\mathcal E)\ge1-2^{-m}$ and its finite
minimum is almost surely unique.  Moreover, whenever the set on the right
is nonempty,
\[
 \tau_*\le\min_{j:\Psi(\vartheta_j)>0}
       \frac{\|\vartheta_j\|^2}{k(k-2)\Psi(\vartheta_j)}.
\]
\end{lemma}
\begin{proof}
Writing $N=\|X\|^2$, Euler's identity and Cauchy--Schwarz give
$N'=2k\Psi$ and
$\Psi'=\|\nabla\Psi\|^2\ge k^2\Psi^2/N$.
While $\Psi>0$, $(\log\Psi-\tfrac k2\log N)'\ge0$.
Integrating $N'\ge2k\Psi(X(0))N^{k/2}/N(0)^{k/2}$ proves
\eqref{swb:eq:escape-bound} and forces finite escape by the polynomial
ODE's continuation criterion.  The zero set of the nonzero polynomial
$\Psi$ is Lebesgue null; changing the sign of $a$ proves its equal sign
probabilities.

Homogeneity gives $\tau_b(s\vartheta)=s^{-(k-2)}\tau_b(\vartheta)$ for
$s>0$.  For a fixed $t\in(0,\infty)$, the level set $\{\tau_b=t\}$
therefore meets each radial ray in at most one point.  Polar coordinates,
Fubini, and absolute continuity of the Gaussian law show that this set
has probability zero.  Escape times are measurable, for example by
expressing the maximal existence time as the increasing limit of exit
times from balls.  Independence then rules out equality between any pair
of finite escape times.  A finite minimum exists whenever one neuron
starts with positive potential, giving the stated probability.
\end{proof}

\begin{lemma}[Uniform field, loss, and AGOP expansions]
\label{swb:lem:expansions}
On every fixed ball of radius $R$ in the rescaled parameters
$Z=(Z_1,\ldots,Z_m)$, uniformly
for $0<\varepsilon\le1$,
\begin{align}
 \varepsilon^{1-k}\mathcal F_j(\varepsilon Z)
    &=\nabla\Psi(Z_j)+O_R(\varepsilon),
       \label{swb:eq:field-expansion}\\
 L(\varepsilon Z)&=1-\frac{\varepsilon^k}{m}\sum_j\Psi(Z_j)
                          +O_R(\varepsilon^{k+1}),
       \label{swb:eq:loss-expansion}\\
 M(\varepsilon Z)&=\varepsilon^6\left(\frac{\kappa}{2m}\right)^2
             \{B(Z)^2+\ell(Z)\ell(Z)^\top\}+O_R(\varepsilon^7),
       \label{swb:eq:agop-expansion}
\end{align}
where the last error is in operator norm and, writing the coordinates of
$Z_j$ as $(a_j,w_j,b_j,z_j,c_j)$,
\[
 B(Z)=\sum_j a_j(w_jz_j^\top+z_jw_j^\top),\qquad
 \ell(Z)=\sum_j a_j(c_jw_j+b_jz_j).
\]
For $k=4$, the error in \eqref{swb:eq:field-expansion} is in fact
$O_R(\varepsilon^2)$.
The subscript on $O_R$ allows its constant to depend on this fixed
radius (and the fixed model parameters); $O_\vartheta$ and $O_\tau$
similarly allow dependence on the realized initialization and fixed
comparator time, respectively. The vector $\ell(Z)$ here is an AGOP
expansion coefficient, not the normalized experimental loss.
\end{lemma}
\begin{proof}
The identity $S(s)=s/2+(s/2)\tanh(s/2)$ gives
\begin{equation}
 S(s)=s/2+s^2/4+R(s),\qquad
 |R(s)|\le |s|^4/48,\qquad |R'(s)|\le |s|^3/12.
 \label{swb:eq:silu-remainder}
\end{equation}
Indeed $|\tanh u-u|\le |u|^3/3$ follows by integrating
$|\tanh'(u)-1|=\tanh^2u\le u^2$; differentiation gives the second
bound.  We also have $|S(s)|\le|s|$, $|S'(s)|\le2$,
$|S(s)-s/2|\le s^2/4$, and $|S'(s)-1/2|\le|s|/2$.
These are analytic bounds, not numerical estimates.

For affine Gaussian forms, every fixed-order moment satisfies
$\|p^\top\widetilde x\|_{L^q}\le C_q\|p\|$; the same bound with
additional fixed powers of $\|\widetilde x\|$ follows from H\"older's
inequality.  Thus $F\in L^2$ suffices for all ensuing covariance bounds
and parameter derivatives, by Cauchy--Schwarz.  Derivatives up to order
two of the features have polynomial envelopes on bounded parameter
sets, so $L$ is twice continuously differentiable and its vector field
is locally Lipschitz.

Put $P_0=w^\top x+b$ and $V_0=z^\top x+c$.  Hermite decomposition yields
\begin{align*}
 \swbCov(y,P_0V_0)&=w^\top H_2z+g_1^\top(cw+bz),\\
 \swbCov(y,P_0^2V_0)&=T_3[w,w,z]+\|w\|^2g_1^\top z
       +2(w^\top z)g_1^\top w\\
 &\quad+2b\,w^\top H_2z+c\,w^\top H_2w
              +b^2g_1^\top z+2bc\,g_1^\top w.
\end{align*}
Consequently $T(p,v)=\swbCov(y,S(P_0)V_0)$ has the leading term
in \eqref{swb:eq:lead}.  Its next contribution to $aT$ is degree four
when $k=3$, and is of order six when $k=4$.  The differentiated remainder
obeys the corresponding bounds by \eqref{swb:eq:silu-remainder}.
The student feedback in \eqref{swb:eq:dynamics} is $O_R(\varepsilon^5)$,
since $\|\widetilde f_{\varepsilon Z}\|_2=O_R(\varepsilon^3)$ and
$\|\nabla_{\theta_j}(a_jg_j)\|_2=O_R(\varepsilon^2)$.
This proves the field expansion, including its sharper cubic-signal
version.  Expanding
$L=1-2\swbCov(y,\widetilde f)+\swbVar(\widetilde f)$, with the last term
$O_R(\varepsilon^6)$, proves the loss formula.

Finally,
$\nabla_xg_j=S'(P_0)V_0w_j+S(P_0)z_j$.
The elementary bounds above give, in $L^2(\gamma_d;\mathbb R^d)$,
\[
 \nabla_x\widetilde f_{\varepsilon Z}
     =\frac{\kappa\varepsilon^3}{2m}\{B(Z)x+\ell(Z)\}
                                      +O_R(\varepsilon^4).
\]
Taking its outer product and expectation proves
\eqref{swb:eq:agop-expansion}, because $B$ is symmetric and
$\swbE[xx^\top]=I_d$, $\swbE[x]=0$.
\end{proof}

\begin{lemma}[Tracking up to a fixed comparator time]
\label{swb:lem:tracking}
Fix a realization with $0<T<\tau_*$, and let $X=(X_j)_j$ be its comparator.
For flow define $Z_\varepsilon(\tau)=\varepsilon^{-1}
\theta(\tau\varepsilon^{-(k-2)})$.  For GD define
$Z_\varepsilon^n=\varepsilon^{-1}\theta^n$ and
$\Delta=h\varepsilon^{k-2}$.  Then, for sufficiently small $\varepsilon$,
\begin{align*}
 \sup_{0\le\tau\le T}\|Z_\varepsilon(\tau)-X(\tau)\|
     &\le C_T\varepsilon,\\
 \max_{0\le n\Delta\le T}\|Z_\varepsilon^n-X(n\Delta)\|
     &\le C_T(\varepsilon+\Delta).
\end{align*}
All these rescaled states lie in one fixed ball.  For $k=4$, the
$\varepsilon$ terms can be replaced by $\varepsilon^2$.
\end{lemma}
\begin{proof}
Choose a ball with radius $R>\sup_{\tau\le T}\|X(\tau)\|+1$.
On this ball $G(Z)=(\nabla\Psi(Z_j))_j$ is Lipschitz, say with constant
$L_0$, and the rescaled actual field is $G(Z)+r_\varepsilon(Z)$ with
$\|r_\varepsilon\|\le C_0\varepsilon$.  Stopping at the ball's first
exit, Gr\"onwall bounds the flow error by
$C_0\varepsilon T e^{L_0T}$.  Taking it smaller than half the margin
closes the stopping argument and proves existence through $T$.

For GD, $Z^{n+1}=Z^n+\Delta\{G(Z^n)+r_\varepsilon(Z^n)\}$.
Along the bounded comparator, the one-step Taylor defect is at most
$C_1\Delta^2$: this follows directly by integrating
$G(X(s))-G(X(n\Delta))$, using a bound for $\|X'\|$.
Writing $e_n=\|Z^n-X(n\Delta)\|$, we therefore obtain
\[
 e_{n+1}\le(1+L_0\Delta)e_n+C_0\varepsilon\Delta+C_1\Delta^2,
 \qquad e_0=0.
\]
Summing this recursion gives
$e_n\le T e^{L_0T}(C_0\varepsilon+C_1\Delta)$.
Induction closes the same radius bootstrap.  Since
$|T-\lfloor T/\Delta\rfloor\Delta|<\Delta$ and $X'$ is bounded,
the rounded GD endpoint also converges to $X(T)$.  The sharper bound
follows by replacing $C_0\varepsilon$ by $C_0\varepsilon^2$.
\end{proof}

\begin{proof}[Proof of Theorem~\ref{swb:thm:alignment}]
Lemma~\ref{swb:lem:escape} gives the escape probability and uniqueness.
Fix a realization in this event with unique
winner.  Every nonwinner extends across $\tau_*$ and stays bounded there.
Write $K(\tau)=\|X_W(\tau)\|$.  Since
$(K^2)'=2k\Psi(X_W)$ and $\Psi'=\|\nabla\Psi\|^2\ge0$, finite escape
forces $\Psi(X_W)$ to become positive: otherwise the norm is bounded.
In fact it becomes unbounded, since a bounded potential would keep
$K^2$ bounded on a finite time interval.  Applying
\eqref{swb:eq:escape-bound} from a positive-potential time gives, for
all sufficiently late $\tau<\tau_*$,
\begin{equation}
 \Psi(X_W(\tau))\ge cK(\tau)^k,\qquad
 K(\tau)\longrightarrow\infty\quad(\tau\uparrow\tau_*),\quad c>0.
 \label{swb:eq:winner-growth}
\end{equation}
This argument permits a winner whose initial potential was negative.

The comparator depends on the spatial weights only through $Pw_j,Pz_j$.
Thus $(I-P)w_j$ and $(I-P)z_j$ are constant along it; for $k=4$,
$b_j,c_j$ are also constant.  For the winner, set
\[
 \begin{aligned}
 B_W&=a_W(w_Wz_W^\top+z_Ww_W^\top),& D&=PB_WP,\\
 \ell_W&=a_W(c_Ww_W+b_Wz_W),& v&=P\ell_W.
 \end{aligned}
\]
Every term removed by these projections contains a frozen transverse
factor.  Since all other factors are bounded by $K$,
\begin{equation}
 \|B_W-D\|+\|\ell_W-v\|=O(K^2).
 \label{swb:eq:transverse-error}
\end{equation}
Here and below constants may depend on the fixed initialization.

For $k=4$, combine \eqref{swb:eq:winner-growth} with
\[
 |\Psi|\le(\kappa/2)\|T_3\|\,|a_W|\|Pw_W\|^2\|Pz_W\|.
\]
Each of $|a_W|,\|Pw_W\|,\|Pz_W\|$ is then at least a positive
constant times $K$.  Using
$\|uv^\top+vu^\top\|=\|u\|\|v\|+|u^\top v|$, we obtain
$\|D\|\ge c_1K^3$.  For $k=3$, instead use the exact identity
\[
 \Psi(X_W)=\kappa\{\tfrac12\operatorname{tr}(H_2D)+g_1^\top v\}.
\]
It implies $\|D\|_{\mathrm F}+\|v\|\ge c_2K^3$.
In both cases the positive semidefinite matrix
$Q=D^2+vv^\top$ has range contained in $U$ and satisfies
\begin{equation}
 \lambda_{\max}(Q)\ge c_3K^6,
 \label{swb:eq:supported-growth}
\end{equation}
using $\|D\|^2\ge\|D\|_{\mathrm F}^2/r$ in the second case.
The bounded nonwinners and \eqref{swb:eq:transverse-error} give
\begin{equation}
 M_c(\tau):=B(X(\tau))^2+\ell(X(\tau))\ell(X(\tau))^\top
             =Q+E,\qquad \|E\|=O(K^5).
 \label{swb:eq:agop-dominance}
\end{equation}
For \emph{any} unit top eigenvector $e$ of $Q+E$,
\[
 \lambda_{\max}(Q)-\|E\|
 \le e^\top(Q+E)e
 \le\lambda_{\max}(Q)\|Pe\|^2+\|E\|.
\]
Hence $\|Pe\|^2\ge1-2\|E\|/\lambda_{\max}(Q)=1-O(K^{-1})$.
This controls the entire top eigenspace without assuming an eigengap.

The hidden-weight claims follow as well.  For $k=4$ the individual
projected norms are of order $K$, whereas the transverse norms are fixed.
For $k=3$, writing
$s=(\|Pw_W\|^2+\|Pz_W\|^2)^{1/2}\le K$, the potential formula yields
$|\Psi(X_W)|\le C_2K^2s$.
Together with \eqref{swb:eq:winner-growth} this gives $s\ge c_4K$.
Thus the combined spatial norm dominates the fixed transverse norm;
no separate lower bound on the two projected norms is needed in this case.

Choose $\tau_1<\tau_*$ sufficiently late that all comparator alignment
bounds hold with strict margin relative to $\delta$.
By Lemmas~\ref{swb:lem:expansions} and~\ref{swb:lem:tracking}, for both
the flow and the rounded GD endpoint,
\[
 \varepsilon^{-6}\left(\frac{2m}{\kappa}\right)^2M
       \longrightarrow M_c(\tau_1)=Q(\tau_1)+E(\tau_1).
\]
The same quadratic-form inequality remains valid after adding the
operator-norm error $o(1)$ to $E$.  It proves the claimed uniform
top-eigenspace alignment, even at multiple top eigenvalues.  Parameter
convergence transfers the hidden-weight inequalities, since their
projected denominators are nonzero at this fixed time.
Uniform boundedness on $[0,\tau_1]$ and
\eqref{swb:eq:loss-expansion} prove the entire-prefix loss bounds and
\eqref{swb:eq:initial-loss}.  Shrinking $\varepsilon_0$ makes the flow
and GD conclusions simultaneous.  Since all the comparator estimates
improve as $K\to\infty$, $\tau_1$ can be chosen arbitrarily close to
$\tau_*$.  Rational choices of these times and countable choices of
cutoffs can be used throughout, so the random endpoints and thresholds
may be chosen measurably.
\end{proof}

\begin{corollary}[Unconditional initial headroom and a same-run gain]
\label{swb:cor:headroom}
For every $\varepsilon>0$, $\swbAtop(\theta(0))$ is stochastically
dominated by $B\sim\operatorname{Beta}(r/2,(d-r)/2)$; in particular its
expectation is at most $r/d$.  Fix $\delta\in(0,1)$ and $q\in(0,1-\delta)$.
Using the random endpoints of Theorem~\ref{swb:thm:alignment},
\begin{equation}
 \liminf_{\varepsilon\downarrow0}\swbP\left(
 \begin{array}{c}
 \text{the theorem's plateau conclusions hold, and}\\
 \swbAtop(\text{endpoint})-\swbAtop(0)\ge1-\delta-q
 \end{array}\right)
 \ge [\swbP(B\le q)-2^{-m}]_+.
 \label{swb:eq:same-run-probability}
\end{equation}
The endpoint event is understood to fail outside $\mathcal E$.
The same bound applies to the event that both the flow and the fixed-step
GD conclusions hold.  Equivalently, for any $\zeta>0$ there is a
deterministic sufficiently small cutoff for which the right side minus
$\zeta$ holds at each initialization scale below it.
\end{corollary}
\begin{proof}
At initialization the joint spatial-weight law is rotationally invariant,
and rotating these weights by $O\in O(d)$ conjugates $M$ to $OMO^\top$.
Conditionally on its top eigenspace choose a uniform unit vector $e$ in
that space using auxiliary randomness.  Its unconditional distribution
is uniform on $S^{d-1}$, and $\swbAtop(0)\le\|Pe\|^2$ pointwise.
The Gaussian representation of a uniform sphere vector shows that
$\|Pe\|^2$ has the displayed Beta distribution.

Apply Lemma~\ref{lem:common-event} with $E=\mathcal E$,
$c=\varepsilon_0$ the theorem's measurable common cutoff, and
$H_\varepsilon=\{\swbAtop(\theta(0))\le q\}$.
The unconditional bound just proved gives
$\swbP(H_\varepsilon^c)\le\swbP(B>q)$ at every scale.
On the resulting intersection both endpoint alignments are at least
$1-\delta$, proving the gain and the probability bound for flow and GD
simultaneously. No rotation law conditional on escape is used.
\end{proof}

\paragraph{Scope for several teacher directions.}
For any choice $P_r$ of a projector onto $r$ top AGOP eigenvectors, the
endpoint satisfies
$\operatorname{tr}(PP_r)/r\ge(1-\delta)/r$, because its range contains
a top eigenvector.  This is an average-alignment lower bound, not a bound
on the smallest principal-angle cosine between the two $r$-dimensional
spaces.  A single winner can contribute rank two, so one cannot attribute
all remaining eigenvalues to nonwinners.  More generally,
$\operatorname{range}(M)\subseteq\operatorname{span}\{w_j,z_j:1\le j\le m\}$
and $\operatorname{rank}(M)\le2m$.  Full-subspace recovery and the later
loss-release regime require additional arguments.

\newtheorem{swrtheorem}[theorem]{Theorem}
\newtheorem{swrlemma}[theorem]{Lemma}
\newtheorem{swrproposition}[theorem]{Proposition}
\newtheorem{swrcorollary}[theorem]{Corollary}
\theoremstyle{remark}
\newtheorem{swrremark}[theorem]{Remark}
\theoremstyle{plain}
\newcommand{\swrQ}{\mathcal Q}
\newcommand{\swrR}{\mathcal R}
\newcommand{\swrP}{\Pi_{\mathcal Q}}
\newcommand{\swrE}{\mathbb E}
\newcommand{\swrHe}{\operatorname{He}}

\subsection{Unrestricted population refitting of small SwiGLU features}
\label{swr:sec:refit}

This section proves an \emph{attained refit-risk bound} on the initial
plateau. It first proves the rank-one bound without asserting a comparison
with initialization. A separate proposition establishes such a comparison
at the boundary width only. All readout coefficients and the intercept are
unrestricted; the constructed coefficients may diverge as initialization
vanishes. These conclusions concern population least squares, not a bounded
readout, ridge regression, or a finite-sample estimator.

\paragraph{Setting and initialization scope.}
Use the student, profiled loss, and population dynamics in
\eqref{swb:eq:student}--\eqref{swb:eq:dynamics}: $p_j=(w_j,b_j)$,
$v_j=(z_j,c_j)$, $g_j=S(p_j^\top\widetilde x)(v_j^\top\widetilde x)$,
$x\sim N(0,I_d)$, and $\widetilde x=(x,1)$. The output scale is
$\kappa>0$, and a fixed mean-field GD step $h$ has raw learning rate $mh$.
Function norms and inner products are Gaussian $L^2$; vector, matrix,
and tensor norms follow the preceding subsection's conventions.
Here the initialization assumption is \emph{broader} than
\eqref{swb:eq:init}: $\theta_j(0)=\varepsilon\vartheta_j$, where the
$\vartheta_j$ are independent copies of a nondegenerate centered Gaussian
law invariant under $a\mapsto-a$. The head and inner-block variances may
differ and are fixed independently of $\varepsilon$. Spatial isotropy
is not assumed unless stated explicitly. The draw is coupled across
$\varepsilon$, and here we allow $1\le r\le d$.

The teacher is $y=F(U^\top x)$ with $U^\top U=I_r$ and
$F\in L^2(\gamma_r)$, $\operatorname{Var}(y)=1$. In this section it is
of type three:
\[
 \swrE[(y-\swrE[y])x]=0,\qquad
 \swrE[(y-\swrE[y])(xx^\top-I_d)]=0,\qquad
 T_3:=\swrE[y\swrHe_3(x)]\ne0.
\]
Here $U\in\mathbb R^{d\times r}$ spans the rank-$r$ teacher subspace,
$(\swrHe_3(x))_{ijk}=x_ix_jx_k-x_i\delta_{jk}-x_j\delta_{ik}
-x_k\delta_{ij}$, and $\delta_{ij}$ is the Kronecker delta.
The notation $T_3[a,b,c]$ denotes contraction with three vectors;
the tensor norm used to bound this contraction is the Frobenius norm.
Let $\swrQ$ be the real polynomials of degree at most two, as a subspace
of $L^2(\gamma_d)$, with orthogonal projector $\swrP$ and dimension
\begin{equation}\label{swr:eq:dimension}
 D=\dim\swrQ=1+d+\frac{d(d+1)}2=\frac{(d+1)(d+2)}2.
\end{equation}
Then $y-\swrE[y]\perp\swrQ$.
The refit risk at a frozen parameter state is
\begin{equation}\label{swr:eq:refit}
 \swrR(\theta)=\inf_{\eta_0\in\mathbb R,\,\eta\in\mathbb R^m}
 \swrE\left[\left(y-\eta_0-\sum_{j=1}^m\eta_jg_j\right)^2\right].
\end{equation}
This infimum is attained because the feature span is finite dimensional;
its coefficients need not be unique. The trained $a_j$ are not constraints
on the refitted coefficients.

\paragraph{Comparator and evaluation times.}
Let $X_j(\tau)$ solve the independent polynomial flows
\begin{equation}\label{swr:eq:comparator}
 X_j'=\nabla\Psi(X_j),\qquad X_j(0)=\vartheta_j,
 \qquad \Psi(a,p,v)=\frac{\kappa}{2}aT_3[w,w,z].
\end{equation}
Let $T_j\in(0,\infty]$ be their maximal forward existence times and
$\tau_* =\min_j T_j$. The event $E_{\rm esc}=\{\tau_*<\infty\}$
has a unique winner $W$ almost surely. At fixed comparator time $\tau$
write
\[
 t_\varepsilon(\tau)=\tau\varepsilon^{-2}\quad\hbox{for flow},
 \qquad
 t_\varepsilon(\tau)=h\left\lfloor\frac{\tau}{h\varepsilon^2}\right\rfloor
 \quad\hbox{for GD}.
\]
All limits below first keep $\tau<\tau_*$ fixed and then send
$\varepsilon\downarrow0$. The selected $\tau$ may depend on the
initialization, but never on $\varepsilon$.
For GD, $\theta(t_\varepsilon(\tau))$ means the iterate with index
$\lfloor\tau/(h\varepsilon^2)\rfloor$. Below $w_\perp=(I-UU^\top)w$
and $z_\perp=(I-UU^\top)z$ denote transverse spatial weights.
Subscripts on $O_\tau$ allow constants to depend on the fixed comparator
time; $O_{L^2}$ specifies the norm in which the remainder is bounded.

\begin{swrlemma}[Expansion, tracking, and escape facts]
\label{swr:lem:tracking}
Write $Z=(Z_1,\ldots,Z_m)$ for the rescaled parameter tuple and
$P=p^\top\widetilde x$, $V=v^\top\widetilde x$ for one neuron's
affine Gaussian forms. For tuples, $\nabla\Psi(Z)$ denotes
$(\nabla\Psi(Z_j))_{j=1}^m$. On every compact set of rescaled parameters,
uniformly on that set,
\begin{align}
 S(\varepsilon P)\varepsilon V
   &=\varepsilon^2 q+\varepsilon^3\chi+O_{L^2}(\varepsilon^5),
 &q&=\tfrac12PV,\quad\chi=\tfrac14P^2V,
 \label{swr:eq:expansion}\\
 \varepsilon^{-3}[-m\nabla L](\varepsilon Z)
   &=\nabla\Psi(Z)+O(\varepsilon^2),
 \label{swr:eq:field}\\
 L(\varepsilon Z)
   &=1-\frac{\varepsilon^4}{m}\sum_j\Psi(Z_j)+O(\varepsilon^6).
 \label{swr:eq:loss}
\end{align}
Here $q_j$ and $\chi_j$ denote the quadratic and cubic terms of neuron $j$;
$\chi_j$ is distinct from its value bias $c_j$.
Consequently, for every realization and fixed $\tau<\tau_*$,
the rescaled flow, and fixed-step GD at its rounded times, satisfy
\begin{equation}\label{swr:eq:tracking}
 \sup_{0\le s\le\tau}
 \max_j\left\|\varepsilon^{-1}
 \theta_j(t_\varepsilon(s))-X_j(s)\right\|
 =O_\tau((1+h)\varepsilon^2),
\end{equation}
where $h=0$ denotes flow.
Moreover $\Pr(E_{\rm esc})\ge1-2^{-m}$, there are no finite atoms in
the law of $T_j$, and, almost surely on $E_{\rm esc}$, the following
hold as $\tau\uparrow\tau_*$:
\begin{enumerate}
 \item all nonwinners extend smoothly to $\tau_*$;
 \item $K(\tau):=\|X_W(\tau)\|\to\infty$ and
       $\Psi(X_W(\tau))\ge cK(\tau)^4$ for some random $c>0$;
 \item the winner's $a$, $U^\top w$, and $U^\top z$ each have norm
       at least $c'K$, for some $c'>0$, while $w_\perp,z_\perp,b,c$
       remain constant.
\end{enumerate}
\end{swrlemma}
\begin{proof}
We use the earlier comparator lemmas through the following interface:
the expansion and tracking arguments are deterministic; the escape
probability uses a density, head-sign symmetry, and independence across
neurons. None of these ingredients requires spatial isotropy or $r<d$.
In particular this is not an invocation of the isotropic initialization
hypothesis in the alignment theorem.

For type three, $k=4$ and $2\kappa aT_{\mathrm{lead}}$ from
\eqref{swb:eq:lead} equals $\Psi$ in \eqref{swr:eq:comparator}. The global SiLU remainder
\eqref{swb:eq:silu-remainder}, with Gaussian affine moments, gives
\eqref{swr:eq:expansion}. The $k=4$ calculation in
Lemma~\ref{swb:lem:expansions} gives \eqref{swr:eq:field}; its loss
calculation gives \eqref{swr:eq:loss}, since the next teacher term and
the student variance are both $O(\varepsilon^6)$. Those calculations
use only $y\in L^2$. Lemma~\ref{swb:lem:tracking}, with mesh
$\Delta=h\varepsilon^2$, gives \eqref{swr:eq:tracking}, including
rounding, for every fixed compact comparator segment.

For the probability assertions, the proof of
Lemma~\ref{swb:lem:escape} applies to the present law: the nonzero
polynomial's zero set is null, head reflection gives probability $1/2$
of positive potential, radial homogeneity removes finite escape-time
atoms under any density, and neuron independence gives the escape
bound and uniqueness. Hence every nonwinner extends past $\tau_*$.
The pathwise argument for \eqref{swb:eq:winner-growth}, with $k=4$,
gives $K\to\infty$ and $\Psi\ge cK^4$ even if the winner started
with negative potential. Finally,
\[
 |\Psi|\le(\kappa/2)\|T_3\|_{\mathrm F}
                  |a|\,\|U^\top w\|^2\|U^\top z\|.
\]
Each factor is at most $K$, so the three active block norms are
bounded below by a positive constant times $K$. All remaining
coordinates are frozen because the potential does not depend on them.
\end{proof}

\begin{swrlemma}[A nonwinner quadratic bank at the random escape endpoint]
\label{swr:lem:bank}
Suppose $m\ge D$. Almost surely on $E_{\rm esc}$, the map
\begin{equation}\label{swr:eq:bank}
 B(\tau):\mathbb R\times\mathbb R^{m-1}\longrightarrow\swrQ,
 \qquad B(\tau)(\beta_0,\beta)
 =\beta_0+\sum_{j\ne W}\beta_jq_j(\tau)
\end{equation}
is surjective at $\tau=\tau_*$. Here $q_j$ is built from the
comparator parameters as in \eqref{swr:eq:expansion}.
There is a random $\tau_a<\tau_*$ such that its smallest positive
singular value is bounded below on $[\tau_a,\tau_*]$.
In particular, the right inverse $B^*(BB^*)^{-1}$ has uniformly
bounded norm there, using the $L^2(\gamma_d)$ norm on $\swrQ$.
Here $B^*$ denotes the Hilbert-space adjoint of $B$.
\end{swrlemma}
\begin{proof}
Choose an orthonormal coordinate basis of $\swrQ$. For any $D-1$
neurons, the determinant with columns $1,q_1,\ldots,q_{D-1}$ is
a polynomial in their affine parameters and is not identically zero.
To see this, realize the nonconstant monomials $x_i,x_i^2,x_ix_j$
as products $PV/2$, one per neuron. Together with $1$ they form a
basis. The rank-deficient set is therefore Lebesgue null in the
full neuron parameter space, including the unused heads.

The time $\tau_*$ is random, so fixed-time nonsingularity alone is
insufficient. Fix a candidate winner $i$ and condition on its entire
initial value $\vartheta_i=\xi$, with finite $t=T(\xi)$.
Conditional on $i$ being the unique winner, the other initial values
are independent copies of their Gaussian law restricted to
\[
 \mathcal D_t=\{\vartheta:T(\vartheta)>t\}.
\]
This conditional law is absolutely continuous whenever its normalizing
probability is positive; zero-probability cases contribute no winning
realizations. The survival domain $\mathcal D_t$ is open. For a smooth
autonomous vector field, its time-$t$ flow map is a smooth
diffeomorphism onto its open image, with inverse supplied by the
backward flow along each surviving segment. In particular its
Jacobian is nonsingular and it preserves absolute continuity.
Thus the joint nonwinner endpoint law at this now fixed $t$ is
absolutely continuous. The determinant polynomial vanishes with
conditional probability zero. Integrate over $\xi$ and sum over
the finitely many possible winners. This proves surjectivity at
the selected random endpoint. Continuity of the surviving
trajectories and of $BB^*$ proves the final assertion.
\end{proof}

\begin{swrlemma}[Exact quadratic cancellation in the actual feature bank]
\label{swr:lem:cancellation}
On the event of Lemma~\ref{swr:lem:bank}, fix
$\tau\in[\tau_a,\tau_*)$. Form $q_j^\varepsilon,\chi_j^\varepsilon$
from the \emph{actual rescaled parameters}
$\varepsilon^{-1}\theta_j(t_\varepsilon(\tau))$.
There are coefficients $\beta^\varepsilon\to\beta$, with
$\beta_W^\varepsilon=\beta_W=1$, such that
\begin{equation}\label{swr:eq:actualcancel}
 \beta_0^\varepsilon+\sum_j\beta_j^\varepsilon q_j^\varepsilon=0,
 \qquad
 C_\varepsilon:=\varepsilon^{-3}
 \left[\varepsilon^2\beta_0^\varepsilon+
             \sum_j\beta_j^\varepsilon g_j\right]
 \longrightarrow C:=\sum_j\beta_j\chi_j
 \quad\text{in }L^2.
\end{equation}
The limiting nonwinner/intercept coefficients obey
$\|(\beta_0,(\beta_j)_{j\ne W})\|=O(K(\tau)^2)$,
uniformly as $\tau\uparrow\tau_*$.
For every $q\in\swrQ$ there is also $Q_\varepsilon(q)$ in the
actual intercept-feature span such that
$Q_\varepsilon(q)\to q$ in $L^2$. Its nonwinner quadratic-bank
coefficients, before the $\varepsilon^{-2}$ readout scaling, are
bounded by a constant times $\|q\|_2$ once
$\varepsilon$ is sufficiently small for the chosen fixed $\tau$;
the bounding constant is uniform in these late comparator times,
but the $\varepsilon$ cutoff need not be.
\end{swrlemma}
\begin{proof}
Replace $q_j$ by $q_j^\varepsilon$ in $B$ to obtain $B_\varepsilon$.
Tracking gives $B_\varepsilon\to B$ for this fixed $\tau$, so it is
surjective for all small $\varepsilon$ and its continuous right
inverse $R_\varepsilon=B_\varepsilon^*
(B_\varepsilon B_\varepsilon^*)^{-1}$ converges to $R$.
Set the nonwinner/intercept vector to
$-R_\varepsilon q_W^\varepsilon$ and the winner coefficient to one.
This solves \eqref{swr:eq:actualcancel} exactly at every
$\varepsilon$. Since $\|q_W\|_2=O(K^2)$ and $R$ is uniformly
bounded, the coefficient estimate follows. Expanding the actual
features now gives
\[
 C_\varepsilon
   =\sum_j\beta_j^\varepsilon \chi_j^\varepsilon+O_\tau(\varepsilon^2)
   \longrightarrow C.
\]
The constant hidden in this error may depend on the fixed time
$\tau$ and is not asserted uniform up to $\tau_*$.

Similarly put $(\eta_0^\varepsilon,\eta^\varepsilon)
=R_\varepsilon q$ and define
\[
 Q_\varepsilon(q)=\varepsilon^{-2}
 \left[\varepsilon^2\eta_0^\varepsilon+
             \sum_{j\ne W}\eta_j^\varepsilon g_j\right].
\]
It equals $q+O_\tau(\varepsilon)$ in $L^2$.
The factor $\varepsilon^2\beta_0^\varepsilon$ in the first witness
is essential: omitting it leaves an uncanceled constant after
division by $\varepsilon^3$. Cancellation of comparator quadratics
alone would also be insufficient under mere $o(1)$ tracking; the
present construction requires no such division of a tracking error.
\end{proof}

\begin{swrtheorem}[Rank-one attained refit risk on the plateau]
\label{swr:thm:rankone}
Suppose $r=1$, $U=u$ is unit, and $m\ge D$.
Write the orthogonal Hermite decomposition
\[
 y=\swrE[y]+A h_3(u^\top x)+F_{\ge4}(u^\top x),
 \qquad h_3(s)=\frac{s^3-3s}{\sqrt6},\qquad A\ne0,
 \qquad A^2+\|F_{\ge4}\|_2^2=1.
\]
Here $A=\swrE[yh_3(u^\top x)]$ and $F_{\ge4}$ contains only
Hermite degrees at least four.
Almost surely on $E_{\rm esc}$, for every $\delta>0$ there is a
fixed $\tau_1<\tau_*$, arbitrarily close to $\tau_*$ if desired,
such that, for flow or any fixed MF GD step $h>0$, all sufficiently
small $\varepsilon$ satisfy
\begin{equation}\label{swr:eq:rankonerisk}
 \swrR\big(\theta(t_\varepsilon(\tau_1))\big)
 \le \|F_{\ge4}\|_2^2+\delta,
 \qquad
 \sup_{0\le t\le t_\varepsilon(\tau_1)}|L(t)-L(0)|
 \le C\varepsilon^4.
\end{equation}
For GD the supremum is over its iterates. The constants and the
smallness threshold may depend on the realized initialization,
$\delta$, and $h$. For pure $h_3$ the attained refit risk is at most
$\delta$. The escape event has probability at least $1-2^{-m}$;
this is not a specified finite-$\varepsilon$ confidence bound.
\end{swrtheorem}
\begin{proof}
Put $s=u^\top x$, $\pi=u^\top w_W$, and $\nu=u^\top z_W$ in the
comparator. Its winner has
$P_W=\pi s+P_\perp$, $V_W=\nu s+V_\perp$, where the two affine
remainders are fixed. Lemma~\ref{swr:lem:tracking} gives
$|\pi|,|\nu|\ge c'K$, so
\[
 \chi_W=\gamma s^3+O_{L^2}(K^2),\qquad
 \gamma=\tfrac14\pi^2\nu,\qquad |\gamma|\ge c''K^3.
\]
The nonwinner $\chi_j$ stay bounded to $\tau_*$ and their cancellation
coefficients are $O(K^2)$. Hence the limit $C$ in
Lemma~\ref{swr:lem:cancellation} satisfies
$\|C/\gamma-s^3\|_2=O(K^{-1})$.
At this fixed comparator time use the actual feature-span witness
\[
 H_\varepsilon
  =\frac1{\sqrt6}\left(C_\varepsilon/\gamma
                        -3Q_\varepsilon(s)\right).
\]
At this fixed $\tau$, it converges in $L^2$ to
\[
 H_\tau=\frac1{\sqrt6}(C/\gamma-3s),\qquad
 \|H_\tau-h_3(s)\|_2=O(K^{-1}).
\]
Every $\chi_j$ is a polynomial of degree at most three, so $H_\tau$
is too. Conditioning such a polynomial on $s$ leaves a polynomial of
degree at most three. Hence $F_{\ge4}(s)\perp h_3(s)-H_\tau$ in
$L^2(\gamma_d)$, including transverse terms in $H_\tau$. For the
allowed refit $\swrE[y]+AH_\varepsilon$, norm continuity therefore gives
\begin{equation}\label{swr:eq:limiting-orthogonality}
 \begin{aligned}
 \lim_{\varepsilon\downarrow0}
 \|y-\swrE[y]-AH_\varepsilon\|_2^2
 &=\|F_{\ge4}\|_2^2+A^2\|h_3(s)-H_\tau\|_2^2\\
 &=\|F_{\ge4}\|_2^2+O(K^{-2}).
 \end{aligned}
\end{equation}
This orthogonality concerns the limiting polynomial $H_\tau$.
At positive $\varepsilon$, the SwiGLU witness need not be a polynomial.
Choose $\tau_1$ sufficiently late that the limiting error is strictly
below $\|F_{\ge4}\|_2^2+\delta$, then take $\varepsilon$ small at this
fixed $\tau_1$. This order of choices proves the refit assertion.
Tracking keeps the full comparator segment compact up to that time;
\eqref{swr:eq:loss} proves the plateau bound and
$L(0)=1+O(\varepsilon^4)$. When the initialization assumptions of Theorem~\ref{swb:thm:alignment}
also hold, the same sufficiently late endpoint can certify its
leading-AGOP-direction conclusion.
\end{proof}

\begin{swrproposition}[The lower-width, fixed-time obstruction]
\label{swr:prop:lowerwidth}
Suppose $m\le D-1$ and the teacher is of type three (any rank).
For every \emph{deterministic} $\tau\ge0$, almost surely on
$\{\tau<\tau_*\}$,
\begin{equation}\label{swr:eq:lowerwidth}
 \lim_{\varepsilon\downarrow0}
 \swrR\big(\theta(t_\varepsilon(\tau))\big)=1.
\end{equation}
In fact $0\le1-\swrR=O_\tau(\varepsilon^2)$.
This includes initialization, but asserts neither a bound at all
adaptive times nor a bound for $\tau=\tau(\varepsilon)\uparrow\tau_*$.
\end{swrproposition}
\begin{proof}
At the fixed $\tau$, conditioning each initialization to survive
past $\tau$ and applying its smooth flow diffeomorphism gives an
absolutely continuous joint law. The $m+1\le D$ functions
$1,q_1(\tau),\ldots,q_m(\tau)$ are linearly independent almost surely:
their Gram determinant is a nonzero polynomial, as witnessed by
distinct members of the monomial construction in
Lemma~\ref{swr:lem:bank}.
Define $A_0:\mathbb R^{m+1}\to L^2$ to have these columns and
$A_\varepsilon$ to have columns $1,\varepsilon^{-2}g_1,\ldots,
\varepsilon^{-2}g_m$ at the actual time. Expansion and tracking give
$\|A_\varepsilon-A_0\|=O_\tau(\varepsilon)$ and
$\sigma=\sigma_{\min}(A_0)>0$.
For $y_0=y-\swrE[y]\perp\swrQ$ and every coefficient vector $v$,
\[
 |\langle y_0,A_\varepsilon v\rangle|
 \le\|y_0\|_2\|A_\varepsilon-A_0\|\,\|v\|,
 \qquad
 \|A_\varepsilon v\|_2
 \ge(\sigma-\|A_\varepsilon-A_0\|)\|v\|.
\]
The norm of the orthogonal projection of $y_0$ onto the actual
feature span is therefore $O_\tau(\varepsilon)$, and its square is
$1-\swrR$. Countably many preselected times may be intersected, but
this argument does not intersect an uncountable family of times.
\end{proof}

\begin{swrproposition}[An initial-risk statement at the boundary width]
\label{swr:prop:initial}
Assume $d\ge2$, $r=1$, and $m=D$. At initialization put
$q_j=\tfrac12P_jV_j$ and $\chi_j=\tfrac14P_j^2V_j$ using the unscaled
Gaussian parameters. Almost surely $1,q_1,\ldots,q_{D-1}$ is a basis
of $\swrQ$. Let the unique coefficients satisfy
\[
 \beta_D=1,\qquad \beta_0+\sum_{j=1}^D\beta_jq_j=0,
 \qquad H_0=(I-\swrP)\sum_{j=1}^D\beta_j\chi_j.
\]
Almost surely $H_0\ne0$ and $H_0$ is not a scalar multiple of
$h_3(u^\top x)$. For the rank-one teacher in
Theorem~\ref{swr:thm:rankone},
\begin{equation}\label{swr:eq:initialrisk}
 R_0:=\lim_{\varepsilon\downarrow0}\swrR(\theta(0))
 =1-A^2\frac{\langle h_3(u^\top x),H_0\rangle^2}{\|H_0\|_2^2}
 >\|F_{\ge4}\|_2^2.
\end{equation}
Thus, on $E_{\rm esc}$ and at this width, the endpoint theorem
does imply a strictly positive, initialization-dependent asymptotic
initial-to-endpoint improvement. No deterministic positive lower
bound on that improvement is asserted.
\end{swrproposition}
\begin{proof}
The basis assertion follows from the same determinant polynomial.
On its nonzero set, the $\beta_j$ and the Hermite coefficients of
$H_0$ are rational functions of the initial affine parameters.
Both exceptional conditions $H_0=0$ and
$H_0\in\operatorname{span}\{h_3(u^\top x)\}$ are algebraic after
multiplying by the basis determinant. They are proper conditions:
rotate coordinates so that $u=e_1$ and realize
$q_1,\ldots,q_{D-1}$ as all nonconstant monomials, with
$q=x_2^2$ realized by $P=x_2,V=2x_2$. Realize the last feature by
$P_D=2x_2,V_D=x_2$, so $q_D=x_2^2$ also. Their cubic difference is
$x_2^3/2$, hence $H_0=\swrHe_3(x_2)/2$, which is nonzero and
orthogonal to $h_3(x_1)$. A nonzero polynomial's zero set is
Lebesgue null, and the Gaussian initialization has a density.

The $D$ actual functions $1,\varepsilon^{-2}g_1,\ldots,
\varepsilon^{-2}g_{D-1}$ converge to a basis of $\swrQ$. The extra
function
\[
 \varepsilon^{-3}\left[\varepsilon^2\beta_0+
                       \sum_{j=1}^D\beta_jg_j\right]
 \longrightarrow C_0:=\sum_{j=1}^D\beta_j\chi_j
\]
converges in $L^2$. These functions span the entire actual refit
space, since $\beta_D=1$. Their limiting Gram matrix is positive
definite because $(I-\swrP)C_0=H_0\ne0$. Inverting their Gram
matrices shows convergence of the orthogonal projectors to the
projector onto $\swrQ\oplus\operatorname{span}\{H_0\}$.
Here $H_0$ lies in the third Hermite chaos. The teacher's higher
Hermite remainder and its centered lower-degree part are therefore
orthogonal to that space except for the displayed cubic projection.
This proves the equality in \eqref{swr:eq:initialrisk}; strict
Cauchy--Schwarz and $A\ne0$ prove the inequality.

For clarity, let $g=R_0-\|F_{\ge4}\|_2^2>0$ on a realized escape
sample. Choose the endpoint theorem with $\delta=g/4$. Its endpoint
risk is at most the tail plus $g/4$ for all small $\varepsilon$,
while convergence of the initial risk makes it at least
$R_0-g/4$. Their difference is at least $g/2$.
\end{proof}

\begin{swrcorollary}[A probability bound for the boundary-width improvement]
\label{swr:cor:initialprobability}
In Proposition~\ref{swr:prop:initial}, assume additionally that the
initialization law is invariant under simultaneous spatial rotations
of all $w_j,z_j$. This holds for independent isotropic Gaussian inner
blocks, with any fixed positive head and inner variances.
For every $q\in(0,1)$ and $\delta>0$, there is an event of probability
at least
\begin{equation}\label{swr:eq:initialprobability}
 \left[1-2^{-D}-\frac{3}{q\,d(d+2)}\right]_+
\end{equation}
on which a fixed random $\tau_1<\tau_*$ can be chosen so that,
for flow or any fixed-step GD and all sufficiently small
$\varepsilon$,
\begin{equation}\label{swr:eq:initialgain}
 \swrR(\theta(0))-
 \swrR\big(\theta(t_\varepsilon(\tau_1))\big)
 \ge A^2(1-q)-\delta,
 \qquad
 \sup_{t\le t_\varepsilon(\tau_1)}|L(t)-L(0)|=O(\varepsilon^4).
\end{equation}
The time and initialization cutoff are sample dependent, and the
probability bound is informative only when its right-hand side is
positive. It is an asymptotic statement at $m=D$, not a guarantee
at a prescribed finite initialization scale.
\end{swrcorollary}
\begin{proof}
The random unoriented line spanned by $H_0$ is rotation invariant:
rotations act on the polynomial bank, its cancellation identity,
and its Hermite projection equivariantly. Normalize it as
$H_0/\|H_0\|_2=\swrHe_3[T]/\sqrt6$ for a symmetric tensor $T$
with Frobenius norm one; its sign will not matter. Here
$\swrHe_3[T]=\sum_{ijk}T_{ijk}\swrHe_3(x)_{ijk}$, and the Hermite
isometry gives
$\langle h_3(u^\top x),H_0/\|H_0\|_2\rangle^2=T[u,u,u]^2$.
Write this squared correlation as $Z$.
For every deterministic symmetric $T$ with $\|T\|_F=1$, the sixth
sphere moment, or Gaussian Wick pairing followed by a radial
integration, gives, with $S^{d-1}$ the Euclidean unit sphere,
\[
 \swrE_{\omega\sim\mathrm{Unif}(S^{d-1})}\left[T[\omega,\omega,\omega]^2\right]
 =\frac{6+9\|\operatorname{tr}T\|^2}{d(d+2)(d+4)}.
\]
There are six pairings connecting all three indices of one tensor
to the other, giving $6\|T\|_F^2$, and nine pairings with an
internal contraction in each tensor, giving
$9\|\operatorname{tr}T\|^2$. The trace map
$L:T\mapsto(\sum_jT_{ijj})_i$ has adjoint
\[
 (L^*v)_{ijk}=(v_i\delta_{jk}+v_j\delta_{ik}+v_k\delta_{ij})/3,
 \qquad LL^*=\frac{d+2}{3}I_d.
\]
Therefore $\|\operatorname{tr}T\|^2\le(d+2)/3$ and the displayed
moment is at most $3/[d(d+2)]$. Averaging an independent uniform
rotation and using invariance of the unoriented line gives
$\swrE_{\vartheta}[Z]\le3/[d(d+2)]$. Markov's inequality yields
$\Pr(Z>q)\le3/[q\,d(d+2)]$.
Lemma~\ref{lem:common-event} gives the lower bound
\eqref{swr:eq:initialprobability} for $\{Z\le q\}\cap E_{\rm esc}$.
No rotation law conditional on escape is used. On this intersection the limiting initial risk exceeds
the Hermite tail by at least $A^2(1-q)$. Apply the endpoint theorem
with tolerance $\delta/2$ and the initial convergence with tolerance
$\delta/2$ to obtain \eqref{swr:eq:initialgain}.
\end{proof}

\begin{swrproposition}[A one-winner projection bound at general rank]
\label{swr:prop:generalrank}
Let $1\le r\le d$, the teacher be of type three, and $m\ge D$.
Almost surely on $E_{\rm esc}$, for every sufficiently late fixed
$\tau<\tau_*$ define unit vectors in the teacher subspace by
\[
 \mu=\frac{UU^\top w_W}{\|U^\top w_W\|},\qquad
 \nu=\frac{UU^\top z_W}{\|U^\top z_W\|},\qquad
 \rho=\mu^\top\nu.
\]
There is a finite random constant $C$, independent of these late
times, such that
\begin{equation}\label{swr:eq:generalrank}
 \limsup_{\varepsilon\downarrow0}
 \swrR\big(\theta(t_\varepsilon(\tau))\big)
 \le 1-\frac{T_3[\mu,\mu,\nu]^2}{2+4\rho^2}
          +\frac{C}{K(\tau)}.
\end{equation}
In particular the quotient in this formula is captured teacher
energy, while the whole right-hand side is residual risk.
For some random $c>0$,
$|T_3[\mu,\mu,\nu]|\ge2c/\kappa$ at all sufficiently late times.
This result asserts neither axis entry nor recovery of a full
teacher subspace.
\end{swrproposition}
\begin{proof}
Replace the rank-one leading product in the cancellation proof by
\[
 \chi_W=\gamma(\mu^\top x)^2(\nu^\top x)+O_{L^2}(K^2),
 \qquad \gamma=\tfrac14\|U^\top w_W\|^2\|U^\top z_W\|
 \ge c'K^3.
\]
The same bounded nonwinner inverse and $O(K^2)$ coefficients give
a feature-span limit within $O(K^{-1})$ of this cubic monomial.
Use $Q_\varepsilon$ to subtract its linear Hermite part and obtain
a limit within $O(K^{-1})$ of
\[
 H_{\mu,\nu}(x)=(\mu^\top x)^2(\nu^\top x)
                  -\nu^\top x-2\rho\mu^\top x.
\]
This function is in the third Hermite chaos. To compute its norm
without numerical approximation, take independent $Z,W\sim N(0,1)$
and write $\mu^\top x=Z$, $\nu^\top x=\rho Z+\sqrt{1-\rho^2}W$.
Then
\[
 H_{\mu,\nu}=\rho\swrHe_3(Z)
       +\sqrt{1-\rho^2}(Z^2-1)W,
 \qquad \|H_{\mu,\nu}\|_2^2=6\rho^2+2(1-\rho^2)=2+4\rho^2.
\]
Its correlation with $y-\swrE[y]$ is exactly
$T_3[\mu,\mu,\nu]$. Projecting onto its approximating single
feature-span direction proves \eqref{swr:eq:generalrank}; the
denominator is at least two, so the $O(K^{-1})$ perturbation is
uniform over unit $\mu,\nu$. Finally
\[
 cK^4\le\Psi(X_W)
 =\frac\kappa2 a_W\|U^\top w_W\|^2\|U^\top z_W\|
                 T_3[\mu,\mu,\nu],
\]
and the absolute product of the four parameter factors is at most
$K^4$, giving the asserted lower bound.
\end{proof}

\begin{swrremark}[Scope and order of limits]\label{swr:rem:scope}
The threshold $m=D$ counts the intercept: $m+1$ leading quadratic
columns first admit a cancellation when $m=D$, and there are
$m-D+1$ such coefficient directions when the quadratic bank has
full rank. They are combinations of neurons, not independently
available residual neurons. At larger widths further cancellation
layers can improve the initial refit, so Proposition~\ref{swr:prop:initial}
is not a lower bound for arbitrary $m\ge D$.
The witness readout coefficients can be $O_\tau(\varepsilon^{-3})$
and its intercept can be $O_\tau(\varepsilon^{-1})$.
The limit first fixes a late pre-escape comparator time and then
shrinks initialization; neither the lower-width proposition nor
compact-time tracking controls an $\varepsilon$-dependent approach
to escape. The high-Hermite remainder is retained in the rank-one
bound and need not be small. No assertion about a subsequent
fixed-size decrease of the \emph{trained} loss is proved here;
release remains an open transition, with only conditional routes
available. One winner's AGOP
contribution can have rank two, so this one-direction refit argument
does not assign all remaining teacher directions to nonwinners.
\end{swrremark}

\subsection{A joint rank-one consequence}
\label{swj:sec:joint}

The initial-risk calculation gives a concrete simultaneous consequence of
the alignment and refit results. This consequence is restricted to the
boundary width and retains the qualitative initialization cutoff.
Here $d$ is the input dimension, $m$ the student width,
$D=(d+1)(d+2)/2$ the dimension of the degree-at-most-two polynomial
space, and $\varepsilon$ the initialization scale. The scalar teacher
has the form $y=F(u^\top x)$, $\|u\|=1$, with unit variance, vanishing
degree-one and degree-two Hermite coefficients, and nonzero cubic
coefficient $a_3=\mathbb E[F(Z)h_3(Z)]$, $Z\sim N(0,1)$.
The statistic $A_{\mathrm{top}}$ is the minimum squared overlap with
$u$ among unit top AGOP eigenvectors, and $\mathcal R$ is the
unrestricted population refit MSE from \eqref{swr:eq:refit}.
Probabilities below refer to the coupled Gaussian initialization;
$q_{\mathrm A},q_{\mathrm R}$ are headroom thresholds and
$\delta_{\mathrm A},\delta_{\mathrm R}$ the endpoint tolerances.

\begin{corollary}[Simultaneous alignment and refit gains]
\label{swj:cor:joint}
Assume the rank-one, type-three setting, $d\ge2$,
$m=D=(d+1)(d+2)/2$, and the independent isotropic Gaussian initialization
of \eqref{swb:eq:init}. Put $a_3=\langle F,h_3\rangle$ and let
$B_d\sim\operatorname{Beta}(1/2,(d-1)/2)$.
Fix $q_{\mathrm A},q_{\mathrm R}\in(0,1)$ and positive tolerances
$\delta_{\mathrm A},\delta_{\mathrm R}$, with
$q_{\mathrm A}+\delta_{\mathrm A}<1$.
For flow or any fixed-step GD, choose a common plateau endpoint as in the
preceding theorems. The limit inferior, as $\varepsilon\downarrow0$,
of the probability that the plateau conclusions and both inequalities
\begin{align*}
 A_{\mathrm{top}}(\mathrm{endpoint})-A_{\mathrm{top}}(0)
   &\ge 1-\delta_{\mathrm A}-q_{\mathrm A},\\
 \mathcal R(0)-\mathcal R(\mathrm{endpoint})
   &\ge a_3^2(1-q_{\mathrm R})-\delta_{\mathrm R}
\end{align*}
hold is at least
\begin{equation}
 \left[\Pr(B_d\le q_{\mathrm A})-2^{-D}
       -\frac{3}{q_{\mathrm R}d(d+2)}\right]_+.
 \label{swj:eq:jointprobability}
\end{equation}
The time is proportional to $\varepsilon^{-2}$ and the entire-prefix
loss variation is $O(\varepsilon^4)$, with sample-dependent constants.
\end{corollary}
\begin{proof}
Use $E_1=E_{\rm esc}$ and $E_2=\{Z\le q_{\mathrm R}\}$,
where $Z$ is the scale-independent squared cubic correlation in
Corollary~\ref{swr:cor:initialprobability}. Their failure bounds are
$2^{-D}$ and $3/[q_{\mathrm R}d(d+2)]$.
On their intersection, discard the null exceptional sets for unique
escape and the nonwinner bank. Choose one sufficiently late rational
$\tau_1<\tau_*$ so that the alignment and attained-refit proofs have
strict margins for $\delta_{\mathrm A}$ and $\delta_{\mathrm R}/2$.
These estimates hold throughout a sufficiently late comparator interval,
so the same time works for both. It is fixed before $\varepsilon$ shrinks.

Initial-risk convergence, with error at most $\delta_{\mathrm R}/2$,
and the two endpoint transfers give a common positive measurable cutoff
$c(\vartheta,\delta_{\mathrm A},\delta_{\mathrm R},h)$. Measurability follows
by countable rational time choices and the compact tracking/error bounds
and nonsingular Gram-matrix margins in those proofs. Below this cutoff,
refit improves by at least $a_3^2(1-q_{\mathrm R})-\delta_{\mathrm R}$,
endpoint alignment is at least $1-\delta_{\mathrm A}$, and the full-prefix
plateau bound holds, for flow and the rounded GD endpoint.

At each scale let $H_\varepsilon=\{A_{\mathrm{top}}(0)\le q_{\mathrm A}\}$.
Corollary~\ref{swb:cor:headroom} gives its unconditional failure bound
$\Pr(B_d>q_{\mathrm A})$. Lemma~\ref{lem:common-event} now proves
\eqref{swj:eq:jointprobability}. No stabilization of $H_\varepsilon$
along a coupled draw, independence, or conditional rotation law is needed.
\end{proof}

For example, with a pure $h_3$ teacher, $d=16$, $m=153$,
$q_{\mathrm A}=0.4$, $\delta_{\mathrm A}=0.1$,
$q_{\mathrm R}=0.25$, and $\delta_{\mathrm R}=0.05$, the two gains
are at least $0.5$ and $0.7$, respectively. The exact probability
expression in \eqref{swj:eq:jointprobability} is approximately $0.9519$.
This is an asymptotic probability lower bound: it supplies neither a
practical numerical initialization scale nor a finite-sample guarantee.
The refit is unrestricted, and the original trained loss still remains
near one throughout the certified window. Its subsequent decrease is
a separate question.

\subsection{What is not established by the plateau theorem}
\label{app:swiglu-release}

For a normalized pure cubic teacher $y=h_3(u^Tx)$, where
$x\sim N(0,I_d)$, $\|u\|=1$, and $h_3(s)=(s^3-3s)/\sqrt6$, the small-initialization
comparison suggests that one neuron dominates the first departure from the
plateau. A natural candidate description then follows the population
gradient flow of one aligned SwiGLU neuron,
\[
 f(x)=b_0+\frac{\kappa}{m}a\,
 \operatorname{SiLU}(\rho u^Tx+b)(q u^Tx+c),
\]
where $m$ is the original student width, $\kappa>0$ its output scale,
$a$ the surviving head, $\rho,q$ the scalar gate and value weights
along $u$, and $b,c$ their biases. The intercept $b_0$ is minimized out.
This is a candidate mechanism for
loss release, not a consequence of the preceding compact-time comparison.

A rigorous transfer requires control from a fixed pre-escape comparison
time to an actual parameter norm independent of the initialization scale,
as well as a characterization of the trajectory leaving the degenerate
origin. The compact-time argument stops before that transition. Thus the
long plateau interval proved above should not be interpreted as a theorem
identifying its endpoint with the first substantial loss decrease.

The candidate one-neuron release mechanism described here is restricted
to pure $h_3$. Higher Hermite components of the teacher can affect the
finite-amplitude vector field; there is no claim that the same curve
applies to every teacher covered by the alignment theorem. Similarly,
fixed-step GD at finite amplitude has a discretization error that does not
vanish merely because initialization tends to zero. A continuous-flow
release curve therefore does not establish the corresponding fixed-step
GD limit without an additional argument. The numerical evidence is
reported with these distinctions in place.

\subsection{Population-quadrature experiments with a SwiGLU student}
\label{swexp:protocol}

These experiments examine finite-initialization trajectories; they do not establish a finite-$\epsilon$ probability bound or a release theorem. We use the student and profiled intercept above, $x\sim\mathcal N(0,I_d)$, $\kappa=1$, and the teacher basis $U=(e_1,\ldots,e_r)$, where $e_i$ are coordinate unit vectors. Here $d$ is the input dimension, $m$ the student width, $r$ the teacher rank, $\kappa$ the output scale, and $\epsilon=\varepsilon$ the initialization scale. The independent Gaussian initializations have coordinate standard deviation $\epsilon/\sqrt{d+1}$ in each inner block and standard deviation $\epsilon$ in the output head. The intercept is minimized analytically at each state; the other three blocks share one learning rate. All expectations and variances in this protocol are population quantities over $x$; $n$ counts optimizer updates and $\theta$ collects the trained neuron parameters.

\paragraph{Reference implementation.}
Here \emph{reference} denotes the archived earlier population-quadrature
implementation and its saved diagnostics, included with the reproduction
code. It is an implementation of the same student model, not an external
published benchmark. Reference reproductions repeat its specified
configurations; new teacher, seed and rank settings are identified
separately. Section~\ref{swexp:outcomes} reports the diagnostic comparisons,
including the width-256 checkpoint replay and the retained validation failures.

\paragraph{Teacher normalization and clock.}
The experimental implementation uses $y=\gamma\sum_{i=1}^r c_i\phi(x_i)$ with $\mathbb E[y^2]=1$, where $\phi$ is the listed teacher link, $c=(c_1,\ldots,c_r)$ its coefficient vector, and $\gamma>0$ the normalizing scalar. We report $\ell=L/V_y$, where $L=\operatorname{Var}(y-\widetilde f)$ is the profiled training MSE and $V_y=\operatorname{Var}(y)$, so the intercept-only baseline is one. This differs from rescaling the teacher itself to have unit variance: for rank-one ReLU, $V_y=1-1/\pi$, whereas all centered teachers here have $V_y=1$. The force is the gradient of the unnormalized $L$. Recorded time is the mean-field clock $t=\sum_n\Delta t_n$ for $\theta_{n+1}=\theta_n-m\Delta t_n\nabla_\theta L$. An ordinary parameter-space learning rate is therefore $\eta=m\Delta t$. Adaptive trajectories use explicit Euler steps with a maximum relative parameter change rule and are labeled as flow approximations. Separate fixed-$\Delta t$ runs test ordinary gradient descent; adaptive trajectories are not evidence for a fixed-step guarantee.

\paragraph{Loss-selected plateau and seed policy.}
For $\delta=10^{-3}$, define the measured plateau as the maximal initial prefix of recorded optimizer states on which $|\ell(t)-\ell(0)|\leq\delta$. The entire update loss history is used, before examining alignment or refit. We also report the stricter $\delta=10^{-4}$ prefix for comparison with the earlier reference diagnostics. Alignment and refit at a prefix endpoint are evaluated at the last saved diagnostic checkpoint inside it; the audit retains the exact update-grid endpoint, diagnostic time and lag. Subsequent crossing times are first observed update times, without interpolation. Seed~0 is the fixed representative; all declared seeds and unsuccessful runs are retained in the tables and supplementary panels.

\paragraph{Population diagnostics.}
The AGOP $M=\mathbb E[\nabla_x f\nabla_x f^\top]$ includes all neuron cross terms; $f=b_0+\widetilde f$ is the trained predictor with profiled intercept. We distinguish $A_{\rm top}=\|U^\top e_1(M)\|^2$ from $A_{\min}=\sigma_{\min}^2(U^\top Q_r(M))$. Here $e_1(M)$ is a unit leading eigenvector, $Q_r(M)$ has orthonormal columns spanning the selected top-$r$ eigenspace, and $\sigma_{\min}$ is the smallest singular value. The first measures whether one leading direction lies in the teacher subspace; the second measures its weakest recovered direction. At rank one they agree. At higher rank we display both, and do not infer full-subspace recovery from $A_{\rm top}$.

\paragraph{Integration and numerical refit.}
The population integrals are evaluated by deterministic quadrature, with no finite training sample. Two-projection feature pairs use tensor Gauss--Hermite rules, diagonal terms use a one-dimensional rule, and teacher terms condition on the relevant teacher coordinate. That coordinate is integrated on $[-10,10]$ using piecewise Gauss--Legendre rules weighted by the Gaussian density, split at $-6,-4,-2,0,2,4,6$ and link kinks. These are numerical approximations, not certified exact integrals. The force uses Gaussian Stein identities before quadrature; at finite order it need not be the exact derivative of the quadrature-discretized loss scalar. Higher-order checks and Gram-matrix diagnostics are retained with the raw records.

For the frozen centered feature vector $\bar g$, with
$\bar g_j=g_j-\mathbb E[g_j]$, write its Gram matrix as
$G=\mathbb E[\bar g\bar g^\top]$, its teacher covariance as
$b=\mathbb E[(y-\mathbb E[y])\bar g]$, and a candidate head as
$w\in\mathbb R^m$. This $w$ is a readout vector, distinct from the
spatial gate weights $w_j$. Its normalized prediction risk is
\[
 R(w)=\frac{V_y-2b^\top w+w^\top G w}{V_y}.
\]
The unrestricted population refit is $\inf_w R(w)$, without a coefficient budget. Numerical ridge or eigencutoff solves approximate this infimum; their attained risks and cutoff sensitivity are reported explicitly. In particular, the reference diagnostic $[V_y-b^\top(G+\lambda I)^{-1}b]/V_y$, where $\lambda$ is the ridge parameter and $I=I_m$, includes the ridge penalty and is not the prediction risk. We retain it for comparison with the reference, but do not label it an exact unrestricted refit. This distinction is most relevant for the $d=16,m=256$ cancellation experiment.

\paragraph{Loss stopping thresholds and observation ceilings.}
The loss prefix above selects a diagnostic window; it does not stop training.
Training stops at the first evaluated update for which
$\ell_n\le\ell_{\rm stop}$, $t_n\ge t_{\max}$, or $n\ge N_{\max}$,
or earlier on a numerical failure. The configuration-specific thresholds are
\[
\ell_{\rm stop}=\begin{cases}
0.5,&\text{rank-one $h_3$, $m=64$, including fixed-step controls},\\
0.7,&\text{rank-one $h_3+0.3h_2$, ReLU, and tanh},\\
0.6,&\text{rank-two $h_3$, $d=16,m=64$, either coefficient vector},\\
0.9,&\text{the two $d=4,m=6$ smokes and the $m=256$ run}.
\end{cases}
\]
The adaptive rank-one, smoke and width-256 runs use
$(t_{\max},N_{\max})=(10^6,20{,}000)$; the fixed-step controls use
$(250,20{,}000)$; the six rank-two scientific runs use the common
analysis ceilings $(500,3000)$. Adaptive steps cap the relative
parameter change at $0.01$ and the mean-field step at $50$.
Thresholds are checked on the update grid, so the last loss can undershoot
$\ell_{\rm stop}$. A ``loss stop'' records threshold attainment, not a
converged or fixed-horizon final loss. Time and update ceilings censor runs
that have not attained their loss target. The timing of the ceiling
amendments and the retained out-of-budget tail are specified below.

\paragraph{Runtime amendments and censoring.}
The seed/configuration matrix was fixed before production, but the execution horizons were amended for runtime and are not presented as wholly preregistered. Before any fixed-step control launched, its nominal mean-field horizon was capped at 250; the original steps and seed were retained. The stopping rule is checked on the accumulated update clock, allowing at most one step of overshoot (the $\Delta t=.1$ run stops at $t=250.1$), and actual endpoint times are retained. After a rank-two trajectory stalled with very small adaptive steps, a common rank-two budget of time 500 and 3000 updates was adopted. The balanced seed~2 run completed naturally at update3060; its raw tail is retained but the endpoint analysis stops at update3000. Only unbalanced seed~0 was interrupted and replayed with the same initialization and force, with the interrupted attempt and its log preserved. Other previously completed trajectories finished inside the amended bounds. The results distinguish final outcome slots from interrupted execution attempts and mark all budget-censored outcomes. These finite observation windows do not establish long-time convergence.

\subsection{Finite-initialization outcomes}

\label{swexp:outcomes}

We retain the notation of Section~\ref{swexp:protocol}: $d,m,r$ are
the input dimension, student width, and teacher rank; $\epsilon$ is the
initialization scale; $c$ is the teacher coefficient vector; and $t$ is
mean-field time. The reported $\ell=L/\operatorname{Var}(y)$ is normalized
trained loss, $A_{\rm top}$ and $A_{\min}$ are leading and weakest-direction
AGOP alignments, and $R$ is normalized attained numerical refit MSE.
The polynomials are $h_2(s)=(s^2-1)/\sqrt2$ and
$h_3(s)=(s^3-3s)/\sqrt6$. Subscripts or superscripts $0,P,D,f$ mark
initialization, the last diagnostic inside a loss prefix, the last
in-budget diagnostic, and the final in-budget update, respectively.

The predeclared matrix contains 24 trajectories: two reference smoke checks, three rank-one cubic reproductions, nine rank-one illustrations with degree-one or degree-two teacher signal, six balanced/unbalanced rank-two illustrations, one width-256 reference reproduction, and three fixed-step controls. Figures~\ref{swexp:rankone0}--\ref{swexp:smoke} retain all groups. Tables~\ref{swexp:tableprimary} and~\ref{swexp:tablestrict} report every seed, including the stricter loss-prefix comparison; Table~\ref{swexp:tablefinal} records final observations.

For rank-one $h_3$, initial leading-direction alignments are 0.128, 0.022, 0.000 (seeds 0, 1, 2). At the primary loss-selected endpoints they are 0.998, 0.994, 0.994, and at the stricter endpoints 0.986, 0.974, 0.977. The primary update-grid prefix durations are 226.4, 266.7, 312.1. All three runs reach the stopping threshold $\ell_{\rm stop}=0.5$; their last evaluated losses are 0.4999, 0.4997, 0.5000 (rounded). These values record threshold crossings beyond the plateau, not unconstrained final-loss outcomes or consequences of the alignment theorem.

For $h_3+.3h_2$, all three primary-prefix leading-direction alignments lie in [0.992,0.996]; all three runs subsequently reach $\ell_{\rm stop}=0.7$, with last evaluated losses in [0.6972,0.7000].

For ReLU, all three primary-prefix leading-direction alignments lie in [0.979,0.997]; all three runs subsequently reach $\ell_{\rm stop}=0.7$, with last evaluated losses in [0.6932,0.6972].

For tanh, all three primary-prefix leading-direction alignments lie in [0.956,0.983]; all three runs subsequently reach $\ell_{\rm stop}=0.7$, with last evaluated losses in [0.6926,0.6998].

For $c=(1,1)$, the primary-prefix weakest-direction alignments are 0.976, 0.151, 0.178 (seeds 0, 1, 2), while the corresponding leading-direction alignments are 0.997, 0.997, 0.995.

For $c=(1,.6)$, the primary-prefix weakest-direction alignments are 0.008, 0.067, 0.004 (seeds 0, 1, 2), while the corresponding leading-direction alignments are 0.997, 0.994, 0.994.

Full-subspace alignment need not be monotone within the low-loss-movement regime: balanced rank-two seed~2 has $A_{\min}=.871$ at the stricter prefix checkpoint, $.178$ at the primary-prefix checkpoint and approximately $.999$ at its last in-budget diagnostic. These endpoints illustrate variability rather than persistent failure or monotone recovery.

For the width-256 reproduction, the numerical refit risk changes from 0.8183 to 0.0435 on the primary loss prefix, with normalized loss 1.000000 to 0.999037. At the reference-comparison tolerance $1.2\times10^{-5}$, the legacy penalized diagnostic changes from 0.8183 to 0.3306, whereas the corrected eigencutoff prediction risk changes from 0.8183 to 0.3306. This additional threshold was fixed from the reference statement before the new width-256 output arrived.

The reference schedule additionally saved update 66, which lies between new diagnostics 52 and 69. An exact 14-update replay from saved state 52 recovers its $t=48.834273$, $\ell=0.999988250$ and legacy diagnostic 0.232107532. The attained eigencutoff prediction risk is 0.232107528 (all four cutoffs agree), changing by only $1.11\times10^{-14}$ at doubled quadrature. Thus the reference decrease from approximately $.82$ to $.23$ is reproduced at its original loss movement $1.17993\times10^{-5}$. This reference-defined replay is a checkpoint validation, not an additional seed, a new primary endpoint or a selected minimum.

All 3830 valid saved-state refit diagnostics have finite normalized risks in $[0,1]$ and positive equilibrated Gram eigenvalues (smallest eigenvalue ratio $2.88\times10^{-4}$). The four prescribed spectral cutoffs retain full rank and give the same risk to recorded precision. However, 955 saved states have Gram relative asymmetry above $10^{-8}$, with maximum $1.07\times10^{-4}$; the solver uses the symmetric part. Positivity and cutoff stability therefore do not certify the underlying Gaussian integrals or the exact unrestricted infimum. Five post-failure diagnostics are explicitly omitted.

Doubling quadrature at the selected rank-one cubic seed-0 states changes refit risk by at most $1.08\times10^{-6}$; at the width-256 primary-prefix checkpoint the change is $3.36\times10^{-8}$, increasing to $4.36\times10^{-4}$ at its final state. Across the six rank-two runs the largest checked change is $4.52\times10^{-4}$ in refit risk and $7.28\times10^{-4}$ in leading alignment; this includes the separately retained out-of-budget tail. These saved-state comparisons are neither integral-error bounds nor higher-order retraining.

The fixed $\Delta t=.1$ endpoint has a material quadrature sensitivity: doubled orders change its normalized loss from $.650360$ to $.653321$ and numerical refit risk from $.412908$ to $.401575$ (difference $.011333$). Its fine numerical endpoint values should not be interpreted as resolved to the displayed digits. The $\Delta t=.05$ endpoint is unchanged at recorded precision, but its trajectory is nonmonotone and returns to loss $1.000056$ by the horizon.

The unchanged reference validation script exits with code zero but prints four failed checks: two finite-difference gradient discrepancies ($1.35\times10^{-4}$ and $1.79\times10^{-4}$) and two Monte Carlo comparisons with largest standardized discrepancies $4.61$ and $4.66$. We retain these failures. The separate higher-order finite-difference diagnostic reduces its four gradient discrepancies to $2.12\times10^{-9}$, $3.86\times10^{-7}$, $4.88\times10^{-8}$ and $7.01\times10^{-7}$, respectively; this convergence evidence does not replace the failed checks. The successful record-integrity audit is a distinct claim.

The fixed $\Delta t=0.3$ control is a retained numerical failure: after exact fixed steps through update 762, the reference relative-gradient safeguard activates and the numerical clock freezes near $t=228.9005$. The recorded loss reaches approximately $6.11\times10^{237}$. Post-safeguard alignment/refit is omitted from the figure and no stability or release conclusion is drawn from that tail.

The fixed $\Delta t=0.1$ control ends at its actual $t=250.1000$ with normalized loss 0.6504 at the nominal observation horizon (censored).

The fixed $\Delta t=0.05$ control ends at its actual $t=250.0000$ with normalized loss 1.0001 at the nominal observation horizon (censored).

The rank-two panels show why leading-direction alignment and recovery of the whole teacher subspace must be separated: the dashed $A_{\min}$ curves can lag far behind the solid $A_{\rm top}$ curves. All three seeds are displayed for each coefficient vector. The width-256 and fixed-step experiments are interpreted separately, with the numerical limitations just described.

Throughout these figures, navy denotes loss, teal denotes AGOP alignment,
and purple denotes numerical refit risk, matching the main figures.
Sparse downward-triangle, diamond and filled-plus markers distinguish
overlaid seeds or controls; the first-row legend gives their assignment.
The divergent fixed-step control uses stars. These markers are placed on
existing saved samples and carry the same identity across rows.
Markers on dotted vertical boundaries identify each run's loss-prefix end.

In all panels, circular markers identify the last diagnostic within the primary loss-selected prefix; a cross at a loss-curve endpoint marks a numerical failure or the common time/step analysis horizon reached before the loss stopping target. Any retained raw states beyond the common budget appear as faint dotted tails and are excluded from the endpoint tables. A single-run panel shades that prefix; a multi-run panel uses a separate vertical dotted boundary for each run. Refit curves use the predetermined $10^{-12}$ relative eigencutoff after diagonal equilibration. Shading around those curves spans the four cutoffs $10^{-8},10^{-10},10^{-12},10^{-14}$; it is a numerical sensitivity range, not a statistical confidence interval or a rigorous risk bound. Dotted purple refit curves, when visibly different, give the actual prediction risk of the reference ridge coefficients. Hollow squares show doubled-quadrature evaluations of the same saved states, without retraining; the width-256 triangle marks the separately reproduced reference checkpoint 66.

\begin{figure}[!htbp]
\centering
\includegraphics[width=\textwidth]{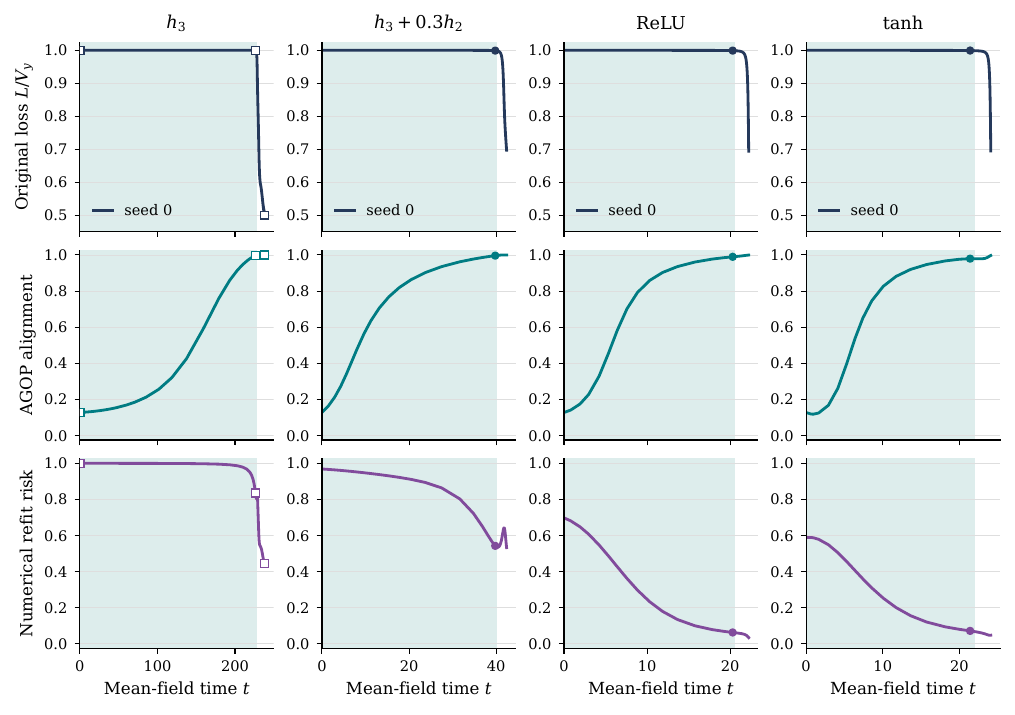}
\caption{Fixed representative seed~0 for the four rank-one teacher links. Rows show original normalized loss, leading AGOP alignment and numerical frozen-feature refit risk. These are adaptive Euler flow approximations, with $d=16,m=64,\epsilon=.05$. The display includes the subsequent measured loss decrease.}
\label{swexp:rankone0}
\end{figure}

\begin{figure}[!htbp]
\centering
\includegraphics[width=\textwidth]{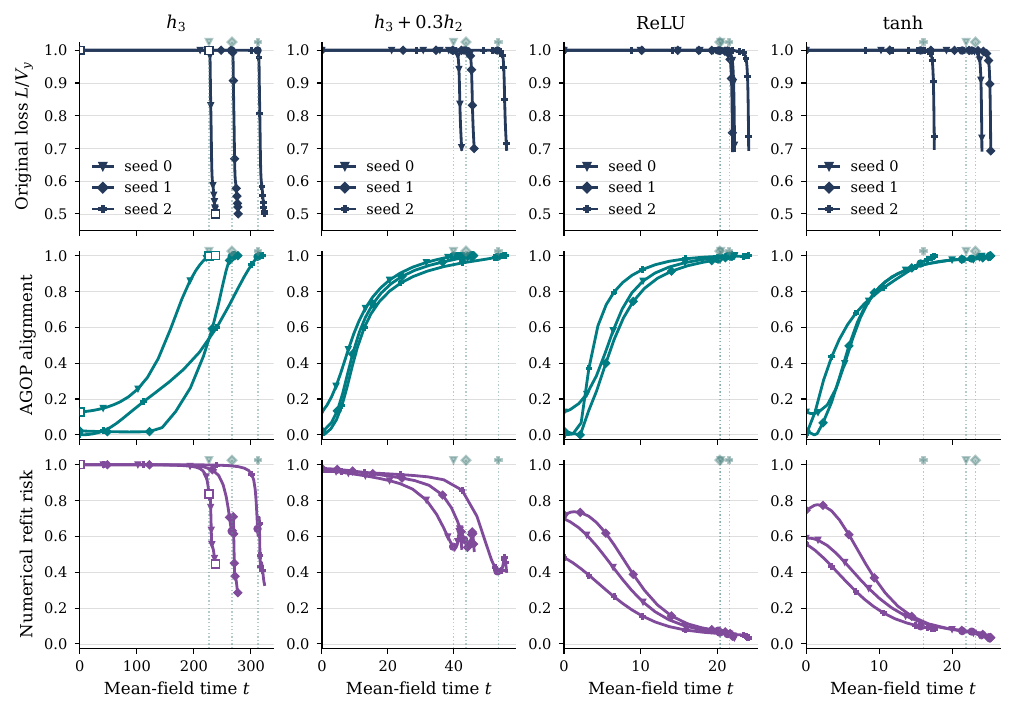}
\caption{All three predeclared seeds for every rank-one link, using the same scales and diagnostics as Figure~\ref{swexp:rankone0}. Each curve is an individual run; no seed is selected by its outcome.}
\label{swexp:rankoneall}
\end{figure}

\begin{figure}[!htbp]
\centering
\includegraphics[width=\textwidth]{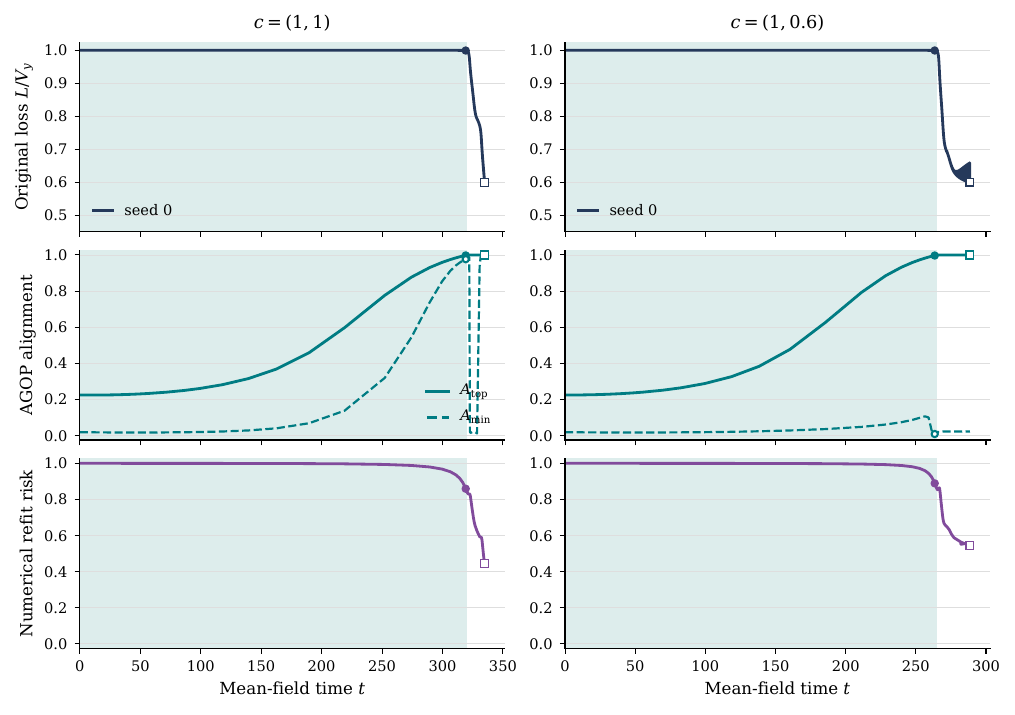}
\caption{Fixed seed~0 for balanced and unbalanced rank-two cubic teachers. In the alignment row, solid curves are $A_{\rm top}$ and dashed curves are $A_{\min}$. Strong leading-direction alignment can coexist with weak recovery of the least aligned teacher direction.}
\label{swexp:ranktwo0}
\end{figure}

\begin{figure}[!htbp]
\centering
\includegraphics[width=\textwidth]{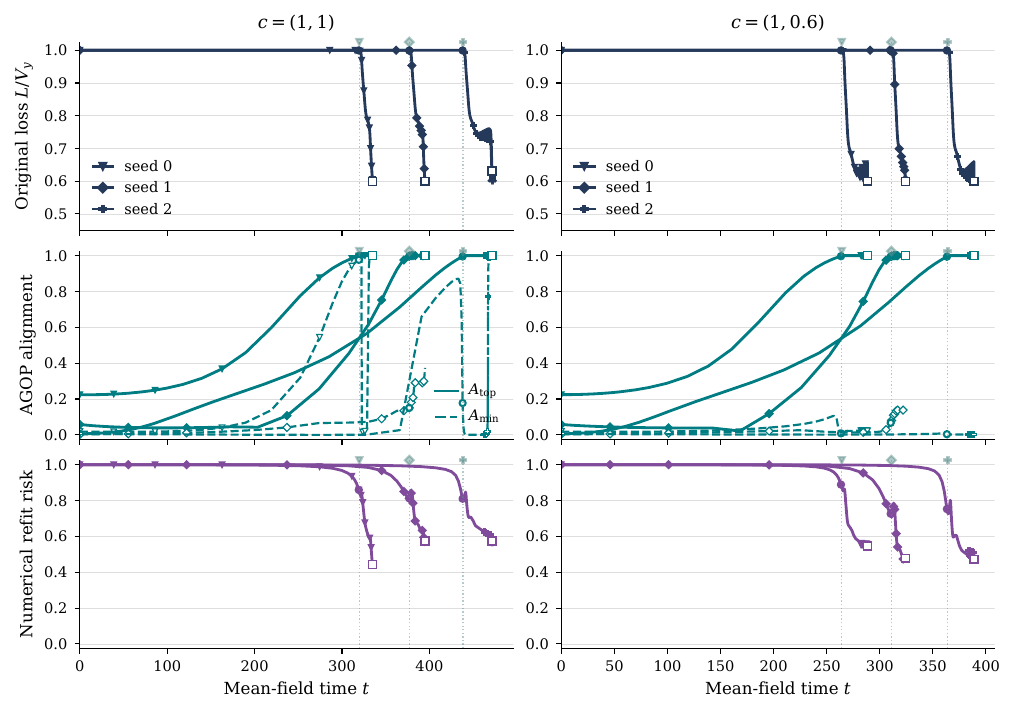}
\caption{All three seeds for each rank-two coefficient vector. The weakest-direction statistic is retained for every seed, including cases with poor simultaneous recovery.}
\label{swexp:ranktwoall}
\end{figure}

\begin{figure}[!htbp]
\centering
\includegraphics[width=\textwidth]{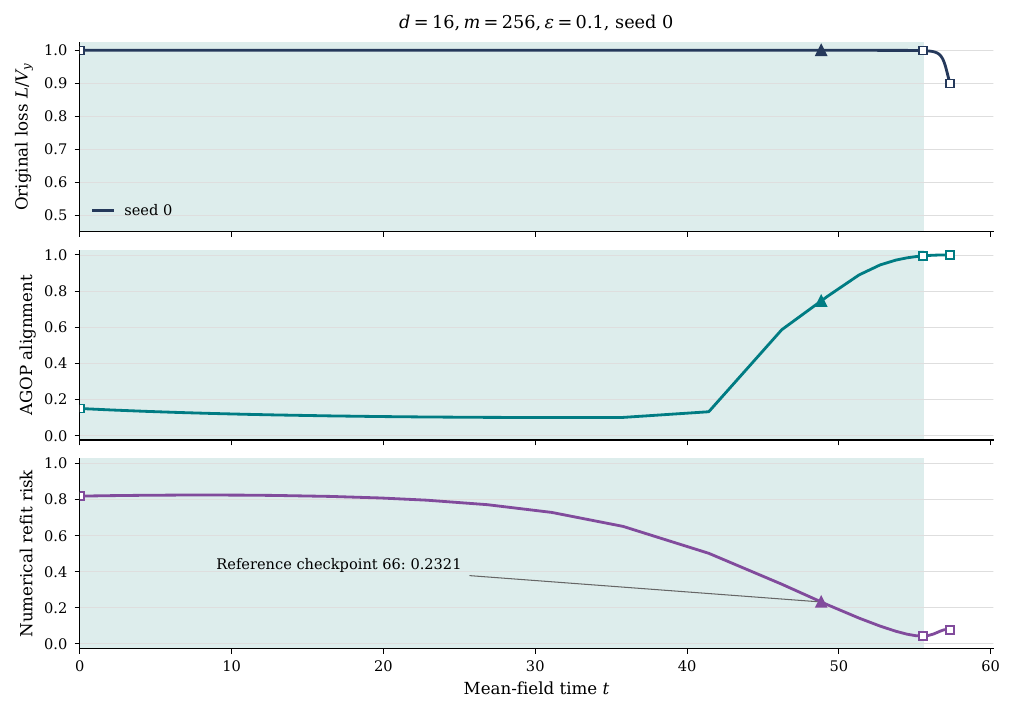}
\caption{Width-256 rank-one reference reproduction ($d=16,\epsilon=.1$, seed~0). The refit panel reports attained prediction risk from the unregularized numerical head solve. The triangle is the separately validated reference checkpoint 66, whose corrected risk is $.232108$; it does not change the primary prefix. The width exceeds $(d+1)(d+2)/2=153$. Quadrature and spectral-cutoff checks accompany these finite-initialization values; the display does not assert an exact zero-risk fit.}
\label{swexp:width256}
\end{figure}

\begin{figure}[!htbp]
\centering
\includegraphics[width=\textwidth]{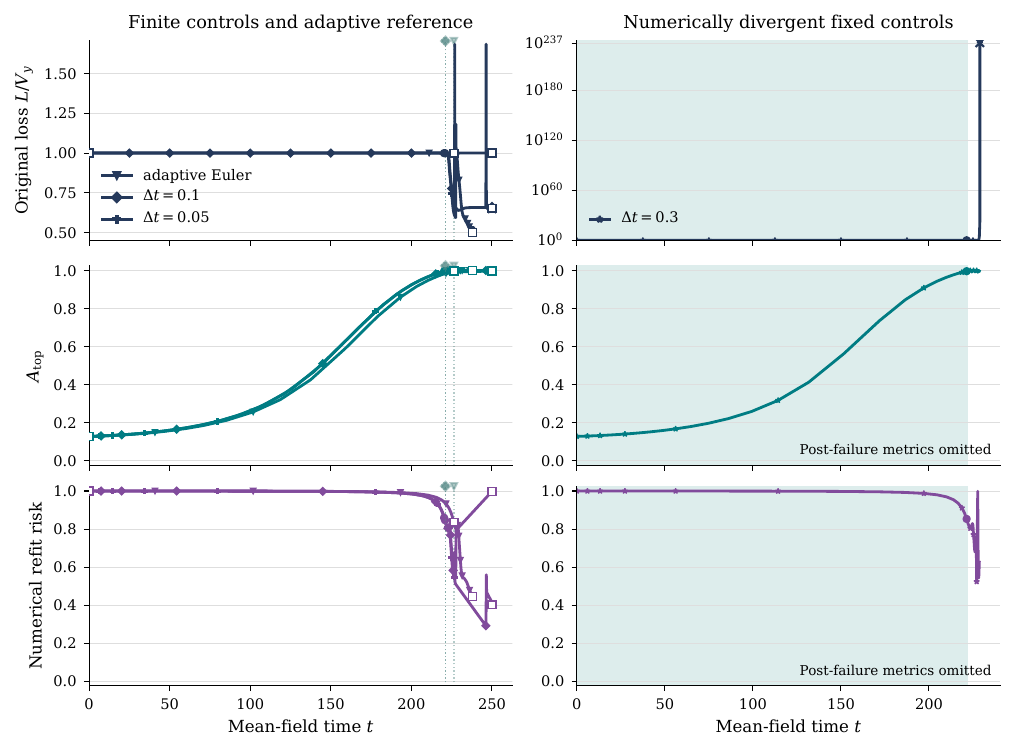}
\caption{Fixed mean-field steps $\Delta t=.3,.1,.05$ versus the adaptive flow approximation, all with the same rank-one cubic seed~0 initialization. Parameter-space learning rates are $m\Delta t$. The x-axis is accumulated mean-field time, not optimizer step count. The right column retains numerically divergent controls on their full loss scale, while omitting alignment/refit after the reference safeguard activates; the left column displays finite controls and the adaptive reference. The nominal time horizon is 250, checked on the accumulated update clock with at most one-step overshoot (actual endpoint times are retained); any trajectory reaching that horizon is censored, and no long-time convergence is claimed. These controls do not convert measured release into a fixed-step release theorem.}
\label{swexp:fixed}
\end{figure}

\begin{figure}[!htbp]
\centering
\includegraphics[width=\textwidth]{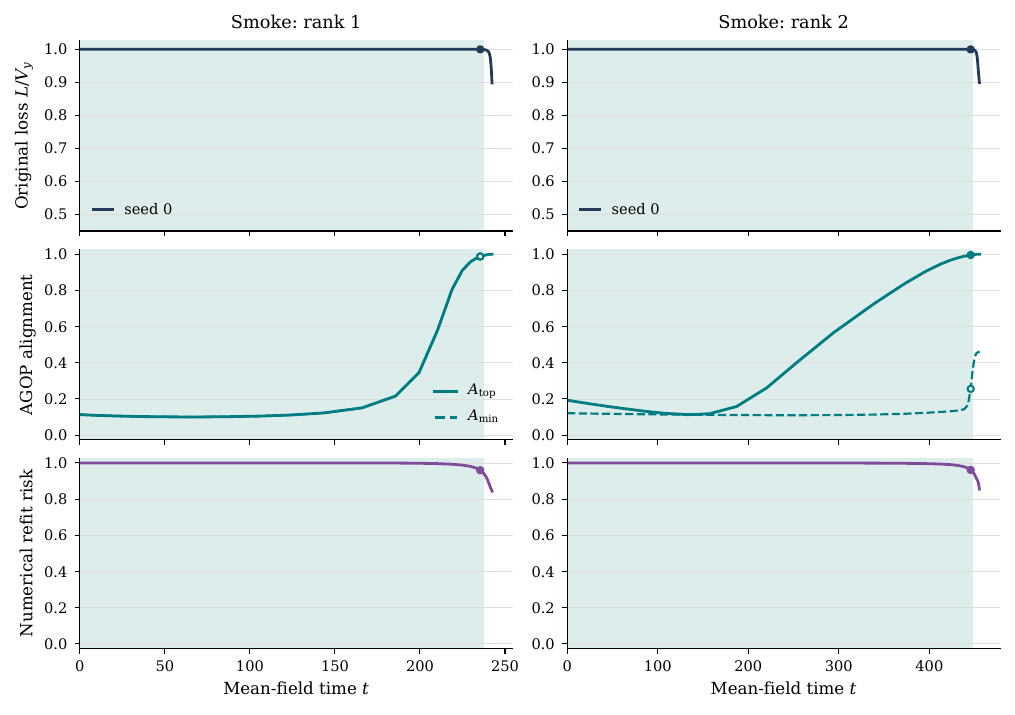}
\caption{The two predeclared reference smoke checks ($d=4,m=6$, seed~0), retained for completeness. The rank-one smoke uses $\epsilon=.1$ and the rank-two smoke $\epsilon=.05$.}
\label{swexp:smoke}
\end{figure}

\begin{table}[!htbp]
\centering
\caption{Every declared SwiGLU trajectory at the loss-selected prefix with $\delta=10^{-3}$. $P$ is the last diagnostic inside the prefix, so $t_P$ may precede the exact last admissible update (both clocks are retained in the audit). $R$ is the attained prediction risk of the equilibrated $10^{-12}$ eigencutoff solve, not a certified unrestricted infimum. \texttt{--} means the weakest-direction statistic duplicates $A_{\rm top}$ at rank one. The reference smokes use $d=4,m=6$; the scientific runs use $d=16,m=64,\epsilon=.05$, except $m=256,\epsilon=.1$. The three GD controls share seed~0; $\dagger$ marks the stated time/step analysis horizon reached before the loss stopping target, rather than convergence; $\ddagger$ denotes the retained numerical divergence of the $\Delta t=.3$ control. Loss values are divided by the teacher variance.}
\label{swexp:tableprimary}

\textit{Loss and time}\par
\begin{tabular}{@{}p{100pt}rrr@{}}
\toprule
Case (seed) & $t_P$ & $\ell_0\to\ell_P$ & $\ell_f$ \\
\midrule
Smoke $r=2$ (0) & 445.6 & $1.000000\to 0.999418$ & 0.8994 \\
Smoke $r=1$ (0) & 235.4 & $1.000000\to 0.999461$ & 0.8977 \\
$h_3$ (0) & 226.4 & $1.000000\to 0.999039$ & 0.4999 \\
$h_3$ (1) & 266.2 & $1.000000\to 0.999403$ & 0.4997 \\
$h_3$ (2) & 311.6 & $1.000000\to 0.999401$ & 0.5000 \\
$h_3$, $m=256$ (0) & 55.5 & $1.000000\to 0.999037$ & 0.8994 \\
$h_3+.3h_2$ (0) & 39.8 & $0.999999\to 0.999128$ & 0.6972 \\
$h_3+.3h_2$ (1) & 43.1 & $1.000000\to 0.999493$ & 0.7000 \\
$h_3+.3h_2$ (2) & 53.2 & $1.000000\to 0.999280$ & 0.6972 \\
ReLU (0) & 20.3 & $0.999999\to 0.999130$ & 0.6939 \\
ReLU (1) & 20.1 & $1.000001\to 0.999240$ & 0.6932 \\
ReLU (2) & 21.3 & $0.999998\to 0.999129$ & 0.6972 \\
tanh (0) & 21.4 & $1.000001\to 0.999313$ & 0.6948 \\
tanh (1) & 22.7 & $1.000001\to 0.999306$ & 0.6926 \\
tanh (2) & 15.7 & $0.999998\to 0.999397$ & 0.6998 \\
$r=2$, $(1,1)$ (0) & 319.2 & $1.000000\to 0.999374$ & 0.5997 \\
$r=2$, $(1,1)$ (1) & 377.0 & $1.000000\to 0.999060$ & 0.5998 \\
$r=2$, $(1,1)$ (2) & 438.0 & $1.000000\to 0.999354$ & 0.6119$\dagger$ \\
$r=2$, $(1,.6)$ (0) & 263.4 & $1.000000\to 0.999486$ & 0.6000 \\
$r=2$, $(1,.6)$ (1) & 310.3 & $1.000000\to 0.999465$ & 0.5994 \\
$r=2$, $(1,.6)$ (2) & 363.3 & $1.000000\to 0.999434$ & 0.6000 \\
GD $\Delta t=0.3$ (0) & 221.4 & $1.000000\to 0.999424$ & failed$\ddagger$ \\
GD $\Delta t=0.1$ (0) & 220.9 & $1.000000\to 0.999282$ & 0.6504$\dagger$ \\
GD $\Delta t=0.05$ (0) & 220.3 & $1.000000\to 0.999507$ & 1.0001$\dagger$ \\
\bottomrule
\end{tabular}
\end{table}

\begin{table}[!htbp]
\centering
\textbf{Table~\ref{swexp:tableprimary} (continued).}\par
\textit{Alignment and refit}\par
\begin{tabular}{@{}p{100pt}rrr@{}}
\toprule
Case (seed) & $A^0_{\rm top}\to A^P_{\rm top}$ & $A^0_{\min}\to A^P_{\min}$ & $R_0\to R_P$ \\
\midrule
Smoke $r=2$ (0) & $0.193\to 0.995$ & $0.121\to 0.256$ & $1.000\to 0.961$ \\
Smoke $r=1$ (0) & $0.113\to 0.987$ & -- & $1.000\to 0.960$ \\
$h_3$ (0) & $0.128\to 0.998$ & -- & $1.000\to 0.836$ \\
$h_3$ (1) & $0.022\to 0.994$ & -- & $1.000\to 0.629$ \\
$h_3$ (2) & $0.000\to 0.994$ & -- & $1.000\to 0.642$ \\
$h_3$, $m=256$ (0) & $0.150\to 0.995$ & -- & $0.818\to 0.043$ \\
$h_3+.3h_2$ (0) & $0.128\to 0.996$ & -- & $0.969\to 0.543$ \\
$h_3+.3h_2$ (1) & $0.022\to 0.992$ & -- & $0.979\to 0.582$ \\
$h_3+.3h_2$ (2) & $0.000\to 0.993$ & -- & $0.963\to 0.405$ \\
ReLU (0) & $0.128\to 0.990$ & -- & $0.698\to 0.063$ \\
ReLU (1) & $0.022\to 0.979$ & -- & $0.721\to 0.074$ \\
ReLU (2) & $0.000\to 0.997$ & -- & $0.483\to 0.055$ \\
tanh (0) & $0.128\to 0.980$ & -- & $0.590\to 0.072$ \\
tanh (1) & $0.022\to 0.983$ & -- & $0.745\to 0.069$ \\
tanh (2) & $0.000\to 0.956$ & -- & $0.558\to 0.098$ \\
$r=2$, $(1,1)$ (0) & $0.224\to 0.997$ & $0.018\to 0.976$ & $1.000\to 0.859$ \\
$r=2$, $(1,1)$ (1) & $0.058\to 0.997$ & $0.003\to 0.151$ & $1.000\to 0.812$ \\
$r=2$, $(1,1)$ (2) & $0.006\to 0.995$ & $0.006\to 0.178$ & $1.000\to 0.810$ \\
$r=2$, $(1,.6)$ (0) & $0.224\to 0.997$ & $0.018\to 0.008$ & $1.000\to 0.889$ \\
$r=2$, $(1,.6)$ (1) & $0.058\to 0.994$ & $0.003\to 0.067$ & $1.000\to 0.726$ \\
$r=2$, $(1,.6)$ (2) & $0.006\to 0.994$ & $0.006\to 0.004$ & $1.000\to 0.752$ \\
GD $\Delta t=0.3$ (0) & $0.128\to 0.997$ & -- & $1.000\to 0.853$ \\
GD $\Delta t=0.1$ (0) & $0.128\to 0.997$ & -- & $1.000\to 0.846$ \\
GD $\Delta t=0.05$ (0) & $0.128\to 0.996$ & -- & $1.000\to 0.859$ \\
\bottomrule
\end{tabular}
\end{table}

\begin{table}[!htbp]
\centering
\caption{Every declared SwiGLU trajectory at the loss-selected prefix with $\delta=10^{-4}$. $P$ is the last diagnostic inside the prefix, so $t_P$ may precede the exact last admissible update (both clocks are retained in the audit). $R$ is the attained prediction risk of the equilibrated $10^{-12}$ eigencutoff solve, not a certified unrestricted infimum. \texttt{--} means the weakest-direction statistic duplicates $A_{\rm top}$ at rank one. The reference smokes use $d=4,m=6$; the scientific runs use $d=16,m=64,\epsilon=.05$, except $m=256,\epsilon=.1$. The three GD controls share seed~0; $\dagger$ marks the stated time/step analysis horizon reached before the loss stopping target, rather than convergence; $\ddagger$ denotes the retained numerical divergence of the $\Delta t=.3$ control. Loss values are divided by the teacher variance.}
\label{swexp:tablestrict}

\textit{Loss and time}\par
\begin{tabular}{@{}p{100pt}rrr@{}}
\toprule
Case (seed) & $t_P$ & $\ell_0\to\ell_P$ & $\ell_f$ \\
\midrule
Smoke $r=2$ (0) & 432.5 & $1.000000\to 0.999930$ & 0.8994 \\
Smoke $r=1$ (0) & 224.9 & $1.000000\to 0.999938$ & 0.8977 \\
$h_3$ (0) & 221.5 & $1.000000\to 0.999945$ & 0.4999 \\
$h_3$ (1) & 262.4 & $1.000000\to 0.999927$ & 0.4997 \\
$h_3$ (2) & 307.6 & $1.000000\to 0.999930$ & 0.5000 \\
$h_3$, $m=256$ (0) & 52.7 & $1.000000\to 0.999945$ & 0.8994 \\
$h_3+.3h_2$ (0) & 34.8 & $0.999999\to 0.999949$ & 0.6972 \\
$h_3+.3h_2$ (1) & 39.3 & $1.000000\to 0.999941$ & 0.7000 \\
$h_3+.3h_2$ (2) & 49.7 & $1.000000\to 0.999914$ & 0.6972 \\
ReLU (0) & 13.7 & $0.999999\to 0.999941$ & 0.6939 \\
ReLU (1) & 16.1 & $1.000001\to 0.999909$ & 0.6932 \\
ReLU (2) & 13.7 & $0.999998\to 0.999939$ & 0.6972 \\
tanh (0) & 15.7 & $1.000001\to 0.999917$ & 0.6948 \\
tanh (1) & 16.8 & $1.000001\to 0.999918$ & 0.6926 \\
tanh (2) & 12.2 & $0.999998\to 0.999927$ & 0.6998 \\
$r=2$, $(1,1)$ (0) & 314.3 & $1.000000\to 0.999926$ & 0.5997 \\
$r=2$, $(1,1)$ (1) & 370.9 & $1.000000\to 0.999944$ & 0.5998 \\
$r=2$, $(1,1)$ (2) & 433.2 & $1.000000\to 0.999925$ & 0.6119$\dagger$ \\
$r=2$, $(1,.6)$ (0) & 259.1 & $1.000000\to 0.999940$ & 0.6000 \\
$r=2$, $(1,.6)$ (1) & 305.9 & $1.000000\to 0.999935$ & 0.5994 \\
$r=2$, $(1,.6)$ (2) & 358.7 & $1.000000\to 0.999934$ & 0.6000 \\
GD $\Delta t=0.3$ (0) & 216.6 & $1.000000\to 0.999949$ & failed$\ddagger$ \\
GD $\Delta t=0.1$ (0) & 217.2 & $1.000000\to 0.999923$ & 0.6504$\dagger$ \\
GD $\Delta t=0.05$ (0) & 216.4 & $1.000000\to 0.999940$ & 1.0001$\dagger$ \\
\bottomrule
\end{tabular}
\end{table}

\begin{table}[!htbp]
\centering
\textbf{Table~\ref{swexp:tablestrict} (continued).}\par
\textit{Alignment and refit}\par
\begin{tabular}{@{}p{100pt}rrr@{}}
\toprule
Case (seed) & $A^0_{\rm top}\to A^P_{\rm top}$ & $A^0_{\min}\to A^P_{\min}$ & $R_0\to R_P$ \\
\midrule
Smoke $r=2$ (0) & $0.193\to 0.981$ & $0.121\to 0.137$ & $1.000\to 0.985$ \\
Smoke $r=1$ (0) & $0.113\to 0.910$ & -- & $1.000\to 0.988$ \\
$h_3$ (0) & $0.128\to 0.986$ & -- & $1.000\to 0.933$ \\
$h_3$ (1) & $0.022\to 0.974$ & -- & $1.000\to 0.706$ \\
$h_3$ (2) & $0.000\to 0.977$ & -- & $1.000\to 0.859$ \\
$h_3$, $m=256$ (0) & $0.150\to 0.945$ & -- & $0.818\to 0.098$ \\
$h_3+.3h_2$ (0) & $0.128\to 0.977$ & -- & $0.969\to 0.723$ \\
$h_3+.3h_2$ (1) & $0.022\to 0.975$ & -- & $0.979\to 0.759$ \\
$h_3+.3h_2$ (2) & $0.000\to 0.980$ & -- & $0.963\to 0.561$ \\
ReLU (0) & $0.128\to 0.936$ & -- & $0.698\to 0.135$ \\
ReLU (1) & $0.022\to 0.944$ & -- & $0.721\to 0.112$ \\
ReLU (2) & $0.000\to 0.968$ & -- & $0.483\to 0.097$ \\
tanh (0) & $0.128\to 0.947$ & -- & $0.590\to 0.121$ \\
tanh (1) & $0.022\to 0.956$ & -- & $0.745\to 0.112$ \\
tanh (2) & $0.000\to 0.864$ & -- & $0.558\to 0.140$ \\
$r=2$, $(1,1)$ (0) & $0.224\to 0.990$ & $0.018\to 0.958$ & $1.000\to 0.916$ \\
$r=2$, $(1,1)$ (1) & $0.058\to 0.976$ & $0.003\to 0.135$ & $1.000\to 0.850$ \\
$r=2$, $(1,1)$ (2) & $0.006\to 0.982$ & $0.006\to 0.871$ & $1.000\to 0.908$ \\
$r=2$, $(1,.6)$ (0) & $0.224\to 0.988$ & $0.018\to 0.100$ & $1.000\to 0.945$ \\
$r=2$, $(1,.6)$ (1) & $0.058\to 0.975$ & $0.003\to 0.029$ & $1.000\to 0.782$ \\
$r=2$, $(1,.6)$ (2) & $0.006\to 0.979$ & $0.006\to 0.008$ & $1.000\to 0.904$ \\
GD $\Delta t=0.3$ (0) & $0.128\to 0.985$ & -- & $1.000\to 0.936$ \\
GD $\Delta t=0.1$ (0) & $0.128\to 0.988$ & -- & $1.000\to 0.923$ \\
GD $\Delta t=0.05$ (0) & $0.128\to 0.987$ & -- & $1.000\to 0.931$ \\
\bottomrule
\end{tabular}
\end{table}

\begin{table}[!htbp]
\centering
\caption{Terminal observations under the stopping rules in Section~\ref{swexp:protocol}, for all 24 conditions. ``Loss stop'' means the prescribed $\ell_{\rm stop}$ was reached; differences below the target reflect update-grid undershoot, not comparable unconstrained final losses. $t_f,\ell_f$ use the complete update history inside the stated analysis budget; $D$ is the latest saved diagnostic within that budget. For balanced rank-two seed~2, $t_f=471.9$ is update 3000, whereas the last diagnostic is update 2928 at $t_D=471.5$; the raw 60-update tail remains archived and dotted in the plot. Failed-control post-safeguard metrics are omitted. Refit risks are numerical attained risks, subject to the quadrature limitations in the text; in particular the fixed $.1$ endpoint risk changes by $.01133$ on doubling quadrature.}
\label{swexp:tablefinal}

\textit{Loss, time, and status}\par
\begin{tabular}{@{}p{100pt}rrl@{}}
\toprule
Case (seed) & $t_f$ & $\ell_f$ & Status \\
\midrule
Smoke $r=2$ (0) & 455.2 & 0.8994 & loss stop \\
Smoke $r=1$ (0) & 242.4 & 0.8977 & loss stop \\
$h_3$ (0) & 238.0 & 0.4999 & loss stop \\
$h_3$ (1) & 277.7 & 0.4997 & loss stop \\
$h_3$ (2) & 323.8 & 0.5000 & loss stop \\
$h_3$, $m=256$ (0) & 57.3 & 0.8994 & loss stop \\
$h_3+.3h_2$ (0) & 42.4 & 0.6972 & loss stop \\
$h_3+.3h_2$ (1) & 46.2 & 0.7000 & loss stop \\
$h_3+.3h_2$ (2) & 56.0 & 0.6972 & loss stop \\
ReLU (0) & 22.2 & 0.6939 & loss stop \\
ReLU (1) & 22.0 & 0.6932 & loss stop \\
ReLU (2) & 24.1 & 0.6972 & loss stop \\
tanh (0) & 24.1 & 0.6948 & loss stop \\
tanh (1) & 25.3 & 0.6926 & loss stop \\
tanh (2) & 17.6 & 0.6998 & loss stop \\
$r=2$, $(1,1)$ (0) & 334.8 & 0.5997 & loss stop \\
$r=2$, $(1,1)$ (1) & 394.9 & 0.5998 & loss stop \\
$r=2$, $(1,1)$ (2) & 471.9 & 0.6119 & censored \\
$r=2$, $(1,.6)$ (0) & 288.4 & 0.6000 & loss stop \\
$r=2$, $(1,.6)$ (1) & 324.2 & 0.5994 & loss stop \\
$r=2$, $(1,.6)$ (2) & 388.7 & 0.6000 & loss stop \\
GD $\Delta t=0.3$ (0) & 228.9 & $6.11\times10^{237}$ & failure \\
GD $\Delta t=0.1$ (0) & 250.1 & 0.6504 & censored \\
GD $\Delta t=0.05$ (0) & 250.0 & 1.0001 & censored \\
\bottomrule
\end{tabular}
\end{table}

\begin{table}[!htbp]
\centering
\textbf{Table~\ref{swexp:tablefinal} (continued).}\par
\textit{Final feature diagnostics}\par
\begin{tabular}{@{}p{100pt}rrrr@{}}
\toprule
Case (seed) & $t_D$ & $A^D_{\rm top}$ & $A^D_{\min}$ & $R_D$ \\
\midrule
Smoke $r=2$ (0) & 455.2 & 1.000 & 0.463 & 0.856 \\
Smoke $r=1$ (0) & 242.4 & 1.000 & -- & 0.845 \\
$h_3$ (0) & 238.0 & 1.000 & -- & 0.445 \\
$h_3$ (1) & 277.7 & 1.000 & -- & 0.291 \\
$h_3$ (2) & 323.8 & 1.000 & -- & 0.332 \\
$h_3$, $m=256$ (0) & 57.3 & 1.000 & -- & 0.077 \\
$h_3+.3h_2$ (0) & 42.4 & 1.000 & -- & 0.533 \\
$h_3+.3h_2$ (1) & 46.2 & 1.000 & -- & 0.521 \\
$h_3+.3h_2$ (2) & 56.0 & 1.000 & -- & 0.404 \\
ReLU (0) & 22.2 & 0.999 & -- & 0.035 \\
ReLU (1) & 22.0 & 0.999 & -- & 0.043 \\
ReLU (2) & 24.1 & 0.999 & -- & 0.033 \\
tanh (0) & 24.1 & 0.997 & -- & 0.048 \\
tanh (1) & 25.3 & 0.999 & -- & 0.036 \\
tanh (2) & 17.6 & 0.998 & -- & 0.083 \\
$r=2$, $(1,1)$ (0) & 334.8 & 1.000 & 1.000 & 0.444 \\
$r=2$, $(1,1)$ (1) & 394.9 & 1.000 & 0.373 & 0.574 \\
$r=2$, $(1,1)$ (2) & 471.5 & 1.000 & 0.999 & 0.574 \\
$r=2$, $(1,.6)$ (0) & 288.4 & 1.000 & 0.022 & 0.546 \\
$r=2$, $(1,.6)$ (1) & 324.2 & 1.000 & 0.138 & 0.478 \\
$r=2$, $(1,.6)$ (2) & 388.7 & 1.000 & 0.002 & 0.473 \\
GD $\Delta t=0.3$ (0) & -- & -- & -- & -- \\
GD $\Delta t=0.1$ (0) & 250.1 & 1.000 & -- & 0.413 \\
GD $\Delta t=0.05$ (0) & 250.0 & 0.997 & -- & 0.998 \\
\bottomrule
\end{tabular}
\end{table}

\clearpage
\subsection{Twenty-seed rank-one extension}
\label{sw20:protocol}

We use $d$ for input dimension, $m$ for student width, $\kappa$ for
the output scale, $\varepsilon$ for initialization scale, $n$ for the
optimizer update index, and $t$ for accumulated mean-field time.
The diagnostics are normalized trained loss
$\ell_n=L(\theta_n)/\operatorname{Var}(y)$, leading alignment
$A_{\rm top}=\|U^\top e_1(M)\|^2$, and normalized numerical refit MSE
$R$ as defined in Section~\ref{swexp:protocol}. Here $M$ is the population
AGOP, $U$ is the rank-one teacher basis, and $e_1(M)$ is its computed
unit leading eigenvector. In the tables,
$\Delta A=\Delta A_{\rm top}=A_{\rm top}(P)-A_{\rm top}(0)$ and
$\Delta R=R(0)-R(P)$, where $P$ is the last saved diagnostic inside
the primary loss prefix. The teacher polynomials are
$h_2(s)=(s^2-1)/\sqrt2$ and $h_3(s)=(s^3-3s)/\sqrt6$.

The four rank-one teacher configurations are extended from seeds 0--2
to seeds 0--19, giving eighty outcomes. The original twelve runs are
retained exactly, and sixty-eight new runs use the same initial distribution,
population force, adaptive step rule, quadrature and stopping criteria
as Section~\ref{swexp:protocol}. The other twelve smoke, rank-two,
wide-student and fixed-step control conditions are unchanged. Thus the
combined study has ninety-two distinct conditions. The new rank-one
initializations are a declared consecutive extension, with no selection
based on alignment, loss or plot appearance. Fifty-nine trajectories
were completed locally and twenty-one on a CPU server, with NumPy 1.24.3
and SciPy 1.10.1 in both environments and unchanged numerical kernels.
The final termination counts are 80 loss-threshold stops.

For each run, $d=16$, $m=64$, $\kappa=1$ and $\varepsilon=.05$.
The relative-change step cap is $.01$ and the maximum mean-field step
is $50$. The cubic stops at normalized loss $.5$, the other links at
$.7$, subject to ceilings of $20{,}000$ updates and mean-field time
$10^6$. These are adaptive Euler flow approximations, with the output
intercept profiled exactly. They do not instantiate the fixed-step
theorem at a quantified practical initialization scale.

\paragraph{Uniform displays and distinct meanings.}
The colors, row order, typography and seed-summary convention match
Figure~\ref{fig:main-trajectories}. Loss and refit divide by teacher
variance, the refit coefficients are unrestricted, and rank one makes
the leading and weakest-direction alignments identical. The teacher-plane
snapping and two-axis angle diagnostics in the ReLU figure do not apply
to this rank-one comparison.

Each loss curve includes every recorded update. AGOP and corrected refit
diagnostics use their original saved states; linear interpolation onto
integer update indices is used only for drawing seed means and pointwise
empirical 10th--90th percentiles. Such interpolation adds no observations.
Means and bands stop at the earliest terminal update among all twenty
seeds for that teacher. Individual recorded tails remain visible beyond
that common support. There is no endpoint holding, extrapolation,
survivor-only averaging or time alignment of individual trajectories.
The display coordinate is $\log(1+n/100)$ with actual update ticks;
adaptive update index is a numerical-work clock, not elapsed flow time.

\paragraph{Numerical checks and endpoint accounting.}
The corrected prediction risk uses the same equilibrated Gram solve at
relative cutoff $10^{-12}$, with sensitivity checks at $10^{-8}$,
$10^{-10}$ and $10^{-14}$. Every new seed also doubles all quadrature
orders at initialization, its primary-prefix diagnostic and its final
state. These are checks of stored states, not retraining or certified
integration-error bounds. Across the 207 resolved comparisons,
the largest absolute changes in normalized loss, leading alignment and
normalized refit risk are $1.09\times10^{-6}$, $3.85\times10^{-8}$ and $2.88\times10^{-3}$, respectively;
0 comparisons remain unresolved.
Nonfinite or negative prediction risks,
nonpositive Gram spectra, and numerically unresolved AGOP leading
directions are flagged. The AGOP flag uses relative leading gap
$>10^{-8}$ and smallest-eigenvalue ratio $\ge-10^{-8}$; it is a
numerical screening rule, not an error certificate. An unresolved seed
suppresses the corresponding ensemble statistic instead of being dropped.
The recorded unresolved-diagnostic count is 0.

For each seed, the primary plateau is the full-update initial prefix
$|\ell_n-\ell_0|\le10^{-3}$. Endpoint gains use the last saved diagnostic
inside that prefix; they are never evaluated at an interpolated threshold
crossing. Table~\ref{sw20:summary} summarizes these gains.
Table~\ref{sw20:seeds} reports every seed's endpoint gains.
Later loss changes are measured outcomes, not conclusions of a
trained-loss release theorem. This unrestricted numerical readout is
distinct from the bounded-refit comparison in the ReLU theorem.

\begin{table}[!htbp]
\centering
\caption{Twenty-seed rank-one outcomes. Alignment and refit gains are
median [minimum, maximum] at the final saved in-prefix diagnostic,
using all twenty values unless a parenthesized resolved count $n$ is shown.}
\label{sw20:summary}

\begin{tabular}{@{}p{80pt}p{135pt}p{135pt}@{}}
\toprule
Teacher & $\Delta A_{\rm top}$ & $\Delta R$\\
\midrule
$h_3$ & 0.962 [0.745, 0.996] & 0.222 [0.146, 0.491] \\
$h_3+.3h_2$ & 0.956 [0.742, 0.996] & 0.475 [0.320, 0.672] \\
ReLU & 0.927 [0.719, 0.997] & 0.519 [0.414, 0.683] \\
$\tanh$ & 0.942 [0.709, 0.983] & 0.476 [0.299, 0.676] \\
\bottomrule
\end{tabular}
\end{table}

\begin{figure}[!htbp]
\centering
\includegraphics[width=\textwidth]{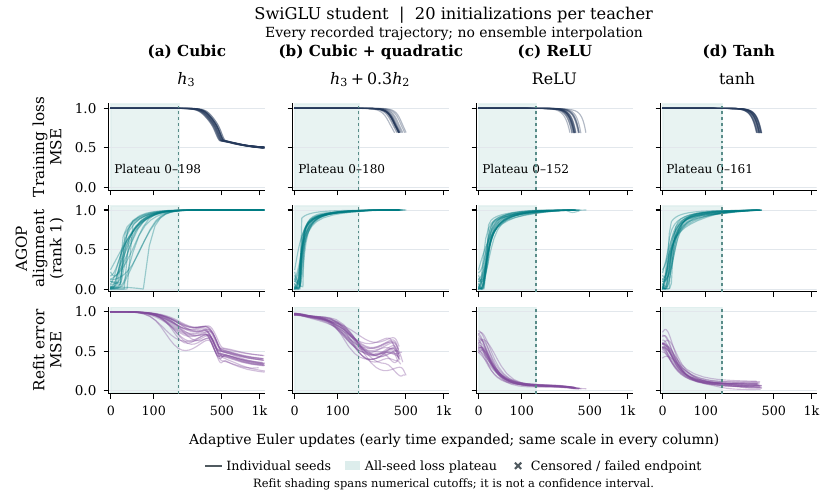}
\caption{All twenty recorded trajectories per rank-one teacher.
Individual curves retain their own saved grids and terminal times.
The three rows and numerical meanings are those of
Figure~\ref{fig:swiglu-main}; all outcomes remain visible.}
\label{sw20:individuals}
\end{figure}

\begin{table}[!htbp]
\centering
\caption{Every seed's endpoint changes in leading AGOP alignment and
normalized numerical refit MSE. Positive $\Delta R$ means lower error.
Values are rounded numerical diagnostics, not error certificates;
\texttt{--} indicates an unresolved value.}
\label{sw20:seeds}

\begin{tabular}{@{}r rr rr rr rr@{}}
\toprule
&\multicolumn{2}{c}{$h_3$}&\multicolumn{2}{c}{$h_3+.3h_2$}&
\multicolumn{2}{c}{ReLU}&\multicolumn{2}{c}{$\tanh$}\\
Seed&$\Delta A$&$\Delta R$&$\Delta A$&$\Delta R$&$\Delta A$&$\Delta R$&$\Delta A$&$\Delta R$\\
\midrule
0 & 0.870 & 0.164 & 0.869 & 0.426 & 0.862 & 0.635 & 0.852 & 0.518 \\
1 & 0.972 & 0.371 & 0.969 & 0.397 & 0.957 & 0.647 & 0.961 & 0.676 \\
2 & 0.994 & 0.358 & 0.993 & 0.558 & 0.997 & 0.428 & 0.956 & 0.460 \\
3 & 0.981 & 0.235 & 0.979 & 0.483 & 0.969 & 0.551 & 0.955 & 0.459 \\
4 & 0.745 & 0.239 & 0.742 & 0.360 & 0.719 & 0.486 & 0.709 & 0.438 \\
5 & 0.899 & 0.146 & 0.894 & 0.459 & 0.882 & 0.619 & 0.869 & 0.539 \\
6 & 0.786 & 0.208 & 0.783 & 0.506 & 0.766 & 0.514 & 0.770 & 0.457 \\
7 & 0.988 & 0.150 & 0.979 & 0.575 & 0.973 & 0.447 & 0.978 & 0.428 \\
8 & 0.996 & 0.349 & 0.996 & 0.557 & 0.984 & 0.442 & 0.983 & 0.510 \\
9 & 0.963 & 0.312 & 0.956 & 0.466 & 0.927 & 0.534 & 0.926 & 0.470 \\
10 & 0.979 & 0.181 & 0.977 & 0.501 & 0.973 & 0.588 & 0.969 & 0.564 \\
11 & 0.978 & 0.288 & 0.980 & 0.505 & 0.975 & 0.637 & 0.972 & 0.623 \\
12 & 0.995 & 0.203 & 0.992 & 0.499 & 0.977 & 0.525 & 0.972 & 0.498 \\
13 & 0.993 & 0.186 & 0.992 & 0.320 & 0.983 & 0.572 & 0.967 & 0.504 \\
14 & 0.873 & 0.491 & 0.873 & 0.468 & 0.871 & 0.683 & 0.836 & 0.633 \\
15 & 0.962 & 0.156 & 0.956 & 0.449 & 0.926 & 0.470 & 0.951 & 0.483 \\
16 & 0.924 & 0.264 & 0.922 & 0.438 & 0.907 & 0.464 & 0.892 & 0.457 \\
17 & 0.941 & 0.367 & 0.935 & 0.539 & 0.913 & 0.472 & 0.933 & 0.420 \\
18 & 0.826 & 0.175 & 0.830 & 0.672 & 0.823 & 0.414 & 0.793 & 0.299 \\
19 & 0.896 & 0.174 & 0.895 & 0.325 & 0.882 & 0.466 & 0.878 & 0.451 \\
\bottomrule
\end{tabular}
\end{table}

\end{document}